%% file: main.tex
\documentclass[11pt]{article}

\usepackage[margin=1in]{geometry}
\usepackage{times}
\usepackage{natbib}
\usepackage{hyperref}
\usepackage{url}
\usepackage{amsmath,amssymb,amsfonts}
\usepackage{graphicx}
\usepackage{booktabs}

\input{macro}

\title{Geometry-Aware Discretization Error of \\ Diffusion Models}
\author{
Samuel Hurault \\
Univ. Gustave Eiffel, CNRS, LIGM \\
\texttt{samuel.hurault@univ-eiffel.fr} 
\and
Thomas Moreau \\
Univ. Paris-Saclay, Inria, CEA \\
\texttt{thomas.moreau@inria.fr} 
\and 
Gabriel Peyré \\
CNRS, ENS, PSL Université \\
\texttt{gabriel.peyre@ens.fr} 
}

\date{}

\begin{document}

\maketitle

\begin{abstract}
Practical diffusion sampling requires simulating a reverse-time ODE or SDE with a limited number of denoising steps, making the choice of sampling parameters crucial for minimizing discretization error. Non-asymptotic convergence bounds characterize sampling complexity, but their worst-case constants can obscure target geometry and thereby limit guidance on parameter optimization. Rather than bounding the error, we derive asymptotically exact small-stepsize expansions of Euler–Maruyama weak and Fréchet errors for general smooth reverse diffusions, with explicit formulas for Gaussian data. These formulas provide tractable objectives for optimizing diffusion parameters, including the noise and rescaling schedules and the stochasticity coefficient, according to the target’s covariance spectrum. In particular, our theory predicts lower optimal stochasticity at smaller step budgets, shows how to adapt the rescaling coefficient to the data power spectrum, and motivates a new family of effective noise schedules. A perturbative extension to Gaussian scale mixtures (GSMs) quantifies how departures from Gaussianity shift the optimal parameters. Finally, experiments on different real image datasets show that FID-optimal parameters agree with the qualitative theoretical predictions.
\end{abstract}

\section{Introduction}
Diffusion and score-based generative models have become a central paradigm for high-dimensional generation~\citep{ho2020denoising, song2020score, lipman2022flow}. Their main practical limit, however, remains sampling time: generating one sample requires integrating a reverse-time ODE or SDE through many steps, each involving an expensive neural network evaluation. Under a fixed inference budget, coarse time discretization can contribute substantially to sampling error. The necessity to reduce the number of steps also motivates recent families of fast-generative models~\citep{song2023consistency, liu2022flow, deng2026generative}. 

Different parameterizations of diffusion sampling can recover the same target distribution in continuous time, yet incur different errors after discretization. Their performance therefore depends jointly on the data distribution and the discretization budget, which should both guide the choice of sampling parameters. We motivate our analysis with two qualitative phenomena illustrating these dependencies.

First, regarding SDE- versus ODE-based sampling, practitioners still lack clear guidance on when and how much noise should be added during generation. Figure~\ref{fig:intro_empirical_motivation}(a) shows that the FID-optimal stochasticity level $\alpha$ (see~\eqref{eq:backward_SDE}, $\alpha=0$ corresponding to ODE sampling and $\alpha=1$ to the standard reverse SDE) depends strongly on the number of sampling steps. In particular, less stochasticity is preferred at smaller step budgets. This has been observed empirically in prior work~\citep{nichol2021improved, xue2023sa, yang2026largelanguagecontinuousdiffusion}.

Similar questions arise when choosing the forward noise rescaling e.g. between variance-preserving (VP) or variance-exploding (VE) diffusions \citep{song2020score}. Figure~\ref{fig:intro_empirical_motivation}(b) shows that the best choice depends on the target distribution. We introduce a  parameter $c_\eta\geq 0$ (see Section~\ref{sec:opt_params}) which controls the variance of the Gaussian initialization of backward sampling: $c_\eta=1$ corresponds to VP, while VE is recovered as $c_\eta\to\infty$. At low step budget ($K=50$), the FID-optimal value of $c_\eta$ strongly varies when sampling different ImageNet classes. This suggests adapting the rescaling to the geometry of the target distribution rather than fixing it to VP or VE.

Our aim is to recover these qualitative phenomena from theory by identifying the role of distribution geometry in discretization error, and use the resulting criteria to guide parameter design.

\begin{figure}[h]
\centering 
\begin{subfigure}[t]{0.37\linewidth}
\centering
\includegraphics[width=1.2\linewidth]{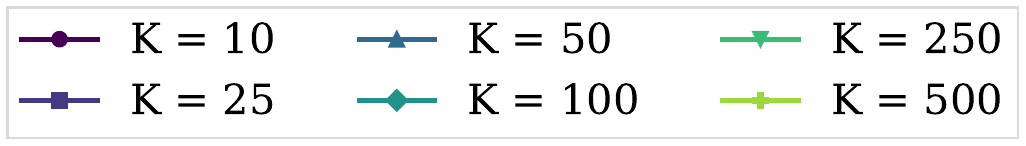} \\[-0.0cm]
\hspace{-0.8cm}\includegraphics[width=\linewidth]{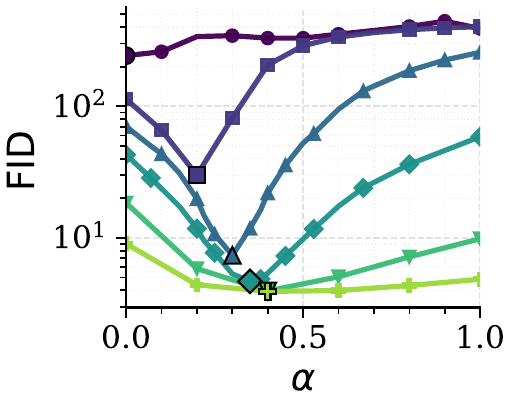}
\end{subfigure} 
\hspace{-0.5cm}
\begin{subfigure}[t]{0.37\linewidth}
\vspace{0.05cm}
\centering
\includegraphics[width=\linewidth]{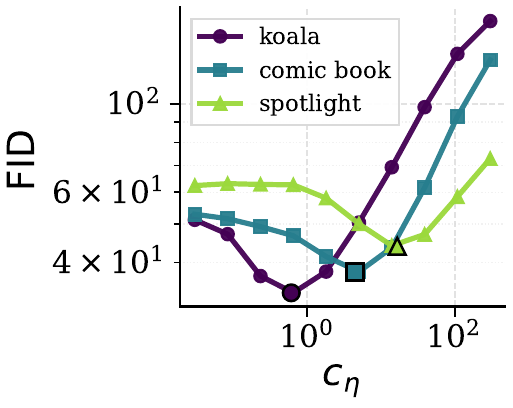}
\end{subfigure}
\caption{(a) FFHQ sampling FID versus stochasticity level $\alpha$ for different step counts $K$, using VE rescaling. (b) Class-conditional ImageNet sampling FID versus the rescaling coefficient $c_\eta$ using $K=50$ steps  (VP is $c_\eta=1$, VE is $c_\eta=+\infty$). See Appendix  Figures~\ref{fig:ffhq_alpha_K_images} and~\ref{fig:class_c_images} for the corresponding image samples.}
\label{fig:intro_empirical_motivation}
\vspace{-0.2cm}
\end{figure}

Most theoretical analyses of discretization error in diffusion sampling provide non-asymptotic convergence guarantees~\citep{de2022convergence, lee2023convergence, li2024towards, benton2023nearly, conforti2025kl, beyler2025convergence}. These results establish sampling complexity, which can guide stepsize choices toward sharper convergence rates~\citep{benton2023nearly}. However, these objectives are typically derived for broad classes of distributions and summarize the target through aggregate or worst-case quantities, such as dimension and regularity constants. This can obscure distribution-specific geometry and sign cancellations over time, limiting their ability to guide diffusion-parameter optimization. Existing theoretical approaches tailored for schedule design optimize error bounds~\citep{xue2024accelerating, sabour2024align} or drift Lipschitz constant~\citep{chen2025lipschitz, tsimpos2025optimal}. Closer in spirit to our approach, \citet{williams2024score} adapt timesteps to local diffusion-path geometry, while \citet{huang2026art} formulate an optimal-control problem based on a local Euler-error surrogate. These objectives quantify local discretization defects but do not characterize their propagation to terminal sampling error.

Our objective is complementary: we characterize the exact leading-order discretization error by tracking how local defects propagate and accumulate to final time, preserving their signed contributions and dependence on the target distribution. For Gaussian targets, this yields explicit objectives for tuning diffusion parameters (noise and rescaling schedules, and stochasticity levels) with respect to the covariance spectrum. Gaussian theory is a natural and well-motivated theoretical testbed: recent empirical and theoretical works suggest that Gaussian or nearly linear score models can provide insight into more general regimes~\citep{wang2024unreasonable, wang2023hidden, li2024understanding}. \citet{pierret2024diffusion} already observed eigenspace decoupling, but do not derive error formulas suitable for parameter optimization. Related scheduler-design approaches use Gaussian simplifications, either through alternative criteria such as the drift Lipschitz constant~\citep{chen2025lipschitz}, which depends only on the smallest covariance eigenvalue, or under isotropic covariance assumptions~\citep{sabour2024align}. Our objective instead retains the full anisotropic covariance spectrum. Closely related, \cite{levy2026correcting} propose a new flow-matching schedule that minimizes Gaussian discretization error with the midpoint sampler.

In this paper, we investigate how discretization budget and target geometry influence diffusion-sampling parameter choices. We focus on the terminal Fréchet discrepancy, which measures mean and covariance mismatch and provides a tractable objective for parameter design. Our contributions are:

\begin{itemize}[leftmargin=2.em]

    \item For general data distributions, we establish in Section~\ref{sec:general_error} a small-stepsize expansion of the weak Euler--Maruyama discretization error, accounting for the propagation and accumulation of local defects. We obtain corresponding expansions for the mean, covariance, and Fréchet errors.
    
  \item For Gaussian data, we make these formulas fully explicit in Proposition~\ref{prop:mean_cov_Gaussian}, obtaining closed-form expressions that expose how discretization error depends on the sampling parameters -- namely, the stochasticity coefficient and the noise and rescaling schedules -- as well as on the covariance spectrum of the target distribution. We extend this analysis perturbatively to Gaussian scale mixtures (GSMs) in Proposition~\ref{prop:mean_cov_GSM}, quantifying how departures from Gaussianity modify the error.

  \item We derive geometry-aware parameter design rules from these
    formulas in Section~\ref{sec:opt_params}. Gaussian theory
    predicts how the optimal parameters vary with the covariance
    spectrum and the sampling budget, matching the trends observed
    in image sampling in Figure~\ref{fig:intro_empirical_motivation}.
    Although the Gaussian minimizers can differ from the empirical
    FID optima, the GSM model provides perturbative corrections
    consistent with the observed shifts. Beyond optimizing the stochasticity and rescaling parameters,
    we derive from the Gaussian error analysis a family of noise
    schedules that generalizes, and  provably improves, the standard polynomial and geometric
    schedules. 
\end{itemize}

 \vspace{-0.1cm}
\section{Preliminaries on diffusion sampling}
\label{sec:background}
\vspace{-0.1cm}
Diffusion models generate samples from a target distribution $p_\data$ by reversing a forward noising process. We use a unified parameterization with a rescaling schedule $\eta_t$ and a noise schedule $\sigma_t$: we define the forward process by $X_0\sim p_{\mathrm{data}}$ and for $t<T$ \vspace{-0.15cm}
\begin{equation}\label{eq:forward_SDE}
dX_t=\frac{\dot\eta_t}{\eta_t}X_t\,dt
+\eta_t\sqrt{2\sigma_t\dot\sigma_t}\,dW_t.
\vspace{-0.1cm}
\end{equation}
Let $\mu_\data$ and $\Sigma_\data$ the mean and covariance of $p_\data$, the marginal distribution of this forward SDE verifies: \vspace{-0.15cm}
\begin{align*}
\mu_t\eqdef\EE[X_t] &= \eta_t \mu_\data \\
\Sigma_t\eqdef\cov[X_t] &= \eta_t^2(\Sigma_\data+\sigma_t^2\id).
\vspace{-0.15cm}
\end{align*}
This unified notation recovers several standard scheduler choices. Variance Preserving (VP) diffusion corresponds to $\eta_t = (1+\sigma_t^2)^{-1/2}$, for which the forward process converges to \( \mathcal{N}(0, \mathrm{Id}) \) as ${T \to \infty}$. Variance Exploding (VE) sets \( \eta_t = 1 \), for which the variance grows with time. The noise schedule~$\sigma_t$ typically starts at~$\sigma_0 = 0$ (or close to $0$) and increases sharply as $t$ approaches $T$, reaching a large terminal value $\sigma_T$, with geometric~\citep{song2020score} or polynomial~\citep{karras2022elucidating} parametrization.  Flow matching (FM) and stochastic interpolants~\citep{lipman2022flow, albergo2023stochastic} also fall within this formulation, with $\eta_t = 1-\frac{t}{T}$ and $\sigma_t = \frac{t}{T-t}$ for linear interpolation. 

In the rest of the analysis, we assume positive rescaling and nondecreasing noise schedules, with \(\eta_0=1\) and \(\sigma_0=0\), and smooth drift and diffusion coefficients of the forward SDE~\eqref{eq:forward_SDE} ensuring its integrability.

\paragraph{Continuous diffusion sampling}

Sampling is done by reversing the dynamics, with a process initialized at $p_T$ whose marginals are \(q_t\eqdef p_{T-t}\). As shown in Appendix~\ref{app:reverse_SDE}, one such family is:
\begin{equation} \label{eq:backward_SDE} %
Y_0\sim p_T, \quad dY_t = v_t(Y_t) dt  + \sqrt{2a_t} \, dW_t \quad t\in [0,T]
\end{equation}
where the drift and diffusion terms write, for $s = T-t$: \vspace{-0.2cm}
\begin{equation*}
\begin{aligned}
  &v_t(x) =  - \frac{\dot \eta_{s}}{\eta_{s} } x  + (1 + \alpha_t) \eta_{s}^2 \sigma_{s} \dot \sigma_{s}\nabla \log p_{s}(x) \\[-0.1cm]
  &a_t = \alpha_t \, \eta_{s}^2 \sigma_{s} \dot \sigma_{s}
\end{aligned}
\vspace{-0.1cm}
\end{equation*}
The parameter $\alpha_t$ controls stochasticity, any nonnegative value produces the same marginals. $\alpha_t=0$ gives an ODE and $\alpha_t=1$ the reverse SDE~\citep{song2020score}. We take $\alpha_t=\alpha$ constant in the rest of the analysis.

\paragraph{Discretized diffusion sampling}
In practice, the sampling dynamics must be discretized. We study Euler--Maruyama (EM) which is the canonical discretization solver. With $K$ iterations and constant step size ${\gamma=T/K}$, the EM discretized version of~\eqref{eq:backward_SDE} is
\begin{align} \label{eq:disc}
 \!  \! \hat Y_{k+1} \! = \! \hat Y_k \!+ \! \gamma \, v_{t_k}(\hat Y_k) \!+\! \! \sqrt{2 \gamma a_{t_k}} \hat W_k, \! \quad \!  \hat W_k  \! \sim  \! \mathcal{N}(0,\id) %
\end{align}
Exponential integrators and higher-order methods are also common in diffusion sampling. Our analysis extends to these schemes; Appendix~\ref{app:gaussian_exponential_integrator} treats the DDIM/DDPM exponential integrator.

\paragraph{Other sources of error.}
Beyond discretization, sampling incurs score-approximation and initialization errors~\citep{hurault2025score, beyler2025convergence}. Our theory assumes exact scores, and our experiments use small steps, where discretization bias is expected to dominate. Appendix~\ref{app:init_error} shows that, for Gaussian targets, initialization error vanishes as the terminal noise level $\sigma_{\max}$ grows, while discretization error persists, supporting the neglect of initialization error at the large $\sigma_{\max}$ used in practice.

\section{First-order asymptotic control of the weak and Fréchet errors}
\label{sec:general_error}

In this section, we derive, for small stepsize $\gamma$, the leading discretization error of the EM scheme. Let $Y_t$ the solution of the continuous SDE~\eqref{eq:backward_SDE} and $\hat Y_k$ the one of the EM discretized version~\eqref{eq:disc}. We first give in Theorem~\ref{thm:general_weak} the expansion of the general weak error, defined for a smooth test function $f$ by $\mathbb{E}[f(\hat Y_k)]-\mathbb{E}[f(Y_{t_k})]$. We then specialize it to  get mean and covariance errors. Throughout, the score and initialization are exact, so that $Y_T \sim p_\data$. We define the stochastic flow map \(\Phi_{t}(x)\) of the SDE~\eqref{eq:backward_SDE} as, the value at final time \(T\) of the solution started from \(x \in\RR^d \) at time \(t\).

We follow the expansion strategy of~\citet{talay1990expansion}. Euler--Maruyama freezes the drift $v_t$ and diffusion $a_t$ at each step, creating an $O(\gamma^2)$ local weak defect that propagates through the remaining dynamics to affect the terminal sampling error. The Jacobian of the flow map describes the linear transport of these perturbations, while its Hessian captures the additional effects of curvature. Summing these propagated defects over the $K=T/\gamma$ steps and approximating the sum with a time integral gives the following weak-error expansion, whose leading term is of order $\gamma$. This result requires sufficient score smoothness and moment bounds; precise assumptions are given in Appendix~\ref{app:general_weak}.

\begin{thm}[First-order weak error, proof in Appendix~\ref{app:general_weak}]
  \label{thm:general_weak}
  Under the regularity conditions stated in
    Appendix~\ref{app:general_weak}, the weak error of the
    EM scheme~\eqref{eq:disc} verifies, when $\gamma \to 0$:   \vspace{-0.3cm}
  \begin{equation*}
  \begin{aligned}
  \EE[f(\hat Y_K)] &-\EE[f(Y_{T})] = \gamma \int_0^{T}\EE\Big[
       e_{t}(Y)\cdot\nabla f(Y_{T})\Big] dt  \\[-0.1cm]
       &+ \gamma\int_0^{T}\EE\Big[\big\langle E_{t}(Y),\nabla^2f(Y_{T})\big\rangle_F
      \Big]dt+O(\gamma^2).
  \end{aligned}
   \vspace{-0.1cm}
  \end{equation*}
where $e_t(Y)\in\RR^d$ and $E_t(Y)\in\RR^{d\times d}$, defined in~\eqref{eq:eq_e} and~\eqref{eq:eq_E}, are independent of~$f$ and depend on $v_t$, $a_t$, their derivatives, and the flow-map Jacobian $J_t(Y)$ and Hessian $H_t(Y)$.
\end{thm}

This formula preserves signed contributions, capturing cancellations between local errors at different times, and depends on the full data distribution rather than aggregate regularity constants. We further discuss its relation to Lipschitz-based objectives~\citep{chen2025lipschitz, tsimpos2025optimal} in Appendix~\ref{app:Lipschitz}.

\paragraph{ODE sampling ($\alpha = 0$)} The mechanism is especially transparent in this setting, since the diffusion term vanishes $a_t = 0$. Define the material acceleration along the ODE trajectory by $D_tv_t(Y_t) =  \frac{d}{dt} v_t(Y_t)$. The weak error from Theorem~\ref{thm:general_weak} simplifies to:
\begin{equation}
\label{eq:weak_ODE}
\begin{aligned}
&\EE\big[f(\hat Y_K)\big] - \EE\big[f(Y_{T})\big]
= \\
&-\frac{\gamma}{2} \int_0^{T}
\EE\Big[
    J_{t}(Y)  D_t v_t(Y_t) \cdot \nabla f(Y_{T})
\Big] dt
+ O(\gamma^2).
\end{aligned}
\end{equation}
The EM local error created by freezing the drift at each step is given by the acceleration $D_t v_t$ while its subsequent amplification or attenuation by the flow is given by the Jacobian of the flow map $J_t(Y)$. ODE discretization error thus vanishes when $D_t v_t = 0$ along the characteristics, i.e., when trajectories remain straight in time~\citep{liu2022flow, tong}.  For SDE sampling, additional terms account for the freezing of the diffusion term $a_t$. Their propagation involves second derivatives of the flow map. 

\vspace{-0.1cm}
\paragraph{Mean and covariance errors} Applying Theorem~\ref{thm:general_weak} with $f(x) = x$ and $f(x) = xx^\top$ gives the first-order errors on the mean and covariance of the discrete diffusion sampler~\eqref{eq:disc}. Denoting $\mu_{\mathrm{data}}=\EE[Y_T]$ and $\Sigma_{\mathrm{data}}=\cov[Y_T]$, we obtain (proof in Appendix~\ref{app:mean_cov}): \vspace{-0.2cm}
  \begin{align}
    \EE[\hat Y_K] - \mu_\data &= \gamma\int_0^T\EE[e_{t}(Y)]dt
      +O(\gamma^2),\label{eq:mean_disc_error}\\
    \cov[\hat Y_K] - \Sigma_\data &= \gamma\int_0^T\Big( 2\EE[E_{t}(Y)]  \label{eq:cov_disc_error} \\[-0.2cm]
       + \cov[e_{t}(Y)&,Y_T] + \cov[e_{t}(Y),Y_T]^\top
       \Big)dt +O(\gamma^2). \vspace{-0.2cm} \nonumber
  \end{align} 
In the rest of the paper, we focus on the errors in these first two moments, which become especially tractable when the drift is linear, as detailed in Section~\ref{sec:gaussian}. We show in Appendix Proposition~\ref{prop:Fréchet_disc} how these moment errors are combined to get the first-order expansion of the Fréchet distance (FD), defined as the sum of the squared $L^2$ distance between means and of the squared Bures distance between covariances.

\section{Explicit formulas for Gaussians and Gaussian Scale Mixtures}
\label{sec:gaussian}

\subsection{Gaussian data}

For Gaussian data, the above general expansions give explicit \emph{geometry-aware} error formulas in terms of the target covariance spectrum. Section~\ref{sec:opt_params} uses them to optimize the diffusion parameters \((\sigma_t,\eta_t,\alpha)\) according to this spectrum and compares the predictions with non-Gaussian image sampling.

For Gaussian data $p_{\mathrm{data}} \sim \mathcal{N}(\mu_\mathrm{data}, \Sigma_\mathrm{data})$, the score is affine and reads, using the notations from Section~\ref{sec:background},  $\nabla \log p_t = - \Sigma_t^{-1} (x - \mu_t)$. The generative SDE~\eqref{eq:backward_SDE} thus has affine drift $v_t(x) = H_t \,x + r_t $, where, for $s = T-t$: \vspace{-0.1cm}
\begin{equation}  \label{eq:vt_linear}
    \begin{aligned}
      H_{t} &\eqdef - \frac{\dot \eta_{s}}{\eta_s}\id - (1 + \alpha) \eta_s^2 \dot \sigma_s \sigma_s \Sigma_{s}^{-1} \\
      r_{t} &\eqdef (1 + \alpha) \eta_s^2 \dot \sigma_s \sigma_s \Sigma_{s}^{-1} \mu_{s}.
    \end{aligned} \vspace{-0.2cm}
  \end{equation}

Proposition~\ref{prop:mean_cov_Gaussian} gives explicit first- and second-order mean and covariance errors for the Gaussian drift~\eqref{eq:vt_linear}. The first-order coefficients follow from Theorem~\ref{thm:general_weak}, while the second-order terms require an additional affine Euler calculation. The errors decompose across eigendirections in the data eigenbasis \(\Sigma_\data=U\diag{\lambda_i}U^\top\).

\begin{prop}[Gaussian data mean and covariance errors, proof in Appendix~\ref{app:mean_cov_Gaussian}]
    \label{prop:mean_cov_Gaussian}
   Assume Gaussian data \(p_{\mathrm{data}} \sim \mathcal{N}(\mu_\data, \Sigma_\data \succ 0 )\). The mean and covariance discretization errors decouple in the eigenbasis of the data covariance
\(\Sigma_\data = U \operatorname{Diag}(\lambda_i) U^\top\) as \vspace{-0.1cm}
\begin{align*}
  \EE[\hat Y_{K}]
  &= \mu_\data  + U \operatorname{Diag}\Big( \Delta^\mu(\lambda_i)\Big) U^\top \mu_\data,\\
  \cov[\hat Y_{K}]
  &= \Sigma_\data + U \operatorname{Diag}\Big( \Delta^\Sigma(\lambda_i)\Big) U^\top .
\end{align*}
For both per-eigendirection defects \(\Delta^\mu(\lambda_i)\) and \(\Delta^\Sigma(\lambda_i)\), we derive explicit first- and second-order small-stepsize expansions of the form \vspace{-0.1cm}
\begin{align*}
\Delta^\mu(\lambda)
&=
\gamma \Delta^{\mu, [1]}(\lambda)
+
\gamma^2 \Delta^{\mu, [2]}(\lambda)
+
O(\gamma^3) \\
\Delta^\Sigma(\lambda)
&=
\gamma \Delta^{\Sigma, [1]}(\lambda)
+
\gamma^2 \Delta^{\Sigma, [2]}(\lambda)
+
O(\gamma^3),
\end{align*}
where each coefficient depends explicitly on the diffusion schedules $\eta_t,\sigma_t$ and on $\alpha$ through time integrals. Their explicit expressions are given in Appendix~\ref{app:mean_cov_Gaussian}.
\end{prop}

\paragraph{ODE sampling ($\alpha=0$).} In this case, the coefficients from Proposition~\ref{prop:mean_cov_Gaussian} considerably simplify. For example, the first-order covariance error only depends on the eigendecomposition $\Delta^H_t$ of the drift matrix $H_t$~\eqref{eq:vt_linear}
\begin{align*}
    \Delta^{\Sigma,[1]}(\lambda)
    =
    -\lambda \int_0^T
    \big(\Delta^H_s(\lambda)\big)^2 ds + o_{\sigma_{\max} \to \infty}(1)
\end{align*}
In comparison~\citet{chen2025lipschitz} optimize schedules based on the control of the drift Lipschitz constant i.e. $\int_0^T \max_i \big(\Delta^H_s(\lambda_i)\big)^2 ds$ (see Appendix~\ref{app:Lipschitz}). Our objective instead depends on the entire spectrum, which will prove important for parameter design.

\paragraph{Typical rescaling schedules. } Appendix~\ref{sec:first_order_power_law} specializes the error coefficients from Proposition~\ref{prop:mean_cov_Gaussian} for usual rescaling schedules, including VE, VP, and FM (see Section~\ref{sec:background} for their definitions). In the VE case, which corresponds to $\eta_t = 1$, the first and second-order terms in the mean error vanish:
$\EE[\hat Y_K] = \mu_\data + O(\gamma^3).$
For the covariance error term~\eqref{eq:cov_error_gauss}, assuming $\alpha >0$ and a large terminal noise $\sigma_T = \sigma_{\max} \to \infty$:
\begin{align} \label{eq:ve_cov_error_beta}
  \! \! \! \! \Delta^{\Sigma,[1]}(\lambda)
  & = \!(\alpha^2\! \! -  \!1)\lambda^{\alpha+1} \! \! \! \int_0^T\! \! \! \!
  \frac{(\sigma_s\dot\sigma_s)^2}{(\lambda+\sigma_s^2)^{\alpha+2}}ds
  \! + \!o_{\sigma_{\max}}(1)
\end{align}
Note that the choice $\alpha=1$ cancels the leading first-order error simultaneously in all eigendirections. 
\paragraph{Fr\'echet distance expansion.}
From Proposition~\ref{prop:mean_cov_Gaussian}, we obtain an explicit expansion of the Fréchet distance. Using the fact that the covariance error is diagonal in the covariance eigenbasis, the Bures distance between covariances decomposes nicely as a single sum over $\lambda_i$-dependent terms (proof in Appendix~\ref{app:Fréchet_dist_exp_gaussian}):
\begin{align}
    \operatorname{FD}&(p_\mathrm{data},\operatorname{Law}(\hat Y_K)) = \gamma^2 \mathcal{E}_{\mathrm{gauss}} + o(\gamma^2) \quad \text{with}\label{eq:gaussian_expansion_Fréchet} \\
    \mathcal{E}_{\mathrm{gauss}} &= \sum_{i=1}^d
      \Delta^{\mu,[1]}(\lambda_i)^2(u_i^\top\mu_\data)^2
      \nonumber + \frac{1}{4}\sum_{i=1}^d
      \frac{\Delta^{\Sigma,[1]}(\lambda_i)^2}{\lambda_i}
\end{align}
  In Appendix~\ref{app:Fréchet_dist_exp_gaussian}, a third-order expansion in $\gamma$ is also derived using the second-order coefficients $\Delta^{\mu, [2]}$ and $\Delta^{\Sigma, [2]}$ from Proposition~\ref{prop:mean_cov_Gaussian}. 
  
    Note that the above result applies to strictly positive covariance eigenvalues $\lambda_i>0$ and does not extend continuously to $\lambda_i=0$, which requires a separate calculation. Appendix~\ref{app:singular_gaussian_endpoint} shows that, for polynomial noise schedule $\sigma_t \propto t^\beta$, zero-variance directions contribute only at $O(\gamma^{2\beta})$ to the Fréchet error, which can be, for large $\beta$, much smaller than the above~$O(\gamma^2)$ for $\lambda_i >0$.
      
  In Section~\ref{sec:opt_params}, we use this Fréchet objective for parameter design. We focus on the Fréchet distance because, for Gaussian data, it coincides with the squared Wasserstein-2 distance. Moreover, its feature-space counterpart, the Fréchet Inception Distance (FID), is the most common metric to evaluate image sampling quality. Expansions in $\gamma$ also hold for other divergences, including KL and smooth $\varphi$-divergences (Appendix~\ref{app:Fréchet_dist_exp_gaussian}). 

  \subsection{Gaussian Scale Mixture}

The above analysis gives explicit and interpretable discretization coefficients, but under a restrictive Gaussian assumption. To assess how these formulas change beyond the Gaussian setting while retaining a tractable model, we consider the Gaussian Scale Mixture (GSM):
\vspace{-0.1cm}
$$
X_\epsilon =\sqrt{1+\varepsilon \zeta} X
\qquad
X\sim\mathcal{N}(0,\Sigma_{\data}),
$$
where $\zeta$ is independent of $X$, $\EE[\zeta]=0$, $\nu=\EE[\zeta^2] < \infty$. GSMs were introduced by~\citet{wainwright1999scale} to model non-Gaussian natural-image statistics. The case $\varepsilon=0$ reduces exactly to the Gaussian model of Proposition~\ref{prop:mean_cov_Gaussian}. We study the perturbation for small $\varepsilon$ of the Gaussian result from Proposition~\ref{prop:mean_cov_Gaussian}.

The next proposition shows that, after applying Theorem~\ref{thm:general_weak}, the Gaussian scale-mixture Fréchet error differs from its Gaussian counterpart~\eqref{eq:gaussian_expansion_Fréchet} at second order in $\varepsilon$, and gives an explicit expression for this correction. This correction term does not decompose per eigendirection, reflecting the coupling introduced by the common random scale. We assume that $\zeta \in L^\infty$ with density bounded away from zero, that the hypotheses of Theorem~\ref{thm:general_weak} hold uniformly near $\varepsilon=0$, and that sufficient regularity allows differentiation twice in $\varepsilon$ under the weak-error expectations.
  
\begin{prop}[GSM Fréchet error, proof in Appendix~\ref{app:mean_cov_GSM}]
\label{prop:mean_cov_GSM}
Under the above GSM model and regularity assumptions, the leading-order Fréchet $O(\gamma^2)$ error coefficient satisfies, as $\varepsilon\to0$: \vspace{-0.1cm}
\begin{equation*}
\mathcal{E}_{\mathrm{GSM}} = \mathcal{E}_{\mathrm{gauss}}
+\varepsilon^2\nu  \sum_i
\omega_i(\{\lambda_j\}) \frac{\Delta^{\Sigma,[1]}(\lambda_i)}{2\lambda_i}
+o(\varepsilon^2). \vspace{-0.1cm}
\end{equation*}
where $\Delta^{\Sigma,[1]}$ is the first-order covariance error of the Gaussian model (Proposition~\ref{prop:mean_cov_Gaussian}) and $\omega_i$ is given by~\eqref{eq:GSM_correction_integral}. $\omega_i$ depends on the full spectrum $\{\lambda_j\}_{j \in \llbracket 1,d \rrbracket}$ and on the diffusion parameters $(\eta_t, \sigma_t, \alpha)$.
\end{prop}

\section{Parameter optimization}
\label{sec:opt_params}
The leading-order Fréchet error is an explicit objective for optimizing the diffusion parameters: stochasticity parameter $\alpha$, noise schedule $\sigma_t$, and rescaling schedule $\eta_t$. We derive optimality criteria in terms of the covariance spectrum and test whether their predictions extend beyond Gaussian and GSM targets.
In particular, we compare with image sampling using pretrained score models~\citep{karras2022elucidating} on CIFAR-10, ImageNet, and FFHQ. For empirical curves, we calculate FID (Fréchet Distance in \emph{Inception-v3} feature space). For theoretical curves, we plot the Gaussian Fréchet Distance (FD) ~\eqref{eq:gaussian_expansion_Fréchet}, with $\Sigma_\data$ estimated empirically from each dataset in image space. See Appendix~\ref{app:details_expe} for experimental details. Our experiments are implemented using the DeepInverse library~\citep{tachella2025deepinverse} and code is available at \url{https://github.com/samuro95/diffusion-discretization-error}.

\paragraph{Optimal stochasticity parameter $\alpha$} %
It controls the magnitude of the diffusion term in the sampling SDE~\eqref{eq:backward_SDE}. We consider here the Variance Exploding setting. The first-order Gaussian theory~\eqref{eq:ve_cov_error_beta} gives a clean reference: $\alpha=1$ cancels the leading-order mean and covariance errors in every eigendirection. This is exactly the choice of $\alpha$ used in the original SDE-based diffusion work~\citep{song2020score}.  However, we have seen in the introduction Figure~\ref{fig:intro_empirical_motivation} (repeated in Figure~\ref{fig:FID_alpha_FFHQ_nstep} (left)) that, for real image sampling, the $\alpha^\star$ minimizing the FID decreases below $1$ as the number of discretization steps $K$ decreases. This empirical decrease of $\alpha^*$ is approximately linear in the stepsize $\gamma=T/K$, as shown in Figure~\ref{fig:FID_alpha_vs_gamma}. This first suggests that a few coarse-step experiments can be used to fit the theoretical scaling law and extrapolate the optimal $\alpha^\star$ to smaller stepsizes.

\begin{figure}[h]
\centering
 \vspace{-0.8cm}
\begin{minipage}{\linewidth}
  \centering
  \vspace{0.6cm}
  \hspace{0.3cm}
  \includegraphics[width=0.5\linewidth]{figures/alpha/n_step/ffhq/empirical_and_theory_nstep_legend.pdf}
\end{minipage}
\hspace{0.3cm}
\begin{subfigure}{0.37\linewidth}
  \centering
  \includegraphics[width=\linewidth]{figures/alpha/n_step/ffhq/empirical_and_theory_nstep_panel_a.pdf}
  \caption{Image Sampling FID}
\end{subfigure}
\begin{subfigure}{0.37\linewidth}
  \centering
  \includegraphics[width=\linewidth]{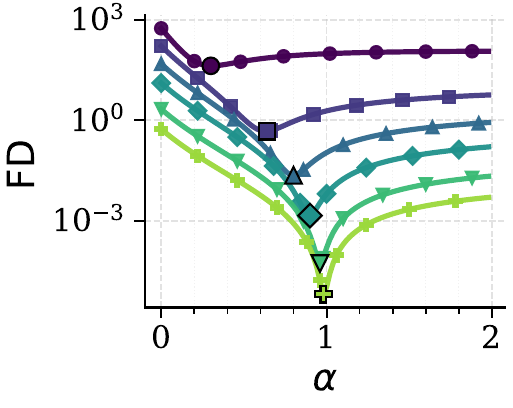}
  \caption{Gaussian Theory FD}
\end{subfigure} %
    \caption{(a) Empirical FID for FFHQ sampling versus $\alpha$ for several step counts $K$ and (b) Gaussian-theory Fréchet Distance (FD) using the empirical FFHQ covariance. Using $\eta_t = 1$ (VE) and  $\sigma_t=\sigma_{\max}(t/T)^5$.
      }
    \label{fig:FID_alpha_FFHQ_nstep}
  \end{figure}

We show that this shift is well captured using our \textit{second-order} theoretical error from Proposition~\ref{prop:mean_cov_Gaussian}.

\begin{prop}[Optimal diffusion parameter,
proof in Appendix~\ref{app:opt_alpha}]
\label{prop:opt_alpha}
Consider VE rescaling with polynomial noise schedule
$\sigma_t=\sigma_{\max}(t/T)^\beta$, $\beta\geq2$ and large $\sigma_{\max}$.

\textup{(i)}   Assuming Gaussian data, the minimizer in $\alpha$ of the second-order Fréchet discretization error~\eqref{eq:gaussian_expansion_Fréchet} verifies \vspace{-0.2cm}
\[
\alpha^*_0
=1+o\!\left(\sigma_{\max}^{-2}\right).
\vspace{-0.2cm}
\]

\textup{(ii)} Pushing the Fr\'echet
expansion~\eqref{eq:gaussian_expansion_Fréchet} to an additional order, the stepsize lowers the above minimizer as:
\vspace{-0.1cm}
\[
\alpha^\star(\gamma)
=\alpha^*_0-\frac{\gamma}{T}\,
\Psi_\beta(\{\lambda_i\})
+o(\gamma),
\]
\vspace{-0.4cm}
\[
\text{where} \quad  \Psi_\beta(\{\lambda_i\})
\! \! \underset{\sigma_{\max}\to\infty}{\sim} \! \! C_\beta \, \sigma_{\max}^{1/\beta} \, 
\frac{\sum_{i=1}^d\lambda_i^{1-3/(2\beta)}}
     {\sum_{i=1}^d\lambda_i^{1-1/\beta}}>0.
\vspace{-0.1cm}
\]

\textup{(iii)} Under the Gaussian scale mixture model of
Proposition~\ref{prop:mean_cov_GSM}, the $\varepsilon$-perturbation of the Gaussian lowers the Gaussian minimizer $\alpha^*_0$ as:
\[
\alpha^\star(\varepsilon)
=\alpha^*_0-\varepsilon^2\nu\,
\Theta
+o(\varepsilon^2),
\]
where $\Theta \! \! \underset{\sigma_{\max}\to\infty}{\to} \! \! \Theta_\infty (\{\lambda_i\}) > 0$. $\Theta_\infty$ is given in~\eqref{eq:opt_alpha_GSM_limit} and explicitly depends on the data spectrum $\{\lambda_i\}$.
\end{prop}

Our Gaussian Fréchet theory thus predicts the decrease in $\alpha^\star$ as the number of discretization steps decreases, as shown Figure~\ref{fig:FID_alpha_FFHQ_nstep}. Proposition~\ref{prop:opt_alpha}(ii) even predicts the slope of the linear decrease of $\alpha^\star(\gamma)$ through the spectrum-dependent constant $\Psi_\beta({\lambda_i})$. Figure~\ref{fig:FID_alpha_vs_gamma} shows that this predicted slope aligns with the empirical linear decrease observed across different image datasets with different spectra. According to $\Psi_\beta$, directions with smaller variance $\lambda_i$ should favor smaller values of $\alpha^\star$. This spectral inhomogeneity is confirmed with Appendix Figure~\ref{fig:FID_alpha_FFHQ_eig} which decomposes the FID per covariance eigendirection.

\begin{figure}[h]
\centering
\includegraphics[width=0.7\linewidth]{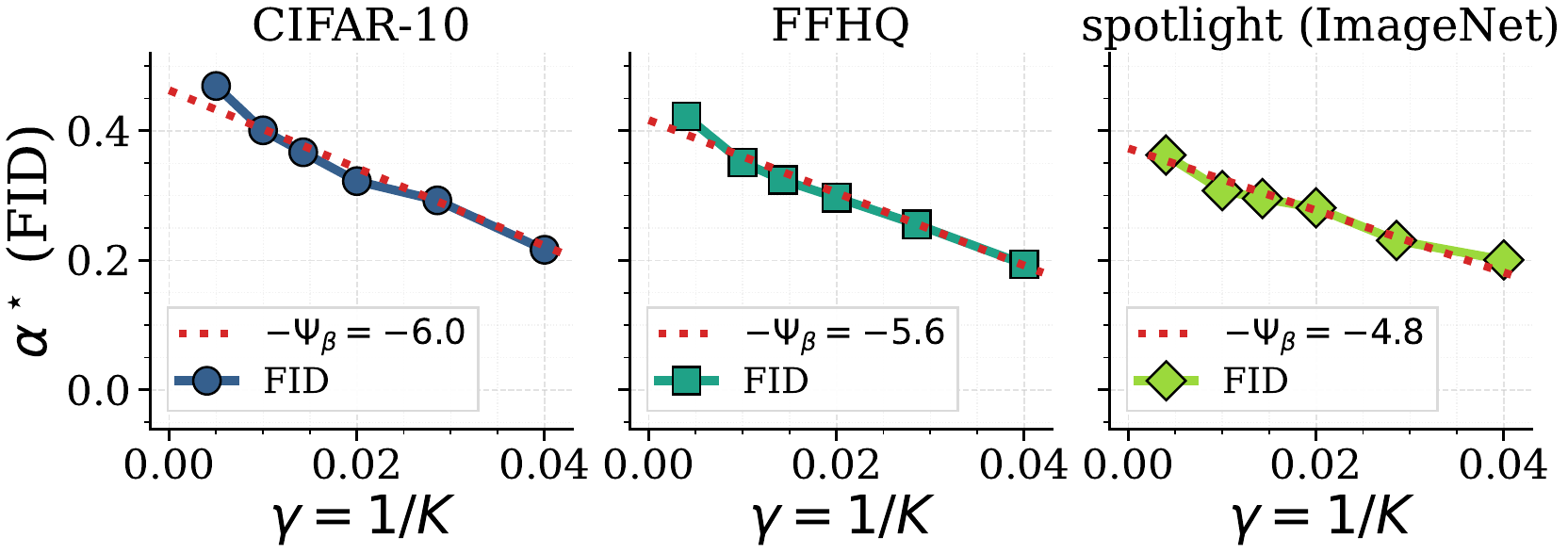}
\caption{FID-optimal $\alpha^\star$ versus $\gamma=1/K$ (VE, $\sigma_t=\sigma_{\max}(t/T)^5$) for CIFAR-10, FFHQ and the ImageNet class spotlight. Dotted: slope $-\Psi_\beta$ of Proposition~\ref{prop:opt_alpha}(ii), computed from each dataset's spectrum}
\label{fig:FID_alpha_vs_gamma}
\vspace{-0.1cm}
\end{figure}

Although Gaussian theory correctly predicts that the optimal $\alpha$ decreases as the stepsize increases, its predicted optima remain systematically above the empirical FID optima. The perturbative GSM analysis in Proposition~\ref{prop:opt_alpha}(iii) gives an explanation for this discrepancy: a small departure from Gaussianity lowers the optimal $\alpha$, consistent with the observed shift.

Finally, note that the covariance shift $\Delta^{\Sigma,[1]}$ increases with $\alpha$: it is negative for $\alpha<1$ and positive for $\alpha>1$. Thus, small $\alpha$ underestimates variance in every eigendirection, while large $\alpha$ overestimates it. This is visually reflected in the samples shown in Appendix Figure~\ref{fig:ffhq_alpha_K_images}. In Appendix~\ref{app:mean_estim}, we show that, in the context of inverse problems, this variance reduction phenomenon at ${\alpha = 0}$ can be leveraged to improve posterior mean estimation using diffusion models.

\paragraph{Optimal rescaling schedule $\eta_t$}

In the large final noise regime of Proposition~\ref{prop:mean_cov_Gaussian}, both first-order mean and covariance errors in one eigendirection of variance $\lambda$ are canceled by \vspace{-0.3cm}
\begin{align}
 \eta_t^\star(\lambda)
  =
  \left(\frac{\lambda}{\lambda+\sigma_t^2}\right)^{\frac{\alpha+1}{2}}
\label{eq:opt_rescaling}
\vspace{-0.1cm}
\end{align}
For ODE sampling ($\alpha=0$), since $\cov(X_t) = \eta_t^2(\Sigma_\data + \sigma_t^2 \id)$, the above $\eta_t^\star(\lambda)$ exactly preserves variance along eigendirection $\lambda$ during sampling, and precisely coincides with VP when $\Sigma_\data=\id$. For anisotropic data, however, this variance preserving property cannot be imposed in all directions simultaneously. 

We thus propose to replace $\lambda$ in~\eqref{eq:opt_rescaling} (with $\alpha = 0$) with a scalar parameter $c_\eta$, and choose $c_\eta^\star$ by minimizing the leading-order Gaussian Fr\'echet
objective~\eqref{eq:gaussian_expansion_Fréchet}.
The forward noising process~\eqref{eq:forward_SDE} with
$\eta_t=\eta_t(c_\eta)$ converges to $\mathcal{N}(0,c_\eta\id)$. Thus, $c_\eta$ controls the variance of the initialization of the backward generative process~\eqref{eq:backward_SDE}. This family includes VP rescaling ($c_\eta=1$) and approaches VE rescaling when $c_\eta\gg\sigma_T^2$.

It was empirically observed in the introduction Figure~\ref{fig:intro_empirical_motivation} that the optimal $c_\eta^*$ for image sampling depends on the target distribution.  Motivated by the approximate power-law behavior of empirical image spectra, we consider $
\lambda_i = \lambda_{\max} i^{-p}.$
The following proposition then characterizes how the minimizer $c_\eta^\star$ depends on the spectral parameters $\lambda_{\max}$ and $p$.

\begin{prop}[Optimal $c_\eta$ parameter for power-law spectra, proof in Appendix~\ref{app:opt_rescaling}]
\label{prop:opt_eta_powerlaw}
For $d \geq 2$, assume a power-law spectrum $\lambda_i=\lambda_{\max}i^{-p}$ with $p>0$ and $\mu_{\data}=0$. Using a polynomial noise schedule $\sigma_t=\sigma_{\max}(t/T)^\beta$ with $\beta\geq 2$, in the limit of large $\sigma_{\max}$

\textup{(i)} The minimizer $c_\eta^\star$ of the leading-order Gaussian Fréchet error~\eqref{eq:gaussian_expansion_Fréchet} is unique and satisfies: 
\[
\begin{aligned}
c_\eta^\star
  &=\lambda_{\max}\tau^\star(p) +o\!\left(\sigma_{\max}^{-2}\right)
\end{aligned}
\]
where $\tau^\star(p)$ minimizes an explicit scalar function given in Appendix equation~\eqref{eq:formula_pstar}.

\textup{(ii)} Under the Gaussian scale mixture model of 
Proposition~\ref{prop:mean_cov_GSM}, the $\varepsilon$-perturbation of the Gaussian changes the above optimal $c_\eta^\star$ as 
\[
c_\eta^\star(\varepsilon)=c_\eta^\star + \varepsilon^2\nu \lambda_{\max} \, \big[ \delta(p) + o_{\sigma_{\max}}(1) \big]+o(\varepsilon^2).
\]
where the sign of $\delta(p)$ depends on $p$. 
\end{prop}
The evolutions of $\tau^\star(p)$ and $\delta(p)$ with $p$ are further described and plotted in Appendix Figures~\ref{fig:tau_p} and~\ref{fig:delta_p}.

\begin{figure}[h]
    \centering
    \begin{subfigure}[t]{0.3\linewidth}
    \captionsetup{skip=0pt}
    \vspace{0pt}
    \centering
    \includegraphics[width=1.1\linewidth]{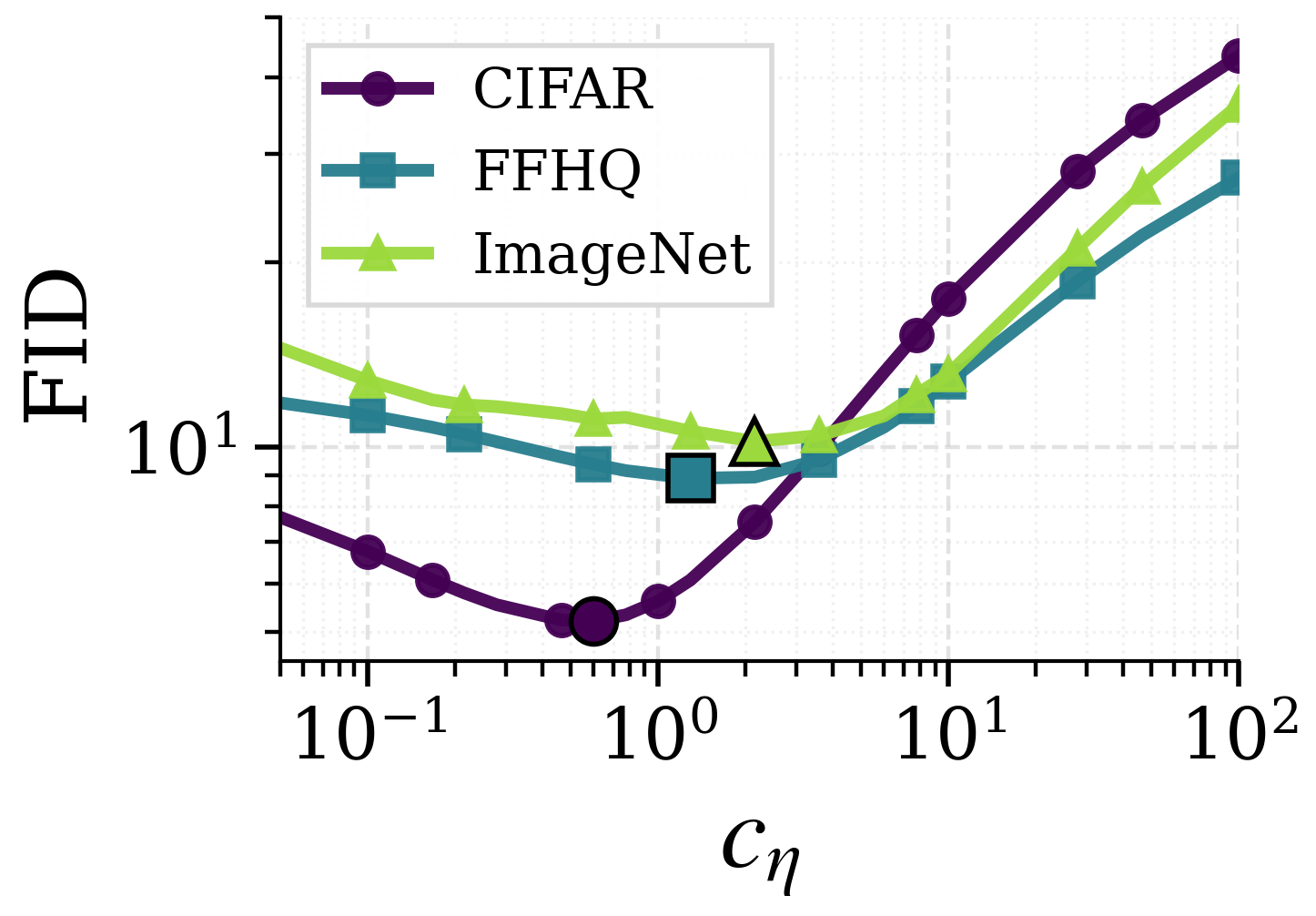}
     \caption{Dataset FID}
     \end{subfigure}
     \hspace{1cm}
    \begin{subfigure}[t]{0.3\linewidth}
    \captionsetup{skip=0pt}
    \vspace{0pt}
    \centering
    \includegraphics[width=1.1\linewidth]{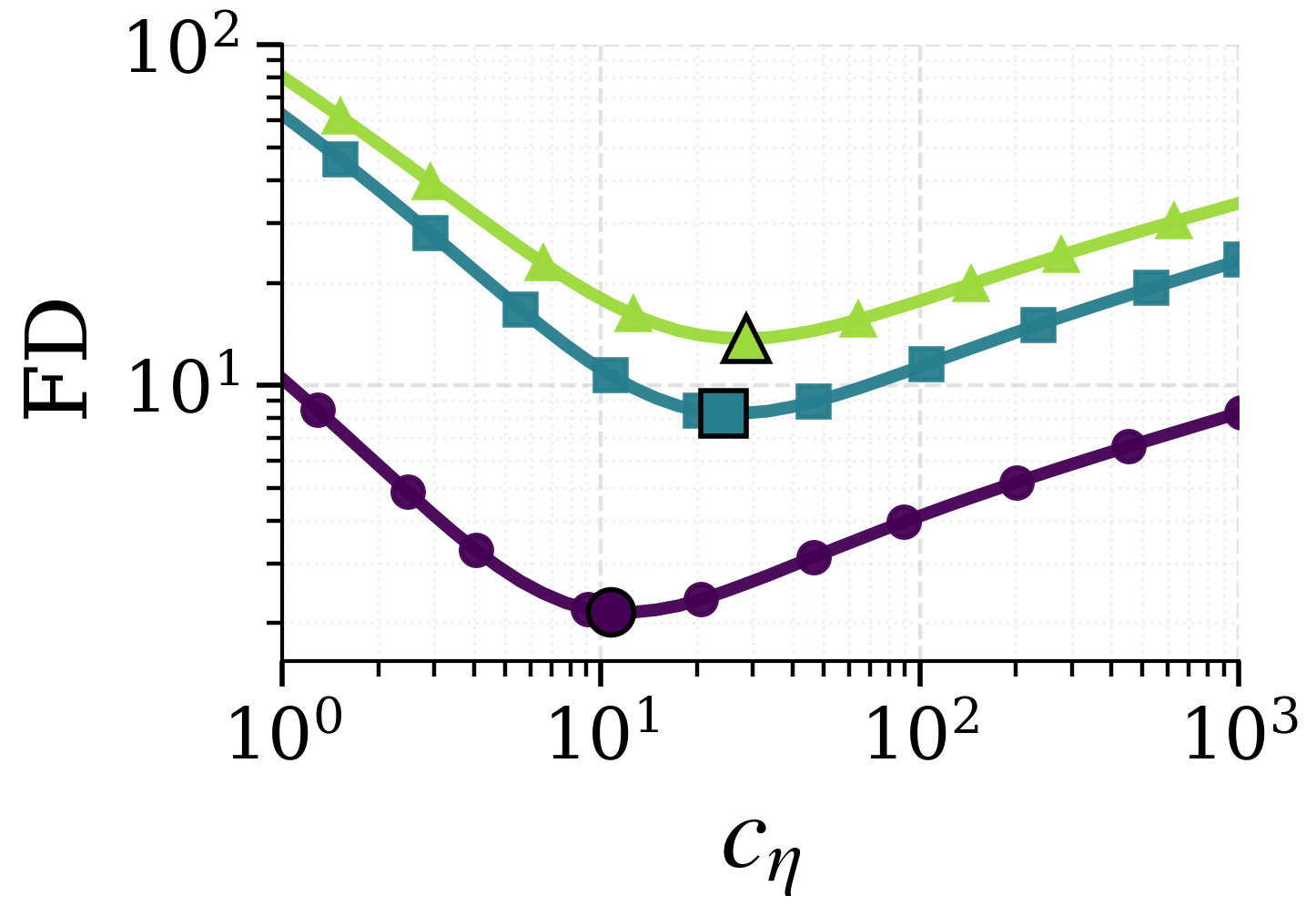}
     \caption{Gaussian FD}
     \end{subfigure}
    \begin{subfigure}[t]{0.3\linewidth}
    \captionsetup{skip=0pt}
    \vspace{0.5cm}
    \centering
    \renewcommand{\arraystretch}{1.15}
    \setlength{\tabcolsep}{3pt}
    \begin{tabular}{@{}lcc@{}}
    \hline
    Dataset & $\lambda_{\max}$ & $p$ \\
    \hline
    {\color{vir1} CIFAR} & 55.3 & 1.43 \\
    {\color{vir2} FFHQ} & 173 & 1.47 \\
    {\color{vir3!75!black} ImageNet} & 225 & 1.35 \\
    \hline
    \end{tabular}
    \end{subfigure}  
     \\ \vspace{0.25cm}
    \begin{subfigure}[t]{0.3\linewidth}
    \captionsetup{skip=0pt}
    \centering
    \includegraphics[width=1.1\linewidth]{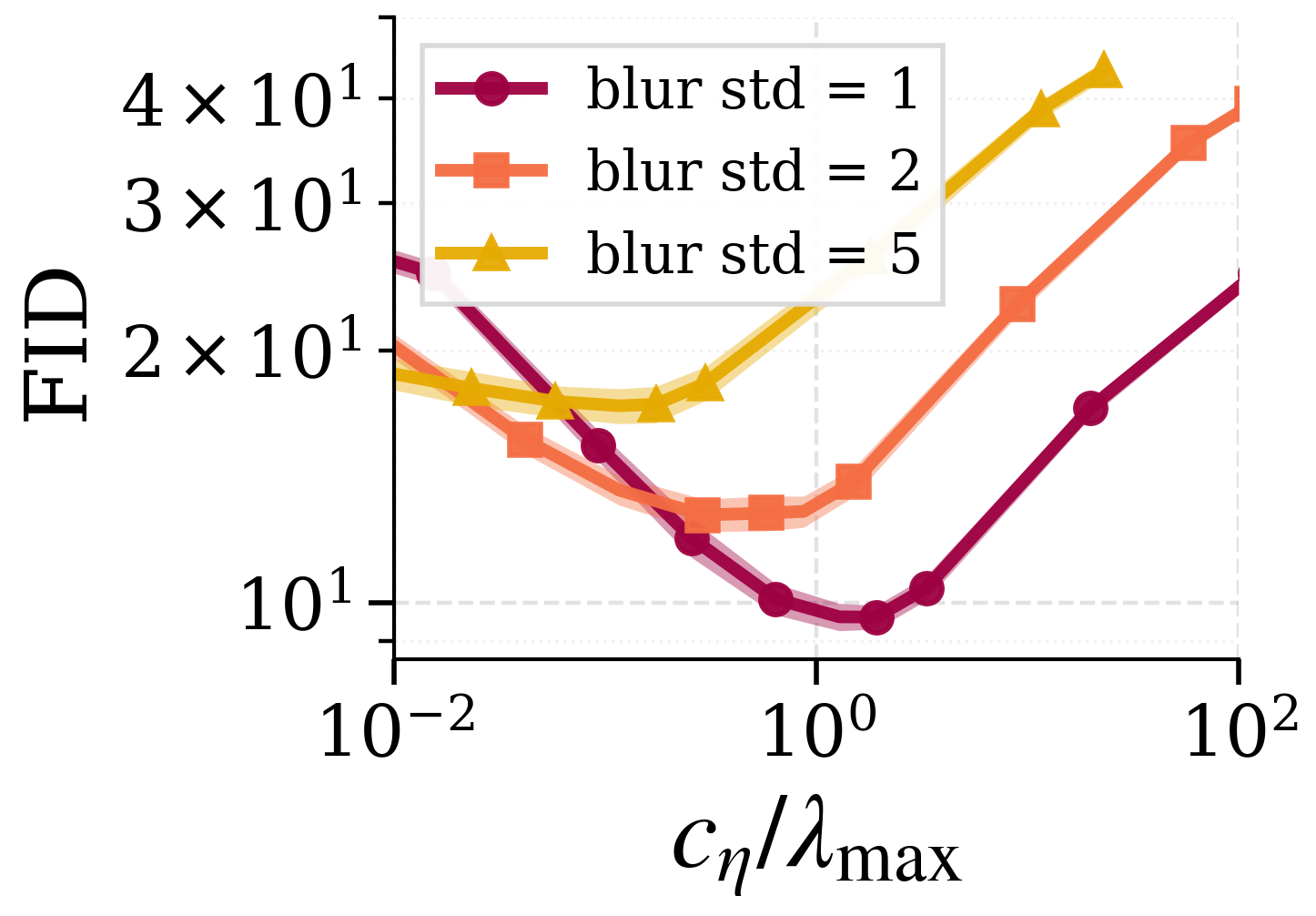}
     \caption{Posterior FID}
     \end{subfigure}
     \hspace{1cm}
    \begin{subfigure}[t]{0.3\linewidth}
    \captionsetup{skip=0pt}
    \centering
    \includegraphics[width=1.1\linewidth]{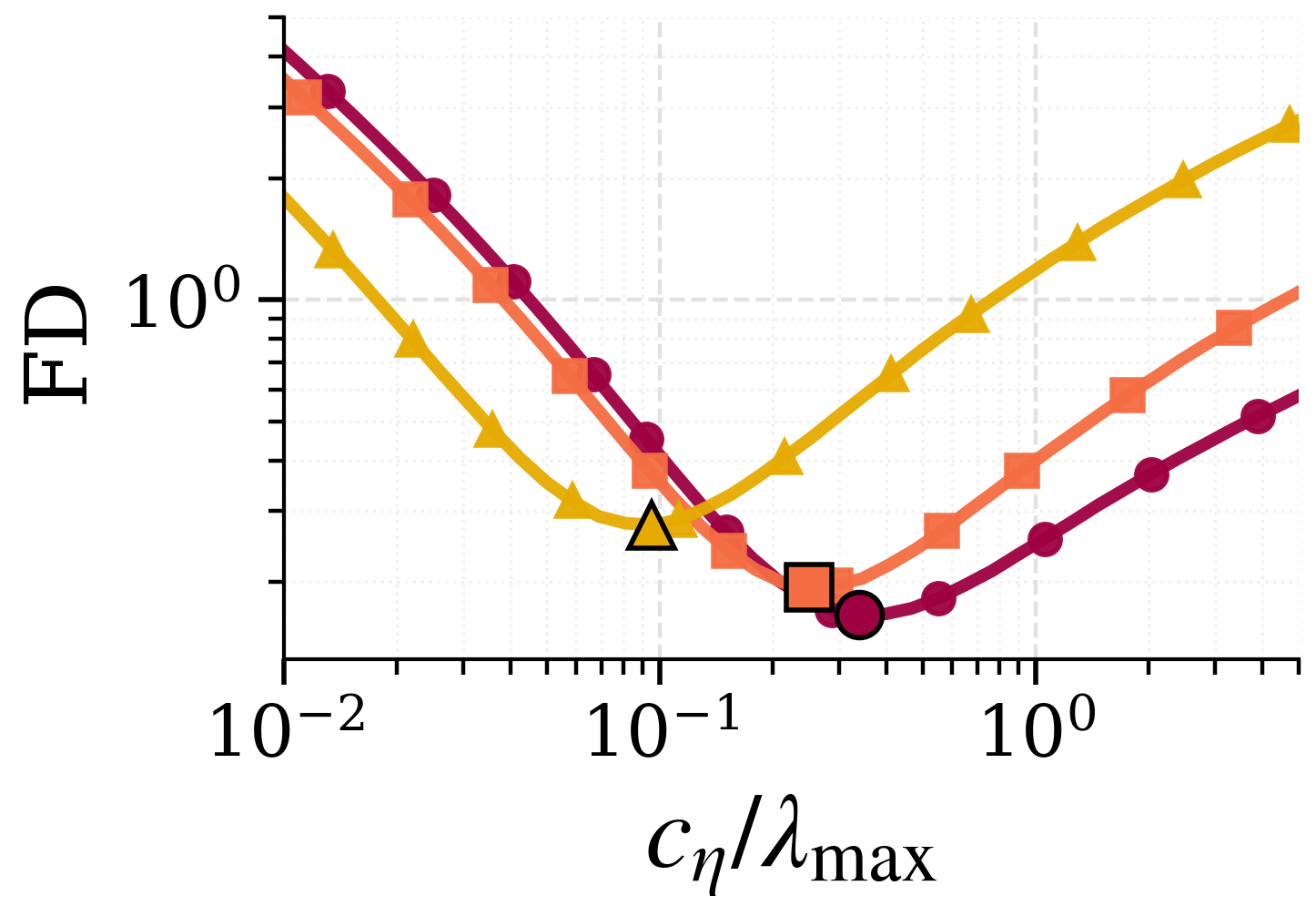}
     \caption{Gaussian FD}
     \end{subfigure} 
     \begin{subfigure}[t]{0.3\linewidth}
     \captionsetup{skip=0pt}
    \vspace{-3.5cm}
    \centering
    \renewcommand{\arraystretch}{1.15}
    \setlength{\tabcolsep}{3pt}
    \begin{tabular}{@{}lcc@{}}
    \hline
    blur std & $\lambda_{\max}$ & $p$ \\
    \hline
    {\color{blur1} 1} & 0.39 & 0.57 \\
    {\color{blur2} 2} & 0.86 & 0.61 \\
    {\color{blur3!85!black} 5} & 4.33 & 0.83 \\
    \hline
    \end{tabular}
    \end{subfigure}
    \caption{Empirical FID and theoretical Fréchet Distance (FD) for image sampling across datasets (top) and for deblurring posteriors (bottom), using $K=50$ steps, $\alpha = 0$ and $\sigma_t = \sigma_{max}\left(\frac{t}{T}\right)^{5}$. Power-law parameters fitted on empirical spectrum (shown Figure~\ref{fig:empiricalspectrum}) are reported in the tables.}
    \label{fig:opt_c_datasets}
\end{figure}

Figure~\ref{fig:opt_c_datasets} empirically supports the
predicted scaling $c_\eta^\star=\lambda_{\max}\tau^\star(p)$
for real image sampling. For CIFAR, FFHQ, and ImageNet,
the top row compares the empirical sampling FID as a
function of $c_\eta$ with the theoretical Gaussian
Fr\'echet error computed from each dataset's estimated
covariance spectrum. We report in the tables, the values of $\lambda_{\max}$ and $p$ of the power-law fitted on each spectrum. Across datasets, $\lambda_{\max}$ changes substantially while the decay rate~$p$ changes only mildly, with limited impact on $\tau^\star(p)$. Correspondingly, both empirical FID and Gaussian FD are minimized at a $c_\eta^*$ value increasing with $\lambda_{\max}$. The same trend holds for the class-conditional plot from the introduction Figure~\ref{fig:intro_empirical_motivation}, as reported in Appendix Figure~\ref{fig:opt_c_classes}.

To isolate the role of the spectral decay $p$ in the prediction $c_\eta^* = \lambda_{\max} \tau^*(p)$, we consider sampling from image deblurring posteriors $p(x\mid y)$ with $y=k*x+w$, for Gaussian blur kernels k. Varying the standard deviation of the blur allows to control the spectral decay of $\cov[x | y]$  (see Figure~\ref{fig:opt_c_datasets}, table). Using the moment-matching approximation~\citep{rozet2024learning}, which is exact for Gaussian data (see Appendix~\ref{app:posterior_sampling} for details), Figure~\ref{fig:opt_c_datasets} again compares FID (c) with theoretical Gaussian FD (d). After rescaling by $\lambda_{\max}$, the remaining variation is due to $p$: both empirical and theoretical errors have similar shapes, and their minimizers $\tau^\star(p)  = c_\eta^*/\lambda_{\max}$ shift to smaller values as $p$ increases. This behavior is consistent with theory: the plot of $\tau^\star(p)$ in Appendix Figure~\ref{fig:tau_p} indeed predicts a decrease of $\tau^\star$ over the range of $p$ values representing these posterior distributions.

Although the Gaussian theory correctly predicts the relative ordering of the optimal $c_\eta$ across datasets, the absolute Gaussian-FD optima do not match those obtained empirically with FID, suggesting a limitation of the Gaussian approximation. Proposition~\ref{prop:opt_eta_powerlaw}(ii) predicts how a non-Gaussian perturbation within the GSM model shifts this Gaussian optimum. The sign of $\delta(p)$, plotted in Appendix Figure~\ref{fig:delta_p}, predicts an upward shift of $c_\eta^*$ for slowly decaying spectra ($p\lesssim0.65$) and a downward shift for steeper spectra. This matches Figure~\ref{fig:opt_c_datasets}: the FID optimum lies above the Gaussian-FD optimum for the two smaller-blur posteriors ($p=0.57,0.61$), and below it for the largest-blur posterior ($p=0.83$) and all three full datasets ($p=1.43,1.47,1.35$).

\paragraph{Optimal noise schedule $\sigma_t$} We follow the same strategy as for the rescaling schedule: first optimize the theoretical Gaussian error for a single eigendirection of variance~$\lambda$, then replace $\lambda$ by a tunable parameter. In the VE setting, denoting $\beta=(1-2\alpha-\alpha^2)^{-1}>1$, the schedule \vspace{-0.2cm}
\begin{equation}
\begin{aligned}
\! \! \sigma_t^*(\lambda, \beta)^2
 \!=  \! \lambda\Bigg[1+\frac tT\Bigg(
\left(1+\frac{\sigma_{\max}^2}{\lambda}\right)^{\frac{1}{2\beta}}\! \! \! \! \! \!-1\Bigg)\Bigg]^{2\beta} \! \! \! \! \! \!-\lambda.
\end{aligned}
\label{eq:sigma_opt_main}
\end{equation}
cancels the first-order Gaussian covariance error along an eigendirection of variance $\lambda$ (proof in Appendix~\ref{app:optimal_noise_schedule}). For anisotropic data, the noise schedule minimizing the multidimensional Fréchet error has no closed-form expression. We therefore optimize within the family $\sigma_t(c_\sigma,  \beta)$ obtained by replacing $\lambda$ in~\eqref{eq:sigma_opt_main} with a tunable parameter~$c_\sigma$. 

This two-parameter family generalizes both polynomial schedules and geometric schedules, which are common in diffusion sampling~\citep{karras2022elucidating,song2020score}. Indeed, as $c_\sigma\downarrow0$, it converges uniformly to $\sigma^{\mathrm{poly}}_t(\beta)=\sigma_{\max}(t/T)^\beta$, while for fixed $c_\sigma>0$, the limit $\beta\to\infty$ yields the geometric evolution $(\sigma^{\mathrm{geo}}_t(c_\sigma))^2 = c_\sigma\left[
\left(1+\frac{\sigma_{\max}^2}{c_\sigma}\right)^{t/T}-1
\right]$, which is the same as~\citet{song2020score}, shifted to ensure $\sigma_0=0$.

We show that a small positive $c_\sigma$ strictly lowers the Gaussian Fréchet error~\eqref{eq:gaussian_expansion_Fréchet} compared to the polynomial schedule. For power-law spectra with decay $p>1$, we also show that, in large dimension, a finite $\beta$ strictly improves upon the optimal geometrical one.

\begin{prop}[Proof in Appendix~\ref{app:sigma_improves}]
\label{prop:sigma_improves}
Assume Gaussian data and VE rescaling schedule. 

\textup{(i)} For every $\beta\geq2$, $\alpha \in [0,1]$ and sufficiently small $c_\sigma>0$, the leading error coefficient in~\eqref{eq:gaussian_expansion_Fréchet} satisfies
\vspace{-0.1cm}
\[
\mathcal E(\sigma_t(c_\sigma, \beta))<\mathcal E(\sigma^\mathrm{poly}_t(\beta)).
\vspace{-0.1cm}
\]

\textup{(ii)} Suppose $\lambda_i=\lambda_{\max}i^{-p}$,
$i=1,\ldots,d$, with $p>1$. For $\alpha\in[0,1]$ and sufficiently
large $d$, let $c_{\sigma}^\star$ minimize
$\mathcal E(\sigma_t^{\mathrm{geo}}(c_\sigma))$ over $c_\sigma>0$.
Then, for sufficiently large $\beta$, \vspace{-0.1cm}
\[
\mathcal E(\sigma_t(c_{\sigma}^\star,\beta))
<\min_{c_\sigma>0}\mathcal E(\sigma_t^{\mathrm{geo}}(c_\sigma)).
\vspace{-0.1cm}
\]
\end{prop}

FFHQ sampling experiments support both predictions.
Figure~\ref{fig:sigma_gen_ffhq} compares empirical FID
with Gaussian Fréchet error~\eqref{eq:gaussian_expansion_Fréchet} computed from the estimated FFHQ covariance
spectrum, varying $c_\sigma$ for several fixed values of $\beta$.
Consistently with Proposition~\ref{prop:sigma_improves}(i), both errors similarly decrease below the polynomial baseline (dashed)
for small positive $c_\sigma$, with larger gains at larger
$\beta$.

Figure~\ref{fig:sigma_geo_ffhq} tests the improvement over
geometric schedules predicted by
Proposition~\ref{prop:sigma_improves}(ii).
We fix $c_\sigma=0.006$, which minimizes the Gaussian Fréchet error~\eqref{eq:gaussian_expansion_Fréchet}
of the geometric schedule, and vary $\beta$.
Every tested finite $\beta\geq2$ reduces both empirical FID
and theoretical Gaussian FD relative to this geometric
baseline (dashed), with both errors approaching the baseline
as $\beta\to\infty$.

\begin{figure}[h]
\captionsetup[subfigure]{skip=1pt}
\centering
\vspace{-0.3cm}
\includegraphics[width=0.5\linewidth]{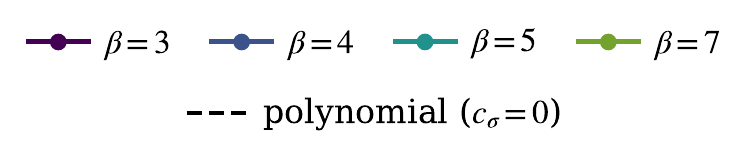}\\[-0.3cm]
\begin{subfigure}{0.37\linewidth}
  \centering
  \includegraphics[width=\linewidth]{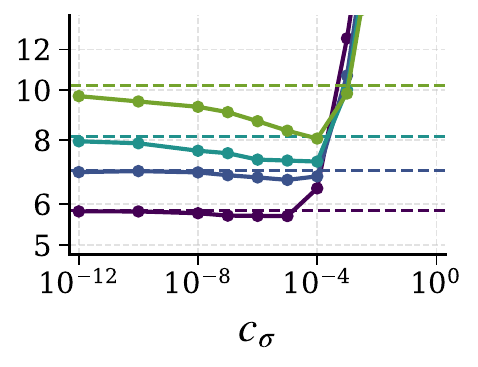}
  \caption{Image FID (FFHQ)}
\end{subfigure}
\begin{subfigure}{0.37\linewidth}
  \centering
  \includegraphics[width=\linewidth]{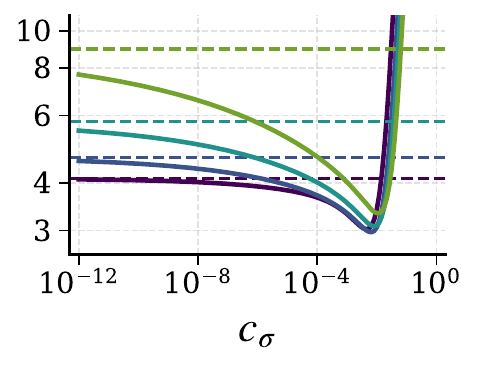}
  \caption{Gaussian Theory FD}
\end{subfigure}
\caption{Sampling error with noise schedule $\sigma_t^*(c_\sigma, \beta)$~\eqref{eq:sigma_opt_main} w.r.t $c_\sigma$ (VE, $\alpha=0.25$, $K=100$). Dashed lines: polynomial schedule  $\sigma^{\mathrm{poly}}_t(\beta)$ ($c_\sigma=0$) of the same color. }
\label{fig:sigma_gen_ffhq}
\vspace{-0.4cm}
\end{figure}

\begin{figure}[h]
\captionsetup[subfigure]{skip=1pt}
\centering
\includegraphics[width=0.5\linewidth]{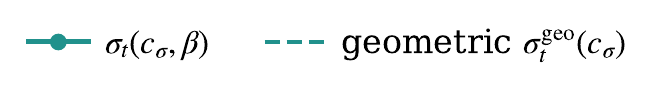}\\[-0.3cm]
\begin{subfigure}{0.37\linewidth}
  \centering
  \includegraphics[width=\linewidth]{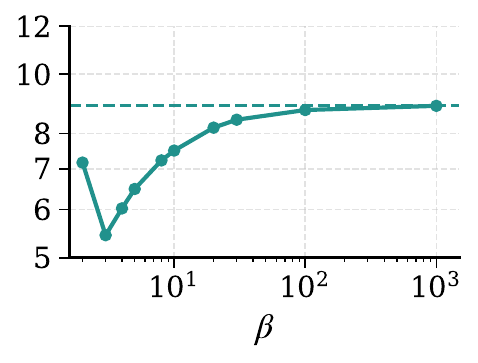}
  \caption{\vspace{-0.1cm} Image FID (FFHQ)}
  \vspace{-0.1cm}
\end{subfigure}
\begin{subfigure}{0.37\linewidth}
  \centering
  \includegraphics[width=\linewidth]{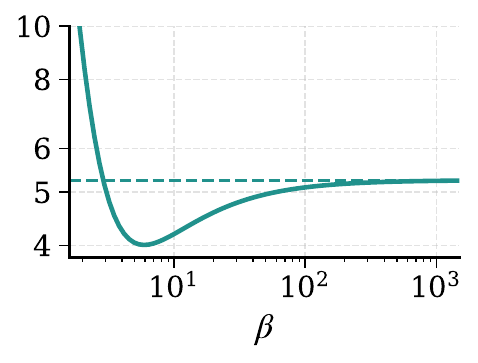}
  \caption{\vspace{-0.1cm} Gaussian Theory FD}
  \vspace{-0.1cm}
\end{subfigure}
\caption{Sampling error with noise schedule $\sigma_t^*(c^*_\sigma, \beta)$~\eqref{eq:sigma_opt_main} w.r.t $\beta$ at the geometric-optimal $c_\sigma^\star=0.006$ (VE, $\alpha=0.25$, $K=100$). Dashed lines: geometric schedule $\sigma^{\mathrm{geo}}_t(c_\sigma^\star)$ ($\beta=\infty$). All samplers end with a final denoising step $\hat x =D(x_{t_K},\sigma_{t_K})$.}
\label{fig:sigma_geo_ffhq} 
\vspace{-0.3cm}
\end{figure}

\section{Conclusion}
We studied how to optimize diffusion model parameters in order to minimize the discretization error. Starting from a general weak error expansion, we derived explicit discretization error formulas under Gaussian data, revealing how the covariance spectrum controls the effect of the stochasticity coefficient $\alpha$, and the diffusion schedules $\eta_t$ and $\sigma_t$. These formulas predict several practical trends observed in experiments, including the decrease of the optimal $\alpha$ under tighter discretization budgets, the dependence of the optimal rescaling $\eta_t$ on spectral anisotropy, and a noise schedule $\sigma_t$ which improves upon polynomial and geometric baselines. Overall, our results provide a simple spectral perspective on diffusion sampler design and a theoretical basis for important empirical heuristics.
\vspace{-0.2cm}
\paragraph{Limitations.} Gaussian approximations and their Gaussian scale mixture extensions capture global trends but miss discretization effects arising from multimodality. We also omit score-estimation
error~\citep{hurault2025score}, which may alter the optimal
sampling parameters through its interaction with discretization
error. Accounting for both effects is a natural next step.

  \bibliographystyle{apalike}
  \bibliography{refs}

\clearpage
\appendix
\thispagestyle{empty}
\onecolumn

\section{Introduction details}

\subsection{Visual samples corresponding to Figure~\ref{fig:intro_empirical_motivation}}

\begin{figure}[h]
\centering
\setlength{\tabcolsep}{1pt}
\renewcommand{\arraystretch}{0.9}
\setlength{\fboxsep}{0pt}\setlength{\fboxrule}{1.2pt}
\newcommand{\aimg}[2]{\includegraphics[width=0.125\linewidth]{figures/images/ffhq_alpha_K/K#1/alpha_#2.png}}
\newcommand{\abest}[2]{\fcolorbox{red}{white}{\includegraphics[width=\dimexpr0.125\linewidth-2.4pt\relax]{figures/images/ffhq_alpha_K/K#1/alpha_#2.png}}}
\begin{tabular}{@{}c@{\hspace{3pt}}ccccccc@{}}
$\alpha$ & $0$ & $0.1$ & $0.2$ & $0.3$ & $0.4$ & $0.6$ & $1$ \\
\rotatebox{90}{\hspace{4pt}\footnotesize $K=10$} &
\abest{10}{0} & \aimg{10}{0p1} & \aimg{10}{0p2} & \aimg{10}{0p3} & \aimg{10}{0p4} & \aimg{10}{0p6} & \aimg{10}{1} \\
\rotatebox{90}{\hspace{4pt}\footnotesize $K=25$} &
\aimg{25}{0} & \aimg{25}{0p1} & \abest{25}{0p2} & \aimg{25}{0p3} & \aimg{25}{0p4} & \aimg{25}{0p6} & \aimg{25}{1} \\
\rotatebox{90}{\hspace{4pt}\footnotesize $K=50$} &
\aimg{50}{0} & \aimg{50}{0p1} & \aimg{50}{0p2} & \abest{50}{0p3} & \aimg{50}{0p4} & \aimg{50}{0p6} & \aimg{50}{1} \\
\rotatebox{90}{\hspace{4pt}\footnotesize $K=100$} &
\aimg{100}{0} & \aimg{100}{0p1} & \aimg{100}{0p2} & \abest{100}{0p3} & \aimg{100}{0p4} & \aimg{100}{0p6} & \aimg{100}{1} \\
\rotatebox{90}{\hspace{4pt}\footnotesize $K=250$} &
\aimg{250}{0} & \aimg{250}{0p1} & \aimg{250}{0p2} & \aimg{250}{0p3} & \abest{250}{0p4} & \aimg{250}{0p6} & \aimg{250}{1} \\
\end{tabular}
\caption{FFHQ $64\times64$ samples (EDM) for different stochasticity $\alpha$ (columns) and number of steps $K$ (rows), using VE ($\eta_t=1$) and $\sigma_t=\sigma_{\max}(t/T)^5$. All images use the same random seed. Red frames mark the grid value closest to the FID-optimal $\alpha$ for each $K$ (Figure~\ref{fig:FID_alpha_FFHQ_nstep}). Small $\alpha$ yields over-smoothed, mean-like samples, while large $\alpha$ creates high-frequency noise artifacts; the range of acceptable $\alpha$ shrinks and moves toward $0$ as $K$ decreases.}
\label{fig:ffhq_alpha_K_images}
\end{figure}

\begin{figure}[h]
\centering
\setlength{\tabcolsep}{1pt}
\renewcommand{\arraystretch}{0.9}
\setlength{\fboxsep}{0pt}\setlength{\fboxrule}{1.2pt}
\newcommand{\cimg}[2]{\includegraphics[width=0.098\linewidth]{figures/images/class_c/#1/c_#2.png}}
\newcommand{\cbest}[2]{\fcolorbox{red}{white}{\includegraphics[width=\dimexpr0.098\linewidth-2.4pt\relax]{figures/images/class_c/#1/c_#2.png}}}
\begin{tabular}{@{}c@{\hspace{3pt}}ccccccccc@{}}
$c_\eta$ & $0.03$ & $0.23$ & $0.65$ & $1.8$ & $5.0$ & $13.9$ & $38.8$ & $108$ & $300$ \\
\rotatebox{90}{\hspace{2pt}\footnotesize koala} &
\cimg{koala}{0p03} & \cimg{koala}{0p2324} & \cbest{koala}{0p6466} & \cimg{koala}{1p799} & \cimg{koala}{5p006} &
\cimg{koala}{13p93} & \cimg{koala}{38p76} & \cimg{koala}{107p8} & \cimg{koala}{300} \\
\rotatebox{90}{\hspace{-2pt}\footnotesize comic book} &
\cimg{comic_book}{0p03} & \cimg{comic_book}{0p2324} & \cimg{comic_book}{0p6466} & \cimg{comic_book}{1p799} & \cbest{comic_book}{5p006} &
\cimg{comic_book}{13p93} & \cimg{comic_book}{38p76} & \cimg{comic_book}{107p8} & \cimg{comic_book}{300} \\
\rotatebox{90}{\hspace{-1pt}\footnotesize spotlight} &
\cimg{spotlight}{0p03} & \cimg{spotlight}{0p2324} & \cimg{spotlight}{0p6466} & \cimg{spotlight}{1p799} & \cimg{spotlight}{5p006} &
\cbest{spotlight}{13p93} & \cimg{spotlight}{38p76} & \cimg{spotlight}{107p8} & \cimg{spotlight}{300} \\
\end{tabular}
\caption{Class-conditional ImageNet $64\times64$ samples (ADM) for different $c_\eta$ (columns), with $K=50$ Euler steps, $\alpha=0$ and $\sigma_t=\sigma_{\max}(t/T)^5$. Each row uses the same random seed across $c_\eta$. Red frames mark the grid value closest to the FID-optimal $c_\eta$ of each class (Figure~\ref{fig:intro_empirical_motivation}). Small $c_\eta$ yields grainy, high-frequency artifacts; large $c_\eta$ over-smooths fine details. Textured classes (koala) are best at small $c_\eta$, while the spotlight class, whose leading eigendirection carries a strong global dark/bright contrast (large $\lambda_{\max}$), needs a much larger $c_\eta$ to reproduce it.}
\label{fig:class_c_images}
\end{figure}

\subsection{Backward diffusion SDE}
\label{app:reverse_SDE}
Let \( (X_t)_{t \geq 0} \) be the stochastic process in \( \mathbb{R}^d \) defined as the solution to the stochastic differential equation:
\begin{align*}
    dX_t = -\beta_t(X_t)\,dt + \sqrt{2 \xi_t}\,dW_t,
\end{align*}
where \( \beta_t : \mathbb{R}^d \to \mathbb{R}^d \) is a measurable drift field,
and \( \xi_t > 0 \) is a time-dependent diffusion coefficient.

We assume throughout that the drift $\beta_t$, the diffusion coefficient $\xi_t$, and the densities $p_t$ are sufficiently regular.

First, the law \( p_t \) of \( X_t \) admits a smooth density that evolves according to the Fokker-Planck equation:
\begin{align*}
\partial_t p_t(x) = \operatorname{div} \left( \beta_t(x)\, p_t(x) \right) + \xi_t \Delta p_t(x),
\end{align*}
Equivalently, using the identity
\begin{align*}
\Delta p_t(x) = \operatorname{div} \left( \nabla \log p_t(x) \cdot p_t(x) \right),
\end{align*}
this equation can be rewritten in conservative form:
\begin{align*}
\partial_t p_t(x) = -\operatorname{div} \left( f_t(x, p_t)\, p_t(x) \right),
\end{align*}
where the effective velocity field \( f_t \) is given by:
\begin{align*}
f_t(x, p_t) = -\beta_t(x) - \xi_t \nabla \log p_t(x).
\end{align*}
Let us consider the reverse process $q_t = p_{T-t}$. We have
$ \partial_t q_t = - \partial_t p_{T-t} $ and the following Fokker--Planck equation for~$q_t$:
\begin{align*}
\partial_t q_t(x) = -\operatorname{div}(\beta_{T-t}(x)q_t(x)) - \xi_{T-t}\Delta q_t
\end{align*}
As the diffusion term is negative, this equation is unstable. We can make it positive by using, for some $\alpha_t \geq 0$:
\begin{align*}
- \Delta q_t = \alpha_t \Delta q_t - (1+\alpha_t)\operatorname{div}( q_t \nabla \log q_t)
\end{align*}
to get
\begin{align*}\partial_t q_t(x) = -\operatorname{div} \left(\left((1+\alpha_t)\xi_{T-t}\nabla \log q_t + \beta_{T-t}(x)\right)q_t(x)\right) + \alpha_t \xi_{T-t} \Delta q_t
\end{align*}
for which the corresponding SDE is
\begin{align*}dY_t = [(1+\alpha_t)\xi_{T-t} \nabla \log q_t(Y_t) + \beta_{T-t}(Y_t)]dt + \sqrt{2\alpha_t \xi_{T-t}} dW_t.
\end{align*}

\subsection{Initialization error}
\label{app:init_error}

Under Gaussian data, we quantify the error induced by replacing the exact terminal law
$$
p_T \propto p_{\mathrm{data}}\!\left(\frac{1}{\eta_T}\cdot\right)
* \mathcal{N}\left(0,(\eta_T\sigma_T)^2\id\right)
$$
by the Gaussian initialization $q_0=\mathcal N(0,(\eta_T\sigma_T)^2\id)$.

\begin{prop}
\label{prop:gaussian_initialization_error_scaling}
Assume Gaussian data $p_{\data}=\mathcal N(\mu_{\data},\Sigma_{\data})$. Let $\widetilde Y_t$ be the solution of the reverse generative SDE~\eqref{eq:backward_SDE}   initialized using $q_0=\mathcal N(0,(\eta_T\sigma_T)^2\id)$
instead of the exact terminal law \(p_T\). At final sampling time, the Fréchet error between $\operatorname{Law}(\widetilde Y_T)$ and $p_{\data}$ satisfies, for any choice of rescaling schedule $\eta_t$:
\[
\FD(p_{\data},\operatorname{Law}(\widetilde Y_T))
=
O\!\left(\sigma_T^{-2(1+\alpha)}\right)
\qquad
\text{as } \sigma_T\to\infty.
\]
Moreover, if the data are centered, \(\mu_{\data}=0\), then:
\[
\FD(p_{\data},\operatorname{Law}(\widetilde Y_T))
=
O\!\left(\sigma_T^{-4(1+\alpha)}\right).
\]
\end{prop}
\begin{proof}
Assuming Gaussian data $p_{\data}=\mathcal N(\mu_{\data},\Sigma_{\data})$, the solution of the forward noising process~\eqref{eq:forward_SDE} at time $T$ is
\begin{align*}
p_T
=
\mathcal N\!\left(
\eta_T\mu_{\rm data},
\eta_T^2(\Sigma_{\rm data}+\sigma_T^2 \id)
\right).
\end{align*}
Instead of initializing from \(p_T\), we initialize from the Gaussian distribution $q_0=\mathcal N(0,(\eta_T\sigma_T)^2 \id)$.
Therefore, at initialization, the mean and covariance errors are
\begin{align*}
\delta m_0 \eqdef  \EE[X_T]  - \EE[\widetilde Y_0] =  \eta_T \mu_\data
\qquad
\delta \Sigma_0 \eqdef \cov[X_T] - \cov[\widetilde Y_0]  =  \eta_T^2 \Sigma_\data
\end{align*}
We now study how this initial error on the mean and covariance propagates during continuous diffusion sampling via~\eqref{eq:backward_SDE}. Let $Y_t$ be the solution of~\eqref{eq:backward_SDE} with exact initialization $Y_0 \sim p_T$, and let $\tilde Y_t$ be the solution with approximate initialization $\widetilde Y_0 \sim q_0$. As detailed in Section~\ref{sec:gaussian}, for Gaussian data the drift in the generative SDE~\eqref{eq:backward_SDE} is affine and is given by
\begin{equation}
  v_t(x) = H_t \,x + r_t \quad \text{with} \quad
  \left\{
    \begin{aligned}
      H_{T-t} &\eqdef - \frac{\dot \eta_{t}}{\eta_t}\id - (1 + \alpha) \eta_t^2 \dot \sigma_t \sigma_t \Sigma_{t}^{-1} \\
      r_{T-t} &\eqdef (1 + \alpha) \eta_t^2 \dot \sigma_t \sigma_t \Sigma_{t}^{-1} \mu_{t}.
    \end{aligned}
    \right.
  \end{equation}
We use $J_{t,s}$, the Jacobian of the flow map between time $s$ and time $t$ (or fundamental matrix) introduced in Section~\ref{sec:general_error}, which is the unique solution of the equation
\begin{equation*}
\frac{d}{dt} J_{t,s} = H_t J_{t,s} \quad J_{s,s} = \id
\end{equation*}
Using the variation-of-constants formula, $Y_t$ following the affine SDE~\eqref{eq:backward_SDE} has the explicit Itô solution at time $t$:
\begin{equation*}
Y_t = J_{t,0} Y_0 + \int_0^t J_{t,s} r_s ds + \int_0^t J_{t,s} \sqrt{2 a_s} dW_s
\end{equation*}
The last two terms do not depend on the initialization law. Therefore, taking mean and covariance on $Y_t$ and $\tilde Y_t$ we easily get the exact propagation formulas for the mean and covariance:
\begin{align*}
\delta m_t &=  \EE[Y_t] - \EE[\tilde Y_t]  = J_{t,0} \delta m_0  \\
\delta \Sigma_t &= \cov[Y_t] - \cov[\tilde Y_t]  = J_{t,0} \delta \Sigma_0 J_{t,0}^\top
\end{align*}
Note that similar propagation equations can be derived in the non-Gaussian case, but we keep it Gaussian here for simplicity. We prove in Section~\ref{app:mean_cov_Gaussian} (equation~\ref{eq:jacobian_gaussian}) that the Jacobian $J_{T,0}$ verifies (at time $T$) in the Gaussian data case:
  \begin{align*}
    J_{T,0}
    &= \eta_{T}^{-1}
    \Sigma_{\data}^{\frac{1+\alpha}{2}}
    \big(\Sigma_\data + \sigma_{T}^2 \id\big)^{-\frac{1+\alpha}{2}}.
  \end{align*}
  We thus get, after replacing $\delta m_0$ and $\delta \Sigma_0$, that the mean and covariance errors, at final diffusion time $T$, due to initialization bias, are:
  \begin{align*}
\delta m_T &= \Sigma_{\data}^{\frac{1+\alpha}{2}}
    \big(\Sigma_\data + \sigma_{T}^2 \id\big)^{-\frac{1+\alpha}{2}} \mu_\data  \\
\delta \Sigma_T &= \Sigma_{\data}^{2+\alpha}
    \big(\Sigma_\data + \sigma_{T}^2 \id\big)^{-(1+\alpha)}
\end{align*}
which rewrites, using the eigendecomposition of the data covariance $\Sigma_{\rm data}=U\operatorname{Diag}(\lambda_i)U^\top$:
  \begin{align*}
\delta m_T &= U \diag{ \left( \frac{\lambda_i}{\lambda_i + \sigma_T^2}\right)^{\frac{\alpha+1}{2}}} U^\top \mu_\data  \\
\delta \Sigma_T &= U \diag{\lambda_i \left( \frac{\lambda_i}{\lambda_i + \sigma_T^2}\right)^{\alpha+1} }U^\top
\end{align*}
First note that the propagated initialization error is independent of the choice of the rescaling schedule $\eta_t$. In particular, it is the same for both Variance Exploding and Variance Preserving.

The mean and covariance biases can be used to evaluate the Fréchet distance between ${p_\data = \operatorname{Law}(Y_T)}$ and $q_T = \operatorname{Law}(\tilde Y_T)$:
  \begin{align*}
  \FD(p_\data, q_T) &= \FD(\operatorname{Law}(Y_T), \operatorname{Law}(\tilde Y_T)) \\
  &= \norm{\delta m_T}^2 + \mathcal{B}^2(\Sigma_\data, \Sigma_\data - \delta \Sigma_T) \\
  &= \sum_{i=1}^d  \left( \frac{\lambda_i}{\lambda_i + \sigma_T^2}\right)^{\alpha+1} (U^\top \mu_\data)_i^2 + \sum_{i=1}^d \left( \sqrt{\lambda_i} -  \sqrt{\lambda_i - \lambda_i \left( \frac{\lambda_i}{\lambda_i + \sigma_T^2}\right)^{\alpha+1}} \right)^2
   \end{align*}
When $\sigma_T \to \infty$, the first term is $O\left(\sigma_T^{-2(\alpha+1)}\right)$ and after expanding the square-root, the second term is $O\left(\sigma_T^{-4(\alpha+1)}\right)$.
\end{proof}

\section{Details and Proof of Theorem~\ref{thm:general_weak}}
\label{app:general_weak}

We first introduce the stochastic calculus quantities that appear in the weak-error formula of Theorem~\ref{thm:general_weak}.
\begin{defn}[Flow derivatives and infinitesimal generator]
Let \(\Phi_{t}\) denote the stochastic flow map associated with SDE~\eqref{eq:backward_SDE}: for each \(x\in\RR^d\), \(\Phi_{t}(x)\) is the value at final time \(T\) of the solution started from \(x\) at time \(t\). Along the trajectory $(Y_t)_{t \in [0,T]}$, we define the stochastic Jacobian, Hessian, and Laplacians of the flow by:
\begin{align*}
J_{t}(Y) &\eqdef \left.\nabla_x\Phi_{t}(x)\right|_{x=Y_t} \\
H_{t}(Y) &\eqdef \left.\nabla_x^2\Phi_{t}(x)\right|_{x=Y_t} \\
\Delta_{t}(Y) &\eqdef
\bigl(\left.\Delta\Phi_{t}(x)_i\right|_{x=Y_t}\bigr)_{i\in\llbracket 1,d\rrbracket}.
\end{align*}
We also denote by $\mathcal{L}_t$ the infinitesimal generator of~\eqref{eq:backward_SDE}, acting on $f \in \mathcal{C}^2(\mathbb{R}^d,\mathbb{R})$ as
\[
\mathcal{L}_t f(y)
\eqdef
v_t(y) \cdot \nabla f(y)
+
a_t \Delta f(y).
\]
Finally, for a third-order tensor \(A\in\mathbb{R}^{d\times d\times d}\) and a matrix
\(M\in\mathbb{R}^{d\times d}\), we use the notation $[A,M] \in \RR^d$ for the contraction over the last two indices of \(A\).
\end{defn}

\paragraph{Detailed assumptions for Theorem~\ref{thm:general_weak}.}

We first impose the following conditions on the schedules, ensuring integrability of the forward SDE~\eqref{eq:forward_SDE}. 
\begin{ass}\label{ass:regularity_sched}
$\eta_t,\sigma_t \in \mathcal C([0,T]) \cap \mathcal C^\infty((0,T])$, $\eta_t>0$, $\sigma_t\ge0$, $\eta_0=1$, $\sigma_0=0$ and denoting $\beta_t \eqdef \frac{\dot \eta_t}{\eta_t}$ and $\xi_t \eqdef \eta_t^2 \dot \sigma_t \sigma_t$ the drift and diffusion terms in the forward SDE~\eqref{eq:forward_SDE}: $\beta_t, \xi_t \in \mathcal{C}^3[0,T]$, $\xi_t \geq 0$.
\end{ass} \vspace{-0.25cm}

Besides Assumption~\ref{ass:regularity_sched}, the conclusion of Theorem~\ref{thm:general_weak} holds under the
following sufficient conditions. 
\begin{ass}[Regularity for the weak-error expansion]
\label{ass:weak_regularity}
Define $B_t(x)=\xi_t\nabla\log p_t(x)$. The mixed derivatives $\partial_t^jD_x^r B_t(x)$ exist
and are continuous for $2j+|r|\le6$, and satisfy
\[
|\partial_t^jD_x^r B_t(x)|
\le
\begin{cases}
C(1+|x|), & |r|=0,\\
C, & |r|\ge1,
\end{cases}
\]
uniformly in $t\in[0,T]$.

Assume $f\in C^6(\RR^d,\RR)$ and, for some $m\ge0$,
\[
|D_x^r f(x)|\le C_f(1+|x|^m),
\qquad |r|\le6,
\]
and $\EE[|Y_0|^{m+6}]<\infty$.
\end{ass}
The smoothness assumptions provide the time and spatial derivatives
needed for the local weak-error expansion. Bounded spatial
derivatives of the drift control the stochastic-flow derivatives
and hence the derivatives of the backward value function.
Linear growth ensures nonexplosion and, together with
$\EE[|Y_0|^{m+6}]<\infty$, uniform moment bounds for the exact
solution and Euler iterates, making the polynomial remainder
terms integrable.

We now re-state and prove Theorem \ref{thm:general_weak}.
\newtheorem*{detailedtheorem1}{Theorem \ref{thm:general_weak}}
\begin{detailedtheorem1}[Detailed version]
  Under Assumptions~\ref{ass:regularity_sched} and~\ref{ass:weak_regularity},
  \begin{equation*}
  \begin{aligned}
  \EE[f(\hat Y_K)] &-\EE[f(Y_{T})] = \gamma \int_0^{T}\EE\Big[
       e_{t}(Y)\cdot\nabla f(Y_{T})\Big] dt + \gamma\int_0^{T}\EE\Big[\langle E_{t}(Y),\nabla^2f(Y_{T})\rangle_F
      \Big]dt+O(\gamma^2).
  \end{aligned}
  \end{equation*}
  where
  \begin{align} \label{eq:eq_e}
    e_{t}(Y)\eqdef-\tfrac12\big(&J_{t}(Y)
      [ (\partial_t+\mathcal L_t)v_t](Y_t) +\dot a_t\Delta_{t}(Y) + 2a_t[H_{t}(Y),\nabla v_t(Y_t)]\big),\\
    E_{t}(Y)\eqdef-\tfrac12&J_{t}(Y)\left[ a_t\big(\nabla v_t(Y_t)+\nabla v_t(Y_t)^\top\big)
      +\dot a_t\id\right]
      J_{t}(Y)^\top \label{eq:eq_E}
  \end{align}
\end{detailedtheorem1}
\begin{proof}
For a fixed terminal time $t$ and test function $f$, define the backward
value function
\[
u_s(x)\coloneqq \mathbb{E}\!\left[f(Y_t^{s,x})\right],
\qquad 0\le s\le t,
\]
where $Y_t^{s,x}$ is the solution at time $t$ started from $x$ at time
$s$. This expectation is over the Brownian motion. It satisfies the backward Kolmogorov equation
\[
(\partial_s+\mathcal L_s)u_s=0,\qquad u_t=f,
\]
where $\mathcal L_s g=v_s\cdot\nabla g+a_s\Delta g$.
We first compute, with the following Lemma, the local weak defect, for small stepsize $\gamma$, of one Euler--Maruyama step, with time coefficients $v_s$ and $a_s$ frozen on $[s, s+\gamma]$. For a fixed time $s\in [0,t_k-\gamma]$ and $x \in \RR^d$ , let
\[
\bar Y_{s+\gamma}^{s,x}
=
x+ \gamma v_s(x)+\sqrt{2\gamma a_s}\,\xi,
\qquad \xi\sim\mathcal N(0,I_d),
\]
\begin{lem}[One-step weak defect] \label{lem:one_step}
\[
\EE\big[u_{s+\gamma}(\bar Y_{s+\gamma}^{s,x})\big]-u_s(x)
=
\gamma^2\psi_s(x)+O\!\big(\gamma^3(1+|x|^{m+6})\big),
\]
where
\[
\psi_s(x)
=
-\frac12\big((\partial_s+\mathcal L_s)v_s\big)(x)\cdot\nabla u_s(x)
-a_s\langle \nabla v_s(x),\nabla^2u_s(x)\rangle
-\frac12\dot a_s\Delta u_s(x).
\]
\end{lem}
\begin{proof}
We start the proof of this lemma by a Taylor expansion of $\EE\big[u_{s+\gamma}(\bar Y_{s+\gamma}^{s,x})\big]$. We directly state here the final form of the expansion with the following Lemma. Because the proof of the lemma is technical but quite straightforward, it is deferred to the end of the section.
\begin{lem} \label{lem:taylor_expansion}
By a Taylor expansion in time and space, with a remainder uniform in $s\in[0,t_k-\gamma]$,
\begin{equation}
\label{eq:taylor_expansion}
\begin{aligned}
&\EE\big[u_{s+\gamma}(\bar Y_{s+\gamma}^{s,x})\big]
=
u_s(x)
+\gamma\big(\partial_su_s+v_s\cdot\nabla u_s+a_s\Delta u_s\big)(x)\\
&+\gamma^2\Big[
\frac12\partial_{ss}u_s
+v_s\cdot\nabla\partial_su_s
+a_s\Delta\partial_su_s
+\frac12 v_s^\top\nabla^2u_s v_s +a_s v_s\cdot\nabla\Delta u_s
+\frac12 a_s^2\Delta^2u_s
\Big](x)
+O\!\big(\gamma^3(1+|x|^{m+6})\big).
\end{aligned}
\end{equation}
\end{lem}
Now, using the backward Kolmogorov equation
\[
(\partial_s+\mathcal L_s)u_s=0
\]
the first order term of~\eqref{eq:taylor_expansion} vanishes. Moreover,
differentiating \(\partial_su_s=-\mathcal L_su_s\) with respect to \(s\)
gives
\[
\partial_{ss}u_s
=
-(\partial_s\mathcal L_s)u_s-\mathcal L_s\partial_su_s.
\]
Hence, using again \(\partial_su_s=-\mathcal L_su_s\),
\[
\frac12\partial_{ss}u_s
+\mathcal L_s\partial_su_s
=
-\frac12(\partial_s\mathcal L_s)u_s
-\frac12\mathcal L_s^2u_s .
\]
Thus, the second-order term
simplifies to:
\[
\begin{aligned}
&\frac12\partial_{ss}u_s
+\mathcal L_s\partial_su_s
+\frac12 v_s^\top\nabla^2u_s v_s
+a_s v_s\cdot\nabla\Delta u_s
+\frac12 a_s^2\Delta^2u_s  \\
&=
-\frac12(\partial_s\mathcal L_s)u_s
-\frac12\mathcal L_s^2u_s
+\frac12 v_s^\top\nabla^2u_s v_s
+a_s v_s\cdot\nabla\Delta u_s
+\frac12 a_s^2\Delta^2u_s .
\end{aligned}
\]
Since
\(
\mathcal L_s u_s
=
v_s\cdot\nabla u_s+a_s\Delta u_s,
\)
we have
\[
\mathcal L_s^2u_s
=
(v_s\cdot\nabla)\big(v_s\cdot\nabla u_s\big)
+a_s\Delta\big(v_s\cdot\nabla u_s\big)
+a_s v_s\cdot\nabla\Delta u_s
+a_s^2\Delta^2u_s .
\]
We now expand the two first terms:
\[
(v_s\cdot\nabla)\big(v_s\cdot\nabla u_s\big)
=
\big((v_s\cdot\nabla)v_s\big)\cdot\nabla u_s
+
v_s^\top\nabla^2u_s\,v_s .
\]
and
\[
\Delta\big(v_s\cdot\nabla u_s\big)
=
(\Delta v_s)\cdot\nabla u_s
+
2\langle\nabla v_s,\nabla^2u_s\rangle
+
v_s\cdot\nabla\Delta u_s .
\]
Therefore,
\[
\mathcal L_s^2u_s
=
\big((v_s\cdot\nabla)v_s+a_s\Delta v_s\big)\cdot\nabla u_s
+v_s^\top\nabla^2u_s v_s
+2a_s v_s\cdot\nabla\Delta u_s
+2a_s\langle\nabla v_s,\nabla^2u_s\rangle
+a_s^2\Delta^2u_s .
\]
Moreover,
\[
(\partial_s\mathcal L_s)u_s
=
(\partial_sv_s)\cdot\nabla u_s+\dot a_s\Delta u_s.
\]
Therefore the second-order coefficient finally simplifies to
\[
-\frac12\big((\partial_s+\mathcal L_s)v_s\big)\cdot\nabla u_s
-a_s\langle \nabla v_s,\nabla^2u_s\rangle
-\frac12\dot a_s\Delta u_s,
\]
which proves the claim.
\end{proof}
We now sum these local defects along the Euler trajectory from time $0$ to time $t_k > 0$. Since $u_{t_k}(x)=f(x)$,
\[
\EE[f(\hat Y_k)]
=
\EE[u_{t_k}(\hat Y_k)].
\]
Moreover, by the definition of \(u_0\),
$u_0(x)=\EE[f(Y_{t_k}^{0,x})]$,
and therefore, using \(Y_{t_k}^{0,Y_0}=Y_{t_k}\),
\[
\EE[f(Y_{t_k})]
=
\EE[u_0(Y_0)].
\]
We get:
\[
\mathbb E[f(\hat Y_k)]-\mathbb E[f(Y_{t_k})]
=
\sum_{\ell=0}^{k-1}
\mathbb E\Big[
u_{t_{\ell+1}}(\hat Y_{\ell+1})
-u_{t_\ell}(\hat Y_\ell)
\Big].
\]
Applying Lemma~\ref{lem:one_step} gives
\[
\mathbb E[f(\hat Y_k)]-\mathbb E[f(Y_{t_k})]
=
\gamma^2\sum_{\ell=0}^{k-1}
\mathbb E[\psi_{t_\ell}(\hat Y_\ell)]
+O(\gamma^2).
\]
Here the remainder is $O(\gamma^2)$. Indeed, the Taylor remainder is $O(\gamma^3)$ per step, with its expectation
bounded by $C\gamma^3\mathbb{E}[1+|\hat Y_\ell|^{m+6}]$. Summing over
the $K=O(\gamma^{-1})$ steps gives a total remainder of $O(\gamma^2)$.

To express the accumulated local defect in terms of the exact process $Y_{t_\ell}$,
we replace its expectation at the Euler iterate by its expectation at
the corresponding exact iterate. $\psi_s$ has four spatial derivatives with growth at most
$1+|x|^{m+1}$. Applying the one-step weak expansion to $\psi_{t_\ell}$
gives a local error of $O(\gamma^2)$; summing over at most $K$ steps
therefore yields
\[
\left|
\mathbb{E}[\psi_{t_\ell}(\hat Y_\ell)]
-
\mathbb{E}[\psi_{t_\ell}(Y_{t_\ell})]
\right|
\le C\gamma,
\]
uniformly in $\ell$. Therefore,
\[
\gamma^2\sum_{\ell=0}^{k-1}
\EE[\psi_{t_\ell}(\widehat Y_\ell)]
=
\gamma^2\sum_{\ell=0}^{k-1}
\EE[\psi_{t_\ell}(Y_{t_\ell})]
+O(\gamma^2).
\]
Our assumptions make
$s\mapsto\EE[\psi_s(Y_s)]$ continuously differentiable with bounded derivative.
Hence the Riemann-sum error satisfies
\[
\gamma^2\sum_{\ell=0}^{k-1}
\EE[\psi_{t_\ell}(Y_{t_\ell})]
=
\gamma\int_0^{t_k} \EE[\psi_s(Y_s)]\,ds+O(\gamma^2).
\]
which gives
\begin{equation}
\label{eq:error_phi_s}
\mathbb E[f(\hat Y_k)]-\mathbb E[f(Y_{t_k})]
=
\gamma\int_0^{t_k}\mathbb E[\psi_s(Y_s)]\,ds+O(\gamma^2).
\end{equation}

It remains to rewrite \(\psi_s(Y_s)\) in terms of derivatives of \(f\) at the
final time \(t_k\).

Let $J_{t,s}(x) = \nabla_x Y_{t}^{s,x}$ be the Jacobian, at time $t$, of the process starting from $x$ at time $s$; as well as $H_{t,s}(x) \eqdef \nabla^2_x Y_{t}^{s,x}$ the Hessian process, which is a 3-tensor and the Laplacian process  $\Delta_{t,s}(x) \eqdef \left(\Delta_c (Y_{t}^{s,x})_i\right)_{i \in \llbracket 1,d \rrbracket}$.
  By differentiating $u_s(x) = \EE\big[f(Y_{t}^{s,x})\big] $ w.r.t the initial condition $x$ we get:
  \begin{align}\label{eq:nabla_u} \nabla u_s(x) = \EE\Big[J_{t,s}(x)^\top \nabla f(Y_{t}^{s,x})\Big]
  \end{align}
  \begin{align}\label{eq:nabla2_u} \nabla^2 u_s(x) &= \EE\Big[H_{t,s}(x)\big[\nabla f(Y_{t}^{s,x})\big] + J_{t,s}(x)^\top \nabla^2 f(Y_{t}^{s,x}) J_{t,s}(x)  \Big]
  \end{align}
  where we denote, for the $3$-tensor $H_{t,s}$ and $v \in \RR^d$, the contraction $H_{t,s}[v]\in \RR^{d \times d}$ with $(H_{t,s}[v])_{k,l} = \sum_{i=1}^d(H_{t,s})_{ikl} v_i$.
  \begin{align}\label{eq:Delta_u} \Delta u_s(x) &= \EE\Big[\Delta_{t,s}(x) \cdot \nabla f(Y_{t}^{s,x}) + \langle J_{t,s}(x) J_{t,s}(x)^\top, \nabla^2 f(Y_{t}^{s,x})\rangle \Big]
  \end{align}
   Substituting \eqref{eq:nabla_u}, \eqref{eq:nabla2_u} and \eqref{eq:Delta_u} into the expression of $\psi_s$ gives
\begin{align}
  \psi_s(x)
  &= - \frac12 \EE\Big[\Big(J_{t,s}(x)\big[(\partial_s+\mathcal{L}_s)v_s\big](x)
    + 2 a_s\, [H_{t,s}(x),\nabla v_s(x)]
      + \dot a_s\, \Delta_{t,s}(x)\Big)\cdot \nabla f(Y_t^{s,x})\Big] \nonumber\\
  &\quad - a_s \EE\Big[\big\langle \nabla v_s(x), J_{t,s}(x)^\top \nabla^2 f(Y_t^{s,x}) J_{t,s}(x)\big\rangle\Big]
      - \frac12 \dot a_s \EE\Big[\big\langle J_{t,s}(x)J_{t,s}(x)^\top,\nabla^2 f(Y_t^{s,x})\big\rangle\Big].
  \label{eq:psi_substitution}
\end{align}
Using that for a matrix $A \in \RR^{d\times d}$
\[
\big\langle \nabla v_s(x), J_{t,s}(x)^\top A J_{t,s}(x)\big\rangle
=
\big\langle J_{t,s}(x)\nabla v_s(x)J_{t,s}(x)^\top, A\big\rangle,
\]
and the fact that, since $\nabla^2 f(Y_t^{s,x})$ is symmetric, only the symmetric part of
$J_{t,s}(x)\nabla v_s(x)J_{t,s}(x)^\top$ contributes, the two Hessian terms in~\eqref{eq:psi_substitution} combine as
\[
\big\langle
-\tfrac12J_{t,s}(x)
\big[a_s(\nabla v_s(x)+\nabla v_s(x)^\top)+\dot a_s I_d\big]
J_{t,s}(x)^\top,\,
\nabla^2 f(Y_t^{s,x})
\big\rangle
=
\big\langle E_{t,s}(x),\nabla^2 f(Y_t^{s,x})\big\rangle,
\]
where we set
\begin{align}
  e_{t,s}(x)
  &\eqdef
  -\frac12 J_{t,s}(x)\big[(\partial_s+\mathcal{L}_s)v_s\big](x)
  - a_s\, [H_{t,s}(x),\nabla v_s(x)]
  - \frac12 \dot a_s\, \Delta_{t,s}(x)
  \in \RR^d, \label{eq:def_e_x}\\
  E_{t,s}(x)
  &\eqdef
  -\frac12 J_{t,s}(x)
  \Big(a_s\big(\nabla v_s(x)+\nabla v_s(x)^\top\big)+\dot a_s I_d\Big)
  J_{t,s}(x)^\top
  \in \RR^{d\times d}. \label{eq:def_E_x}
\end{align}
Hence \eqref{eq:psi_substitution} rewrites as
\begin{align}
  \psi_s(x)
  =
  \EE\Big[
    e_{t,s}(x)\cdot \nabla f(Y_t^{s,x})
    + \big\langle E_{t,s}(x), \nabla^2 f(Y_t^{s,x})\big\rangle
  \Big].
  \label{eq:psi_det_x}
\end{align}

We now replace the deterministic initial condition \(x\) by the random state
\(Y_s\). Denoting by \(\mathcal F_s\) the natural filtration of the process,
\(Y_s\) is \(\mathcal F_s\)-measurable and the future Brownian increments are
independent of \(\mathcal F_s\). Therefore, \eqref{eq:psi_det_x} yields
\begin{align*}
  \psi_s(Y_s)
  &=
  \EE\Big[
    e_{t,s}(Y_s)\cdot \nabla f(Y_t^{s,Y_s})
    + \big\langle E_{t,s}(Y_s), \nabla^2 f(Y_t^{s,Y_s})\big\rangle
    \,\Big|\, \mathcal F_s
  \Big].
\end{align*}
Using the flow property $Y_t^{s,Y_s}=Y_t$, and the notations
\[
J_{t,s}(Y)\eqdef J_{t,s}(Y_s), \qquad
H_{t,s}(Y)\eqdef H_{t,s}(Y_s), \qquad
\Delta_{t,s}(Y)\eqdef \Delta_{t,s}(Y_s),
\]
we obtain
\begin{align*}
  \psi_s(Y_s)
  &=
  \EE\Big[
    e_{t,s}(Y)\cdot \nabla f(Y_t)
    + \big\langle E_{t,s}(Y), \nabla^2 f(Y_t)\big\rangle
    \,\Big|\, \mathcal F_s
  \Big].
\end{align*}
Taking expectation gives
\begin{align}
  \EE\big[\psi_s(Y_s)\big]
  =
  \EE\Big[
    e_{t,s}(Y)\cdot \nabla f(Y_t)
    + \big\langle E_{t,s}(Y), \nabla^2 f(Y_t)\big\rangle
  \Big].
  \label{eq:psi_along_flow}
\end{align}
  and finally:
\begin{align*}
  \EE\big[f(\hat Y_k)\big]-\EE\big[f(Y_{t_k})\big]
  &=
  \gamma \int_0^{t_k}
  \EE\Big[
    e_{t_k,s}(Y)\cdot \nabla f(Y_{t_k})
    + \big\langle E_{t_k,s}(Y), \nabla^2 f(Y_{t_k})\big\rangle
  \Big] ds
  + O(\gamma^2),
\end{align*}
which concludes the proof.
\end{proof}

We finally give here the proof of the Taylor expansion~\eqref{eq:taylor_expansion} from Lemma~\ref{lem:taylor_expansion}.
\begin{proof}[Proof of Lemma~\ref{lem:taylor_expansion}]
Fix $s$ and $x$, and introduce the frozen, constant-coefficient operator
\[
\bar{\mathcal L}_{s,x}g(z)
\coloneqq v_s(x)\cdot\nabla g(z)+a_s\Delta g(z).
\]
Let $P_h^{s,x}$ be its semigroup, so that
\[
P_h^{s,x}g(x)
=\EE\big[g(x+h v_s(x)+\sqrt{2h a_s}\,\xi)\big].
\]
Set $F(h)\coloneqq P_h^{s,x}u_{s+h}(x)$ for $0\leq h\leq\gamma$. Since the coefficients of $\bar{\mathcal L}_{s,x}$ are frozen, differentiation under the expectation gives, for $j\leq3$,
\[
F^{(j)}(h)
=P_h^{s,x}\big[(\partial_s+\bar{\mathcal L}_{s,x})^j u_{s+h}\big](x).
\]
Under Assumption~\ref{ass:weak_regularity}, $u_s$ has continuous mixed derivatives
$\partial_s^jD_x^ru_s$ for $2j+|r|\le 6$, with
\[
|\partial_s^jD_x^ru_s(x)|\le C(1+|x|^{m+j}) \leq C(1+|x|^{m+3})
\]
Together with the linear growth of $v_s$ and Gaussian moment bounds, this yields
$|F^{(3)}(h)|\leq C(1+|x|^{m+6})$, uniformly in $s,x,h$ in the stated ranges. 

Moreover 
\[
F(0)=u_s(x),\qquad
F'(0)=(\partial_s+\bar{\mathcal L}_{s,x})u_s(x), \qquad 
F''(0)
=\bigl(\partial_{ss}
+2\bar{\mathcal L}_{s,x}\partial_s
+\bar{\mathcal L}_{s,x}^{2}\bigr)u_s(x).
\]

Taylor's theorem therefore gives
\[
\begin{aligned}
F(\gamma)
&=u_s(x)
+\gamma(\partial_s+\bar{\mathcal L}_{s,x})u_s(x)\\
&\quad+\frac{\gamma^2}{2}
\big(\partial_{ss}+2\bar{\mathcal L}_{s,x}\partial_s
+\bar{\mathcal L}_{s,x}^{2}\big)u_s(x)
+O\!\big(\gamma^3(1+|x|^{m+6})\big).
\end{aligned}
\]
Finally,
\[
\bar{\mathcal L}_{s,x}^{2}u_s(x)
=v_s(x)^\top\nabla^2u_s(x)v_s(x)
+2a_s v_s(x)\cdot\nabla\Delta u_s(x)
+a_s^2\Delta^2u_s(x).
\]
Substituting this identity and expanding $\bar{\mathcal L}_{s,x}\partial_su_s$ proves~\eqref{eq:taylor_expansion}.
\end{proof}

\section{Proof of the mean and covariance errors}
  \label{app:mean_cov}

  \begin{proof}

  \textbf{Error on the mean.} For $i \in \llbracket 1,d \rrbracket$, we apply Theorem~\ref{thm:general_weak} with $f^{(i)}(x)=x_i$ to get, in vector form,
  \begin{align} \label{eq:mean_disc_error_proof}
     \EE[\hat Y_{k}] - \EE[Y_{t_k}]
    &= \gamma  \int_0^{t_k}  \EE\Big[e_{t_k,s}(Y)  \Big] ds +O(\gamma^2)
  \end{align}
 The corollary directly applies the above result at final iteration $K$ i.e. $t_k = T$, for which $\EE[Y_{T}] = \mu_\data$.

  \textbf{Error on the covariance.}
  For $i,j \in \llbracket 1,d \rrbracket$, we now apply the above result with $f^{(i,j)}(x)=x_i x_j$.  Then $\nabla f^{(i,j)}(x) = x_j e_i + x_i e_j$ and $\nabla^2 f^{(i,j)}(x) = E_{ij} + E_{ji}$. Applying Theorem~\ref{thm:general_weak} componentwise gives
  \begin{align*}
      \EE[ (\hat Y_{k})_i (\hat Y_{k})_j] &- \EE[(Y_{t_k})_i (Y_{t_k})_j] \\
    &= \gamma \int_0^{t_k}
  \EE\Big[
    e_{t_k,s}(Y)\cdot \nabla f^{(i,j)}(Y_{t_k})
    + \big\langle E_{t_k,s}(Y), \nabla^2 f^{(i,j)}(Y_{t_k})\big\rangle
  \Big] ds
  + O(\gamma^2). \nonumber
  \end{align*}
  Using that $ E_{t_k,s}(Y)$ is symmetric,
  \begin{align*}
  \langle E_{t_k,s}(Y), \nabla^2 f^{(i,j)}(Y_{t_k})\big\rangle = 2 \langle E_{t_k,s}(Y), E_{i,j} \rangle = 2 (E_{t_k,s}(Y))_{ij}.
  \end{align*}
  In matrix form:
  \begin{align*}
    \EE[\hat Y_{k} \hat Y_{k}^\top] - \EE[Y_{t_k} Y_{t_k} ^\top]  &= \gamma \int_0^{t_k} \, \EE\Big[e_{t_k,s}(Y) Y_{t_k}^\top +  Y_{t_k} e_{t_k,s}(Y)^\top  + 2E_{t_k,s}(Y) \Big] ds + O(\gamma^2)
  \end{align*}
  We now deduce the error on the covariance. Denoting $m_t = \EE[Y_{t}]$ and $K_t = \EE[Y_{t} Y_{t}^\top]$, as well as $\hat m_k = \EE[\hat Y_{k}]$ and $\hat K_k \eqdef \EE[\hat Y_{k} \hat Y_{k}^\top]$, the covariance verifies $\cov(Y_t) = K_t-m_tm_t^\top$ and $\cov(\hat Y_{k}) = \hat K_k- \hat m_k \hat m_k^\top$. Then:
  \begin{align}\label{eq:cov-decomp}
  \cov(\hat Y_{k}) - \cov(Y_{t_k})
    &=(\hat  K_k - K_{t_k} )-\big(  \hat m_k \hat m_k^\top - m_{t_k}m_{t_k}^\top\big).
  \end{align}
  Using
  \begin{align*}
    m_{t_k} m_{t_k}^\top - \hat m_k \hat m_k^\top=  (m_{t_k}-\hat m_k)m_{t_k}^\top + m_{t_k}(m_{t_k}-\hat m_k)^\top - (m_{t_k}-\hat m_k)(m_{t_k} - \hat m_k)^\top,
  \end{align*}
  and the fact that $m_{t_k} - \hat m_k =O(\gamma)$, we have:
  \begin{align*}
    m_{t_k}m_{t_k}^\top-\hat  m_k \hat  m_k^\top
    =(m_{t_k}-\hat  m_k)m_{t_k}^\top+m_{t_k}(m_{t_k}-\hat  m_k)^\top+ O(\gamma^2)
  \end{align*}
  Thus, using \eqref{eq:mean_disc_error}, the first order discretization error on the covariance writes
  \begin{align}  \label{eq:cov_disc_error_proof}
    \cov[\hat Y_{k}] - \cov[Y_{t_k}]  &= \gamma \int_0^{t_k}  \Big(\cov\Big[e_{t_k,s}(Y), Y_{t_k}\Big]  +  \cov\Big[e_{t_k,s}(Y), Y_{t_k}\Big]^\top + 2 \EE\Big[E_{t_k,s}(Y)\Big] \Big) ds + O(\gamma^2)
  \end{align}
  The corollary directly applies the above result at final iteration $K$ i.e. $t_k = T$, for which $\cov[Y_{T}] = \Sigma_\data$.
   \end{proof}

\section{Fréchet distance expansions}
  \label{app:Fréchet_dist_exp_general}

The first-order errors on the mean and on the covariance calculated in~\eqref{eq:mean_disc_error} and ~\eqref{eq:cov_disc_error} can be combined into a first-order expansion of the Fr\'echet distance to the target distribution. The Fr\'echet distance compares the first and second order moments of two distributions, and is defined by
\begin{align} \label{eq:frechet_distance}
  \FD(p_\mathrm{data}, \operatorname{Law}(y)) = \Big \|\mu_\mathrm{data} - \EE[y]\Big \|^2 + \mathcal{B}^2\Big(\Sigma_\mathrm{data},\cov[y]\Big)
\end{align}
with the Bures distance between covariances:
\begin{align} \label{eq:Bures}
  \mathcal{B}^2(\Sigma_{\mathrm{data}}, \Sigma) &= \tr{ \Sigma_{\mathrm{data}} + \Sigma - 2 (\sqrt{\Sigma_{\mathrm{data}}} \Sigma \sqrt{\Sigma_{\mathrm{data}}})^{1/2} } .
\end{align}
We consider the Fréchet sampling error between the distribution of the final iterate $\operatorname{Law}(\hat Y_K)$ and the data distribution $p_{\data}$. We express this Fréchet error with respect to the errors on the mean~\eqref{eq:mean_disc_error} and covariance~\eqref{eq:cov_disc_error} errors, which we denote:
   \begin{align*} d^\mu_k = \EE[\hat Y_k] - \mu_{\data} = \gamma d^{\mu,[1]}_k + o(\gamma)  \quad \text{and} \quad
      D^\Sigma_k = \cov[\hat Y_k] -  \Sigma_{\data} &= \gamma D^{\Sigma,[1]}_k + o(\gamma)
    \end{align*}
We also decompose the covariance error making use of the eigencomposition of the data covariance  $\Sigma_\data = U \operatorname{Diag}(\lambda_i)U^\top$.
\begin{prop}[Fréchet distance expansion]
\label{prop:Fréchet_disc}
 With the previous notations:
\begin{align*}
  \FD(p_\mathrm{data}, \operatorname{Law}(\hat Y_K))
  &=
  \gamma^2 \Big \|d^{\mu,[1]}_K\Big \|^2  + \frac{\gamma^2}{2} \sum_{i,j =1}^d
  \frac{\Big(u_i^\top D^{\Sigma,[1]}_K u_j\Big)^2}{\lambda_i +\lambda_j}
  + o(\gamma^2).
\end{align*}
\end{prop}
\begin{proof}
\cite{malago2018wasserstein} give the first order expansion of the Bures distance: for $\Sigma_0  \in \operatorname{Sym}^{++}(d)$, and $\Sigma_1 \in \operatorname{Sym}(d)$ such that for small enough $\varepsilon$, ${\Sigma_0 \pm \varepsilon\Sigma_1 \in  \operatorname{Sym}^{+}(d)}$,
    \begin{align*}
      \mathcal{B}^2(\Sigma_0, \Sigma_0 + \varepsilon \Sigma_1) &= \frac{\varepsilon^2 }{2}\tr{ \Sigma_1 L_{\Sigma_0}^{-1}[\Sigma_1] } +o(\varepsilon^2)
    \end{align*}
where we used the matrix Lyapunov operator, defined for a real matrix $C$ by $L_C[X] = CX + XC^\top$.

   Using $\cov[\hat Y_K] = \Sigma_{\data} + \gamma D^{\Sigma,[1]}_K + o(\gamma)$, we get
   \begin{align*}
     \mathcal{B}^2 \Big(\Sigma_\data, \cov[\hat Y_K] \Big) =  \mathcal{B}^2 \Big(\Sigma_\data, \Sigma_\data + \gamma D^{\Sigma,[1]}_K \Big) = \frac{\gamma^2}{2} \tr{D^{\Sigma,[1]}_K L_{\Sigma_\mathrm{data}}^{-1}\left[D^{\Sigma,[1]}_K \right]} + o(\gamma^2)
  \end{align*}
  We can then use the decomposition of the inverse Lyapunov operator $L_{C}^{-1}$ in the eigenbasis of the data covariance, given in the following Lemma:
  \begin{lem} \label{lem:SGD_3} For $C \in \operatorname{Sym}^{++}(d)$, denoting $C = U \diag{\lambda_i}_{1 \leq i\leq d} U^\top$ the diagonalization of $C$ in its eigenbasis, the inverse operator of $L_C[X]$ decomposes as
    \begin{align*}
      (L_{ C})^{-1}[Z] = U \left( U^\top Z U \odot \left(\frac{1}{\lambda_i + \lambda_j}\right)_{1 \leq i,j \leq d} \right) U^\top.
    \end{align*}
  \end{lem}
  The proof follows directly from $L_C^{-1}[Z] = X \Leftrightarrow U^\top(CX + XC^{\top}) U = U^\top Z U$. When applied to our $L_{\Sigma_\mathrm{data}}^{-1}\left[D^{\Sigma,[1]}_K \right]$, it gives:
    \begin{align*}
   L_{\Sigma_\mathrm{data}}^{-1}\left[D^{\Sigma,[1]}_K \right] &= U \left( U^\top D^{\Sigma,[1]}_K U \odot \left(\frac{1}{\lambda_i + \lambda_j}\right)_{1 \leq i,j \leq d} \right) U^\top
  \end{align*}
  We can finally deduce:
  \begin{align*}
    \frac{\gamma^2 }{2} \tr{ D^{\Sigma,[1]}_K L_{\Sigma_\mathrm{data}}^{-1}\left[D^{\Sigma,[1]}_K \right]  } &= \frac{\gamma^2 }{2} \sum_{i,j} \frac{(U^\top D^{\Sigma,[1]}_K U)^2_{i,j}}{\lambda_i + \lambda_j}
  \end{align*}
\end{proof}

\subsection{For Gaussian data}
\label{app:Fréchet_dist_exp_gaussian}

 \paragraph{First-order expansion} Using the above notations, under Gaussian data ${p_\data \sim \mathcal{N}(\mu_\data, \Sigma_\data)}$, Proposition~\ref{prop:mean_cov_Gaussian} gives the mean and covariance errors decomposed along the eigenspace of the data covariance ${\Sigma_\data = U \diag{\lambda_i} U^\top}$:
  \begin{align*}
  d^{\mu,[1]}_K &= U \diag{\Delta^{\mu,[1]}_K (\lambda_i)}U^\top \mu_\data\\
  D^{\Sigma,[1]}_K &= U \diag{\Delta^{\Sigma,[1]}_K(\lambda_i)} U^\top
    \end{align*}
Applying the above Proposition~\ref{prop:Fréchet_disc} gives directly
  \begin{align*}
    \operatorname{FD}(p_\mathrm{data}, \operatorname{Law}(\hat Y_K))
    =
    \gamma^2\sum_{i=1}^d \Delta^{\mu, [1]}(\lambda_i)^2 (u_i^\top \mu_\data)^2
    +
    \frac{\gamma^2}{4}\sum_{i=1}^d \frac{\Delta^{\Sigma, [1]}(\lambda_i)^2}{\lambda_i} + o(\gamma^2)
  \end{align*}
Instead of using the general decomposition of Proposition~\ref{prop:Fréchet_disc}, one could also simply use the per-eigenvalue decomposition of the Bures distance, as used below.

\paragraph{Other smooth divergences.}
Assume $\lambda_i>0$ for every $i$. Since both laws in this subsection are Gaussian, their KL divergence and any $\varphi$-divergence whose generator is twice differentiable at $1$, with $\varphi(1)=0$, have the local expansion
\begin{align*}
D_\varphi\!\left(p_\data\,\middle\|\,\operatorname{Law}(\hat Y_K)\right)
&=\frac{\varphi''(1)\gamma^2}{2}
\sum_{i=1}^d\frac{\Delta^{\mu,[1]}(\lambda_i)^2(u_i^\top\mu_\data)^2}{\lambda_i}\\
&\quad+\frac{\varphi''(1)\gamma^2}{4}
\sum_{i=1}^d\frac{\Delta^{\Sigma,[1]}(\lambda_i)^2}{\lambda_i^2}
+o(\gamma^2).
\end{align*}
For KL, $\varphi(s)=s\log s$ and $\varphi''(1)=1$. To obtain the displayed terms, expand the Gaussian log-density around $(\mu_\data,\Sigma_\data)$ with mean perturbation $\gamma d_K^{\mu,[1]}$ and covariance perturbation $\gamma D_K^{\Sigma,[1]}$, then diagonalize in the common eigenbasis. Relative to FD, these objectives give additional weight to low-variance directions. Total variation is not covered by this quadratic $\varphi$ expansion because its generator is nonsmooth at $1$.

 \paragraph{Higher order expansion}
  In Proposition~\ref{prop:mean_cov_Gaussian},
  we have seen that the mean and covariance errors fully decompose in the eigenbasis of $\Sigma_\data = U \operatorname{Diag}(\lambda_i) U^\top $ as:
  \begin{align*}
  \EE[\hat Y_{K}]
  &= \mu_\data  + U \operatorname{Diag}\Big( \Delta^\mu(\lambda_i)\Big) U^\top \mu_\data,\\
  \cov[\hat Y_{K}]
  &= \Sigma_\data + U \operatorname{Diag}\Big( \Delta^\Sigma(\lambda_i)\Big) U^\top .
\end{align*}
The Fréchet distance~\eqref{eq:frechet_distance} simplifies directly to:
\begin{align*}
  \FD(p_\mathrm{data}, \operatorname{Law}(y)) &= \Big \| U \operatorname{Diag}\Big( \Delta^\mu(\lambda_i)\Big) U^\top \mu_\data \Big \|^2 + \mathcal{B}^2\Big(U \operatorname{Diag}(\lambda_i) U^\top,U \operatorname{Diag}\Big(\lambda_i +  \Delta^\Sigma(\lambda_i)\Big) U^\top\Big) \\
  &= \sum_{i=1}^d \Delta^\mu(\lambda_i)^2 (u_i^\top \mu_\data)^2 + \sum_{i=1}^d \left( \sqrt{\lambda_i} -  \sqrt{\lambda_i + \Delta^\Sigma(\lambda_i)} \right)^2
\end{align*}
Proposition~\ref{prop:mean_cov_Gaussian} also gives a second-order expansion of the mean and covariance defects:
\begin{align*}
\Delta^{\mu}(\lambda)
    =
    \gamma \Delta^{\mu, [1]}(\lambda)
    +
    \gamma^2 \Delta^{\mu, [2]}(\lambda)
    +
    O(\gamma^3) \\
    \Delta^{\Sigma}(\lambda)
    =
    \gamma \Delta^{\Sigma, [1]}(\lambda)
    +
    \gamma^2 \Delta^{\Sigma, [2]}(\lambda)
    +
    O(\gamma^3)
\end{align*}
Expanding the square for the mean term,
\begin{align*}
  \Delta^\mu(\lambda_i)^2
  &=
  \gamma^2
  \left(\Delta^{\mu,[1]}(\lambda_i)\right)^2
  +
  2\gamma^3
  \Delta^{\mu,[1]}(\lambda_i)
  \Delta^{\mu,[2]}(\lambda_i)
  +
  O(\gamma^4).
\end{align*}
For the covariance term, using the square-root expansion
\begin{align*}
  \left(\sqrt{\lambda+\delta}-\sqrt{\lambda}\right)^2
  =
  \frac{\delta^2}{4\lambda}
  -
  \frac{\delta^3}{8\lambda^2}
  +
  O(\delta^4),
\end{align*}
we get
\begin{align*}
  \left(
  \sqrt{\lambda_i}
  -
  \sqrt{\lambda_i+\Delta^\Sigma(\lambda_i)}
  \right)^2
  &=
  \frac{\gamma^2}{4\lambda_i}
  \left(\Delta^{\Sigma,[1]}(\lambda_i)\right)^2  \\
  &\quad
  +
  \gamma^3
  \left[
  \frac{
  \Delta^{\Sigma,[1]}(\lambda_i)
  \Delta^{\Sigma,[2]}(\lambda_i)
  }{2\lambda_i}
  -
  \frac{
  \left(\Delta^{\Sigma,[1]}(\lambda_i)\right)^3
  }{8\lambda_i^2}
  \right]
  +
  O(\gamma^4).
\end{align*}
Therefore,
\begin{align} \label{eq:third_order}
  \FD(p_\mathrm{data}, \operatorname{Law}(\hat Y_K))
  &=
  \gamma^2 \Big(\mathcal E^{[0]}
  +
  \gamma \mathcal E^{[1]} \Big)
  +
  O(\gamma^4),
\end{align}
where
\begin{align*}
  \mathcal E^{[0]}
  &\eqdef
  \sum_{i=1}^d
  \left(\Delta^{\mu,[1]}(\lambda_i)\right)^2
  (u_i^\top \mu_\data)^2
  +
  \frac14
  \sum_{i=1}^d
  \frac{
  \left(\Delta^{\Sigma,[1]}(\lambda_i)\right)^2
  }{\lambda_i},
  \\
  \mathcal E^{[1]}
  &\eqdef
  2
  \sum_{i=1}^d
  \Delta^{\mu,[1]}(\lambda_i)
  \Delta^{\mu,[2]}(\lambda_i)
  (u_i^\top \mu_\data)^2
  +
  \sum_{i=1}^d
  \left[
  \frac{
  \Delta^{\Sigma,[1]}(\lambda_i)
  \Delta^{\Sigma,[2]}(\lambda_i)
  }{2\lambda_i}
  -
  \frac{
  \left(\Delta^{\Sigma,[1]}(\lambda_i)\right)^3
  }{8\lambda_i^2}
  \right].
\end{align*}

\section{Affine drift mean and covariance expansions}
  \label{app:affine_velocity}

  Proposition~\ref{prop:affine_velocity} gives the mean and
    covariance errors for a general linear drift, deriving the
    first-order terms from the general weak error Theorem~\ref{thm:general_weak} and the
    second-order terms from the discrete moment recursions.

  \begin{prop}[Affine drift mean and covariance errors]
  \label{prop:affine_velocity}
  Under the affine drift $v_t(x)=H_t x+r_t$, the mean and covariance biases admit the expansions
  \begin{align}
    \EE[\hat Y_k] - \EE[Y_{t_k}]
    &=
    \gamma d^{\mu,[1]}_{t_k} + \gamma^2 d^{\mu,[2]}_{t_k} + O(\gamma^3),
    \label{eq:affine_mean_disc_error} \\
    \cov[\hat Y_k] - \cov[Y_{t_k}]
    &=
    \gamma D^{\Sigma,[1]}_{t_k} + \gamma^2 D^{\Sigma,[2]}_{t_k} + O(\gamma^3),
    \label{eq:affine_cov_disc_error}
  \end{align}
  where, denoting $m_s \eqdef \EE[Y_s]$ and $C_s \eqdef \cov(Y_s)$,
  \begin{align}
    d^{\mu,[1]}_{t_k}
    &= -\frac12 \int_0^{t_k} J_{t_k,s} \ddot m_s \, ds,  \label{eq:affine_first_order_mean}
    \\
    D^{\Sigma,[1]}_{t_k}
    &= -\int_0^{t_k} J_{t_k,s}
    \left(\frac12 \ddot C_s - H_s C_s H_s^\top\right)
    J_{t_k,s}^\top \, ds,
     \label{eq:affine_first_order_cov}
    \\
    d^{\mu,[2]}_{t_k}
    &=
    \int_0^{t_k} J_{t_k,s}
    \left[
      - \frac12(\dot H_s + H_s^2)d^{\mu,[1]}_s
      + \frac14 H_s \ddot m_s
      + \frac1{12}\dddot m_s
    \right] ds,
    \label{eq:affine_second_order_mean} \\
    D^{\Sigma,[2]}_{t_k}
    &=
    \int_0^{t_k} J_{t_k,s}
    \left[
    \begin{aligned}
      &- \frac12(\dot H_s + H_s^2)D^{\Sigma,[1]}_s
      - \frac12 D^{\Sigma,[1]}_s(\dot H_s^\top + (H_s^\top)^2) \\
      &+ \frac14 H_s \ddot C_s
      + \frac14 \ddot C_s H_s^\top
      + \frac1{12} \dddot C_s \\
      &- \frac12 \dot H_s C_s H_s^\top
      - \frac12 H_s C_s \dot H_s^\top \\
      &- H_s^2 C_s H_s^\top
      - H_s C_s (H_s^\top)^2
      - a_s H_s H_s^\top
    \end{aligned}
    \right]
    J_{t_k,s}^\top\, ds.
    \label{eq:affine_second_order_cov}
  \end{align}
  \end{prop}
  \begin{proof}
  We consider the affine drift $v_s(x) = H_s x + r_s$. The Jacobian $J_{t,s}$ then verifies
  \begin{align*}
    \frac{dJ_{t,s}}{dt} = H_t J_{t,s}, \qquad J_{s,s} = \id,
  \end{align*}
  and is deterministic. Moreover, the Hessian and Laplacian of the flow vanish:
  \[
  H_{t,s}(Y) = 0, \qquad \Delta_{t,s}(Y) = 0.
  \]
  The quantities $e_{t,s}$ and $E_{t,s}$ from Theorem~\ref{thm:general_weak} simplify to:
  \begin{align}
    e_{t,s}(Y) &= -\frac12 J_{t,s}\big[\dot H_s Y_s + \dot r_s + H_s(H_s Y_s + r_s) \big], \label{eq:e_linear}\\
    E_{t,s}(Y) &= -\frac12 J_{t,s}\big[ a_s(H_s + H_s^\top) + \dot a_s \id \big] J_{t,s}^\top. \label{eq:E_linear}
  \end{align}

  \paragraph{First order coefficients} For the first order error, we specialize to the linear drift the mean~\eqref{eq:mean_disc_error} and covariance errors~\eqref{eq:cov_disc_error} obtained via the general weak error expansion of Theorem~\ref{thm:general_weak}. Note that the mean $m_s \eqdef \EE[Y_s]$ and covariance $C_s \eqdef \cov(Y_s)$ along the continuous SDE~\eqref{eq:backward_SDE} follow the ODEs:
  \begin{align} \label{eq:mean_ODE}
    \dot m_s &= H_s m_s + r_s, \\
    \dot C_s &= H_s C_s + C_s H_s^\top + 2 a_s \id.
    \label{eq:cov_ODE}
  \end{align}
  Differentiating~\eqref{eq:mean_ODE} gives $$
  \ddot m_s = \dot H_s m_s + H_s (H_s m_s + r_s) + \dot r_s
  $$
  Therefore, applying~\eqref{eq:mean_disc_error_proof} with~\eqref{eq:e_linear} gives
  \begin{align*}
    \EE[\hat Y_{k}] - \EE[Y_{t_k}]
    &= \gamma \int_0^{t_k} \EE[e_{t_k,s}(Y)] ds + O(\gamma^2) \\
    &= -\frac{\gamma}{2} \int_0^{t_k} J_{t_k,s}\big[\dot H_s m_s + \dot r_s + H_s(H_s m_s + r_s)\big] ds + O(\gamma^2) \nonumber \\
    &= -\frac{\gamma}{2} \int_0^{t_k} J_{t_k,s} \ddot m_s \, ds + O(\gamma^2), \nonumber
  \end{align*}
  which proves \eqref{eq:affine_first_order_mean}.

 For the covariance error, we specialize~\eqref{eq:cov_disc_error_proof} for affine drift. Using the variation-of-constants formula~\citep{davis1977linear}, the solution of the SDE~\eqref{eq:backward_SDE} with linear drift takes the form
  \[
  Y_t = J_{t,s} Y_s + b_{t,s}.
  \]
  Thus $\cov(Y_s,Y_t) = C_s J_{t,s}^\top$ and using~\eqref{eq:e_linear}:
  \begin{align*}
    \cov[e_{t_k,s}(Y), Y_{t_k}]
    = -\frac12 J_{t_k,s}\big[\dot H_s + H_s^2\big] C_s J_{t_k,s}^\top.
  \end{align*}
  The covariance error~\eqref{eq:cov_disc_error_proof} becomes
  \begin{align*}
    \cov[\hat Y_{k}] - \cov[Y_{t_k}]
    &= - \gamma \int_0^{t_k} J_{t_k,s}
    \Big[
    +\frac12(\dot H_s + H_s^2)C_s
    + \frac12 C_s(\dot H_s^\top + (H_s^\top)^2)
    \\
    &+ a_s(H_s + H_s^\top)
    + \dot a_s \id
    \Big]
    J_{t_k,s}^\top ds
    + O(\gamma^2). \nonumber
  \end{align*}
  Moreover, differentiating the covariance ODE~\eqref{eq:cov_ODE} gives
  \begin{align*}
    \frac12 \ddot C_s - H_s C_s H_s^\top
    =
    \frac12(\dot H_s + H_s^2)C_s
    + \frac12 C_s(\dot H_s^\top + (H_s^\top)^2)
    + a_s(H_s + H_s^\top)
    + \dot a_s \id,
  \end{align*}
  which proves \eqref{eq:affine_first_order_cov}.

  \paragraph{Second-order coefficients.} For the second-order coefficients, we need to re-write the expansion of the mean and covariance from scratch in the linear-drift setting.
  Recall that the exact mean $m_s = \EE[Y_s]$ and covariance $C_s = \cov[Y_s]$ follow the ODEs
  \begin{align}
    \dot m_s &= H_s m_s + r_s,
   \\
    \dot C_s &= H_s C_s + C_s H_s^\top + 2 a_s \id.
  \end{align}
  We can write similar equations for the mean $\hat m_{k} = \EE[\hat Y_k]$ and covariance $\hat C_{k} = \cov[\hat Y_k]$ of the discretized process. Taking expectations in~\eqref{eq:disc} with affine drift:
    \[
    \hat Y_{k+1}
    =
    \hat Y_k+\gamma(H_{t_k}\hat Y_k+r_{t_k})
    +\sqrt{2\gamma a_{t_k}}\,\xi_k,
    \qquad \xi_k\sim\mathcal N(0,I),
    \]
    gives
   \begin{align}
    \hat m_{k+1}
    =
    \hat m_k+\gamma(H_{t_k}\hat m_k+r_{t_k}).
    \label{eq:disc_mean_eq}
   \end{align}
    Similarly, taking covariance:
    \begin{align}
    \hat C_{k+1}
    =
    (I+\gamma H_{t_k}) \hat C_k(I+\gamma H_{t_k})^\top
    +2\gamma a_{t_k} \id. \label{eq:disc_cov_eq}
    \end{align}
  Moreover, Taylor expanding the exact mean and covariance at time $t_k$ gives
  \begin{align}
    m_{t_{k+1}}
    &=
    m_{t_k}
    + \gamma \dot m_{t_k}
    + \frac{\gamma^2}{2}\ddot m_{t_k}
    + \frac{\gamma^3}{6}m^{(3)}_{t_k}
    + O(\gamma^4), \label{eq:Taylor_mean} \\
    C_{t_{k+1}}
    &=
    C_{t_k}
    + \gamma \dot C_{t_k}
    + \frac{\gamma^2}{2}\ddot C_{t_k}
    + \frac{\gamma^3}{6}C^{(3)}_{t_k}
    + O(\gamma^4). \label{eq:Taylor_cov}
  \end{align}

  \textit{Second-order mean error.}
  We look for an expansion of the form
  \[
    d^\mu_k \eqdef \EE[\hat Y_k] - \EE[Y_{t_k}] = \gamma d^{\mu,[1]}_{t_k} + \gamma^2 d^{\mu,[2]}_{t_k} + \gamma^3 d^{\mu,[3]}_{t_k} + O(\gamma^4).
  \]
  We subtract the above Taylor expansion~\eqref{eq:Taylor_mean} from the recursion~\eqref{eq:disc_mean_eq}:
  \begin{align} \label{eq:d_mu_1}
    d^\mu_{k+1}
    &=
    (I+\gamma H_{t_k})d^\mu_k
    - \frac{\gamma^2}{2}\ddot m_{t_k}
    - \frac{\gamma^3}{6}m^{(3)}_{t_k}
    + O(\gamma^4).
  \end{align}
  We also  Taylor expand $d^{\mu,[1]}_{t_{k+1}}$ and $d^{\mu,[2]}_{t_{k+1}}$:
  \begin{align*}
  d^{\mu,[1]}_{t_{k+1}} &= d^{\mu,[1]}_{t_k} + \gamma \dot d^{\mu,[1]}_{t_k} +  \frac12 \gamma^2 \ddot d^{\mu,[1]}_{t_k} + O(\gamma^3) \\
  d^{\mu,[2]}_{t_{k+1}} &= d^{\mu,[2]}_{t_k} + \gamma \dot d^{\mu,[2]}_{t_k} + O(\gamma^2),
  \end{align*}
  to get:
  \begin{align} \label{eq:d_mu_2}
    d^\mu_{k+1}
    &=
    \gamma d^{\mu,[1]}_{t_k}
    + \gamma^2 \big(\dot d^{\mu,[1]}_{t_k} + d^{\mu,[2]}_{t_k}\big)
    + \gamma^3 \big(\tfrac12 \ddot d^{\mu,[1]}_{t_k} + \dot d^{\mu,[2]}_{t_k} + d^{\mu,[3]}_{t_{k}}\big)
    + O(\gamma^4),
  \end{align}
  We then compare the coefficients of $\gamma^2$ and $\gamma^3$ in \eqref{eq:d_mu_1} and \eqref{eq:d_mu_2}:
  \begin{align*}
    \dot d^{\mu,[1]}_s &= H_s d^{\mu,[1]}_s - \frac12 \ddot m_s, \qquad d^{\mu,[1]}_0 = 0, \\
    \dot d^{\mu,[2]}_s &= H_s d^{\mu,[2]}_s - \frac12 \ddot d^{\mu,[1]}_s - \frac16 m_s^{(3)}, \qquad d^{\mu,[2]}_0 = 0.
  \end{align*}
  Differentiating the first equation above
gives
  \begin{align*}
\ddot d^{\mu,[1]}_s
    &=
\dot H_s d^{\mu,[1]}_s + H_s \dot d^{\mu,[1]}_s - \frac12 m_s^{(3)} \\
&=
(\dot H_s + H_s^2)d^{\mu,[1]}_s
- \frac12 H_s \ddot m_s
- \frac12 m_s^{(3)}.
  \end{align*}
We use this identity into $\dot d^{\mu,[2]}_s$ to get
  \begin{align*}
\dot d^{\mu,[2]}_s
=
H_s d^{\mu,[2]}_s
-\frac12 (\dot H_s + H_s^2)d^{\mu,[1]}_s
+\frac14 H_s \ddot m_s
+\frac1{12} m_s^{(3)},
  \end{align*}
  and conclude by variation of constants to finally get \eqref{eq:affine_second_order_mean}.

  \textit{Second-order covariance error.} We follow the exact same steps for the covariance error.  We look for the expansion of
  \[
    D^\Sigma_k = \cov[\hat Y_k] - \cov[Y_{t_k}] = \gamma D^{\Sigma,[1]}_{t_k} + \gamma^2 D^{\Sigma,[2]}_{t_k} + \gamma^3 D^{\Sigma,[3]}_{t_k} + O(\gamma^4 ).
  \]
   We get for $D^\Sigma_{k+1}$, using Taylor expansions:
  \begin{align}
    D^\Sigma_{k+1}
    &=
    \gamma D^{\Sigma,[1]}_{t_{k+1}} + \gamma^2 D^{\Sigma,[2]}_{t_{k+1}} + \gamma^3 D^{\Sigma,[3]}_{t_{k+1}} + O(\gamma^4) \nonumber \\
    &=
    \gamma D^{\Sigma,[1]}_{t_k}
    + \gamma^2 \big(\dot D^{\Sigma,[1]}_{t_k} + D^{\Sigma,[2]}_{t_k}\big)
    + \gamma^3 \big(\tfrac12 \ddot D^{\Sigma,[1]}_{t_k} + \dot D^{\Sigma,[2]}_{t_k} + D^{\Sigma,[3]}_{t_k}\big)
    + O(\gamma^4), \label{eq:cov_1}
  \end{align}
  On the other side, combining~\eqref{eq:Taylor_cov} and~\eqref{eq:disc_cov_eq}:
  \begin{align} \label{eq:cov_2}
    D^\Sigma_{k+1}
    &=
    (I+\gamma H_{t_k})D^\Sigma_k(I+\gamma H_{t_k})^\top
    - \gamma^2 \left(\frac12 \ddot C_{t_k} - H_{t_k}C_{t_k}H_{t_k}^\top\right)
    - \frac{\gamma^3}{6}C^{(3)}_{t_k}
    + O(\gamma^4).
  \end{align}
  We define \(\Lambda_s \eqdef \frac12 \ddot C_s - H_s C_s H_s^\top\) and
  compare the coefficients of $\gamma^2$ and $\gamma^3$ in~\eqref{eq:cov_1} and~\eqref{eq:cov_2}:
  \begin{align*}
    \dot D^{\Sigma,[1]}_s
    &=
    H_s D^{\Sigma,[1]}_s + D^{\Sigma,[1]}_s H_s^\top - \Lambda_s,
    \qquad
    D^{\Sigma,[1]}_0 = 0, \\
    \dot D^{\Sigma,[2]}_s
    &=
    H_s D^{\Sigma,[2]}_s + D^{\Sigma,[2]}_s H_s^\top
    + H_s D^{\Sigma,[1]}_s H_s^\top
    - \frac16 C_s^{(3)}
    - \frac12 \ddot D^{\Sigma,[1]}_s,
    \qquad
    D^{\Sigma,[2]}_0 = 0.
  \end{align*}
  To replace \(\ddot D^{\Sigma,[1]}_s\) in the equation above, we differentiate the first-order equation:
  \begin{align*}
    \ddot D^{\Sigma,[1]}_s
    &=
    \dot H_s D^{\Sigma,[1]}_s + D^{\Sigma,[1]}_s \dot H_s^\top +
    H_s \bigl(H_s D^{\Sigma,[1]}_s + D^{\Sigma,[1]}_s H_s^\top - \Lambda_s\bigr)
    \\ \nonumber
    &+ \bigl(H_s D^{\Sigma,[1]}_s + D^{\Sigma,[1]}_s H_s^\top - \Lambda_s\bigr) H_s^\top \nonumber
    - \dot \Lambda_s \nonumber \\
    &=
    (\dot H_s + H_s^2)D^{\Sigma,[1]}_s
    + D^{\Sigma,[1]}_s(\dot H_s^\top + (H_s^\top)^2)
    + 2 H_s D^{\Sigma,[1]}_s H_s^\top
    - H_s \Lambda_s
    - \Lambda_s H_s^\top
    - \dot \Lambda_s.
  \end{align*}
  Since
  \[
    \Lambda_s
    =
    \frac12 \ddot C_s - H_s C_s H_s^\top,
    \qquad
    H_s \Lambda_s + \Lambda_s H_s^\top
    =
    \frac12 H_s \ddot C_s
    + \frac12 \ddot C_s H_s^\top
    - H_s^2 C_s H_s^\top
    - H_s C_s (H_s^\top)^2,
  \]
  and
  \[
    \dot \Lambda_s
    =
    \frac12 C_s^{(3)}
    - \dot H_s C_s H_s^\top
    - H_s \dot C_s H_s^\top
    - H_s C_s \dot H_s^\top,
  \]
  using \(\dot C_s = H_s C_s + C_s H_s^\top + 2 a_s \id\), this becomes
  \begin{align*}
    \ddot D^{\Sigma,[1]}_s
    &=
    (\dot H_s + H_s^2)D^{\Sigma,[1]}_s
    + D^{\Sigma,[1]}_s(\dot H_s^\top + (H_s^\top)^2)
    + 2 H_s D^{\Sigma,[1]}_s H_s^\top \nonumber \\
    &\quad
    - \frac12 H_s \ddot C_s
    - \frac12 \ddot C_s H_s^\top
    - \frac12 C_s^{(3)}
    + \dot H_s C_s H_s^\top
    + H_s C_s \dot H_s^\top \nonumber \\
    &\quad
    + 2 H_s^2 C_s H_s^\top
    + 2 H_s C_s (H_s^\top)^2
    + 2 a_s H_s H_s^\top.
  \end{align*}
  Substituting this into the equation for \(\dot D^{\Sigma,[2]}_s\) yields the integrand in \eqref{eq:affine_second_order_cov}, and variation of constants gives the claimed formula.

  Finally, we justify that the remainder terms are $O(\gamma^3)$. 
Define
\[
R^\mu_k=d^\mu_k-\gamma d^{\mu,[1]}_{t_k}
-\gamma^2d^{\mu,[2]}_{t_k},\qquad
R^\Sigma_k=D^\Sigma_k-\gamma D^{\Sigma,[1]}_{t_k}
-\gamma^2D^{\Sigma,[2]}_{t_k}.
\]
Substituting the truncated expansions into the exact moment
recursions gives
\[
R^\mu_{k+1}=(\id+\gamma H_{t_k})R^\mu_k+O(\gamma^4),\qquad
R^\Sigma_{k+1}=(\id+\gamma H_{t_k})R^\Sigma_k(\id+\gamma H_{t_k})^\top+O(\gamma^4).
\]
Under Assumption~\ref{ass:regularity_sched} $H_t$ is uniformly bounded on $[0,T]$, $\|\id+\gamma H_{t_k}\|\leq1+C\gamma$, while
$R^\mu_0=R^\Sigma_0=0$. Discrete Gronwall therefore yields
\[
\max_{0\leq k\leq K}
\bigl(\|R^\mu_k\|+\|R^\Sigma_k\|\bigr)
\leq C K\gamma^4=O(\gamma^3).
\]
This establishes the claimed global remainder bounds,
uniformly over the time grid.

  \paragraph{Equivalent form for the first-order errors}
  Recall that the Jacobian satisfies
  \begin{align*}
    \frac{dJ_{t,s}}{ds} = - J_{t,s} H_s.
  \end{align*}
  Using $\dot m_s = H_s m_s + r_s$, we get
  \begin{align*}
    \frac{d}{ds}\Big(J_{t_k,s}(H_s m_s + r_s)\Big)
    &= - J_{t_k,s} H_s(H_s m_s + r_s) + J_{t_k,s} \ddot m_s.
  \end{align*}
  Plugging this identity into \eqref{eq:affine_mean_disc_error} gives
  \begin{align}
    \EE[\hat Y_{k}] - \EE[Y_{t_k}]
    &= -\frac{\gamma}{2} \Big[ J_{t_k,s} (H_s m_s + r_s)\Big]_0^{t_k}
    - \frac{\gamma}{2}\int_0^{t_k} J_{t_k,s}H_s(H_s m_s + r_s) ds
    + O(\gamma^2).
    \label{eq:affine_mean_disc_error_ipp}
  \end{align}

  For the covariance, from
  \[
  \dot C_s = H_s C_s + C_s H_s^\top + 2 a_s \id.
  \]
  Then
  \begin{align*}
    \frac{d}{ds}\Big(J_{t_k,s} \dot C_s J_{t_k,s}^\top\Big)
    &= J_{t_k,s}\big(\dot H_s C_s + C_s \dot H_s^\top + 2 \dot a_s \id\big) J_{t_k,s}^\top.
  \end{align*}
  Since
  \begin{align*}
    \frac12 \ddot C_s - H_s C_s H_s^\top
    &=
    \frac12\big(\dot H_s C_s + C_s \dot H_s^\top + 2 \dot a_s \id\big) \\
    &\quad + \frac12\big(H_s^2 C_s + C_s (H_s^\top)^2 + 2 a_s(H_s + H_s^\top)\big), \nonumber
  \end{align*}
  \eqref{eq:affine_cov_disc_error} becomes
  \begin{align}
    \cov[\hat Y_{k}] - \cov[Y_{t_k}]
    &= -\frac{\gamma}{2} \Big[J_{t_k,s} (H_s C_s + C_s H_s^\top + 2 a_s \id)  J_{t_k,s}^\top \Big]_0^{t_k}   \nonumber \\
    &\quad - \frac\gamma 2 \int_0^{t_k} J_{t_k,s} \Big[ H_s^2C_s  +  C_s (H_s^\top)^2 + 2 a_s(H_s + H_s^\top) \Big]J_{t_k,s}^\top ds +O(\gamma^2).
    \label{eq:affine_cov_disc_error_ipp}
  \end{align}
  \end{proof}

\section{Proof of Proposition~\ref{prop:mean_cov_Gaussian}}
  \label{app:mean_cov_Gaussian}

We first propose the following additional optional boundary conditions  on the schedules $(\eta_t, \sigma_t)$. It will be useful to simplify the final Gaussian error formulas in the regime of large $\sigma_T=\sigma_{\max}$. 

\begin{ass}\label{ass:regularity_sched_opt}
\begin{enumerate}
Assume for $t \to 0$: $\frac{\dot \eta_t}{\eta_t} \to 0$ and $\sigma_t \dot \sigma_t \to 0$, and for $t = T$, as $\sigma_T=\sigma_{\max}\to\infty$:
\[
\sigma_T^{-(1+\alpha)}\frac{\dot\eta_T}{\eta_T}\to0,
\qquad
\sigma_T^{-2\alpha}\left(\frac{\dot\eta_T}{\eta_T}+\frac{\dot\sigma_T}{\sigma_T}\right)\to0.
\]
\end{enumerate}
\end{ass} 
It holds, for example, for VP and VE with polynomial noise. It is optional and may fail for other schedules, including linear-interpolation flow matching.

We now re-state a more detailed version of Proposition \ref{prop:mean_cov_Gaussian}.
\newtheorem*{detailedprop1}{Proposition \ref{prop:mean_cov_Gaussian}}
\begin{detailedprop1}[Detailed version]
   Assume Gaussian data \(p_{\mathrm{data}} \sim \mathcal{N}(\mu_\data, \Sigma_\data  \succ 0 )\). Then the mean and covariance discretization errors decouple in the eigenbasis of the data covariance
\(\Sigma_\data = U \operatorname{Diag}(\lambda_i) U^\top\):
\begin{align*}
  \EE[\hat Y_{K}]
  &= \mu_\data  + U \operatorname{Diag}\Big( \Delta^\mu(\lambda_i)\Big) U^\top \mu_\data,\\
  \cov[\hat Y_{K}]
  &= \Sigma_\data + U \operatorname{Diag}\Big( \Delta^\Sigma(\lambda_i)\Big) U^\top .
\end{align*}
For both per-eigendirection defects \(\Delta^\mu(\lambda_i)\) and \(\Delta^\Sigma(\lambda_i)\), we derive explicit first- and second-order small-stepsize expansions of the form
\begin{align*}
\Delta^\mu(\lambda)
&=
\gamma \Delta^{\mu, [1]}(\lambda)
+
\gamma^2 \Delta^{\mu, [2]}(\lambda)
+
O(\gamma^3) \\
\Delta^\Sigma(\lambda)
&=
\gamma \Delta^{\Sigma, [1]}(\lambda)
+
\gamma^2 \Delta^{\Sigma, [2]}(\lambda)
+
O(\gamma^3),
\end{align*}
In particular, denoting
\begin{align*}
  A_s \eqdef \frac{\dot \eta_s}{\eta_s}, \quad
  B_s(\lambda) \eqdef \frac{\sigma_s\dot\sigma_s}{\lambda+\sigma_s^2}, \quad N_s(\lambda) \eqdef \frac{\lambda}{\lambda+\sigma_s^2}
\end{align*}
the first-order coefficients are
\begin{align}
\Delta^{\mu,[1]}(\lambda)
  &=-\tfrac12\int_0^T N_s(\lambda)^{(1+\alpha)/2}A_s[A_s+(1+\alpha)B_s(\lambda)]ds-\tfrac12[N_s(\lambda)^{(1+\alpha)/2}A_s]_0^T,
  \label{eq:mean_error_gauss} \\
\Delta^{\Sigma,[1]}(\lambda)
  &=-\lambda\int_0^T N_s(\lambda)^\alpha
  \big[(A_s+B_s(\lambda))^2-\alpha^2 B_s(\lambda)^2\big]ds-[\lambda N_s(\lambda)^\alpha(A_s+B_s(\lambda))]_0^T.
  \label{eq:cov_error_gauss}
\end{align}
Under Assumption~\ref{ass:regularity_sched_opt}, the boundary terms $[ \, \cdot \,]_0^T$ in both above expansions vanish as $\sigma_T = \sigma_{\max} \to \infty$. 

For the second-order formulas, we use the first-order errors at intermediate sampling times $t = k \gamma$:
\[
 \EE[\hat Y_{k}] -  \EE[Y_t]
=U\diag{\Delta_t^{\mu,[1]}(\lambda_i)}U^\top\mu_\data,
\qquad
 \cov[\hat Y_{k}] -  \cov[Y_t]
=U\diag{\Delta_t^{\Sigma,[1]}(\lambda_i)}U^\top.
\]
Denoting
\[
Q_s(\lambda)\eqdef A_s+(1+\alpha)B_s(\lambda),
\qquad V_s(\lambda)\eqdef A_s+B_s(\lambda),
\]
they write
\begin{align}
\Delta^{\mu,[2]}(\lambda)
&=\int_0^T\Bigg[
-\frac{N_s(\lambda)^{\frac{1+\alpha}{2}}}{2\eta_s}
\big(\dot Q_s(\lambda)+Q_s(\lambda)^2\big)
\Delta_{T-s}^{\mu,[1]}(\lambda) \nonumber\\
&\qquad-N_s(\lambda)^{\frac{1+\alpha}{2}}\Big(
\frac14Q_s(\lambda)(\dot A_s+A_s^2)
+\frac1{12}(\ddot A_s+3A_s\dot A_s+A_s^3)
\Big)\Bigg]\,ds.
\label{eq:gaussian_second_order_mean}
\end{align}
\begin{align}
\Delta^{\Sigma,[2]}(\lambda)
&=\int_0^T\Bigg[
-\frac{N_s(\lambda)^{1+\alpha}}{\eta_s^2}
\big(\dot Q_s(\lambda)+Q_s(\lambda)^2\big)
\Delta_{T-s}^{\Sigma,[1]}(\lambda) \nonumber\\
&\quad+\lambda N_s(\lambda)^\alpha\Big(
-\frac16\big(\ddot V_s(\lambda)
+6V_s(\lambda)\dot V_s(\lambda)+4V_s(\lambda)^3\big)
\nonumber\\
&\qquad-Q_s(\lambda)\big(\dot V_s(\lambda)+2V_s(\lambda)^2\big)
+Q_s(\lambda)\dot Q_s(\lambda)+2Q_s(\lambda)^3
-\alpha B_s(\lambda)Q_s(\lambda)^2
\Big)\Bigg]\,ds.
\label{eq:gaussian_second_order_cov}
\end{align}
\end{detailedprop1}

  \begin{proof}
  We first define the notations: for $t\in [0,T]$ and $k \in \llbracket 0, K \rrbracket$, $m_t = \EE[Y_t]$, $C_t = \cov[Y_t]$, $\hat m_k = \EE[\hat Y_{t_k}]$ and  $\hat C_k = \cov[\hat Y_k]$.
  Here $\Sigma_t=\cov[X_t]$ denotes the forward covariance, whereas
  $C_t=\cov[Y_t]=\Sigma_{T-t}$ is the reverse covariance. We also recall that the affine drift $v_t(x) = H_t \,x + r_t $, where, for $s = T-t$: 
\begin{equation}  
    \begin{aligned}
      H_{t} &\eqdef - \frac{\dot \eta_{s}}{\eta_s}\id - (1 + \alpha) \eta_s^2 \dot \sigma_s \sigma_s \Sigma_{s}^{-1} \\
      r_{t} &\eqdef (1 + \alpha) \eta_s^2 \dot \sigma_s \sigma_s \Sigma_{s}^{-1} \mu_{s}.
    \end{aligned} \vspace{-0.2cm}
  \end{equation}
  
  These quantities satisfy the evolution equations \eqref{eq:mean_ODE}--\eqref{eq:cov_ODE} and~\eqref{eq:disc_mean_eq}--\eqref{eq:disc_cov_eq} (see Appendix~\ref{app:affine_velocity} for details):
  \begin{align*}
    \dot m_s &= H_s m_s + r_s, \\
    \hat m_{k+1}
    &=
    \hat m_k+\gamma(H_{t_k}\hat m_k+r_{t_k}) \\
    \dot C_s &= H_s C_s + C_s H_s^\top + 2 a_s \id \\
    \hat C_{k+1}
    &=
    (I+\gamma H_{t_k}) \hat C_k(I+\gamma H_{t_k})^\top
    +2\gamma a_{t_k} \id.
  \end{align*}
  
     Moreover, $H_t$ co-diagonalizes with $\Sigma_\data = U \operatorname{Diag}(\lambda_i) U^\top$ and $r_t$ takes the form ${r_t = P_t \mu_\data}$ with $P_t$ which  co-diagonalizes with $\Sigma_\data$. 
    We thus get, subtracting the continuous and discrete equations, the following decomposition at time $t_k = k \gamma$:
\begin{align*}
  d^\mu_{t_k} \eqdef \hat m_k-m_{t_k}
  &=  U \operatorname{Diag}\Big( \Delta_{t_k}^\mu(\lambda_i)\Big) U^\top \mu_\data,\\
 d^\Sigma_{t_k} \eqdef \hat C_k - C_{t_k}
  &= U \operatorname{Diag}\Big( \Delta_{t_k}^\Sigma(\lambda_i)\Big) U^\top
\end{align*}
for scalar per-eigendirection errors $\Delta_{t_k}^\mu(\lambda_i)$ and
$\Delta_{t_k}^\Sigma(\lambda_i)$. Proposition~\ref{prop:affine_velocity}
gives their small-stepsize expansions
\[
\Delta_{t_k}^{\mu}(\lambda)
=\gamma\Delta_{t_k}^{\mu,[1]}(\lambda)
+\gamma^2\Delta_{t_k}^{\mu,[2]}(\lambda)+O(\gamma^3),
\qquad
\Delta_{t_k}^{\Sigma}(\lambda)
=\gamma\Delta_{t_k}^{\Sigma,[1]}(\lambda)
+\gamma^2\Delta_{t_k}^{\Sigma,[2]}(\lambda)+O(\gamma^3).
\]
The integral formulas define the coefficient functions for every
$t\in[0,T]$; their values at $t=t_k$ give the coefficients in the
grid-time expansions above.
We first calculate both first-order coefficients at arbitrary sampling
time $t$, then specialize to $t=T$ and reuse the intermediate-time
coefficients in the second-order formulas.

First, we calculate explicitly the Jacobian matrix $J_{t,s}$ which appears in the means and covariance decomposition of Proposition~\ref{prop:affine_velocity}. Recall that $J_{t,s}$ is the unique solution of the linear ODE:
  \begin{align*}
    \frac{dJ_{t,s}(Y)}{dt} = \nabla v_t(Y_t)  J_{t,s}(Y) = H_t J_{t,s}(Y), \qquad J_{s,s} = \id
  \end{align*}
  Note that all the $H_s$ commute in time and thus $$J_{t,s}(Y) = \exp\left(\int_s^t H_\tau d\tau \right).$$

  Using $\Sigma_t = \eta_t^2 (\Sigma_\data + \sigma_t^2 \id)$
  gives
  \begin{align*}
    \frac{d}{dt}\log\Sigma_t
    =
    \Sigma_t^{-1}\frac{d}{dt}\Sigma_t
    = 2 \frac{\dot \eta_t}{\eta_t} \id + 2 \eta_t^2 \sigma_t \dot \sigma_t \Sigma_t^{-1},
  \end{align*}
  hence
  \begin{align*}
    H_t
    &= \alpha \frac{\dot \eta_{T-t}}{\eta_{T-t}} \id
    -\frac{1+\alpha}{2}\frac{d}{ds}\log\Sigma_s\bigg|_{s=T-t}.
  \end{align*}
  Integrating the logarithmic derivatives gives
  \begin{align*}
    \int_s^t H_\tau\,d\tau
    &=\alpha\log\left(\frac{\eta_{T-s}}{\eta_{T-t}}\right)\id
    -\frac{1+\alpha}{2}\log\left(\Sigma_{T-s}\Sigma_{T-t}^{-1}\right),
  \end{align*}
  and therefore
  \begin{align*}
    J_{t,s}
    &=
    \left(\frac{\eta_{T-s}}{\eta_{T-t}} \right)^{\alpha}
    \Sigma_{T-t}^{\frac{1+\alpha}{2}}
    \Sigma_{T-s}^{-\frac{1+\alpha}{2}}
  \end{align*}
  and at final time $t=T$
\begin{align} \label{eq:jacobian_gaussian}
    J_{T,s}
    &= \eta_{T-s}^{-1}
    \Sigma_{\data}^{\frac{1+\alpha}{2}}
    \big(\Sigma_\data + \sigma_{T-s}^2 \id\big)^{-\frac{1+\alpha}{2}}.
  \end{align}
The first-order coefficients from
\eqref{eq:affine_first_order_mean} and~\eqref{eq:affine_first_order_cov}
satisfy
\begin{align*}
d_t^{\mu,[1]}
&=-\frac12\int_0^t J_{t,s}\ddot m_s\,ds
=U\diag{\Delta_t^{\mu,[1]}(\lambda_i)}U^\top\mu_\data,\\
D_t^{\Sigma,[1]}
&=-\int_0^t J_{t,s}
\left(\frac12\ddot C_s-H_sC_sH_s^\top\right)J_{t,s}^\top\,ds
=U\diag{\Delta_t^{\Sigma,[1]}(\lambda_i)}U^\top.
\end{align*}

\paragraph{First-order mean coefficient.}
Since $m_s=\eta_{T-s}\mu_\data$, differentiation gives
\begin{equation}
\ddot m_s=(\dot A_{T-s}+A_{T-s}^2)\eta_{T-s}\mu_\data.
\label{eq:ddotms}
\end{equation}
Using $N_s(\lambda) \eqdef \frac{\lambda}{\lambda+\sigma_s^2}$, the general Jacobian calculated above has eigenvalue
\[
j_{t,T-s}(\lambda)
=\frac{\eta_{T-t}}{\eta_s}
\left(\frac{N_s(\lambda)}{N_{T-t}(\lambda)}\right)^{\frac{1+\alpha}{2}}.
\]
Changing variables $s \gets T-s$ gives
\begin{equation}
\Delta_t^{\mu,[1]}(\lambda)
=-\frac{\eta_{T-t}}{2N_{T-t}(\lambda)^{\frac{1+\alpha}{2}}}
\int_{T-t}^T N_s(\lambda)^{\frac{1+\alpha}{2}}
(\dot A_s+A_s^2)\,ds.
\label{eq:gaussian_partial_first_order_mean}
\end{equation}
Using
\[
\frac{d}{ds}N_s(\lambda)^{\frac{1+\alpha}{2}}
=-(1+\alpha)B_s(\lambda)N_s(\lambda)^{\frac{1+\alpha}{2}},
\]
integration by parts gives
\begin{align}
\Delta_t^{\mu,[1]}(\lambda)
={}&-\frac{\eta_{T-t}}{2N_{T-t}(\lambda)^{\frac{1+\alpha}{2}}}
\Bigg\{
\Big[N_s(\lambda)^{\frac{1+\alpha}{2}}A_s\Big]_{s=T-t}^{s=T}
+\int_{T-t}^T N_s(\lambda)^{\frac{1+\alpha}{2}}
A_s\big[A_s+(1+\alpha)B_s(\lambda)\big]\,ds
\Bigg\}.
\label{eq:gaussian_partial_mean_ibp}
\end{align}

\paragraph{First-order covariance coefficient.}
We start from~\eqref{eq:affine_first_order_cov} at sampling time $t$:
\[
D_t^{\Sigma,[1]}
=-\int_0^t J_{t,s}
\left[\frac12\ddot C_s-H_sC_sH_s^\top\right]
J_{t,s}^\top\,ds
=U\diag{\Delta_t^{\Sigma,[1]}(\lambda)}U^\top.
\]
Changing variables $s\mapsto T-s$ and using commutation
of the matrices gives
\[
U\diag{\Delta_t^{\Sigma,[1]}(\lambda)}U^\top
=-\int_{T-t}^T J_{t,T-s}^2
\left[\frac12\ddot C_{T-s}-H_{T-s}^2C_{T-s}\right]\,ds.
\]
Using the previously introduced notations,
\begin{align*}
H_{T-s}
&=-U\diag{A_s+(1+\alpha)B_s(\lambda)}U^\top,\\
J_{t,T-s}^2C_{T-s}
&=U\diag{
\lambda_{T-t}
\left(\frac{N_s(\lambda)}{N_{T-t}(\lambda)}\right)^\alpha
}U^\top.
\end{align*}
Moreover, differentiating
$\lambda_s=\eta_s^2(\lambda+\sigma_s^2)$ gives
\begin{align*}
\dot\lambda_s
&=2\big(A_s+B_s(\lambda)\big)\lambda_s,\\
\frac12\ddot\lambda_s
&=\Big[\dot A_s+\dot B_s(\lambda)
+2\big(A_s+B_s(\lambda)\big)^2\Big]\lambda_s.
\end{align*}
We get:
\begin{equation}
\Delta_t^{\Sigma,[1]}(\lambda)
=-\frac{\lambda_{T-t}}{N_{T-t}(\lambda)^\alpha}
\int_{T-t}^T N_s(\lambda)^\alpha
\Big[\dot A_s+\dot B_s(\lambda)
+2\big(A_s+B_s(\lambda)\big)^2-(A_s+(1+\alpha)B_s(\lambda))^2\Big]\,ds.
\label{eq:gaussian_partial_first_order_cov}
\end{equation}
Since
\[
\frac{d}{ds}\Big[
N_s(\lambda)^\alpha\big(A_s+B_s(\lambda)\big)
\Big]
=N_s(\lambda)^\alpha
\Big[\dot A_s+\dot B_s(\lambda)
-2\alpha B_s(\lambda)\big(A_s+B_s(\lambda)\big)\Big],
\]
integration by parts gives
\begin{align}
\Delta_t^{\Sigma,[1]}(\lambda)
={}&-\frac{\lambda_{T-t}}{N_{T-t}(\lambda)^\alpha}
\Bigg\{
\Big[N_s(\lambda)^\alpha
\big(A_s+B_s(\lambda)\big)\Big]_{s=T-t}^{s=T} +\int_{T-t}^T N_s(\lambda)^\alpha
\Big[\big(A_s+B_s(\lambda)\big)^2
-\alpha^2B_s(\lambda)^2\Big]\,ds
\Bigg\}.
\label{eq:gaussian_partial_cov_ibp}
\end{align}

At $t=T$, we have $\eta_0=N_0(\lambda)=1$, and
$\lambda_0=\lambda$. Thus, we get~\eqref{eq:mean_error_gauss} and~\eqref{eq:cov_error_gauss}.

\paragraph{Second-order mean coefficient.}
Departing from~\eqref{eq:affine_second_order_mean}, with $s \gets T-s$, 
\begin{align}
d_T^{\mu,[2]}
&=\int_0^T J_{T,T-s}\left[
-\frac12(\dot H_{T-s}+H_{T-s}^2)d_{T-s}^{\mu,[1]}
+\frac14H_{T-s}\ddot m_{T-s}+\frac1{12}\dddot m_{T-s}
\right]ds.
\label{eq:inter_snd_mean}
\end{align}
Using~\eqref{eq:jacobian_gaussian},
\[
J_{T,T-s}d_{T-s}^{\mu,[1]}
=U\diag{
\frac{N_s(\lambda_i)^{\frac{1+\alpha}{2}}}{\eta_s}
\Delta_{T-s}^{\mu,[1]}(\lambda_i)
}U^\top\mu_\data.
\]
Moreover,
\[
H_{T-s}=-U\diag{Q_s(\lambda_i)}U^\top,
\qquad
\dot H_{T-s}=U\diag{\dot Q_s(\lambda_i)}U^\top,
\]
and differentiation of~\eqref{eq:ddotms} gives
\[
\dddot m_{T-s}=-(\ddot A_s+3A_s\dot A_s+A_s^3)\eta_s\mu_\data.
\]
Substituting these expressions into~\eqref{eq:inter_snd_mean} yields
\begin{align}
\Delta^{\mu,[2]}(\lambda)
&=\int_0^T\Bigg[
-\frac{N_s(\lambda)^{\frac{1+\alpha}{2}}}{2\eta_s}
\big(\dot Q_s(\lambda)+Q_s(\lambda)^2\big)
\Delta_{T-s}^{\mu,[1]}(\lambda) \nonumber\\
&\qquad-N_s(\lambda)^{\frac{1+\alpha}{2}}\Big(
\frac14Q_s(\lambda)(\dot A_s+A_s^2)
+\frac1{12}(\ddot A_s+3A_s\dot A_s+A_s^3)
\Big)\Bigg]\,ds.
\end{align}

\paragraph{Second-order covariance coefficient.}
Applying~\eqref{eq:affine_second_order_cov} at final time and making
the same change of variable gives
\begin{align}
D_T^{\Sigma,[2]}
&=\int_0^T J_{T,T-s}\Bigg[
-\frac12(\dot H_{T-s}+H_{T-s}^2)D_{T-s}^{\Sigma,[1]}
-\frac12D_{T-s}^{\Sigma,[1]}(\dot H_{T-s}^\top+(H_{T-s}^\top)^2)
\nonumber\\
&\quad+\frac14H_{T-s}\ddot C_{T-s}+\frac14\ddot C_{T-s}H_{T-s}^\top
+\frac1{12}\dddot C_{T-s}
-\frac12\dot H_{T-s}C_{T-s}H_{T-s}^\top
-\frac12H_{T-s}C_{T-s}\dot H_{T-s}^\top
\nonumber\\
&\quad-H_{T-s}^2C_{T-s}H_{T-s}^\top-H_{T-s}C_{T-s}(H_{T-s}^\top)^2
-a_{T-s}H_{T-s}H_{T-s}^\top
\Bigg]J_{T,T-s}^\top\,ds.
\label{eq:inter_snd_cov}
\end{align}
The intermediate-time covariance coefficient is transported as
\[
J_{T,T-s}D_{T-s}^{\Sigma,[1]}J_{T,T-s}^\top
=U\diag{
\frac{N_s(\lambda_i)^{1+\alpha}}{\eta_s^2}
\Delta_{T-s}^{\Sigma,[1]}(\lambda_i)
}U^\top.
\]
Moreover differentiating $C_{T-s}= U \operatorname{Diag}(\lambda_s)U^\top$ gives
\[
\ddot C_{T-s}
=2 U \operatorname{Diag}\Big( \lambda_s\big[\dot V_s(\lambda)+2V_s(\lambda)^2\big]\Big) U^\top 
\qquad
\dddot C_{T-s}
=-2 U \operatorname{Diag}\Big( \lambda_s\big[
\ddot V_s(\lambda)+6V_s(\lambda)\dot V_s(\lambda)
+4V_s(\lambda)^3\big] \Big) U^\top.
\]
Finally, $j_{T,T-s}(\lambda)^2\lambda_s=\lambda N_s(\lambda)^\alpha$ and $a_{T-s}=\alpha B_s(\lambda)\lambda_s$.
Substitution into~\eqref{eq:inter_snd_cov} gives
\begin{align}
\Delta^{\Sigma,[2]}(\lambda)
&=\int_0^T\Bigg[
-\frac{N_s(\lambda)^{1+\alpha}}{\eta_s^2}
\big(\dot Q_s(\lambda)+Q_s(\lambda)^2\big)
\Delta_{T-s}^{\Sigma,[1]}(\lambda) \nonumber\\
&\quad+\lambda N_s(\lambda)^\alpha\Big(
-\frac16\big(\ddot V_s(\lambda)
+6V_s(\lambda)\dot V_s(\lambda)+4V_s(\lambda)^3\big)
\nonumber\\
&\qquad-Q_s(\lambda)\big(\dot V_s(\lambda)+2V_s(\lambda)^2\big)
+Q_s(\lambda)\dot Q_s(\lambda)+2Q_s(\lambda)^3
-\alpha B_s(\lambda)Q_s(\lambda)^2
\Big)\Bigg]\,ds.
\end{align}

  \paragraph{Large-$\sigma_T$ simplification.}
  The residual terms obtained in the first order mean and covariance coefficients are $\Big[ \Omega^\mu_s(\lambda)\Big]_0^T$ and $\Big[ \Omega^\Sigma_s(\lambda)\Big]_0^T$ with
  \[
    \Omega^\mu_s(\lambda)
    \eqdef
    -\frac12
    \left(\frac{\lambda}{\lambda+\sigma_s^2}\right)^{\frac{1+\alpha}{2}}
    \frac{\dot\eta_s}{\eta_s},
  \]
  and
  \[
    \Omega^\Sigma_s(\lambda)
    \eqdef
    -\lambda
    \left(\frac{\lambda}{\lambda+\sigma_s^2}\right)^\alpha
    \left(
      \frac{\dot\eta_s}{\eta_s}
      +
      \frac{\sigma_s\dot\sigma_s}{\lambda+\sigma_s^2}
    \right).
  \]
  First, for $s = T$, as $\sigma_T = \sigma_{\max} \to\infty$, under Assumption~\ref{ass:regularity_sched_opt}, we have
  \begin{align*}
    \Omega^\mu_s(\lambda)
    &=
    -\frac12 \lambda^{\frac{1+\alpha}{2}}
    \sigma_T^{-(1+\alpha)}
    \frac{\dot\eta_T}{\eta_T}\big(1+o(1)\big)
    =
    o_{\sigma_{\max}\to\infty}(1), \\
    \Omega^\Sigma_s(\lambda)
    &=
    -\lambda^{\alpha+1}
    \sigma_T^{-2\alpha}
    \left(
      \frac{\dot\eta_T}{\eta_T}
      +
      \frac{\dot\sigma_T}{\sigma_T}
    \right)\big(1+o(1)\big)
    =
    o_{\sigma_{\max}\to\infty}(1).
  \end{align*}
  Moreover, as $s\to0$, if $\dot \sigma_s \sigma_s \to 0$ and $\dot \eta_s \to 0$ (Assumption~\ref{ass:regularity_sched_opt}), $\Omega^\mu_s$ and  $\Omega^\Sigma_s$ vanish in $s \to 0$.
  Thus:
  \[
    \Big[ \Omega^\mu_s(\lambda)\Big]_0^T
    = o_{\sigma_{\max}\to\infty}(1),
    \qquad
   \Big[ \Omega^\Sigma_s(\lambda)\Big]_0^T
    = o_{\sigma_{\max}\to\infty}(1).
  \]
  \end{proof}

  \subsection{First-order scheduler specializations}
  \label{sec:first_order_power_law}

  \paragraph{Variance Exploding \citep{song2020score}}
  It corresponds to $\eta_t = 1$. Then the first and second order mean errors from Proposition~\ref{prop:mean_cov_Gaussian} vanish:
  \begin{align*}
    \Delta^{\mu,[1]}(\lambda) = \Delta^{\mu,[2]}(\lambda) = 0
  \end{align*}
  Moreover, if $\dot \sigma_t \sigma_t \to 0$ as $t\to 0$, Assumption~\ref{ass:regularity_sched_opt} holds and the covariance error~\eqref{eq:cov_error_gauss} simplifies to:
  \begin{align}
    \Delta^{\Sigma,[1]}(\lambda)
    &=
    (\alpha^2-1)\lambda^{\alpha+1}
    \int_0^T \frac{(\sigma_s\dot\sigma_s)^2}{(\lambda+\sigma_s^2)^{\alpha+2}}\,ds
    + o_{\sigma_{\max} \to \infty}(1)
    \label{eq:ve_cov_error_sigma}
  \end{align}
Assume the polynomial schedule
  \begin{align*}
    \sigma_t = \sigma_{\max}\left(\frac{t}{T}\right)^\beta,
    \qquad
    \beta > \frac12,
    \qquad
    \alpha > 0.
  \end{align*}
  Then
  \begin{align*}
    \Delta^{\Sigma,[1]}(\lambda)
    &=
    (\alpha^2-1)\frac{\beta^2\sigma_{\max}^4}{T}\lambda^{\alpha+1}
    \int_0^1 \frac{x^{4\beta-2}}{(\lambda+\sigma_{\max}^2 x^{2\beta})^{\alpha+2}}\,dx
    + o_{\sigma_{\max} \to \infty}(1).
  \end{align*}
  With the change of variable $v=\frac{\sigma_{\max}^2}{\lambda} x^{2\beta}$, we obtain
  \begin{align*}
    \Delta^{\Sigma,[1]}(\lambda)
    &=
    (\alpha^2-1)\frac{\beta}{2T}
    \lambda^{1-\frac{1}{2\beta}}
    \sigma_{\max}^{1/\beta}
    \int_0^{\sigma_{\max}^2/\lambda}
    \frac{v^{1-\frac{1}{2\beta}}}{(1+v)^{\alpha+2}}\,dv
    + o_{\sigma_{\max} \to \infty}(1).
  \end{align*}
  Since $\beta>\frac12$ and $\alpha>0$,
  \begin{align*}
    \int_0^{\sigma_{\max}^2/\lambda}
    \frac{v^{1-\frac{1}{2\beta}}}{(1+v)^{\alpha+2}}\,dv
    \to
    \operatorname{B}\!\left(2-\frac{1}{2\beta},\alpha+\frac{1}{2\beta}\right),
  \end{align*}
  hence
  \begin{align*}
    \Delta^{\Sigma,[1]}(\lambda)
    &=
    (\alpha^2-1)\frac{\beta \lambda}{2T}
    \operatorname{B}\!\left(
      2-\frac{1}{2\beta},
      \alpha+\frac{1}{2\beta}
    \right)
    \left( \frac{\sigma_{\max}^2}{\lambda}\right)^{\frac{1}{2\beta}}
   + o_{\sigma_{\max} \to \infty}(\sigma_{\max}^{1/\beta}).
  \end{align*}

  \paragraph{Variance Preserving~\citep{song2020score,ho2020denoising}}It corresponds to $\eta_t = (1+\sigma_t^2)^{-1/2}$. Then, if $\dot \sigma_t \sigma_t \to 0$ as $t\to 0$, Assumption~\ref{ass:regularity_sched_opt} holds and the mean~\eqref{eq:mean_error_gauss} and covariance~\eqref{eq:cov_error_gauss} errors simplify to
  \begin{align} \label{eq:vp_error_mu}
    \Delta^{\mu,[1]}(\lambda)
    &=
    -\frac{1}{2\lambda}
    \int_0^T
    \frac{(\sigma_s\dot\sigma_s)^2}{1+\sigma_s^2}
    \left(\frac{\lambda}{\lambda+\sigma_s^2}\right)^{\frac{\alpha+3}{2}}
    \left(
      \frac{\lambda-1}{1+\sigma_s^2}
      -
      \alpha
    \right)\,ds
    + o_{\sigma_{\max}\to\infty}(1), \\
    \Delta^{\Sigma,[1]}(\lambda)
    &=
    -\frac{1}{\lambda}
    \int_0^T
    \left(\frac{\lambda}{\lambda+\sigma_s^2}\right)^{\alpha+2}
    (\sigma_s\dot\sigma_s)^2
    \left[
      \left(\frac{1-\lambda}{1+\sigma_s^2}\right)^2
      -
      \alpha^2
    \right]\,ds
    + o_{\sigma_{\max}\to\infty}(1)
    \label{eq:vp_error_sigma}
  \end{align}
  We now further specialize when $\alpha=0$ for the polynomial schedule
  \begin{align*}
    \sigma_t = \sigma_{\max}\left(\frac{t}{T}\right)^\beta,
    \qquad
    \beta > \frac12.
  \end{align*}
  Setting $x=t/T$, we get
  \begin{align*}
    \Delta^{\mu,[1]}(\lambda)
    &=
    -\frac{\beta^2(\lambda-1)\sqrt{\lambda}\,\sigma_{\max}^4}{2T}
    \int_0^1
    \frac{x^{4\beta-2}}
    {(1+\sigma_{\max}^2x^{2\beta})^2(\lambda+\sigma_{\max}^2x^{2\beta})^{3/2}}
    \,dx
    + o(1), \\
    \Delta^{\Sigma,[1]}(\lambda)
    &=
    -\frac{\beta^2(1-\lambda)^2\lambda\,\sigma_{\max}^4}{T}
    \int_0^1
    \frac{x^{4\beta-2}}
    {(1+\sigma_{\max}^2x^{2\beta})^2(\lambda+\sigma_{\max}^2x^{2\beta})^2}
    \,dx
    + o(1),
  \end{align*}
  With the change of variable
  \begin{align*}
    u = \sigma_{\max}^2 x^{2\beta},
    \qquad
    dx = \frac{1}{2\beta}\sigma_{\max}^{-1/\beta}u^{\frac{1}{2\beta}-1}\,du, \quad  x^{4\beta-2}
    =
    \sigma_{\max}^{-4+\frac{2}{\beta}}u^{2-\frac{1}{\beta}},
  \end{align*}
  we obtain
  \begin{align*}
    \Delta^{\mu,[1]}(\lambda)
    &=
    -\frac{\beta(\lambda-1)\sqrt{\lambda}}{4T}\,
    \sigma_{\max}^{1/\beta}
    \int_0^{\sigma_{\max}^2}
    \frac{u^{\,1-\frac{1}{2\beta}}}{(1+u)^2(\lambda+u)^{3/2}}
    \,du
    + o\!\left(\sigma_{\max}^{1/\beta}\right), \\
    \Delta^{\Sigma,[1]}(\lambda)
    &=
    -\frac{\beta(1-\lambda)^2\lambda}{2T}\,
    \sigma_{\max}^{1/\beta}
    \int_0^{\sigma_{\max}^2}
    \frac{u^{\,1-\frac{1}{2\beta}}}{(1+u)^2(\lambda+u)^2}
    \,du
    + o\!\left(\sigma_{\max}^{1/\beta}\right).
  \end{align*}
  Since $\beta>\frac12$, both integrands are integrable on $(0,\infty)$, so dominated convergence yields
  \begin{align}
    \Delta^{\mu,[1]}(\lambda)
    &=
    -\frac{\beta(\lambda-1)\sqrt{\lambda}}{4T}\,
    \sigma_{\max}^{1/\beta}
    \int_0^\infty
    \frac{u^{\,1-\frac{1}{2\beta}}}{(1+u)^2(\lambda+u)^{3/2}}
    \,du
    + o\!\left(\sigma_{\max}^{1/\beta}\right), \label{eq:vp_mean_error_poly} \\
    \Delta^{\Sigma,[1]}(\lambda)
    &=
    -\frac{\beta(1-\lambda)^2\lambda}{2T}\,
    \sigma_{\max}^{1/\beta}
    \int_0^\infty
    \frac{u^{\,1-\frac{1}{2\beta}}}{(1+u)^2(\lambda+u)^2}
    \,du
    + o\!\left(\sigma_{\max}^{1/\beta}\right). \label{eq:vp_cov_error_poly}
  \end{align}

  \paragraph{Flow Matching~\citep{lipman2022flow}}
  For the standard deterministic linear-interpolation Flow Matching schedule, the error integrals become fully tractable. Let 
  \begin{align*}
    \alpha = 0,
    \qquad
    \eta_t = 1-\frac{t}{T},
    \qquad
    \sigma_t = \frac{t}{T-t}.
  \end{align*}
  Then $A_t = \frac{\dot \eta_t}{\eta_t}$ verifies $ \dot A_t + A_t^2 =  0$ and thus the first-order mean formula from Proposition~\ref{prop:mean_cov_Gaussian} cancels:
  \begin{align*}
    \Delta^{\mu,[1]}(\lambda)
    &= 0.
  \end{align*}
  For the covariance error, for $Q_t(\lambda)
    \eqdef A_t + B_t(\lambda)$, since $\alpha=0$,
  \begin{align*}
    \Delta^{\Sigma,[1]}(\lambda)
    &=
    -\lambda\int_0^T \big(\dot Q_t(\lambda) + Q_t(\lambda)^2\big)\,dt.
  \end{align*}
   We derive for the above choice of $\eta_t$ and $\sigma_t$:
  \begin{align*}
   \dot Q_t(\lambda) + Q_t(\lambda)^2 = \frac{T^2 \lambda}{(\lambda (T-t)^2 + t^2)^2}
  \end{align*}
  Hence, the first order error becomes fully explicit:
  \begin{align*}
    \Delta^{\Sigma,[1]}(\lambda)
    &=
    -\lambda^2 T^2
    \int_0^T \frac{dt}{(\lambda (T-t)^2 + t^2)^2} = - \frac1T\left(\lambda
    +
    \frac{(1+\lambda)\sqrt{\lambda}}{4}\pi \right).
  \end{align*}

  \subsection{Second-order specialization for VE}
    \label{sec:snd_order_VE}

  For the analysis of the optimal diffusion-term parameter $\alpha$ of Proposition~\ref{prop:opt_alpha}, we need the second-order covariance coefficient $\Delta^{\Sigma,[2]}$ at \(\alpha=1\) for VE ($\eta_t=1$). In this case, using the notations from the proof of Proposition~\ref{prop:mean_cov_Gaussian} in Appendix~\ref{app:mean_cov_Gaussian}:
  \[
    A_s = 0,
    \quad
    B_s(\lambda) = \frac{\sigma_s\dot\sigma_s}{\lambda+\sigma_s^2},
    \quad
    Q_s(\lambda) = 2B_s(\lambda), \quad
    N_s(\lambda) = \frac{\lambda}{\lambda+\sigma_s^2}.
  \]
  The first-order covariance coefficient transported to final sampling time $T$ satisfies
  \[
    N_s(\lambda)^2\Delta_{T-s}^{\Sigma,[1]}(1,\lambda)
    =
    -\lambda \Big[N_r(\lambda)B_r(\lambda)\Big]_s^T
  \]
  Plugging this identity into~\eqref{eq:gaussian_second_order_cov} yields
  \begin{align*}
    \Delta^{\Sigma,[2]}(1,\lambda)
    &=
    \int_0^T
    \Big[
      -2\Big(\dot B_s(\lambda)+2B_s(\lambda)^2\Big)N_s(\lambda)^2\Delta_{T-s}^{\Sigma,[1]}(1,\lambda)
      \\
      &+ \lambda N_s(\lambda)\Big(
        -\frac16 \ddot B_s(\lambda)
        + B_s(\lambda)\dot B_s(\lambda)
        + \frac{22}{3} B_s(\lambda)^3
      \Big)
    \Big] ds.
  \end{align*}
  We also use the polynomial schedule
  $
    \sigma_t = \sigma_{\max}\left(\frac{t}{T}\right)^\beta,
  $
  and assume \(\beta>1\), with large \(\sigma_{\max}\to\infty\). Set
  \[
    z \eqdef \frac{\sigma_{\max}^2}{\lambda},
    \qquad
    x \eqdef \frac{s}{T}.
  \]
  Then
  \[
    B_s(\lambda)
    =
    \frac{\beta z x^{2\beta-1}}{T(1+z x^{2\beta})},
  \]
  and
  \begin{align*}
    \dot B_s(\lambda)
    &=
    \frac{\beta z x^{2\beta-2}\big((2\beta-1)-z x^{2\beta}\big)}
         {T^2(1+z x^{2\beta})^2}, \\
    \ddot B_s(\lambda)
    &=
    \frac{2\beta z x^{2\beta-3}}{T^3(1+z x^{2\beta})^3}
    \Big(
      (\beta-1)(2\beta-1)
      - (\beta+2)(2\beta-1)z x^{2\beta}
      + z^2 x^{4\beta}
    \Big).
  \end{align*}
  Moreover,
  \[
    \lambda N_T(\lambda)B_T(\lambda)
    =
    \frac{\beta\lambda z}{T(1+z)^2},
  \]
  and
  \[
    N_s(\lambda)^2\Delta_{T-s}^{\Sigma,[1]}(1,\lambda)
    =
    \lambda N_s(\lambda)B_s(\lambda)-\lambda N_T(\lambda)B_T(\lambda).
  \]
  Let \(R_z(\lambda)\) denote the induced change in \(\Delta^{\Sigma,[2]}(1,\lambda)\) when \(N_s(\lambda)^2\Delta_{T-s}^{\Sigma,[1]}(1,\lambda)\) is replaced by \(\lambda N_s(\lambda)B_s(\lambda)\), i.e. removing the second $T$ term. Then
  \[
    R_z(\lambda)
    =
    2\lambda N_T(\lambda)B_T(\lambda)
    \int_0^T\big(\dot B_s(\lambda)+2B_s(\lambda)^2\big)\,ds.
  \]
  Since \(B_0(\lambda)=0\),
  \[
    \int_0^T \dot B_s(\lambda)\,ds
    =
    B_T(\lambda)-B_0(\lambda)
    =
    \frac{\beta z}{T(1+z)}.
  \]
  Also,
  \begin{align*}
    \int_0^T B_s(\lambda)^2\,ds
    &=
    \frac{\beta^2 z^2}{T}
    \int_0^1 \frac{x^{4\beta-2}}{(1+z x^{2\beta})^2}\,dx \\
    &=
    \frac{\beta}{2T}z^{1/(2\beta)}
    \int_0^z \frac{u^{1-\frac{1}{2\beta}}}{(1+u)^2}\,du
    \\
    &\le
    \frac{\beta}{2T}
    B\!\left(2-\frac{1}{2\beta},\,\frac{1}{2\beta}\right)
    z^{1/(2\beta)},
  \end{align*}
  Hence
  \begin{align*}
    |R_z(\lambda)|
    &\le
    2\frac{\beta\lambda z}{T(1+z)^2}
    \left(
      \frac{\beta z}{T(1+z)}
      +
      \frac{\beta}{T}
      B\!\left(2-\frac{1}{2\beta},\,\frac{1}{2\beta}\right)
      z^{1/(2\beta)}
    \right) \\
    &\le
    C_{\beta,\lambda,T}
    \left(
      z^{-1}
      +
      z^{1/(2\beta)-1}
    \right)
  \end{align*}
  for some constant \(C_{\beta,\lambda,T}>0\) independent of \(z\). And thus:
  \[
    R_z(\lambda)
    =
    o\!\left(z^{1/\beta}\right)
    =
    o\!\left(\sigma_{\max}^{2/\beta}\right),
  \]
  Combining the first-order contribution with the rightmost parenthesis gives
  \begin{align*}
    &-2\big(\dot B_s(\lambda)+2B_s(\lambda)^2\big)B_s(\lambda)
      -\frac16\ddot B_s(\lambda)
      +B_s(\lambda)\dot B_s(\lambda)
      +\frac{22}{3}B_s(\lambda)^3 \\
    &\qquad=
      -\frac16\ddot B_s(\lambda)
      -B_s(\lambda)\dot B_s(\lambda)
      +\frac{10}{3}B_s(\lambda)^3.
  \end{align*}
  Therefore,
  \[
    \Delta^{\Sigma,[2]}(1,\lambda)
    =\lambda\int_0^T N_s(\lambda)
      \Big(-\frac16\ddot B_s(\lambda)
        -B_s(\lambda)\dot B_s(\lambda)
        +\frac{10}{3}B_s(\lambda)^3\Big)\,ds
      +R_z(\lambda).
  \]
  We now calculate $\Delta^{\Sigma,[2]}(1,\lambda)$ up to $O\!\left(\sigma_{\max}^{2/\beta}\right)$. Using the change of variables \(u=z x^{2\beta}\), we obtain
  \begin{align} \label{eq:before_beta}
    \Delta^{\Sigma,[2]}(1,\lambda)
    &=
    \frac{\lambda}{6T^2} z^{1/\beta}
    \int_0^{z}
    \frac{u^{-1/\beta}}{(1+u)^4}
    \Big(
      -(\beta-1)(2\beta-1)
      - 2(\beta-1)(2\beta-1)u \\
      &\qquad\qquad\qquad\qquad
      + (10\beta^2+3\beta-1)u^2
    \Big)\,du
    + o\!\left(z^{1/\beta}\right). \nonumber
  \end{align}
  Since \(\beta>1\), the integral converges at \(0\) and at \(+\infty\), so the upper integral limit \(z\) can be replaced by \(+\infty\) up to \(o(1)\). Now using $B$ the Beta function, we use
  \[
    I_k(\beta)
    \eqdef
    \int_0^\infty \frac{u^{k-\frac1\beta}}{(1+u)^4}\,du
    =
    B\!\left(k+1-\frac1\beta,\,3-k+\frac1\beta\right),
    \qquad k=0,1,2.
  \]
  and the identities:
  \begin{align*}
    I_0(\beta)
    &=
    \frac{(\beta+1)(2\beta+1)}{(\beta-1)(2\beta-1)}
    B\!\left(3-\frac1\beta,\,1+\frac1\beta\right), \\
    I_1(\beta)
    &=
    \frac{\beta+1}{2\beta-1}
    B\!\left(3-\frac1\beta,\,1+\frac1\beta\right), \\
    I_2(\beta)
    &=
    B\!\left(3-\frac1\beta,\,1+\frac1\beta\right).
  \end{align*}
  Therefore, the terms of~\eqref{eq:before_beta} can be simplified using $B\!\left(3-\frac1\beta,\,1+\frac1\beta\right)$:
  \begin{align*}
    &-(\beta-1)(2\beta-1)I_0(\beta)
    - 2(\beta-1)(2\beta-1)I_1(\beta)
    + (10\beta^2+3\beta-1)I_2(\beta) \\
    &\qquad=
    \Big(
      -(\beta+1)(2\beta+1)
      - 2(\beta^2-1)
      + 10\beta^2+3\beta-1
    \Big)
    B\!\left(3-\frac1\beta,\,1+\frac1\beta\right) \\
    &\qquad=
    6\beta^2
    B\!\left(3-\frac1\beta,\,1+\frac1\beta\right).
  \end{align*}
  Substituting this identity into~\eqref{eq:before_beta} gives
  \begin{align}
    \Delta^{\Sigma,[2]}(1,\lambda)
    &=
    C^{(2)}(\beta)\frac{\lambda}{T^2}
    \left(\frac{\sigma_{\max}^2}{\lambda}\right)^{\frac{1}{\beta}}
    + o\!\left(\sigma_{\max}^{2/\beta}\right),
    \label{eq:second_order_VE_powerlaw}
  \end{align}
  with
  \begin{align}
    C^{(2)}(\beta)
    &\eqdef
    \beta^2
    B\!\left(3-\frac1\beta,\,1+\frac1\beta\right) > 0.
    \label{eq:C2_beta}
  \end{align}

\section{Gaussian discretization error of the DDIM/DDPM exponential integrator}
\label{app:gaussian_exponential_integrator}

Deterministic DDIM/DPM-Solver-1~\citep{song2020denoising,lu2022dpm} and posterior-variance DDPM~\citep{ho2020denoising} can be recovered by the same exponential-integrator family. Define the MMSE denoiser in image space by
\[
D_s(z)\eqdef\EE[X_0\mid X_0+\sigma_s\epsilon=z],
\qquad \epsilon\sim\mathcal N(0,I_d),
\]
where $\epsilon$ is independent of $X_0\sim p_\data$. This denoiser depends on the noise schedule but not on $\eta_s$. The noise prediction satisfies
\[
\epsilon_s(x)\eqdef-\eta_s\sigma_s\nabla\log p_s(x),
\qquad
D_s\left(\frac{x}{\eta_s}\right)=\frac{x}{\eta_s}-\sigma_s\epsilon_s(x).
\]
The generative drift becomes, for $s=T-t>0$,
\[
v_t(x)=-\left(\frac{\dot\eta_s}{\eta_s}
+(1+\alpha)\frac{\dot\sigma_s}{\sigma_s}\right)x
+(1+\alpha)\eta_s\frac{\dot\sigma_s}{\sigma_s}D_s\left(\frac{x}{\eta_s}\right).
\]
Thus, freezing only $D_{T-t}(Y_t/\eta_{T-t})$ leaves a linear SDE whose time-dependent coefficients and Gaussian noise can be integrated exactly. Set
\[
s_k=T-t_k,\qquad
r_k=\frac{\sigma_{s_{k+1}}}{\sigma_{s_k}}.
\]
The resulting update is
\begin{equation}
\label{eq:stochastic_EI_update}
\begin{aligned}
\hat Y_{k+1}=\eta_{s_{k+1}}\Bigg[
&r_k^{1+\alpha}\frac{\hat Y_k}{\eta_{s_k}}
+(1-r_k^{1+\alpha})D_{s_k}\left(\frac{\hat Y_k}{\eta_{s_k}}\right)+\sigma_{s_{k+1}}\sqrt{1-r_k^{2\alpha}}\,\hat W_k
\Bigg],\qquad \hat W_k\sim\mathcal N(0,I_d).
\end{aligned}
\end{equation}
For $\alpha=0$ the noise term is identically zero. At the terminal step, $r_{K-1}=0$ and $\sigma_{s_K}=0$, so the update is $\hat Y_K=D_{s_{K-1}}(\hat Y_{K-1}/\eta_{s_{K-1}})$ for every $\alpha\ge0$.

\paragraph{Exact DDIM and DDPM endpoints.}
For $\alpha=0$,~\eqref{eq:stochastic_EI_update} reduces to
\[
\hat Y_{k+1}=\frac{\eta_{s_{k+1}}}{\eta_{s_k}}\hat Y_k
+\eta_{s_{k+1}}(\sigma_{s_{k+1}}-\sigma_{s_k})
\epsilon_{s_k}(\hat Y_k),
\]
which is exactly deterministic DDIM/DPM-Solver-1. Now take $\alpha=1$,
VP schedules $\eta_s=\sqrt{\frac{1}{1+\sigma_s^2}}$ and let $\bar a_s \eqdef \frac{1}{1+\sigma_s^2}$. Defining
\[
\beta_k=1-\frac{\bar a_{s_k}}{\bar a_{s_{k+1}}},\qquad
\widetilde\beta_k=\frac{1-\bar a_{s_{k+1}}}{1-\bar a_{s_k}}\beta_k,
\]
the same update becomes
\[
\hat Y_{k+1}=\frac{1}{\sqrt{1-\beta_k}}
\left[\hat Y_k-\frac{\beta_k}{\sqrt{1-\bar a_{s_k}}}
\epsilon_{s_k}(\hat Y_k)\right]
+\sqrt{\widetilde\beta_k}\,\hat W_k.
\]
This is exactly DDPM with noise variance $\widetilde\beta_k$. For intermediate $\alpha$,~\eqref{eq:stochastic_EI_update} is a continuous-SDE interpolation between DDIM and DDPM. Generalized DDIM with stochasticity parameter $\eta_{\mathrm{DDIM}}\in[0,1]$ is recovered by choosing a time-dependent $\alpha=\alpha_k=\frac{\log(1-\eta_{\mathrm{DDIM}}^2(1-r_k^2))}{2\log r_k}$, yielding noise variance $\eta_{\mathrm{DDIM}}^2\widetilde\beta_k$.

\begin{rem}[Independence of the rescaling schedule]
Set $Z_t=Y_t/\eta_{T-t}$. The SDE in image coordinates is
\[
dZ_t=-(1+\alpha) \frac{\dot\sigma_{T-t}}{\sigma_{T-t}}
\bigl(Z_t-D_{T-t}(Z_t)\bigr)\,dt
+\sqrt{2\alpha\sigma_{T-t}\dot\sigma_{T-t}}\,dW_t.
\]
Freezing $D_{T-t}(Z_t)$ gives
\[
\hat Z_{k+1}=r_k^{(1+\alpha)}\hat Z_k+(1-r_k^{(1+\alpha)})D_{s_k}(\hat Z_k)
+\sigma_{s_{k+1}}\sqrt{1-r_k^{2\alpha}}\,\hat W_k.
\]
The exact predictor, initialization $\Law(Z_0)=p_\data*\mathcal N(0,\sigma_T^2I_d)$, and update coefficients are independent of $\eta$. Since $\eta_0=1$, the terminal variable is $\hat Y_K=\hat Z_K$. Hence the full terminal discretization error is independent of the rescaling schedule.
\end{rem}

\begin{prop}[First-order Gaussian exponential-integrator error]
\label{prop:gaussian_EI_first_order}
Assume $p_\data=\mathcal N(\mu_\data,\Sigma_\data)$ with $\Sigma_\data\succ0$. Assume that $\dot\sigma_0=0$. Then the mean is exact at every step, and the terminal covariance error decouples in the data eigenbasis as in Proposition~\ref{prop:mean_cov_Gaussian}, with
\[
\Delta^\mu_{\mathrm{EI}}(\alpha,\lambda)=0,
\qquad
\Delta^\Sigma_{\mathrm{EI}}(\alpha,\lambda)
=\gamma\Delta^{\Sigma,[1]}_{\mathrm{EI}}(\alpha,\lambda)+O(\gamma^2),
\]
where
\begin{equation}
\label{eq:EI_gaussian_error}
\Delta^{\Sigma,[1]}_{\mathrm{EI}}(\alpha,\lambda)
=-(1+\alpha)^2\lambda^{2+\alpha}\int_0^T
\frac{\dot\sigma_s^2}{(\lambda+\sigma_s^2)^{2+\alpha}}\,ds.
\end{equation}
\end{prop}

\begin{proof}
We follow the moment-recursion and error-propagation calculations used in the proof of Proposition~\ref{prop:mean_cov_Gaussian}, working here in image coordinates
\[
Z_t=\frac{Y_t}{\eta_{T-t}},\qquad
\hat Z_k=\frac{\hat Y_k}{\eta_{s_k}},\qquad
t_k=k\gamma,\qquad s_k=T-t_k.
\]
In particular, $s_{k+1}=s_k-\gamma$. Throughout this proof, denote the moments of these image-coordinate variables by
\[
m_t=\EE[Z_t],\qquad C_t=\cov[Z_t],\qquad
\hat m_k=\EE[\hat Z_k],\qquad \hat C_k=\cov[\hat Z_k].
\]
The exact marginals give
\[
m_t=\mu_\data,\qquad C_t=\Sigma_\data+\sigma_{T-t}^2I_d.
\]
The initialization is exact, so $\hat m_0=m_0$ and $\hat C_0=C_0$.

For Gaussian data, the exact denoiser is affine:
\[
D_s(z)=\mu_\data+\Sigma_\data(\Sigma_\data+\sigma_s^2I_d)^{-1}(z-\mu_\data).
\]
Substituting this expression into~\eqref{eq:stochastic_EI_update} and dividing by $\eta_{s_{k+1}}$ gives
\[
\hat Z_{k+1}
=M_k\hat Z_k+(I_d-M_k)\mu_\data
+\sigma_{s_{k+1}}\sqrt{1-r_k^{2\alpha}}\,\hat W_k,
\]
where
\[
M_k=(\Sigma_\data+\sigma_{s_k}^2r_k^{1+\alpha}I_d)
(\Sigma_\data+\sigma_{s_k}^2I_d)^{-1}.
\]
Taking expectations and using independence of $\hat W_k$ and $\hat Z_k$ yields the mean and covariance recursions
\begin{align*}
\hat m_{k+1}&=M_k\hat m_k+(I_d-M_k)\mu_\data,\\
\hat C_{k+1}&=M_k\hat C_kM_k^\top
+\sigma_{s_{k+1}}^2(1-r_k^{2\alpha})I_d.
\end{align*}
Since $\hat m_0=\mu_\data$, the mean recursion gives $\hat m_k=m_{t_k}=\mu_\data$ at every step.

Write $\Sigma_\data=U\operatorname{Diag}(\lambda_i)U^\top$. All the matrices $M_k$ co-diagonalize with $\Sigma_\data$, with eigenvalues
\[
M_k(\lambda)=\frac{\lambda+\sigma_{s_k}^2r_k^{1+\alpha}}{\lambda+\sigma_{s_k}^2}.
\]
The affine update preserves Gaussianity, and the covariance recursion preserves eigenspace decoupling. As in Proposition~\ref{prop:mean_cov_Gaussian}, subtract the exact moments and write
\begin{align*}
d^\mu_{t_k}&\eqdef\hat m_k-m_{t_k}=0,\\
d^\Sigma_{t_k}&\eqdef\hat C_k-C_{t_k}
=U\operatorname{Diag}\bigl(\Delta^\Sigma_{\mathrm{EI},t_k}(\alpha,\lambda_i)\bigr)U^\top.
\end{align*}
In one eigendirection of variance $\lambda>0$, the exact variance at time $t_k$ is $\lambda+\sigma_{s_k}^2$. Applying one discrete step to this exact variance and subtracting the exact variance at $t_{k+1}$ gives
\begin{equation}
\label{eq:EI_exact_local_defect}
\begin{aligned}
&M_k(\lambda)^2(\lambda+\sigma_{s_k}^2)
+\sigma_{s_{k+1}}^2(1-r_k^{2\alpha})
-(\lambda+\sigma_{s_{k+1}}^2) =-\frac{\lambda\sigma_{s_k}^2}{\lambda+\sigma_{s_k}^2}(1-r_k^{1+\alpha})^2.
\end{aligned}
\end{equation}
Consequently, the covariance error satisfies the exact recursion
\[
\begin{aligned}
\Delta^\Sigma_{\mathrm{EI},t_{k+1}}(\alpha,\lambda)
&=M_k(\lambda)^2\Delta^\Sigma_{\mathrm{EI},t_k}(\alpha,\lambda)-\frac{\lambda\sigma_{s_k}^2}{\lambda+\sigma_{s_k}^2}(1-r_k^{1+\alpha})^2,
\qquad \Delta^\Sigma_{\mathrm{EI},0}(\alpha,\lambda)=0.
\end{aligned}
\]
The first term propagates the covariance error already present at $t_k$; the second is the local defect created by the new step.

We next make the dependence on the stepsize explicit. Although $M_k(\lambda)$ does not display $\gamma$ in its definition, it depends on it through
\[
r_k=\frac{\sigma_{s_k-\gamma}}{\sigma_{s_k}},\qquad
\sigma_{s_k-\gamma}=\sigma_{s_k}-\gamma\dot\sigma_{s_k}+O(\gamma^2).
\]
This gives the weighted expansion
\[
\sigma_{s_k}(1-r_k^{1+\alpha})
=(1+\alpha)\gamma\dot\sigma_{s_k}+O(\gamma^2).
\]
Using
\[
M_k(\lambda)
=1-\frac{\sigma_{s_k}}{\lambda+\sigma_{s_k}^2}
\bigl[\sigma_{s_k}(1-r_k^{1+\alpha})\bigr],
\]
we obtain
\begin{align}
M_k(\lambda)^2 \label{eq:Mlambdaexp}
&=1-2\gamma(1+\alpha)
\frac{\sigma_{s_k}\dot\sigma_{s_k}}{\lambda+\sigma_{s_k}^2}
+O(\gamma^2),\\
-\frac{\lambda\sigma_{s_k}^2}{\lambda+\sigma_{s_k}^2}(1-r_k^{1+\alpha})^2
&=-\gamma^2(1+\alpha)^2
\frac{\lambda\dot\sigma_{s_k}^2}{\lambda+\sigma_{s_k}^2}
+O(\gamma^3).
\end{align}

Iterating the covariance-error recursion gives
\[
\Delta^\Sigma_{\mathrm{EI},T}(\alpha,\lambda)
=-\gamma^2(1+\alpha)^2
\sum_{k=0}^{K-1}
\frac{\lambda\dot\sigma_{s_k}^2}{\lambda+\sigma_{s_k}^2}
\prod_{j=k+1}^{K-1}M_j(\lambda)^2
+O(\gamma^2).
\]
 Since $0<M_j(\lambda)\le1$, the sum of the $K=T/\gamma$ local $O(\gamma^3)$ remainders is $O(\gamma^2)$.

It remains to evaluate this product. Taking logarithms and approximating the sum by an integral:
\[
\begin{aligned}
\log\left(\prod_{j=k+1}^{K-1}M_j(\lambda)^2\right)
&=-2(1+\alpha)\int_{t_{k+1}}^T
\frac{\sigma_{T-t}\dot\sigma_{T-t}}{\lambda+\sigma_{T-t}^2}\,dt+O(\gamma)\\
&=-2(1+\alpha)\int_0^{s_{k+1}}
\frac{\sigma_s\dot\sigma_s}{\lambda+\sigma_s^2}\,ds+O(\gamma)\\
&=(1+\alpha)\log\left(\frac{\lambda}{\lambda+\sigma_{s_{k+1}}^2}\right)+O(\gamma).
\end{aligned}
\]
And thus:
\[
\prod_{j=k+1}^{K-1}M_j(\lambda)^2
=\left(\frac{\lambda}{\lambda+\sigma_{s_{k+1}}^2}\right)^{1+\alpha}+O(\gamma)
=\left(\frac{\lambda}{\lambda+\sigma_{s_k}^2}\right)^{1+\alpha}+O(\gamma).
\]
The last equality follows from $s_{k+1}=s_k-\gamma$. Substituting the product into the accumulated-error formula yields
\[
\begin{aligned}
\Delta^\Sigma_{\mathrm{EI},T}(\alpha,\lambda)
&=-\gamma^2(1+\alpha)^2\lambda^{2+\alpha}
\sum_{k=0}^{K-1}\frac{\dot\sigma_{s_k}^2}{(\lambda+\sigma_{s_k}^2)^{2+\alpha}}
+O(\gamma^2) \\
&= -\gamma(1+\alpha)^2\lambda^{2+\alpha}
\int_0^T\frac{\dot\sigma_s^2}{(\lambda+\sigma_s^2)^{2+\alpha}}\,ds
+O(\gamma^2)
\end{aligned}
\]
which proves~\eqref{eq:EI_gaussian_error}.
\end{proof}

\paragraph{Consequences for the Fr\'echet error.}
Since the mean is exact, the leading Gaussian Fr\'echet error is
\begin{equation}
\label{eq:EI_Frechet_expansion}
\FD\bigl(p_\data,\Law(\hat Y_K)\bigr)
=\frac{\gamma^2}{4}\sum_{i=1}^d
\frac{\bigl(\Delta^{\Sigma,[1]}_{\mathrm{EI}}(\alpha,\lambda_i)\bigr)^2}{\lambda_i}
+O(\gamma^3).
\end{equation}
For polynomial noise schedules with exponent $\beta>1$, this leading error is minimized over $\alpha\ge0$ at $\alpha=0$. Thus deterministic DDIM strictly improves the leading-order Gaussian discretization error over its stochastic version. This is consistent with the empirical observations from \cite{nichol2021improved, yang2026largelanguagecontinuousdiffusion} in image and text generation.

\section{Discretization error in zero-variance directions}
\label{app:singular_gaussian_endpoint}

The discretization-error formulas derived for strictly positive covariance
eigenvalues can diverge as $\lambda\downarrow0$ and do not, in general,
extend continuously to $\lambda=0$. A zero-variance direction therefore requires
separate treatment: the target is a point mass in that direction, and the reverse process
must remove all remaining variance by the end of sampling. This produces a
singular terminal drift, even when the forward drift and diffusion variance
coefficients are smooth. Consequently, the expansions derived for fixed
$\lambda>0$ cannot simply be evaluated at $\lambda=0$.

For simplicity, assume $\mu_\data=0$ and use the polynomial noise schedule
of the paper,
\[
\sigma_s=\sigma_{\max}\left(\frac{s}{T}\right)^\beta,
\qquad \beta>0.
\]
The numerical mean remains exactly zero, so only the covariance error needs to be studied.

In a direction with $\lambda=0$, the forward variance is
$\eta_s^2\sigma_s^2$. Since $\dot\sigma_s/\sigma_s=\beta/s$, the scalar
reverse drift in this direction is, with $s=T-t$,
\[
v_t(y)=-\left(\frac{\dot\eta_s}{\eta_s}
+\frac{(1+\alpha)\beta}{s}\right)y.
\]
It becomes singular as $s\downarrow0$. In particular, Theorem~\ref{thm:general_weak} cannot be applied here, since the reverse drift does not satisfy its regularity assumptions at final time. Nevertheless, the discrete Euler covariance recursion remains well defined and can be analyzed directly.

We use the notation of Proposition~\ref{prop:mean_cov_Gaussian}: $t_k=k\gamma$, $H_t$ for the linear drift
coefficient, and $\hat C_k$ for the numerical covariance. In the following
calculations, we consider that they are 1D scalar because only evaluated in a direction with
$\lambda=0$:
\[
H_t=-\frac{\dot\eta_{T-t}}{\eta_{T-t}}
-\frac{(1+\alpha)\beta}{T-t},
\qquad
a_t=\frac{\alpha\beta}{T-t}\eta_{T-t}^2\sigma_{T-t}^2.
\]
The Euler update reads
\[
\hat Y_{k+1}=(1+\gamma H_{t_k})\hat Y_k
+\sqrt{2\gamma a_{t_k}}\,\hat W_k,
\qquad \hat W_k\sim\mathcal N(0,1)
\]
which gives the covariance recursion
\begin{equation}
\hat C_{k+1}=(1+\gamma H_{t_k})^2\hat C_k+2\gamma a_{t_k},
\qquad \hat C_0=\eta_T^2\sigma_{\max}^2.
\label{eq:singular_cov_recursion}
\end{equation}
Iterating this identity gives the exact terminal variance
\begin{equation}
\begin{aligned}
\Delta^\Sigma(0)=\hat C_K
&=\eta_T^2\sigma_{\max}^2
\prod_{\ell=0}^{K-1}(1+\gamma H_{t_\ell})^2+2\gamma\sum_{k=0}^{K-1}a_{t_k}
\prod_{\ell=k+1}^{K-1}(1+\gamma H_{t_\ell})^2.
\end{aligned}
\label{eq:singular_cov_products}
\end{equation}
To see the endpoint rate explicitly, first take VE, $\eta_s=1$.
Reindexing the steps by their remaining time $j=K-k$ i.e. $T-t_k=j\gamma = t_j$, gives
\begin{equation}
\begin{aligned}
\Delta^\Sigma(0)
&=\eta_T^2\sigma_{\max}^2
\prod_{j=1}^K
\left(
1-\frac{(1+\alpha)\beta}{j}
-\gamma\frac{\dot\eta_{t_j}}{\eta_{t_j}}
\right)^2+2\alpha\beta\sigma_{\max}^2
\left(\frac{\gamma}{T}\right)^{2\beta}
\sum_{j=1}^K \eta_{t_j}^2j^{2\beta-1}
\prod_{\ell=1}^{j-1}
\left(
1-\frac{(1+\alpha)\beta}{\ell}
-\gamma\frac{\dot\eta_{t_\ell}}{\eta_{t_\ell}}
\right)^2.
\end{aligned}
\label{eq:singular_cov_exact}
\end{equation}
This formula does not involve any Taylor expansion of the singular drift. We show that the leading rate of this term is $\gamma^{2\beta}$.

For fixed $\alpha>0$, the product of the unsquared Euler multipliers
over the last $j-1$ steps,
\[
\prod_{\ell=1}^{j-1}
\left(
1-\frac{(1+\alpha)\beta}{\ell}
-\gamma\frac{\dot\eta_{t_\ell}}{\eta_{t_\ell}}
\right),
\]
is bounded in magnitude by a constant times
$j^{-(1+\alpha)\beta}$, uniformly for $1\leq j\leq K$.
Here we use boundedness of $\dot\eta_s/\eta_s$ on $[0,T]$.
The propagated initial variance is therefore
$O(K^{-2(1+\alpha)\beta})$, which is smaller than $K^{-2\beta}$.
After division by $\sigma_{\max}^2(\gamma/T)^{2\beta}$,
the summands are thus bounded by a constant times
$j^{-1-2\alpha\beta}$, a summable sequence.
For each fixed $j$, $\eta_{t_j}\to1$ and
$\gamma\dot\eta_{t_\ell}/\eta_{t_\ell}\to0$.
Dominated convergence therefore gives, when $\gamma \to 0$
\begin{equation}
\frac{\Delta^\Sigma(0)}
{\sigma_{\max}^2(\gamma/T)^{2\beta}}
\longrightarrow
2\alpha\beta\sum_{j=1}^{\infty}j^{2\beta-1}
\prod_{\ell=1}^{j-1}
\left(1-\frac{(1+\alpha)\beta}{\ell}\right)^2>0.
\label{eq:singular_cov_limit}
\end{equation}

\paragraph{Contribution to the Fr\'echet distance.}
For centered Gaussian data, the exact covariance contribution in each
eigendirection is
\[
\left(\sqrt{\lambda_i+\Delta^\Sigma(\lambda_i)}
-\sqrt{\lambda_i}\right)^2.
\]
At $\lambda_i=0$, this is simply $\Delta^\Sigma(0)$, so the total
contribution of zero eigenvalues is $\sum_{i:\lambda_i=0}\Delta^\Sigma(0)$.
For $\alpha>0$, it therefore has order $\gamma^{2\beta}$, whereas a
fixed positive eigenvalue typically contributes at order $\gamma^2$.
In particular, with the polynomial exponent $\beta=5$ used in the
experiments, zero-variance directions contribute only at order
$\gamma^{10}$ or smaller. %

\section{On the relation with Lipschitz constant optimization}
\label{app:Lipschitz}

To control the discretization error of ODE diffusion sampling ($\alpha = 0$), \cite{chen2025lipschitz} proposed optimizing parameters by minimizing an average squared local Lipschitzness of the drift along the trajectory:
\begin{align}
\label{eq:Lip_obj}
    \min \int_0^T \EE\!\left[\norm{\nabla v_s(Y_s)}^2\right] ds .
\end{align}

\paragraph{Connection with the general weak error formula.}
The objective~\eqref{eq:Lip_obj} is a natural surrogate for our weak error expansion~\eqref{eq:weak_ODE}. Indeed, defining
\[
L_s(Y_s) \coloneqq \|\nabla v_s(Y_s)\|,
\]
we have the pointwise bounds
\[
\|(v_s\cdot\nabla)v_s(Y_s)\|
\leq
L_s(Y_s)\,\|v_s(Y_s)\|,
\qquad
\|J_{t,s}(Y)\|
\leq
\exp\!\left(\int_s^t L_\tau(Y_\tau)\,d\tau\right).
\]
Thus, decreasing the local Jacobian norm can reduce both the local acceleration term and its amplification by the flow.

However, this Lipschitz-type quantity does not capture the time-derivative contribution \(\partial_s v_s\) appearing in the material derivative of the velocity. Moreover, the above inequalities are norm-based upper bounds: they replace the multi-dimensional geometric action of \(\nabla v_s\) and \(J_{t,s}\) by operator-norm estimates. Consequently, the Lipschitz objective may be conservative. In particular, the first-order weak error can vanish even when the local Lipschitz factor is not small.

\paragraph{Comparison in the Gaussian case.}
This distinction becomes even more explicit in the Gaussian setting. For \(\alpha=0\), the first-order covariance error in Proposition~\ref{prop:mean_cov_Gaussian} reads
\begin{align*}
    \Delta^{\Sigma,[1]}(\lambda)
    &=
    -\lambda \big[A_s+B_s(\lambda)\big]_0^T
    -\lambda \int_0^T
    \big(A_s+B_s(\lambda)\big)^2 ds .
\end{align*}
For \(\alpha=0\), the linear drift~\eqref{eq:vt_linear} has the form
\[
v_t(x)=H_t x+r_t,
\qquad
H_t = U \diag{\Delta^H_t(\lambda)} U^\top,
\]
where
\begin{align*}
\Delta^H_{T-t}(\lambda)
&=
-\frac{\dot\eta_t}{\eta_t}
-\frac{\dot\sigma_t\sigma_t}{\lambda+\sigma_t^2}  =
-\big(A_t+B_t(\lambda)\big).
\end{align*}
Therefore,
\begin{align*}
    \Delta^{\Sigma,[1]}(\lambda)
    &=
    \lambda \big[\Delta^H_s(\lambda)\big]_0^T
    -\lambda \int_0^T
    \big(\Delta^H_s(\lambda)\big)^2 ds .
\end{align*}
Under the endpoint conditions in Assumption~\ref{ass:regularity_sched_opt}, and in the large terminal-noise regime \(\sigma_T\to\infty\), the boundary term vanishes, yielding
\begin{align*}
    \Delta^{\Sigma,[1]}(\lambda)
    &=
    -\lambda \int_0^T
    \big(\Delta^H_s(\lambda)\big)^2 ds
    + o(1).
\end{align*}
Hence the covariance contribution to the target Fréchet objective~\eqref{eq:gaussian_expansion_Fréchet} becomes
\begin{align*}
    \min
    \sum_i
    \lambda_i
    \left(
    \int_0^T
    \big(\Delta^H_s(\lambda_i)\big)^2 ds
    \right)^2 .
\end{align*}

By contrast, in the Gaussian case, the averaged Lipschitz objective~\eqref{eq:Lip_obj} reduces to
\begin{align*}
    \min \int_0^T \|H_s\|_{\mathrm{op}}^2\,ds
    =
    \min \int_0^T
    \max_i \big(\Delta^H_s(\lambda_i)\big)^2\,ds .
\end{align*}
At each time, the maximum is attained at either $\lambda_{\min}$ or $\lambda_{\max}$. Our Fréchet objective, instead, depends on the whole spectrum through the weights \(\{\lambda_i\}\). In this sense, our criterion is more geometry-aware: it accounts for the full covariance structure rather than only the extreme eigenvalues controlling the operator norm.

\section{Proof of Proposition~\ref{prop:mean_cov_GSM}}
\label{app:mean_cov_GSM}
\begin{proof}

Choose $\varepsilon_0>0$ such that
$\varepsilon_0\|\zeta\|_\infty<1$, and consider the family
$|\varepsilon|\leq\varepsilon_0$. We first take
$\gamma\to0$ at each fixed $\varepsilon$, then expand the
resulting weak-error coefficient as $\varepsilon\to0$.
The uniform regularity hypotheses above the proposition justify the
differentiations and integrations used below.

Write $\Sigma_\data=U\operatorname{Diag}(\lambda_i)U^\top$.
Since $\EE[\zeta]=0$, $\cov[X_\varepsilon] = \Sigma_\data$.  We consider zero mean data which gives zero mean discretization error. We thus only consider the covariance error here. The general covariance error~\eqref{eq:cov_disc_error} obtained from Theorem~\ref{thm:general_weak}, written at intermediate time $t=k\gamma$ instead of final time $t=T$ (see equation~\eqref{eq:cov_disc_error_proof} in its proof) gives:
\begin{align}
\cov[\hat Y_k]- \cov[Y_t]
&=\gamma D^\varepsilon(t) +O(\gamma^2) %
\label{eq:GSM_matrix_weak_coefficient}
\end{align}

We first show that $D^\varepsilon$ is diagonal in the data covariance eigenbasis $\Sigma_\data = U \operatorname{Diag}(\lambda_i) U^\top$.
Conditionally on $\zeta$, the forward law at time $t$ is Gaussian with
covariance
\[
\Sigma_t^\varepsilon(\zeta)
=\eta_t^2U\operatorname{Diag}
\bigl((1+\varepsilon\zeta)\lambda_i+\sigma_t^2\bigr)U^\top.
\]
Let $S_j$ change the sign of coordinate $j$, and put $R_j=US_jU^\top$.
The conditional Gaussian density satisfies
$p_t^\varepsilon(R_jx\mid\zeta)=p_t^\varepsilon(x\mid\zeta)$.
Averaging over $\zeta$ shows that the unconditional density has the
same symmetry. Consequently
\[
\nabla\log p_t^\varepsilon(R_jx)
=R_j\nabla\log p_t^\varepsilon(x),
\qquad
v_s^\varepsilon(R_jx)=R_jv_s^\varepsilon(x).
\]
The initialization has this symmetry,
and the Euler noise is isotropic. The Euler recursion
therefore gives $R_j\hat Y_k\stackrel{\rm law}{=}\hat Y_k$ for every
$j,k$. In particular $\EE[\hat Y_k]=0$ and
$U^\top\cov[\hat Y_k]U$ is diagonal. 
Thus $U^\top D^\varepsilon U$ is diagonal, we write:
\[
D^\varepsilon(t)
=U\operatorname{Diag}(\Delta_i^\varepsilon(t))U^\top, \quad \Delta_i^\varepsilon(t) \in \RR
\]
In the rest of the proof, to ease the notations and without loss of generality, we work directly in this eigenbasis, as if $U = \id$. We restore $U$ in the final covariance formula.

For the exact continuous variance at reverse
time $t$, define
\[
C_i(t)\eqdef \EE[(Y_t)^2]_i = \eta_{T-t}^2(\lambda_i+\sigma_{T-t}^2)
\]
such that~\eqref{eq:GSM_matrix_weak_coefficient} rewrites for each eigendirection:
\[
\EE[(\hat Y_k)_i^2]
=C_i(t)+\gamma\Delta_i^\varepsilon(t)+O(\gamma^2).
\]
The goal is to calculate an expansion in $\varepsilon$ of $\Delta_i^\varepsilon(t)$ of the form:
\[
\Delta_i^\varepsilon(t)
=\Delta_i(t)+\varepsilon^2\Delta_i^{\varepsilon,[1]}(t)
+o(\varepsilon^2).
\]
where, in the notation of
Proposition~\ref{prop:mean_cov_Gaussian},
$\Delta_i(T)=\Delta^{\Sigma,[1]}(\lambda_i)$ denotes the first-order covariance error in the Gaussian case.

More generally, for a fixed test function $f$, denote
the first-order weak error at time $t$ by
\[
\EE[f(\hat Y_k)]-\EE[f(Y_t)]
=\gamma \,  \mathcal{E}_t^\varepsilon(f)+O(\gamma^2),
\qquad t=k\gamma.
\]
This is the coefficient in Theorem~\ref{thm:general_weak} with
terminal time $t$ instead of $T$. For $f(x)=x_i^2$, we get our covariance error
$\mathcal{E}_t^\varepsilon(f)=\Delta_i^\varepsilon(t)$. To write its small $\varepsilon$ perturbation, we come back to the proof of Theorem~\ref{thm:general_weak}. We defined the backward value function by
\begin{align} u_{s,t}(x)\coloneqq \EE\big[f(Y_{t})\mid Y_s=x\big] = \EE\big[f(Y_{t}^{s,x})] ,\qquad 0\le s\le t.
\end{align}
It represents the expected value of $f$ at time $t$ for the process started from $x$ at time $s$. We showed in the proof of Theorem~\ref{thm:general_weak}, equation~\eqref{eq:error_phi_s}, a first formula for the weak error:
\begin{align}
\label{eq:mathcaleerror}
\mathcal{E}_t^\varepsilon(f)
=\int_0^t\mathbb E[\psi_{s,t}^{\varepsilon}(Y_s)]\,ds
\end{align}
where
\begin{align*}
\psi^\varepsilon_{s,t}(x)
={}&-\frac12
\bigl((\partial_s+\mathcal L^\varepsilon_s)
v_s^\varepsilon\bigr)(x)\cdot
\nabla u_{s,t}(x)-a_s\langle\nabla v_s^\varepsilon(x),
             \nabla^2u_{s,t}(x)\rangle
-\frac12\dot a_s\Delta u_{s,t}(x).
\end{align*}
with generator $\mathcal L^\varepsilon_s g=v_s^\varepsilon\cdot\nabla g+a_s\Delta g$.  We first rewrite this equation in terms of a differential equation.

For fixed \(s<t\), differentiation with respect to the terminal time $t$
gives
\[
\partial_tu_{s,t} (x)
=\EE[(\mathcal L_t  f)(Y_t^{s,x})].
\]
Thus differentiating~\eqref{eq:mathcaleerror} in time and using
$u_{t,t}(x)=f(x)$ gives
\begin{align} \nonumber
\frac{d}{dt}\mathcal E_t^\varepsilon(f)
&= \int_0^t \mathbb E[\partial_t \psi_{s,t}^{\varepsilon}(Y_s)]ds
+\EE[\psi_{t,t}^\varepsilon(Y_t)] \\
&= \mathcal E_t^\varepsilon(\mathcal L^\varepsilon_t f)
+\EE[\psi_{t,t}^\varepsilon(Y_t)],
\qquad \mathcal E_0^\varepsilon(f)=0.
\label{eq:GSM_dual_weak_evolution}
\end{align}
The first term propagates the error accumulated before time
$t$, while the second is the local Euler defect at time $t$.

In the rest of the proof, we write an expansion at order two in $\varepsilon$ of all the terms  in \eqref{eq:GSM_dual_weak_evolution} for $f(x)=x_i^2$. 

We start by expanding the reverse drift $v_t^\varepsilon$. Denote
\[
\lambda_i(t)\eqdef\eta_t^2(\lambda_i+\sigma_t^2),
\qquad y_i(t)\eqdef\frac{x_i}{\sqrt{\lambda_i(t)}},
\qquad q_i(t)\eqdef\frac{\lambda_i}{\lambda_i+\sigma_t^2}
\]
$\varphi_t$ the density of $\mathcal{N}(0, \operatorname{Diag}(\lambda_i(t))$). Conditionally on $\zeta$, the coordinates of $X_t^\varepsilon$
in the eigenbasis of $\Sigma_{\mathrm{data}}$ are independent
centered Gaussians with variances
\[
\eta_t^2\big((1+\varepsilon\zeta)\lambda_j+\sigma_t^2\big)
=\lambda_j(t)\big(1+\varepsilon\zeta q_j(t)\big).
\]
Consequently,
\[
p_t^\varepsilon(x\mid\zeta)
=\prod_{j=1}^d
\frac{
\exp\!\left[-\frac{y_j(t)^2}
{2(1+\varepsilon\zeta q_j(t))}\right]
}{
\sqrt{2\pi\lambda_j(t)(1+\varepsilon\zeta q_j(t))}
}.
\]
Dividing by the reference Gaussian density $\varphi_t$
gives
\[
\frac{p_t^\varepsilon(x\mid\zeta)}{\varphi_t(x)}
=\prod_{j=1}^d(1+\varepsilon\zeta q_j(t))^{-1/2}
\exp\!\left[
\frac12\sum_{j=1}^d y_j(t)^2
\frac{\varepsilon\zeta q_j(t)}
{1+\varepsilon\zeta q_j(t)}
\right].
\]
We use the probabilists' Hermite polynomials
$H_2(y)=y^2-1$, $H_3(y)=y^3-3y$, and
$H_4(y)=y^4-6y^2+3$. Taking the logarithm gives
\[
L_t^\varepsilon(x,\zeta)
\eqdef \log\frac{p_t^\varepsilon(x\mid\zeta)}{\varphi_t(x)}
=-\frac12\sum_j\log(1+\varepsilon\zeta q_j(t))
+\frac12\sum_j y_j(t)^2
\frac{\varepsilon\zeta q_j(t)}
{1+\varepsilon\zeta q_j(t)}.
\]
Using $\log(1+u)=u-u^2/2+o(u^2)$ and
$u/(1+u)=u-u^2+o(u^2)$, we obtain, with
$S_t(y)=\sum_jq_j(t)H_2(y_j(t))$:
\[
L_t^\varepsilon(x,\zeta)
=\frac{\varepsilon\zeta}{2}S_t(y)
+\frac{\varepsilon^2\zeta^2}{4}
\sum_jq_j(t)^2(1-2y_j(t)^2)+o(\varepsilon^2).
\]
Boundedness of $\zeta$ gives integrable bounds
for the density derivatives, uniformly in time and
small $\varepsilon$. Thus the conditional Taylor expansion
can be averaged over $\zeta$. For the $O(\varepsilon)$ term:
\[
\left.\partial_\varepsilon p_t^\varepsilon(x)\right|_{\varepsilon=0}
=\varphi_t(x)\EE\!\left[
e^{L_t^0}
\left.\partial_\varepsilon L_t^\varepsilon\right|_{\varepsilon=0}
\right]
=\frac12\varphi_t(x)S_t(y)\EE[\zeta]=0.
\]
thus the unconditional density has no term at \(O(\varepsilon)\). For the $O(\varepsilon^2)$ term, differentiating twice gives
\[
\left.\partial_\varepsilon^2p_t^\varepsilon(x)\right|_{\varepsilon=0}
=\varphi_t(x)\EE\!\left[
e^{L_t^0}
\left(
\left(\left.\partial_\varepsilon L_t^\varepsilon
\right|_{\varepsilon=0}\right)^2
+\left.\partial_\varepsilon^2L_t^\varepsilon
\right|_{\varepsilon=0}
\right)
\right]
=\frac{\nu}{4}\varphi_t(x)R_t(y),
\]
where $\nu=\EE[\zeta^2]$ and
\[
R_t(y)
=S_t(y)^2+2\sum_jq_j(t)^2(1-2y_j(t)^2).
\]
Thus, the unconditional density satisfies
\[
p_t^\varepsilon(x)
=\varphi_t(x)\left[1+\frac{\varepsilon^2\nu}{8}R_t(y)\right]
+o(\varepsilon^2).
\]
The identities $H_2(y)^2=H_4(y)+4H_2(y)+2$ and
$y^2=H_2(y)+1$ give the equivalent form
\[
R_t(y)=\sum_jq_j^2H_4(y_j)
+2\sum_{j<\ell}q_jq_\ell H_2(y_j)H_2(y_\ell).
\]
Hence $R_t$ is orthogonal to every polynomial of
total degree at most two in $\varphi_t$. Since $\log(1+u)=u+o(u)$,
\[
\log p_{t}^\varepsilon(x)
=\log\varphi_t(x)+\frac{\varepsilon^2\nu}{8}R_t(y)
+o(\varepsilon^2).
\]
Using $\partial_{x_i}=\lambda_i(t)^{-1/2}\partial_{y_i}$,
$H_4'=4H_3$, and $H_2'=2y$, we obtain
\[
\partial_{x_i}R_t(y)
=\frac{4q_i}{\sqrt{\lambda_i(t)}}
\left[q_iH_3(y_i)
+\sum_{j\ne i}q_jy_iH_2(y_j)\right].
\]
The reverse drift is, for $r = T-t$
\[
v_t^\varepsilon(x)
=-\frac{\dot\eta_r}{\eta_r}x
+(1+\alpha)\eta_r^2\sigma_r\dot\sigma_r
\nabla\log p_r^\varepsilon(x)
\]
Therefore
\[
 v_t^\varepsilon(x)
= v_t^{0}(x)
+\varepsilon^2  v_t^{[1]}(x)
+o(\varepsilon^2),
\]
where (recall that we are assuming that we are in the data eigenbasis) $v_{t}^{0}(x)_i= \Delta^H_i(t)x_i$ with, for $r = T-t$:
\[
\Delta^H_i(t) 
=-\frac{\dot\eta_r}{\eta_r}
-(1+\alpha)\frac{\sigma_r\dot\sigma_r}
{\lambda_i+\sigma_r^2},
\]
and the correction term is
\begin{align}
v_{t}^{[1]}(x)_i
&=\nu\frac{(1+\alpha)\eta_r\sigma_r\dot\sigma_r}
{2\sqrt{\lambda_i+\sigma_r^2}}q_i(r)
\left[q_i(r)H_3(y_i(r))
+\sum_{j\ne i}q_j(r)y_i(r)H_2(y_j(r))\right], \quad y_i(t) = \frac{x_i}{\sqrt{\lambda_i(t)}}
\label{eq:GSM_score_correction}
\end{align}

The drift expansion above implies
$\partial_\varepsilon\mathcal L_t^\varepsilon|_{\varepsilon=0}=0$,
since the diffusion coefficient is independent of $\varepsilon$.
Under the assumed regularity, differentiating the backward
equation gives 
$\partial_\varepsilon u_{s,t}^\varepsilon|_{\varepsilon=0}=0$.
This also implies
$\partial_\varepsilon\mathcal E_t^\varepsilon(f)|_{\varepsilon=0}=0$.
Thus the weak-error coefficient also has no term linear
in $\varepsilon$.

We can now write the $\varepsilon$-expansion of the second term $\EE[\psi_{t,t}^\varepsilon(Y_t)]$ in the right side of~\eqref{eq:GSM_dual_weak_evolution}.  For our $f(x) = x_i^2$, we have 
\(\nabla f(x)=2x_i e_i\), \(\nabla^2f(x)=2e_ie_i^\top\), and
\(\Delta f(x)=2\).

At \(s=t\), \(u_{t,t}(x)=f(x)\), hence
\begin{align}
\psi_{t,t}^\varepsilon(x)
&=-\frac12
 \bigl[(\partial_t+\mathcal L^\varepsilon_t)
       v_t^\varepsilon\bigr](x)\cdot(2x_i e_i)
 -a_t\langle\nabla v_t^\varepsilon(x),2e_ie_i^\top\rangle
 - \dot a_t \nonumber \\
&=-x_i
 \bigl[(\partial_t+\mathcal L^\varepsilon_t)
       v_t^\varepsilon\bigr]_i(x)
 -2a_t\partial_{x_i}v_{t}^\varepsilon(x)_i
 -\dot a_t.
 \label{eq:phi_tt_1}
\end{align}
We derive the expansion $\psi_{t,t}^\varepsilon=\psi_{t,t}^0+
\varepsilon^2\psi_{t,t}^{[1]}+o(\varepsilon^2)$ of the above, using
\[
v_t^\varepsilon=v_t^{0}+\varepsilon^2v_t^{[1]}+o(\varepsilon^2),
\qquad
\mathcal L_t^\varepsilon
=\mathcal L_t^0+\varepsilon^2v_t^{[1]}\cdot\nabla
+o(\varepsilon^2).
\]
We thus have, for the first term in~\eqref{eq:phi_tt_1}:
\[
(\partial_t+\mathcal L^\varepsilon_t)
       v_t^\varepsilon  = (\partial_t+\mathcal L^0_t)
       v_t^0 + \varepsilon^2 \left( (\partial_t+\mathcal L_t^0)v_t^{[1]} 
+\bigl(v_t^{[1]}\cdot\nabla\bigr)v_{t}^{0} \right) + o(\varepsilon^2).
\]
Moreover, using \(v_{t}^{0}(x)_i =\Delta_i^H(t)x_i\), we can simplify above the term
$$\bigl[(v_t^{[1]}\cdot\nabla)v_t^0\bigr](x)_i
=\sum_j v_t^{[1]}(x)_j
  \partial_{x_j}\bigl(v_t^0(x)_i\bigr)
=\Delta_i^H(t)v_t^{[1]}(x)_i.$$
Finally, for the second term in~\eqref{eq:phi_tt_1}, using again \(v_{t}^{0}(x)_i =\Delta_i^H(t)x_i\):
\[
-2a_t\partial_{x_i}v_t^\varepsilon(x)_i
=-2a_t\Delta_i^H(t)
-2a_t\varepsilon^2\partial_{x_i}v_t^{[1]}(x)_i
+o(\varepsilon^2).
\]
Collecting all the above terms gives $\psi_{t,t}^\varepsilon=\psi_{t,t}^0+
\varepsilon^2\psi_{t,t}^{[1]}+o(\varepsilon^2)$ with
\begin{align}
\label{eq:phitt0}
\psi_{t,t}^{0}(x)
&=-\bigl(\dot\Delta_i^H(t)+(\Delta_i^H(t))^2\bigr)x_i^2
  -2a_t\Delta_i^H(t)-\dot a_t,\\
\psi_{t,t}^{[1]}(x)
&=-x_i\bigl[(\partial_t+\mathcal L_t^0)v_t^{[1]}\bigr](x)_i
  -x_i\Delta_i^H(t)v_t^{[1]}(x)_i
  -2a_t\partial_{x_i}v_t^{[1]}(x)_i.
\label{eq:phitt1_pre}
\end{align}
The generator product rule
gives
\begin{align*}
(\partial_t+\mathcal L_t^0)
 \bigl(x_i v_t^{[1]}(x)_i\bigr)
={}&x_i\bigl[(\partial_t+\mathcal L_t^0)v_t^{[1]}\bigr](x)_i +\bigl(\mathcal L_t^0x_i\bigr)v_t^{[1]}(x)_i
 +2a_t\partial_{x_i}v_t^{[1]}(x)_i.
\end{align*}
Here $\mathcal L_t^0x_i=v_t^0(x)_i=\Delta_i^H(t)x_i$.
Comparison with~\eqref{eq:phitt1_pre} therefore shows that, for $g_t(x) \eqdef x_i v_t^{[1]}(x)_i$, 
\begin{align}
\psi_{t,t}^{[1]}
=-(\partial_t+\mathcal L_t^0)g_t \label{eq:phitt1}
\end{align}

Now we look for the expansion for $\EE[\psi_{t,t}^\varepsilon(Y_t)]$. Note that the $\varepsilon$ contribution comes from both the change in the local defect $\psi_{t,t}^\varepsilon$ and the change in the density $p_{T-t}^\varepsilon$ of $Y_t$. The assumed regularity permits their expansions to be
multiplied and integrated. We have seen that both first-order perturbations in $\varepsilon$ vanish, but both second-order contributions must be retained. We can write:
\begin{align*}
\EE[\psi_{t,t}^\varepsilon(Y_t)] = E_0 + \varepsilon^2 E_1 + o(\varepsilon^2)
\end{align*}
with 
\begin{align*}
E_0 &= \EE[\psi_{t,t}^0(Y^{0}_t)] \\
E_1 &=
\underbrace{\lim_{\varepsilon\to0}
\frac{\EE[\psi_{t,t}^0(Y_t)]
      -\EE[\psi_{t,t}^0(Y_t^0)]}{\varepsilon^2}}_{\text{change in the law}}
+
\underbrace{\EE[\psi_{t,t}^{[1]}(Y_t^0)]}_{\text{change in the local defect}}.
\end{align*}
where we denote $Y^{0}_t$ the solution of the Gaussian SDE, with linear velocity $v_{t}^{0}$. We first derive $E_0$. 
The exact Gaussian variance in the eigenspace $C(t) \eqdef  \lambda_{T-t}$ satisfies
\[
\dot C_i(t)
=2\Delta_i^H(t)C_i(t)+2a_t.
\]
Differentiating this identity once more gives
\[
\frac12\ddot C_i(t)
=\dot\Delta_i^H(t)C_i(t)
+2(\Delta_i^H(t))^2C_i(t)
+2a_t\Delta_i^H(t)+\dot a_t.
\]
As $\EE[(Y^0_t)_i^2]=C_i(t)$, it follows from~\eqref{eq:phitt0} that
\begin{equation}
E_0 = \EE[\psi_{t,t}^0(Y^0_t)]
=(\Delta_i^H(t))^2C_i(t)
-\frac12\ddot C_i(t).
\label{eq:GSM_gaussian_local_defect}
\end{equation}
We next prove that $E_1 = 0 $. First, the first term (change in the law) indeed contributes zero. Indeed, from~\eqref{eq:phitt0}, \(\psi_{t,t}^0(x)\) is a constant plus a multiple of \(x_i^2\), and \(\mathbb E[(Y_t)_i^2]= \mathbb E[(Y^0_t)_i^2]\). Now, let's prove that the change in the local defect also vanishes. 
The Gaussian process $Y_t^0$ has generator $\mathcal L_t^0$, so (recall that $g_t(x) = x_i v_t^{[1]}(x)_i$)
\[
\EE[\psi_{t,t}^{[1]}(Y_t^0)] = - \EE[((\partial_t+\mathcal L_t^0)g_t)(Y_t^0)]
= - \frac{d}{dt}
  \EE\bigl[g_t(Y_t^0)\bigr] = - \frac{d}{dt}
  \EE\bigl[(Y_t^0)_i v_t^{[1]}(Y_t^0)_i\bigr].
\]
Under the law of $Y_t^0$, the coordinates $y_j$ in~\eqref{eq:GSM_score_correction} are independent standard normals.
The two types of terms in $(Y_t^0)_i v_t^{[1]}(Y_t^0)_i$ have expectations
\[
\EE[y_iH_3(y_i)]=\EE[y_i^4-3y_i^2]=3-3=0,
\qquad
\EE[y_i^2H_2(y_j)]
=\EE[y_i^2]\EE[H_2(y_j)]=0\quad(j\ne i).
\]
Consequently, $\EE[\psi_{t,t}^{[1]}(Y_t^0)]=0$. Thus, we get overall $E_1 = 0$ and 
$$
\EE[\psi_{t,t}^\varepsilon(Y_t)]
=\EE[\psi_{t,t}^0(Y_t^0)]+o(\varepsilon^2).
$$

It remains to handle the first propagation term in~\eqref{eq:GSM_dual_weak_evolution}.
Using again
$
g_t(x)=x_i v_t^{[1]}(x)_i,
$
the generator acts on \(f(x)=x_i^2\) as
\[
\mathcal L_t^\varepsilon f(x)
=2x_i v_t^\varepsilon(x)_i+2a_t
=2\Delta_i^H(t)x_i^2+2a_t
 +2\varepsilon^2g_t(x)+o(\varepsilon^2).
\]Apply the weak-error functional to each term, holding \(t\) fixed:
\[
\begin{aligned}
\mathcal E_t^\varepsilon(\mathcal L_t^\varepsilon f)
={}&2\Delta_i^H(t)\mathcal E_t^\varepsilon(x_i^2)
 +2a_t\mathcal E_t^\varepsilon(1)
 +2\varepsilon^2\mathcal E_t^\varepsilon(g_t)
 +o(\varepsilon^2)\\
={}&2\Delta_i^H(t)\Delta_i^\varepsilon(t)
 +2\varepsilon^2\mathcal E_t^0(g_t)
 +o(\varepsilon^2).
\end{aligned}
\]
To justify the above remainder, under the regularity and domination assumptions of the
proposition, applying $\mathcal E_t^\varepsilon$ to the
remainder in the generator expansion gives $o(\varepsilon^2)$.
Moreover, continuity in $\varepsilon$ yields
$\mathcal E_t^\varepsilon(g_t)=\mathcal E_t^0(g_t)+o(1)$.

Let's finally derive the expression of $\mathcal E_t^0(g)$. At \(\varepsilon=0\), the laws of the exact $Y^0_t$ and discretized  $\hat Y^0_k$ variables are centered Gaussians. Their variances in direction \(j\) are respectively
\[
C_j(t)
\quad\text{and}\quad
C_j(t)+\gamma\Delta_j(t)+O(\gamma^2).
\]
Write their densities as
\[
\rho_t(x)=\prod_j\varphi_{C_j(t)}(x_j),
\qquad
\hat\rho_t(x)=\prod_j
\varphi_{C_j(t)+\gamma\Delta_j(t)+O(\gamma^2)}(x_j).
\]
For one coordinate, using again $H_2(y)=y^2-1$,
\[
\frac{\partial_C\varphi_C(x)}{\varphi_C(x)}
=\partial_C\left(-\tfrac12\log C-\frac{x^2}{2C}\right)
=-\frac1{2C}+\frac{x^2}{2C^2}
=\frac1{2C}H_2\!\left(\frac{x}{\sqrt C}\right).
\]Taylor expansion of each factor, followed by multiplication, gives
\[
\hat\rho_t(x)
=\rho_t(x)\left[
1+\gamma\sum_j
\frac{\Delta_j(t)}{2C_j(t)}
H_2\!\left(\frac{x_j}{\sqrt{C_j(t)}}\right)
\right]+O(\gamma^2).
\]
Using the definition
\[
\mathcal E_t^0(g)
=\lim_{\gamma\to0}
\frac{\EE[g(\hat Y_k^0)]-\EE[g(Y_t^0)]}{\gamma}.
\]
we get,
\begin{align}
\mathcal E_t^0(g)
=\sum_j\frac{\Delta_j(t)}{2C_j(t)}
\EE\!\left[
g(Y_t^0)
H_2\!\left(\frac{(Y_t^0)_j}{\sqrt{C_j(t)}}\right)
\right].
\label{eq:etg0}
\end{align}
Multiplying~\(\eqref{eq:GSM_score_correction}\) by
\(x_i=\eta_r\sqrt{\lambda_i+\sigma_r^2}\,y_i(r)\) yields
\[
g_t(x)=x_i v_t^{[1]}(x)_i = 
\frac{\nu(1+\alpha)\eta_r^2\sigma_r\dot\sigma_r}{2}\,
q_i(r)\left[
q_i(r)y_i(r)H_3(y_i(r))
+\sum_{j\ne i}q_j(r)y_i(r)^2H_2(y_j(r))
\right],  \quad y_i(t) = \frac{x_i}{\sqrt{\lambda_i(t)}}.
\]Under \(Y_t^0\), the \(y_j(r)\) are independent standard normals. In~\eqref{eq:etg0}, only two types of terms survive:
\[
\EE[y_iH_3(y_i)H_2(y_i)]
=\EE[y_i^6-4y_i^4+3y_i^2]=6,
\]and, for \(j\ne i\), using $\lambda_j(r) = C_j(t)$
\[
\EE[y_i^2H_2(y_j)^2]
=\EE[y_i^2]\EE[H_2(y_j)^2]=2.
\]All other terms cancel by polynomial orthogonality. We thus get:
\begin{align*}
\mathcal E_t^0(g_t)
&=\frac12 \nu(1+\alpha)\eta_r^2\sigma_r\dot\sigma_r\,q_i(r)
\left[
\frac{3q_i(r)\Delta_i(t)}{C_i(t)}
+\sum_{j\ne i}\frac{q_j(r)\Delta_j(t)}{C_j(t)}
\right] \\
&= \frac12 \nu(1+\alpha)\sigma_r\dot\sigma_r
\frac{\lambda_i}{\lambda_i+\sigma_r^2}
\left[
\frac{3\lambda_i\Delta_i(t)}{(\lambda_i+\sigma_r^2)^2}
+\sum_{j\ne i}
\frac{\lambda_j\Delta_j(t)}{(\lambda_j+\sigma_r^2)^2}
\right].
\end{align*}
All previous expansions considered, we have overall, uniformly for $t\in[0,T]$,
\[
\Delta_i^\varepsilon(t)
=\Delta_i(t)
+\varepsilon^2\Delta_i^{\varepsilon,[1]}(t)
+o(\varepsilon^2).
\]
Differentiating twice the integral form of
\eqref{eq:GSM_dual_weak_evolution} then gives the following
coefficient equations
\begin{align}
\frac{d}{dt}\Delta_i(t)
&=2\Delta_i^H(t)\Delta_i(t)
+(\Delta_i^H(t))^2C_i(t)
-\frac12\ddot C_i(t),
\qquad \Delta_i(0)=0,
\label{eq:GSM_gaussian_accumulated_defect}\\
\frac{d}{dt}\Delta_i^{\varepsilon,[1]}(t)
&=2\Delta_i^H(t)\Delta_i^{\varepsilon,[1]}(t)
+2\mathcal E_t^0
  \bigl(g_t\bigr),
\qquad \Delta_i^{\varepsilon,[1]}(0)=0.
\label{eq:GSM_correction_evolution}
\end{align}
Equation~\eqref{eq:GSM_gaussian_accumulated_defect} is the Gaussian
covariance-error equation, so
$\Delta_i(t)=\Delta_t^{\Sigma,[1]}(\lambda_i)$.
The partial-time formula~\eqref{eq:gaussian_partial_cov_ibp} from the proof
of Proposition~\ref{prop:mean_cov_Gaussian} therefore gives directly
\[
\Delta_i(T-t)=-\lambda_i(t)\mathcal G(\lambda_i,t),
\qquad \lambda_i(t)=\eta_t^2(\lambda_i+\sigma_t^2),
\]
where, using the same notations $A_s$, $B_s$, and $N_s$,
\begin{align}
\mathcal G(\lambda,t)
\eqdef{}&N_t(\lambda)^{-\alpha}\Bigg\{
\Big[N_s(\lambda)^\alpha V_s(\lambda)\Big]_{s=t}^{s=T}
\nonumber\\
&\qquad+\int_t^T N_s(\lambda)^\alpha
\big[V_s(\lambda)^2-\alpha^2B_s(\lambda)^2\big]\,ds
\Bigg\}.
\label{eq:Gj}
\end{align}

Finally, apply variation of constants to
\eqref{eq:GSM_correction_evolution} at the final sampling time:
\[
\Delta_i^{\varepsilon,[1]}(T)
=\int_0^T
\exp\!\left(2\int_s^T\Delta_i^H(u)\,du\right)
2\mathcal E_s^0(g_s)\,ds.
\]
After changing to forward schedule time, its integrating factor is
\[
\exp\!\left(2\int_{T-t}^T\Delta_i^H(u)\,du\right)
=\frac{N_t(\lambda_i)^{1+\alpha}}{\eta_t^2}.
\]
The expression for \(\mathcal E_{T-t}^0(g_{T-t})\) obtained above is
\begin{align*}
2\mathcal E_{T-t}^0(g_{T-t})
={}&\nu(1+\alpha)\sigma_t\dot\sigma_t
N_t(\lambda_i)
\left[
\frac{3\lambda_i\Delta_i(T-t)}
     {(\lambda_i+\sigma_t^2)^2}
+\sum_{j\ne i}
\frac{\lambda_j\Delta_j(T-t)}
     {(\lambda_j+\sigma_t^2)^2}
\right].
\end{align*}
For each \(j\), we got previously 
\[
\frac{1}{\eta_t^2}
\frac{\lambda_j\Delta_j(T-t)}
     {(\lambda_j+\sigma_t^2)^2}
=-N_t(\lambda_j)\mathcal G(\lambda_j,t).
\]
Consequently, \(\Delta_i^{\varepsilon,[1]}(T)=\nu\omega_i\), where
\begin{align}
\omega_i(\Lambda,\sigma_t,\eta_t,\alpha)
={}&-(1+\alpha)\int_0^T
\sigma_t\dot\sigma_t\,N_t(\lambda_i)^{2+\alpha}
\Bigg[
3N_t(\lambda_i)\mathcal G(\lambda_i,t)
+\sum_{j\ne i}
N_t(\lambda_j)\mathcal G(\lambda_j,t)
\Bigg]dt.
\label{eq:GSM_correction_integral}
\end{align}
where we recall that $G(\lambda_i,t)$ is defined in~\eqref{eq:Gj} and $N_t(\lambda_i) = \frac{\lambda}{\lambda+\sigma_t^2}$.

As $\Delta_i^\varepsilon(T)=\Delta_i^\varepsilon$ and
$\Delta_i(T)=\Delta^{\Sigma,[1]}(\lambda_i)$, restoring the original
coordinates gives
\[
D^\varepsilon
=U\operatorname{Diag}\!\left(
\Delta^{\Sigma,[1]}(\lambda_i)
+\varepsilon^2\nu \omega_i +o(\varepsilon^2)
\right)U^\top.
\]

Finally, the exact and Euler terminal means are zero. By
Proposition~\ref{prop:Fréchet_disc}, the coefficient of $\gamma^2$
in their Fr\'echet distance is
\[
\frac12\sum_{i,j}
\frac{(U^\top D^\varepsilon U)_{ij}^2}{\lambda_i+\lambda_j}
=\frac14\sum_i\frac{(\Delta_i^\varepsilon)^2}{\lambda_i}.
\]
Substituting
$\Delta_i^\varepsilon
=\Delta^{\Sigma,[1]}(\lambda_i)
+\varepsilon^2\nu\omega_i +o(\varepsilon^2)$ gives
\[
\sum_i\frac{\Delta^{\Sigma,[1]}(\lambda_i)^2}{4\lambda_i}
+\frac{\varepsilon^2\nu}{2}
\sum_i\frac{\omega_i \Delta^{\Sigma,[1]}(\lambda_i)}{\lambda_i}
+o(\varepsilon^2),
\]
which is the formula of the Proposition.
\end{proof}
\begin{rem}
Note that this expansion is taken at fixed dimension $d$. Under the Gaussian
reference law, the leading log-density perturbation
has variance $\varepsilon^2\nu\sum_j q_j(t)^2/2$, equal to
$\varepsilon^2\nu d/2$ at $t=0$. This suggests the regime
$\varepsilon^2\nu d\ll1$ for quantitative use in large dimension. 
The expansion is being proved for fixed $d$; its accuracy as $d$
increases is not quantified.
\end{rem}

\section{Optimization of diffusion parameters}

\subsection{Proof of Proposition~\ref{prop:opt_alpha}}
\label{app:opt_alpha}

Before turning to the proof, we state the following lemma that characterizes the perturbation of the minimizer of a function changed by a small quantity.
\begin{lem}[stability of a nondegenerate minimizer]
\label{lem:minimizer_stability}
Let \(I\) be a compact interval, and let \(F_\varepsilon,F_0\in C^2(I)\). Assume that \(F_0\) has a unique global minimizer \(x_0\) in the interior of \(I\), with
\[
F_0''(x_0)>0,
\qquad
\|F_\varepsilon-F_0\|_{C^2(I)}\longrightarrow0.
\]
Then, for sufficiently small \(\varepsilon\), \(F_\varepsilon\) has a unique global minimizer \(x_\varepsilon\), and \(x_\varepsilon\to x_0\). Moreover,
\[
|x_\varepsilon-x_0|
\leq
\frac{2}{F_0''(x_0)}
\|F_\varepsilon'-F_0'\|_{\infty,I}.
\]
If, additionally,
\[
F_\varepsilon=F_0+\varepsilon G+o(\varepsilon)
\quad\text{in }C^1(I),
\]
then
\[
x_\varepsilon
=x_0-\varepsilon\frac{G'(x_0)}{F_0''(x_0)}
+o(\varepsilon).
\]
\end{lem}
\begin{proof}
    Uniform convergence and uniqueness of the limiting minimum imply that every minimizer of \(F_\varepsilon\) lies in an arbitrarily small neighborhood of \(x_0\). By continuity of \(F_0''\) and \(C^2\) convergence, we can choose this neighborhood so that
\[
F_\varepsilon''\geq \tfrac12F_0''(x_0)>0
\]
for sufficiently small \(\varepsilon\). Thus \(F_\varepsilon\) is strictly convex there, proving uniqueness. Since the minimizer is interior, the mean value theorem gives
\[
x_\varepsilon-x_0
=-\frac{F_\varepsilon'(x_0)}{F_\varepsilon''(\xi_\varepsilon)},
\]
where \(\xi_\varepsilon\) lies between \(x_0\) and \(x_\varepsilon\). This proves the bound, since \(F_0'(x_0)=0\). The expansion follows from
\(F_\varepsilon'(x_0)=\varepsilon G'(x_0)+o(\varepsilon)\) and
\(F_\varepsilon''(\xi_\varepsilon)\to F_0''(x_0)\)
\end{proof}
We now turn to the proof of Proposition~\ref{prop:opt_alpha}. 
\begin{proof}
Throughout, let $\lambda_i>0$, $\eta_t=1$, and
$\sigma_t=\sigma_{\max}(t/T)^\beta$ with $\beta>1$. Write
\[
\lambda_i(t)\eqdef\lambda_i+\sigma_t^2,
\qquad
\Delta_i(\alpha)\eqdef\Delta^{\Sigma,[1]}(\alpha,\lambda_i),
\qquad
D_i(\alpha)\eqdef\Delta^{\Sigma,[2]}(\alpha,\lambda_i),
\]
and recall the notations from Proposition~\ref{prop:mean_cov_Gaussian} $B_t(\lambda)=\sigma_t\dot\sigma_t/(\lambda+\sigma_t^2)$
and $N_t(\lambda)=\lambda/(\lambda+\sigma_t^2)$.
All large-noise estimates keep $T$, $\beta$, the dimension, and the
positive eigenvalues fixed. The perturbation expansions in $\gamma$
and $\varepsilon$ are taken at fixed $\sigma_{\max}$ first; we then
study their coefficients as $\sigma_{\max}\to\infty$.

\paragraph{Part (i)}
For VE the Euler mean is exact, so the $\gamma^2$ coefficient of the
Fr\'echet error is
\[
\mathcal{E}_0(\alpha)\eqdef\frac14\sum_i\frac{\Delta_i(\alpha)^2}{\lambda_i}.
\]
Since $B_0(\lambda)=0$,
\eqref{eq:cov_error_gauss} gives:
\begin{align*}
\Delta_i(\alpha)
={}&\lambda_i(\alpha^2-1)
\int_0^T N_t(\lambda_i)^\alpha B_t(\lambda_i)^2\,dt
-\lambda_i N_T(\lambda_i)^\alpha B_T(\lambda_i),\\
\text{with} \quad &\int_0^T N_t(\lambda_i)^\alpha B_t(\lambda_i)^2\,dt
={}\frac{\beta}{2T}\sigma_{\max}^{1/\beta}\lambda_i^\alpha
\int_0^{\sigma_{\max}^2}
\frac{u^{1-\frac1{2\beta}}}{(\lambda_i+u)^{\alpha+2}}\,du.
\end{align*}
For $\alpha$ in a neighborhood of $1$, the last integral
and its first two $\alpha$-derivatives converge uniformly as
$\sigma_{\max}\to\infty$: differentiating only adds powers of
$\log(\lambda_i/(\lambda_i+u))$, which remain integrable. Moreover, the last term and its first two derivatives are
$O(\sigma_{\max}^{-2\alpha}(1+\log^2\sigma_{\max}))$. Thus $\Delta_i(\alpha)$, $\Delta'_i(\alpha)$ and $\Delta''_i(\alpha)$ are $O(\sigma_{\max}^{1/\beta})$.

We apply Lemma~\ref{lem:minimizer_stability} to the normalized objective \( \mathcal{E}_0(\alpha)/\sigma_{\max}^{2/\beta}\). The limiting normalized covariance error is
\[
\delta_i(\alpha)
\eqdef\lim_{\sigma_{\max}\to\infty}
\sigma_{\max}^{-1/\beta}\Delta_i(\alpha)
=\frac{\beta}{2T}\lambda_i^{1-a}
(\alpha^2-1)\operatorname{B}(2-a,\alpha+a), \quad a = \frac{1}{2\beta}
\]
Consequently, the limiting normalized objective (when $\sigma_{\max} \to \infty$) is
\[
\overline F_\infty(\alpha)
=\frac14\sum_i\frac{\delta_i(\alpha)^2}{\lambda_i}.
\]
and the convergence is in $\mathcal{C}^2$.
On a fixed neighborhood of $1$ contained in $(0,\infty)$,
the beta factor is smooth and strictly positive.
Thus $\overline F_\infty$ vanishes only at $\alpha=1$,
which is its unique minimum. Moreover,
\[
\delta_i(1)=0,\qquad
\delta_i'(1)
=\frac{\beta}{T}\lambda_i^{1-a}
\operatorname{B}(2-a,1+a).
\]
Differentiating the objective twice therefore gives
\[
\overline F_\infty''(1)
=\frac12\sum_i\frac{\delta_i'(1)^2}{\lambda_i}
=\frac{\beta^2}{2T^2}
\operatorname{B}(2-a,1+a)^2
\sum_i\lambda_i^{1-1/\beta}>0.
\]
Hence the limiting minimum is nondegenerate. We can apply  Lemma~\ref{lem:minimizer_stability}. We obtain a unique local minimizer \(\alpha_0^\star\to1\). Moreover, since
\[
\mathcal{E}_0'(1)=O(\sigma_{\max}^{-2+1/\beta}), \quad \mathcal{E}_0''(1) = O(\sigma_{\max}^{2/\beta})
\]
the displacement bound in the lemma gives
\[
\alpha_0^\star-1
=O\!\left(\frac{|\mathcal{E}_0'(1)|}{\sigma_{\max}^{2/\beta}}\right)
=O(\sigma_{\max}^{-2-1/\beta})
=o(\sigma_{\max}^{-2}).
\]

\paragraph{Part (ii): the stepsize correction.}
By Appendix~\ref{app:Fréchet_dist_exp_gaussian}, the Gaussian
Fr\'echet error has the expansion
\[
\mathcal E(\alpha,\gamma)
=\gamma^2 \big( \mathcal{E}_0(\alpha)+\gamma \mathcal{E}_1(\alpha) \big) +o(\gamma^3),
\qquad
\mathcal{E}_1(\alpha)\eqdef\sum_i
\left(\frac{\Delta_i(\alpha)D_i(\alpha)}{2\lambda_i}
-\frac{\Delta_i(\alpha)^3}{8\lambda_i^2}\right).
\]
Fix $\sigma_{\max}$ and a compact neighborhood $I$ of
the nondegenerate Gaussian minimizer $\alpha_0^\star$.
Differentiating the exact covariance recursion twice in
$\alpha$, one can prove that 
\[
\gamma^{-2}\mathcal E(\cdot,\gamma)
=\mathcal{E}_0+\gamma \mathcal{E}_1+O(\gamma^2)
\quad\text{in }C^2(I).
\]
We can then apply the perturbation formula of Lemma~\ref{lem:minimizer_stability} and get 
\begin{align}
\alpha^\star(\gamma)
&=\alpha_0^*-\gamma\,
\frac{\mathcal{E}_1'(\alpha_0^*)}{\mathcal{E}_0''(\alpha_0^*)}+o(\gamma),
\label{eq:generic_alpha_star_expansion}
\end{align}
where 
\begin{align*}
\mathcal{E}_1'(\alpha)
&=\sum_i\left[
\frac{\Delta_i'(\alpha)D_i(\alpha)
+\Delta_i(\alpha)D_i'(\alpha)}{2\lambda_i}
-\frac{3\Delta_i(\alpha)^2\Delta_i'(\alpha)}{8\lambda_i^2}
\right],\nonumber\\
\mathcal{E}_0''(\alpha)
&=\frac12\sum_i\frac{\Delta_i'(\alpha)^2
+\Delta_i(\alpha)\Delta_i''(\alpha)}{\lambda_i}.
\end{align*}

Set $a=1/(2\beta)$ and define the positive constants
\[
C^{(1)}(\beta)\eqdef\beta\operatorname{B}(2-a,1+a),
\qquad
C^{(2)}(\beta)\eqdef\beta^2\operatorname{B}(3-2a,1+2a).
\]
The integral in part~(i), evaluated by $u=\lambda_i v$, yields
\[
\Delta_i'(1)
=\frac{C^{(1)}(\beta)}{T}
\sigma_{\max}^{1/\beta}\lambda_i^{1-a}
+o(\sigma_{\max}^{1/\beta}).
\]
We proved in Appendix~\ref{sec:snd_order_VE}, equation~\eqref{eq:second_order_VE_powerlaw} that
\[
D_i(1)
=\frac{C^{(2)}(\beta)}{T^2}
\sigma_{\max}^{2/\beta}\lambda_i^{1-2a}
+o(\sigma_{\max}^{2/\beta}).
\]
The same change of variable $u=\sigma_t^2$ in
\eqref{eq:gaussian_second_order_cov} bounds $D_i$ and $D_i'$
by $O(\sigma_{\max}^{2/\beta})$ uniformly near $\alpha = 1$.
Indeed, the integrands have the common factor
$\sigma_{\max}^{2/\beta}$; their $\alpha$-derivatives add logarithms
of $N_t$, which remain integrable, and the terminal terms are
of lower order. Together with part~(i), these bounds imply
\[
D_i(\alpha_0^*)=D_i(1)+o(\sigma_{\max}^{2/\beta}),
\qquad
\Delta_i'(\alpha_0^*)=\Delta_i'(1)
+o(\sigma_{\max}^{1/\beta}).
\]
Since $\Delta_i(\alpha_0^*)=O(\sigma_{\max}^{-2})$, the remaining
terms in the numerator and denominator of
\eqref{eq:generic_alpha_star_expansion} are negligible on their
respective scales $\sigma_{\max}^{3/\beta}$ and
$\sigma_{\max}^{2/\beta}$. Consequently,
\begin{align*}
\frac{\mathcal{E}_1'(\alpha_0^*)}{\mathcal{E}_0''(\alpha_0^*)}
&=\frac{\sum_i\lambda_i^{-1}\Delta_i'(1)D_i(1)}
{\sum_i\lambda_i^{-1}\Delta_i'(1)^2}
+o(\sigma_{\max}^{1/\beta})\\
&=\frac{C^{(2)}(\beta)}{T C^{(1)}(\beta)}
\sigma_{\max}^{1/\beta}
\frac{\sum_i\lambda_i^{1-3/(2\beta)}}
{\sum_i\lambda_i^{1-1/\beta}}
+o(\sigma_{\max}^{1/\beta}).
\end{align*}
We obtain (ii) using 
\begin{align}
C_\beta
= \beta
\frac{\operatorname{B}\!\left(3-\frac{1}{\beta},\,1+\frac{1}{\beta}\right)}
     {\operatorname{B}\!\left(2-\frac{1}{2\beta},\,1+\frac{1}{2\beta}\right)} > 0.
\end{align}

\paragraph{Part (iii): the scale-mixture correction.}

By Proposition~\ref{prop:mean_cov_GSM}, the $\gamma^2$ coefficient, expanded in $\varepsilon$ is
\[
\mathcal{E}(\alpha,\varepsilon)
=\mathcal{E}_0(\alpha)+\varepsilon^2\nu\,G(\alpha)+o(\varepsilon^2),
\qquad
G(\alpha)\eqdef\frac12\sum_i
\frac{\omega_i(\alpha)\Delta_i(\alpha)}{\lambda_i},
\]
where $\omega_i$ is defined in~\eqref{eq:GSM_correction_integral}.
Under the
stated mixed regularity and domination assumptions,
differentiating the backward equations and the weak-error
formula justifies Taylor expansion in $\varepsilon$ after
one $\alpha$-derivative, and continuity after two. Thus
\[
\|\mathcal{E}(\cdot,\varepsilon)-\mathcal{E}_0-\varepsilon^2\nu G\|_{C^1(I)}
=o(\varepsilon^2),
\qquad
\|\mathcal{E}(\cdot,\varepsilon)-\mathcal{E}_0\|_{C^2(I)}\longrightarrow0.
\]
Moreover, the Gaussian minimizer $\alpha_0^*$ is nondegenerate by part~(i).
Lemma~\ref{lem:minimizer_stability}, applied to  $\mathcal{E}(\cdot,\varepsilon)$ with perturbation parameter
$\varepsilon^2$, therefore gives
\[
\alpha^\star(\varepsilon)
=\alpha_0^\star-\varepsilon^2\nu\,\Theta(\sigma_{\max})
+o(\varepsilon^2),
\qquad
\Theta(\sigma_{\max})
\eqdef\frac{G'(\alpha_0^\star)}{\mathcal{E}_0''(\alpha_0^\star)}.
\]
The $\varepsilon$ expansion is taken at fixed
$\sigma_{\max}$ before the large-noise limit. In the numerator,
\[
G'(\alpha_0^*)
=\frac12\sum_i\frac{
\omega_i'(\alpha_0^*)\Delta_i(\alpha_0^*)
+\omega_i(\alpha_0^*)\Delta_i'(\alpha_0^*)
}{\lambda_i}.
\]

We first show that $\omega_i(1)>0$. Recall from~\eqref{eq:GSM_correction_integral} that
\begin{align}
\omega_i(\alpha)
={}&-(1+\alpha)\int_0^T
\sigma_t\dot\sigma_t\,N_t(\lambda_i)^{2+\alpha}
\Bigg[
3N_t(\lambda_i)\mathcal G(\lambda_i,t)
+\sum_{j\ne i}
N_t(\lambda_j)\mathcal G(\lambda_j,t)
\Bigg]dt.
\label{eq:omega_i}
\end{align}
At $\alpha=1$, this becomes
\begin{align*}
\omega_i(1)
={}&-2\int_0^T
\sigma_t\dot\sigma_t\,N_t(\lambda_i)^{3}
\Bigg[
3N_t(\lambda_i)\mathcal G(\lambda_i,t)
+\sum_{j\ne i}
N_t(\lambda_j)\mathcal G(\lambda_j,t)
\Bigg]dt.
\end{align*}
where, in the formula of $\mathcal G$~\eqref{eq:Gj}, the integrand vanishes and only the boundary term remains:
\begin{align*}
\mathcal G(\lambda,t)
&= N_t(\lambda)^{-1}
\Big[N_s(\lambda)
B_s(\lambda)\Big]_{s=t}^{s=T}%
\end{align*}
and 
\begin{align*}
\omega_i(1)
={}&-2\int_0^T
\sigma_t\dot\sigma_t\,N_t(\lambda_i)^{3}
\Bigg[
3\Big[N_s(\lambda_i)
B_s(\lambda_i)\Big]_{s=t}^{s=T}
+\sum_{j\ne i}
\Big[N_s(\lambda_j)
B_s(\lambda_j)\Big]_{s=t}^{s=T}
\Bigg]dt.
\end{align*}
Since \(\eta_t=1\), we have
\(\lambda_j(t)=\lambda_j+\sigma_t^2\), and
\[
\Big[N_s(\lambda_j)B_s(\lambda_j)\Big]_{s=t}^{s=T}
=
\frac{\lambda_j\sigma_T\dot\sigma_T}{\lambda_j(T)^2}
-
\frac{\lambda_j\sigma_t\dot\sigma_t}{\lambda_j(t)^2}.
\]
The two terms of this difference give
\begin{align*}
\omega_i(1)
={}&
2\int_0^T(\sigma_t\dot\sigma_t)^2
\left(\frac{\lambda_i}{\lambda_i(t)}\right)^3
\left[
\frac{3\lambda_i}{\lambda_i(t)^2}
+\sum_{j\ne i}\frac{\lambda_j}{\lambda_j(t)^2}
\right]dt -
2\sigma_T\dot\sigma_T
\left[
\frac{3\lambda_i}{\lambda_i(T)^2}
+\sum_{j\ne i}\frac{\lambda_j}{\lambda_j(T)^2}
\right]
\int_0^T\sigma_t\dot\sigma_t
\left(\frac{\lambda_i}{\lambda_i(t)}\right)^3dt.
\end{align*}
To compare these terms on the same scale, the polynomial schedule gives
\[
\dot\sigma_t=\frac{\beta}{T}\sigma_{\max}^{1/\beta}
(\sigma_t^2)^{\frac12-\frac1{2\beta}},
\qquad
(\sigma_t\dot\sigma_t)^2dt
=\frac{\beta}{2T}\sigma_{\max}^{1/\beta}
u^{1-\frac1{2\beta}}\,du,
\quad u=\sigma_t^2.
\]
Hence the positive first term is exactly
\begin{align*}
&\frac{\beta}{T}\sigma_{\max}^{1/\beta}
\int_0^{\sigma_{\max}^2}
u^{1-\frac1{2\beta}}
\left(\frac{\lambda_i}{\lambda_i+u}\right)^3
\left[
\frac{3\lambda_i}{(\lambda_i+u)^2}
+\sum_{j\ne i}\frac{\lambda_j}{(\lambda_j+u)^2}
\right]du.
\end{align*}
The integrand is positive. It behaves as
$O(u^{1-1/(2\beta)})$ at zero and
$O(u^{-4-1/(2\beta)})$ at infinity, so the integral
converges to a strictly positive constant. For the other term,
the same substitution gives
\begin{align*}
\int_0^T\sigma_t\dot\sigma_t
\left(\frac{\lambda_i}{\lambda_i(t)}\right)^3dt
&=\frac{\lambda_i^3}{2}
\int_0^{\sigma_T^2}\frac{du}{(\lambda_i+u)^3}\\
&=\frac{\lambda_i}{4}
\left[
1-\left(\frac{\lambda_i}{\lambda_i(T)}\right)^2
\right]
\leq\frac{\lambda_i}{4}.
\end{align*}
Since \(\lambda_j(T)\geq\sigma_T^2\) and
$\dot\sigma_T=\beta\sigma_T/T$, its absolute value is at most
\begin{align*}
\frac{\dot\sigma_T}{2\sigma_T^3}\,
\lambda_i\bigl(\operatorname{tr}\Sigma_\data+2\lambda_i\bigr)
= \frac{\beta\,\lambda_i\bigl(\operatorname{tr}\Sigma_\data+2\lambda_i\bigr)}
{2T\sigma_{\max}^2}
\longrightarrow0 \quad \text{as } \sigma_{\max}\to\infty.
\end{align*}
The first positive term is of exact order
$\sigma_{\max}^{1/\beta}$, whereas this second boundary term is
$O(\sigma_{\max}^{-2})$. In particular,
$\omega_i(1)>0$ for all sufficiently large $\sigma_{\max}$. 

We now analyze the large $\sigma_{\max}$ behavior of $\omega_i$ and $\omega'_i$ for $\alpha $ close to $1$. Remind
\[
\mathcal G(\lambda,t)
=N_t(\lambda)^{-\alpha}
\left\{
\big[N_s(\lambda)^\alpha B_s(\lambda)\big]_{s=t}^{s=T}
+(1-\alpha^2)\int_t^T
N_s(\lambda)^\alpha B_s(\lambda)^2\,ds
\right\}.
\]
Changing to $u=\sigma_t^2$ we have the identities:
\begin{align*}
B_t(\lambda)
&=\frac{\beta}{T}\sigma_{\max}^{1/\beta}
\frac{u^{1-\frac1{2\beta}}}{\lambda+u},\\
\int_t^T N_s(\lambda)^\alpha B_s(\lambda)^2\,ds
&=\frac{\beta}{2T}\sigma_{\max}^{1/\beta}\lambda^\alpha
\int_u^{\sigma_{\max}^2}
\frac{v^{1-\frac1{2\beta}}}{(\lambda+v)^{\alpha+2}}\,dv.
\end{align*}
These two identities exhibit a common factor \(\sigma_{\max}^{1/\beta}\) in \(\omega_i(\alpha)\). Differentiating with respect to \(\alpha\) introduces logarithmic factors, which remain integrable under the same weights. The boundary terms in $\mathcal G(\lambda,t)$ are of lower order. Hence \(
\omega_i'(\alpha)=O(\sigma_{\max}^{1/\beta}) \) uniformly near $\alpha = 1$.

Coming back to the numerator of the implicit theorem
\[
G'(\alpha_0^*)
=\frac12\sum_i\frac{
\omega_i'(\alpha_0^*)\Delta_i(\alpha_0^*)
+\omega_i(\alpha_0^*)\Delta_i'(\alpha_0^*)
}{\lambda_i}.
\]
Using the above analysis and (i), the first product is
$O(\sigma_{\max}^{-2+1/\beta})
=o(\sigma_{\max}^{2/\beta})$; the second has exact order
$\sigma_{\max}^{2/\beta}$. Similarly,
$\mathcal{E}_0''=\frac12\sum_i(\Delta_i'^2+
\Delta_i\Delta_i'')/\lambda_i$ has leading order
$\sigma_{\max}^{2/\beta}$. More explicitly,
\begin{align*}
G'(\alpha_0^*)
={}&\sum_i\lambda_i
\left[\int_0^T\frac{(\sigma_t\dot\sigma_t)^2}
{\lambda_i(t)^3}\,dt\right]\omega_i(1)
+o(\sigma_{\max}^{2/\beta}) = \frac{\beta}{2T}\sigma_{\max}^{1/\beta}
\sum_i\lambda_i
\left[\int_0^{\sigma_{\max}^2}
\frac{u^{1-\frac1{2\beta}}}{(\lambda_i+u)^3}\,du\right]
\omega_i(1)
+o(\sigma_{\max}^{2/\beta}) \\
\mathcal{E}_0''(\alpha_0^*)
={}&2\sum_i\lambda_i^3
\left[\int_0^T\frac{(\sigma_t\dot\sigma_t)^2}
{\lambda_i(t)^3}\,dt\right]^2
+o(\sigma_{\max}^{2/\beta}) = \frac{\beta^2}{2T^2}\sigma_{\max}^{2/\beta}
\sum_i\lambda_i^3
\left[\int_0^{\sigma_{\max}^2}
\frac{u^{1-\frac1{2\beta}}}{(\lambda_i+u)^3}\,du\right]^2
+o(\sigma_{\max}^{2/\beta}).
\end{align*}
Both have the identical leading prefactor
$(\beta^2/(2T^2))\sigma_{\max}^{2/\beta}$,
which cancels in $\Theta(\sigma_{\max})$. Sending $\sigma_{\max}^2$ to infinity gives
\begin{align}
\label{eq:opt_alpha_GSM_limit}
\lim_{\sigma_{\max}\to\infty}\Theta(\sigma_{\max})
={}&
\frac{
\displaystyle\sum_i\lambda_i
\left[\int_0^\infty
\frac{u^{1-1/(2\beta)}}{(\lambda_i+u)^3}\,du\right]
\left[\int_0^\infty
u^{1-1/(2\beta)}
\left(\frac{\lambda_i}{\lambda_i+u}\right)^3
\left(
\frac{3\lambda_i}{(\lambda_i+u)^2}
+\sum_{j\ne i}\frac{\lambda_j}{(\lambda_j+u)^2}
\right)du\right]
}{
\displaystyle\sum_i\lambda_i^3
\left[\int_0^\infty
\frac{u^{1-1/(2\beta)}}{(\lambda_i+u)^3}\,du\right]^2
}>0.
\end{align}
Thus $\Theta(\sigma_{\max})$ is positive
for all sufficiently large $\sigma_{\max}$. This proves part~(iii).

Note that for \(\Sigma_\data=\lambda\id\), it simplifies to
\[
\lim_{\sigma_{\max}\to\infty}\Theta(\sigma_{\max})
=(d+2)
\frac{
\operatorname{B}\!\left(
2-\frac1{2\beta},\,3+\frac1{2\beta}
\right)
}{
\operatorname{B}\!\left(
2-\frac1{2\beta},\,1+\frac1{2\beta}
\right)
}
=
\frac{(d+2)
\left(1+\frac1{2\beta}\right)
\left(2+\frac1{2\beta}\right)}{12}.
\]

\end{proof}

\subsection{Proof of Proposition~\ref{prop:opt_eta_powerlaw}}
\label{app:opt_rescaling}
\begin{proof}
Throughout, $\alpha=0$, $\eta_t(c)=(c/(c+\sigma_t^2))^{1/2}$,
and $\sigma_t=\sigma_{\max}(t/T)^\beta$ with $\beta>1$.
All large $\sigma_{\max}$ estimates keep $T$, $\beta$, $d$, and the positive
eigenvalues fixed, and are locally uniform in $c>0$.

\paragraph{Part (i).}
We get from equation~\eqref{eq:cov_error_gauss}
\begin{align}
\Delta^{\Sigma,[1]}(\lambda;c)
={}&-\lambda(c-\lambda)^2\int_0^T
\frac{(\sigma_t\dot\sigma_t)^2}
{(\lambda+\sigma_t^2)^2(c+\sigma_t^2)^2}\,dt -\frac{\beta\lambda\sigma_{\max}^2(c-\lambda)}
{T(\lambda+\sigma_{\max}^2)(c+\sigma_{\max}^2)}.
\label{eq:opt_c_Delta_start}
\end{align}
The second boundary term and its first two $c$-derivatives are
$O(\sigma_{\max}^{-2})$. Moreover, the change of variable $u=\sigma_t^2$
in the integral term gives
\[
-\frac{\beta\sigma_{\max}^{1/\beta}}{2T}\,
\lambda(c-\lambda)^2\int_0^{\sigma_{\max}^2}
\frac{u^{1-\frac1{2\beta}}}{(\lambda+u)^2(c+u)^2}\,du.
\]
and thus taking $\sigma_{\max} \to \infty$:
\begin{align}
\Delta^{\Sigma,[1]}(\lambda;c)
={}&-\frac{\beta\sigma_{\max}^{1/\beta}}{2T}\,
\lambda^{1-\frac1{2\beta}}h_\beta\!\left(\frac c\lambda\right)
+O(\sigma_{\max}^{-2}),
\label{eq:opt_c_Delta_finite}
\end{align}
where
\begin{align}
h_\beta(z)
&\eqdef(z-1)^2\int_0^\infty
\frac{v^{1-\frac1{2\beta}}}{(1+v)^2(z+v)^2}\,dv
\label{eq:ha_integral}
\end{align}
Note that we get the closed-form $h_\beta$ 
\begin{align}
h_\beta(z)
&=\frac{\pi}{\sin(\pi/(2\beta))}
\left[\left(1-\frac1{2\beta}\right)
\left(1+z^{-\frac1{2\beta}}\right)
-2\frac{z^{1-\frac1{2\beta}}-1}{z-1}\right].
\label{eq:ha_closed}
\end{align}
which is extended by continuity at $z=1$.
It follows from the partial-fraction identity
\[
\frac{(z-1)^2}{(1+v)^2(z+v)^2}
=\frac1{(1+v)^2}+\frac1{(z+v)^2}
-\frac2{(1+v)(z+v)}
\]
and the evaluations
\begin{align*}
\int_0^\infty\frac{v^{1-\frac1{2\beta}}}{(1+v)^2}\,dv
&=\left(1-\frac1{2\beta}\right)
\frac{\pi}{\sin(\pi/(2\beta))},\\
\int_0^\infty\frac{v^{1-\frac1{2\beta}}}{(z+v)^2}\,dv
&=z^{-\frac1{2\beta}}\left(1-\frac1{2\beta}\right)
\frac{\pi}{\sin(\pi/(2\beta))},\\
\int_0^\infty\frac{v^{1-\frac1{2\beta}}}{(1+v)(z+v)}\,dv
&=\frac{\pi}{\sin(\pi/(2\beta))}
\frac{z^{1-\frac1{2\beta}}-1}{z-1}.
\end{align*}

For the power-law spectrum $\lambda_i=\lambda_{\max}i^{-p}$,
set $\tau=c/\lambda_{\max}$. The objective~\eqref{eq:gaussian_expansion_Fréchet} is
\begin{align}  \label{eq:Ec}
\mathcal E(c) = \sum_i
\frac{\Delta^{\Sigma,[1]}(\lambda_i;c)^2}{4\lambda_i}
\mathcal E(\lambda_{\max}\tau) =\mathcal{E}_0(c) + O(\sigma_{\max}^{-2+1/\beta})
\end{align}
where
\begin{equation}
\mathcal{E}_0(c) = \frac{\beta^2\sigma_{\max}^{2/\beta}}{4T^2}
\lambda_{\max}^{1-\frac1\beta}\Psi_{p,\beta}(\tau)
+O(\sigma_{\max}^{-2+1/\beta}), \quad 
\Psi_{p,\beta}(\tau)\eqdef
\sum_{i=1}^d i^{-p(1-1/\beta)}h_\beta(\tau i^p)^2, \quad \text{$h_\beta$ defined in~\eqref{eq:ha_closed}}.
\label{eq:Psi_p}
\end{equation}

Differentiating~\eqref{eq:ha_integral} gives
\begin{equation}
h_\beta'(z)=2(z-1)\int_0^\infty
\frac{v^{1-\frac1{2\beta}}}{(1+v)(z+v)^3}\,dv.
\label{eq:ha_derivative}
\end{equation}
Thus $h_\beta\geq0$ vanishes only at $1$, and $h_\beta h_\beta'$ is
negative below $1$ and positive above $1$. Since
\[
\Psi_{p,\beta}'(\tau)=2\sum_i i^{p/\beta}h_\beta(\tau i^p)h_\beta'(\tau i^p),
\]
this derivative is negative for $\tau\leq d^{-p}$ and positive
for $\tau\geq1$. Every minimizer lies in $(d^{-p},1)$.

To prove uniqueness of the minimizer, we start with
\begin{align*}
\frac{d}{d\tau}\left[\tau^{1+1/\beta}\Psi_{p,\beta}'(\tau)\right]
&=\tau^{1/\beta-1}\left[\tau^2\Psi_{p,\beta}''(\tau)
+\left(1+\frac1\beta\right)\tau\Psi_{p,\beta}'(\tau)\right].
\end{align*}
To prove positivity of the expression in brackets, we first study the sign of
$z^2h_\beta''(z)+(1+1/(2\beta))zh_\beta'(z)$.

Set $C=\pi/\sin(\pi/(2\beta))$.
The identity
\[
\frac{z^{1-\frac1{2\beta}}-1}{z-1}
=\left(1-\frac1{2\beta}\right)\int_0^1[1+s(z-1)]^{-\frac1{2\beta}}\,ds
\]
allows us to rewrite~\eqref{eq:ha_closed} as
\[
h_\beta(z)=C\left(1-\frac1{2\beta}\right)\left[
1+z^{-\frac1{2\beta}}-2\int_0^1[1+s(z-1)]^{-\frac1{2\beta}}\,ds
\right].
\]
Differentiating under the integral sign gives
\begin{align*}
h_\beta'(z)
&=\frac{C}{2\beta}\left(1-\frac1{2\beta}\right)\left[
-z^{-\frac1{2\beta}-1}
+2\int_0^1\frac{s}{[1+s(z-1)]^{1+\frac1{2\beta}}}\,ds
\right],\\
h_\beta''(z)
&=\frac{C}{2\beta}\left(1-\frac1{2\beta}\right)\left(1+\frac1{2\beta}\right)\left[
z^{-\frac1{2\beta}-2}
-2\int_0^1\frac{s^2}{[1+s(z-1)]^{2+\frac1{2\beta}}}\,ds
\right].
\end{align*}
In the combination $z^2h_\beta''(z)+(1+1/(2\beta))zh_\beta'(z)$,
the terms involving $z^{-1/(2\beta)}$ cancel. Combining the two
integrals and using $1+s(z-1)-zs=1-s$, we obtain
\[
z^2h_\beta''(z)+\left(1+\frac1{2\beta}\right)zh_\beta'(z)
=\frac{C}{\beta}\left(1-\frac1{2\beta}\right)\left(1+\frac1{2\beta}\right)\,z
\int_0^1\frac{s(1-s)}{[1+s(z-1)]^{2+\frac1{2\beta}}}\,ds>0.
\]
Here the coefficient and the integrand are positive for
$z>0$ and $0<s<1$.

Set $F(z)=h_\beta(z)^2$.
By the product rule,
\[
z^2F''(z)+\left(1+\frac1\beta\right)zF'(z)
=2z^2h_\beta'(z)^2
+2h_\beta(z)\left[
z^2h_\beta''(z)+\left(1+\frac1{2\beta}\right)zh_\beta'(z)
\right]
+\frac1\beta zh_\beta(z)h_\beta'(z).
\]
For $z>1$, both $h_\beta(z)$ and $h_\beta'(z)$ are positive.
The expression in brackets is positive by the identity above,
so the right-hand side is positive. To extend this positivity to $0<z<1$, set $v=u/z$
in the integral defining $h_\beta(1/z)$
in~\eqref{eq:ha_integral}. This gives
$h_\beta(1/z)=z^{1/(2\beta)}h_\beta(z)$, hence
$F(1/z)=z^{1/\beta}F(z)$. Differentiating and applying the chain rule,
\begin{align*}
F'(1/z)
&=-\frac1\beta z^{1+\frac1\beta}F(z)-z^{2+\frac1\beta}F'(z),\\
F''(1/z)
&=\frac1\beta\left(1+\frac1\beta\right)z^{2+\frac1\beta}F(z)
+2\left(1+\frac1\beta\right)z^{3+\frac1\beta}F'(z)+z^{4+\frac1\beta}F''(z).
\end{align*}
Combining these identities yields
\[
\frac1{z^2}F''(1/z)+\frac{1+1/\beta}{z}F'(1/z)
=z^{1/\beta}\left[z^2F''(z)+\left(1+\frac1\beta\right)zF'(z)\right].
\]
For $0<z<1$, the left-hand side is positive because $1/z>1$.
So the expression in brackets is also positive.
At $z=1$, it vanishes because $h_\beta(1)=h_\beta'(1)=0$
implies $F'(1)=F''(1)=0$.

Writing $z_i=\tau i^p$, we have
$\Psi_{p,\beta}(\tau)=\sum_{i=1}^d i^{-p(1-1/\beta)}F(z_i)$.
The chain rule therefore gives
\begin{align*}
&\tau^2\Psi_{p,\beta}''(\tau)
+\left(1+\frac1\beta\right)\tau\Psi_{p,\beta}'(\tau)\\
&\quad=\sum_{i=1}^d i^{-p(1-1/\beta)}
\left[z_i^2F''(z_i)+\left(1+\frac1\beta\right)z_iF'(z_i)\right]>0.
\end{align*}
Indeed, each term in brackets is nonnegative and vanishes
only when $z_i=1$. Since $p>0$, the arguments
$z_i=\tau i^p$ are distinct. For $d\geq2$, at least one
term is therefore strictly positive. Consequently,
\begin{align*}
\frac{d}{d\tau}\left[\tau^{1+1/\beta}\Psi_{p,\beta}'(\tau)\right]
&=\tau^{1/\beta-1}\left[\tau^2\Psi_{p,\beta}''(\tau)
+\left(1+\frac1\beta\right)\tau\Psi_{p,\beta}'(\tau)\right] > 0 
\end{align*}
We saw before that $\Psi_{p,\beta}'(d^{-p})<0$ and $\Psi_{p,\beta}'(1)>0$. This proves that
$\Psi_{p,\beta}$ has a unique global minimizer
$\tau^\star(p)\in(d^{-p},1)$, with
$\Psi_{p,\beta}''(\tau^\star(p))>0$.

So we proved that $\mathcal{E}_0(c)$ has a unique minimizer. 
With the same argument, we show that, for sufficiently large $\sigma_{\max}$, the exact objective $\mathcal E(c) = \mathcal{E}_0(c)  + O(\sigma_{\max}^{-2+1/\beta})$ also has a unique global minimizer. For sufficiently large $\sigma_{\max}$, \(\mathcal E'(c)<0\) for \(c\leq\lambda_{\max}d^{-p}/2\) and \(\mathcal E'(c)>0\) for \(c\geq\lambda_{\max}\). Thus every minimizer lies in a fixed compact interval.  On this compact interval, the expansion holds in $C^2$, with
\[
\left\|
\frac{\mathcal E-\mathcal E_0}{\sigma_{\max}^{2/\beta}}
\right\|_{C^2}
=O(\sigma_{\max}^{-2-1/\beta})\longrightarrow0.
\]
Since the normalized limiting objective
$\mathcal E_0/\sigma_{\max}^{2/\beta}$ has a unique nondegenerate
global minimizer $\lambda_{\max}\tau^\star(p)$,
Lemma~\ref{lem:minimizer_stability}, applied to
$\mathcal E/\sigma_{\max}^{2/\beta}$, implies that the exact
objective $\mathcal E$ also has a unique global minimizer
$c_\eta^\star$ for sufficiently large $\sigma_{\max}$.
The displacement bound in the same lemma gives
\[
c_\eta^\star
=\lambda_{\max}\tau^\star(p)
+O(\sigma_{\max}^{-2-1/\beta})
=\lambda_{\max}\tau^\star(p)+o(\sigma_{\max}^{-2}),
\]
with $\tau^\star(p)$ minimizer of
\begin{align}
\Psi_{p,\beta}(\tau)\eqdef
\sum_{i=1}^d i^{-p(1-1/\beta)}h_\beta(\tau i^p)^2, \quad %
h_\beta(z) =\frac{\pi}{\sin(\pi/(2\beta))}
\left[\left(1-\frac1{2\beta}\right)
\left(1+z^{-\frac1{2\beta}}\right)
-2\frac{z^{1-\frac1{2\beta}}-1}{z-1}\right].
\label{eq:formula_pstar}
\end{align}
\paragraph{Part (ii): the scale-mixture correction.}
\label{app:opt_c_GSM}
By Proposition~\ref{prop:mean_cov_GSM}, the $\gamma^2$ Fréchet
coefficient has the expansion
\[
\mathcal{E}(c,\varepsilon)
=\mathcal{E}_0(c)+\varepsilon^2\nu\,G(c)+o(\varepsilon^2),
\]
where
\[
\mathcal{E}_0(c)\eqdef\frac14\sum_i
\frac{\Delta^{\Sigma,[1]}(\lambda_i;c)^2}{\lambda_i}
=\frac14\mathcal E(c),
\qquad
G(c)\eqdef\frac12\sum_i
\frac{\omega_i(c)\Delta^{\Sigma,[1]}(\lambda_i;c)}{\lambda_i}.
\]
Here $\omega_i$ is defined in~\eqref{eq:GSM_correction_integral},
with $\alpha=0$ and $\eta_t(c)=(c/(c+\sigma_t^2))^{1/2}$.
As in Proposition~\ref{prop:opt_alpha}(iii), we first expand in
$\varepsilon$ at fixed $\sigma_{\max}$, then study the displacement
coefficient as $\sigma_{\max}\to\infty$.

Fix $\sigma_{\max}$ and a compact neighborhood
$I\subset(0,\infty)$ of $c_\eta^\star$.
To apply Lemma~\ref{lem:minimizer_stability}, we also need
control of the first two $c$-derivatives.
Here the dependence on $c$ enters through the scale
$\eta_t(c)$. On $I$, this scale, its inverse, and their
first two $c$-derivatives are uniformly bounded in time.
The density satisfies
\[
p_t^{\varepsilon,c}(x)
=\eta_t(c)^{-d}
\bar p_t^\varepsilon(x/\eta_t(c)),
\]
where $\bar p_t^\varepsilon$ is independent of $c$.
Under the stated mixed regularity and domination assumptions,
differentiating the backward equations and the weak-error
formula (see proof of Proposition~\ref{prop:mean_cov_GSM}) therefore justifies Taylor expansion in $\varepsilon$
after one $c$-derivative, and continuity after two. Thus
\[
\|\mathcal{E}(\cdot,\varepsilon)-\mathcal{E}_0-\varepsilon^2\nu G\|_{C^1(I)}
=o(\varepsilon^2),
\qquad
\|\mathcal{E}(\cdot,\varepsilon)-\mathcal{E}_0\|_{C^2(I)}\longrightarrow0.
\]
Since $\mathcal{E}_0=\mathcal E/4$, part~(i) gives
\[
\mathcal{E}_0''(c_\eta^\star)
=\frac{\beta^2}{16T^2}\sigma_{\max}^{2/\beta}
\lambda_{\max}^{-1-1/\beta}
\Psi_{p,\beta}''(\tau^\star(p))(1+o(1))>0
\]
for sufficiently large $\sigma_{\max}$.
Applying Lemma~\ref{lem:minimizer_stability} to the actual
objective $\mathcal{E}(\cdot,\varepsilon)$ therefore yields
\[
c_\eta^\star(\varepsilon)
=c_\eta^\star-\varepsilon^2\nu
\frac{G'(c_\eta^\star)}{\mathcal{E}_0''(c_\eta^\star)}
+o(\varepsilon^2).
\]
In the numerator,
\[
G'(c)=\frac12\sum_i\frac{
\omega_i'(c)\Delta^{\Sigma,[1]}(\lambda_i;c)
+\omega_i(c)\partial_c\Delta^{\Sigma,[1]}(\lambda_i;c)
}{\lambda_i}.
\]
Recall
\[
A_t=-\frac{\sigma_t\dot\sigma_t}{c+\sigma_t^2},
\qquad B_t(\lambda)=\frac{\sigma_t\dot\sigma_t}{\lambda+\sigma_t^2},
\qquad N_t(\lambda)=\frac{\lambda}{\lambda+\sigma_t^2}.
\]
At $\alpha=0$, equation~\eqref{eq:Gj} reduces to
\[
\mathcal G(\lambda,t)
=\big[A_s+B_s(\lambda)\big]_{s=t}^{s=T}
+\int_t^T\big(A_s+B_s(\lambda)\big)^2\,ds,
\]
with
\[
A_s+B_s(\lambda)
=\sigma_s\dot\sigma_s
\frac{c-\lambda}{(\lambda+\sigma_s^2)(c+\sigma_s^2)}.
\]
Setting $u=\sigma_t^2$ and $v=\sigma_s^2$ gives
\begin{align*}
-\mathcal G(\lambda,t)
={}&\frac{\beta\sigma_{\max}^{1/\beta}}{T}
\Bigg[\frac{u^{1-\frac1{2\beta}}(c-\lambda)}
{(\lambda+u)(c+u)}
-\frac{(c-\lambda)^2}{2}
\int_u^{\sigma_{\max}^2}
\frac{v^{1-\frac1{2\beta}}}{(\lambda+v)^2(c+v)^2}\,dv\Bigg]\\
&-\frac{\beta\sigma_{\max}^2(c-\lambda)}
{T(\lambda+\sigma_{\max}^2)(c+\sigma_{\max}^2)}.
\end{align*}
As in part~(i), the last term is $O(\sigma_{\max}^{-2})$ and
extending the integral to infinity changes the expression by
$O(\sigma_{\max}^{-2})$. Then, using,
\begin{align}
\psi_\beta(\lambda,u;c)
\eqdef{}&\frac{u^{1-\frac1{2\beta}}(c-\lambda)}
{2(\lambda+u)(c+u)}
-\frac{(c-\lambda)^2}{4}\int_u^\infty
\frac{v^{1-\frac1{2\beta}}}{(\lambda+v)^2(c+v)^2}\,dv.
\label{eq:psi_integral}
\end{align}
we have 
\[
\mathcal G(\lambda,t)
=-\frac{2\beta\sigma_{\max}^{1/\beta}}{T}
\psi_\beta(\lambda,\sigma_t^2;c)+O(\sigma_{\max}^{-2}).
\]
At $u=0$,
\[
\psi_\beta(\lambda,0;c)
=-\frac14\lambda^{-\frac1{2\beta}}h_\beta(c/\lambda).
\]
Since $\Delta^{\Sigma,[1]}(\lambda;c)=-\lambda\mathcal G(\lambda,0)$,
this recovers the Gaussian defect from part~(i):
\begin{equation}
\Delta^{\Sigma,[1]}(\lambda;c)
=-\frac{\beta\sigma_{\max}^{1/\beta}}{2T}
\lambda^{1-\frac1{2\beta}}h_\beta(c/\lambda)
+O(\sigma_{\max}^{-2}).
\label{eq:opt_c_GSM_gaussian_defect}
\end{equation}

Now, for the weight $\omega_i(c)$, at $\alpha=0$, equation~\eqref{eq:GSM_correction_integral} reads
\begin{align*}
\omega_i(c)
=-\int_0^T\sigma_t\dot\sigma_t\,N_t(\lambda_i)^2
\left[2N_t(\lambda_i)\mathcal G(\lambda_i,t)
+\sum_{j=1}^dN_t(\lambda_j)\mathcal G(\lambda_j,t)\right]dt.
\end{align*}
Substituting the preceding expansion and using
$\sigma_t\dot\sigma_t\,dt=du/2$ gives
\begin{align}
\omega_i(c)
={}&\frac{\beta\sigma_{\max}^{1/\beta}}{T}
\int_0^\infty\left(\frac{\lambda_i}{\lambda_i+u}\right)^2
\Bigg[2\frac{\lambda_i}{\lambda_i+u}\psi_\beta(\lambda_i,u;c)
+\sum_{j=1}^d\frac{\lambda_j}{\lambda_j+u}
\psi_\beta(\lambda_j,u;c)\Bigg]du
+O(\sigma_{\max}^{-2}).
\label{eq:opt_c_GSM_weight}
\end{align}
For $\lambda_i=\lambda_{\max}i^{-p}$ and $c=\lambda_{\max}\tau$,
the homogeneity
\[
\psi_\beta(\lambda_{\max}\lambda,\lambda_{\max}u;
\lambda_{\max}c)
=\lambda_{\max}^{-1/(2\beta)}\psi_\beta(\lambda,u;c)
\]
yields
\[
\omega_i(\lambda_{\max}\tau)
=\frac{\beta\sigma_{\max}^{1/\beta}}{T}
\lambda_{\max}^{1-\frac1{2\beta}}J_{i,\beta}(\tau)
+O(\sigma_{\max}^{-2}),
\]
where
\begin{align*}
J_{i,\beta}(\tau)\eqdef{}&\int_0^\infty
\left(\frac{i^{-p}}{i^{-p}+u}\right)^2
\Bigg[2\frac{i^{-p}}{i^{-p}+u}\psi_\beta(i^{-p},u;\tau) +\sum_{j=1}^d\frac{j^{-p}}{j^{-p}+u}
\psi_\beta(j^{-p},u;\tau)\Bigg]du.
\end{align*}
Substituting these expansions into the definitions of $\mathcal{E}_0$ and
$G$ gives
\begin{align*}
\mathcal{E}_0(c)
&=\frac{\beta^2\sigma_{\max}^{2/\beta}
\lambda_{\max}^{1-1/\beta}}{16T^2}\,
\Psi_{p,\beta}(\tau)+O(\sigma_{\max}^{-2+1/\beta}),\\
G(c)
&=-\frac{\beta^2\sigma_{\max}^{2/\beta}
\lambda_{\max}^{1-1/\beta}}{4T^2}\,
\Phi_{p,\beta}(\tau)+O(\sigma_{\max}^{-2+1/\beta}),
\end{align*}
where $\Psi_{p,\beta}$ is defined in \eqref{eq:Psi_p} and 
\begin{equation}
\Phi_{p,\beta}(\tau)\eqdef
\sum_{i=1}^d i^{p/(2\beta)}h_\beta(\tau i^p)J_{i,\beta}(\tau).
\label{eq:Phi_p}
\end{equation}
Since $c_\eta^\star/\lambda_{\max}\to\tau^\star(p)$, the chain rule gives
\[
-\frac{G'(c_\eta^\star)}{\mathcal{E}_0''(c_\eta^\star)}
=4\lambda_{\max}\,
\frac{\Phi_{p,\beta}'(\tau^\star(p))}
{\Psi_{p,\beta}''(\tau^\star(p))}
+o_{\sigma_{\max}\to\infty}(1).
\]
Thus, defining
\[
\delta(p)\eqdef
\frac{4\Phi_{p,\beta}'(\tau^\star(p))}
{\Psi_{p,\beta}''(\tau^\star(p))},
\]
the displacement formula becomes
\[
c_\eta^\star(\varepsilon)
=c_\eta^\star+\varepsilon^2\nu\lambda_{\max}
\big[\delta(p)+o_{\sigma_{\max}\to\infty}(1)\big]
+o(\varepsilon^2),
\]
\end{proof}

\paragraph{Evolution of $\tau^\star(p)$ with $p$.}
The first two derivatives of~\eqref{eq:Psi_p} are
\begin{align*}
\Psi_{p,\beta}'(\tau)
&=2\sum_{i=1}^d i^{p/\beta}h_\beta(\tau i^p)h_\beta'(\tau i^p),\\
\Psi_{p,\beta}''(\tau)
&=2\sum_{i=1}^d i^{p(1+1/\beta)}
\left[h_\beta'(\tau i^p)^2+h_\beta(\tau i^p)h_\beta''(\tau i^p)\right].
\end{align*}
Differentiating $\Psi_{p,\beta}'(\tau^\star(p))=0$ with respect to $p$ gives
\begin{align*}
\frac{d\tau^\star(p)}{dp}
={}&-\frac{2}{\Psi_{p,\beta}''(\tau^\star(p))}
\sum_{i=1}^d i^{p/\beta}(\log i)\Bigg[
\frac1\beta h_\beta(\tau^\star(p)i^p)h_\beta'(\tau^\star(p)i^p)\\
&\qquad+\tau^\star(p)i^p\left[
 h_\beta'(\tau^\star(p)i^p)^2
+h_\beta(\tau^\star(p)i^p)h_\beta''(\tau^\star(p)i^p)
\right]\Bigg].
\end{align*}
The denominator is positive by part (i), so the sign of $\frac{d\tau^\star(p)}{dp}$ is the opposite of the sign of the numerator, which is difficult to follow analytically.

Figure~\ref{fig:tau_p} plots $\tau^\star(p)$ and $\frac{d\tau^\star(p)}{dp}$ as functions of $p$ for $\beta=5$, the exponent used in all the $c_\eta$ experiments, and different values of $d$.

In the posterior sampling experiment of Figure~\ref{fig:opt_c_datasets}, the fitted power-law exponents are $p\approx 0.57$, $0.61$, and $0.83$ for the three blur levels (see the table in Figure~\ref{fig:opt_c_datasets}). Figure~\ref{fig:tau_p} shows that, for $d=12288$, $\beta=5$ and throughout this range of $p$, the minimizer $\tau^\star(p)$ decreases as $p$ increases, equivalently $\frac{d\tau^\star(p)}{dp}<0$. Since Proposition~\ref{prop:opt_eta_powerlaw} gives $c_\eta^\star=\lambda_{\max}\tau^\star(p)$, this implies that the normalized optimum $c_\eta^\star/\lambda_{\max}=\tau^\star(p)$ also decreases with $p$ and, for fixed $\lambda_{\max}$, so does $c_\eta^\star$. This is consistent with Figure~\ref{fig:opt_c_datasets} (bottom), where the minimizer of the curves plotted against $c_\eta/\lambda_{\max}$ shifts to the left as the posterior spectrum becomes steeper.

\begin{figure}[h]
    \centering
    \begin{subfigure}[t]{0.45\linewidth}
    \centering
    \includegraphics[width=\linewidth]{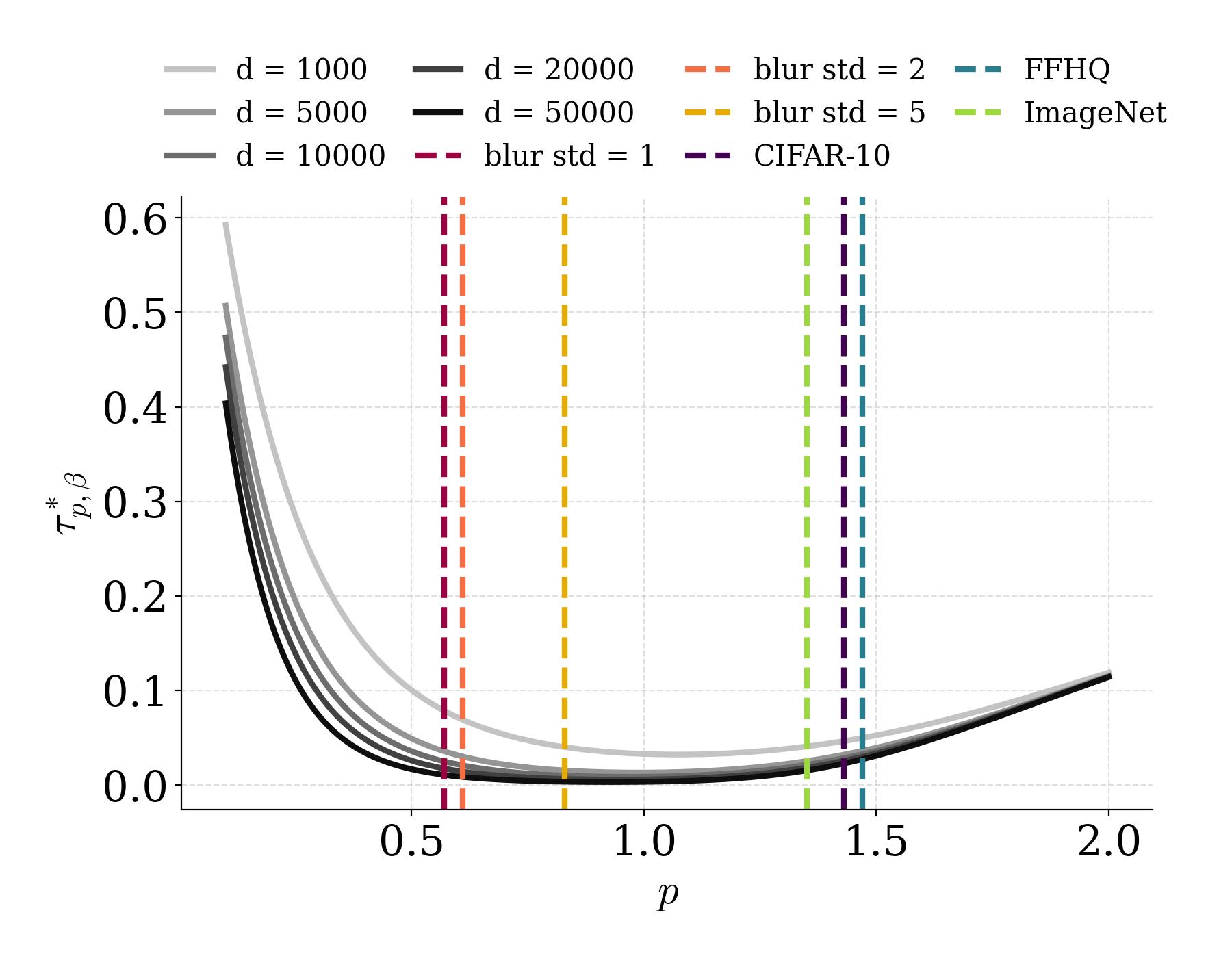}
     \end{subfigure}
     \hfill
    \begin{subfigure}[t]{0.45\linewidth}
    \centering
    \includegraphics[width=\linewidth]{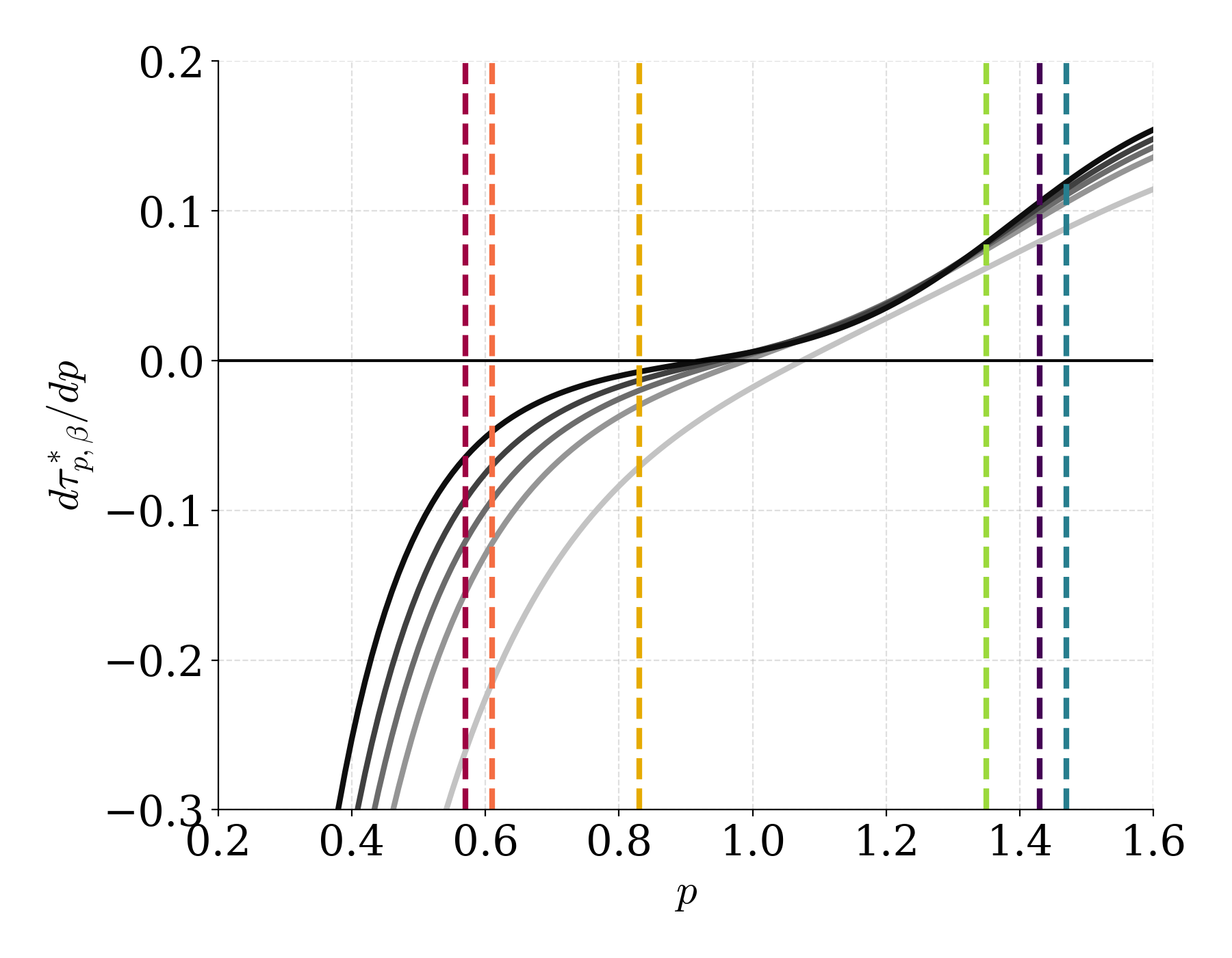}
     \end{subfigure}
    \caption{Evolution of the minimizer $\tau^\star(p)$ of $\Psi_{p,\beta}$~\eqref{eq:Psi_p} w.r.t.\ $p$ (left) and of its derivative $\frac{d\tau^\star(p)}{dp}$ (right), for $\beta=5$ as in the experiments of Figure~\ref{fig:opt_c_datasets}. The plot is done for different values of the dimension $d$ appearing in the sum. Dashed vertical lines mark the values of $p$ of the deblurring posteriors (blur std 1, 2, 5) and of CIFAR-10, FFHQ and ImageNet, with the colors of Figure~\ref{fig:opt_c_datasets} (for the posterior sampling experiment $d = 12288$).}
    \label{fig:tau_p}
\end{figure}

\paragraph{Sign of $\delta(p)$.}
$\delta(p)$ involves one-dimensional integrals only. We evaluate it numerically by quadrature on a logarithmic $u$ grid. Figure~\ref{fig:delta_p} plots $\delta(p)$ as a function of $p$ for $\beta=5$ and several dimensions $d$. For every $d$, $\delta(p)$ is positive for small $p$, changes sign at a threshold that slowly decreases with $d$ and is negative beyond. For $d=12288$, the deblurring posteriors with $p=0.57$ and $p=0.61$ are on the positive side while the one with $p=0.83$ and the full datasets are on the negative side. This is perfectly consistent with the optimal $c_\eta$ values shown in Figure~\ref{fig:opt_c_datasets}: the real FID-optimal $c_\eta$ lies below the Gaussian-FD optimum for full dataset samplings and for the  posterior sampling with blur std 5 ($p=0.83$). On the other side, the FID optimal $c_\eta$ lies above the Gaussian optimum for the two other posterior sampling experiments ($p=0.57$ and $p=0.61$).

\begin{figure}[h]
    \centering
    \includegraphics[width=0.4\linewidth]{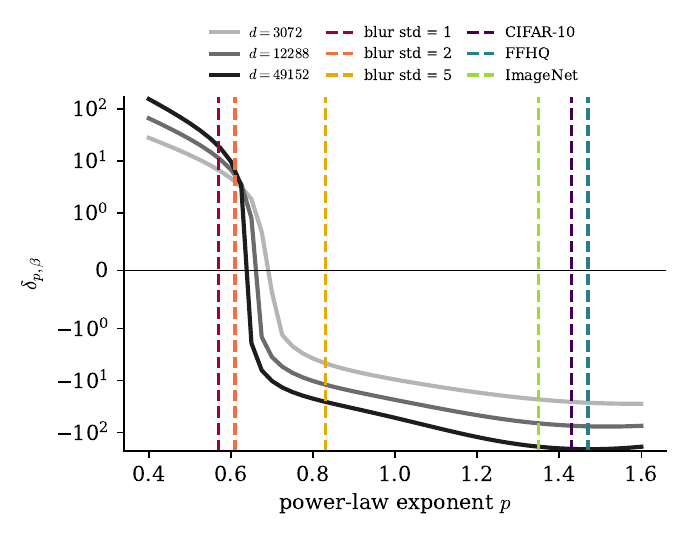}
    \caption{Scale-mixture shift $\delta(p)$ of the optimal rescaling (Proposition~\ref{prop:opt_eta_powerlaw}(ii)) as a function of the power-law exponent $p$, in the limit $\sigma_{\max}\to\infty$ with $\beta=5$ as in the experiments, for different dimensions $d$. Dashed vertical lines: $p$ values of the deblurring posteriors and of CIFAR-10, FFHQ and ImageNet, with the colors of Figure~\ref{fig:opt_c_datasets}. $\delta(p)>0$ predicts an optimal $c_\eta$ above the Gaussian one, $\delta(p)<0$ below it.}
    \label{fig:delta_p}
\end{figure}

\subsection{Optimal noise schedule $\sigma_t$}
\label{app:optimal_noise_schedule}

We optimize the noise schedule $\sigma_t$ for fixed diffusion parameter $\alpha$ and rescaling $\eta_t$. We start with the one-dimensional problem: for one eigendirection of variance $\lambda$, the following proposition gives the optimal choice of $\sigma_t^*$ that minimizes the first-order Fr\'echet error~\eqref{eq:gaussian_expansion_Fréchet} for VE (the VP analogue is given in Appendix~\ref{app:optimal_noise_schedule_VP}).

\begin{prop}[Single-eigendirection optimal noise schedule in the variance exploding case]
  \label{prop:opt_sigma_ve}
  Fix one eigendirection of variance $\lambda$, let $\eta_t=1$ (VE), and define $q_\alpha\eqdef1-2\alpha-\alpha^2$.
  The following schedule cancels the first-order one-dimensional Fréchet error:
  \begin{align} \label{eq:sigma_opt_1D}
    \sigma_t^*(\lambda)^2
    =\begin{cases}
      \lambda \left( 1+ \frac{\sigma_{\max}^2}{\lambda} \right)^{t/T}-\lambda,&q_\alpha=0,\\[0.4em]
      \lambda\left[1+\frac{t}{T}\left(\left( 1+ \frac{\sigma_{\max}^2}{\lambda} \right)^{q_\alpha/2}-1\right)\right]^{2/q_\alpha}-\lambda,
      &q_\alpha\neq0.
    \end{cases}
  \end{align}
\end{prop}

  \begin{proof}
    In the variance exploding case, the first-order mean error vanishes. Starting from~\eqref{eq:cov_error_gauss}, the first-order covariance error in one eigendirection of variance $\lambda$ becomes
    \begin{align}
      \Delta^{\Sigma,[1]}(\lambda)
      = -\lambda \int_0^{T}
      \left(\frac{\lambda}{\lambda+\sigma_s^2}\right)^\alpha
      \left[
      \frac{d}{ds}\left(\frac{\sigma_s \dot \sigma_s}{\lambda+\sigma_s^2}\right)
      +
      (1-2\alpha-\alpha^2)
      \left(\frac{\sigma_s \dot \sigma_s}{\lambda+\sigma_s^2}\right)^2
      \right] ds. \label{eq:ve_cov_error_full}
    \end{align}
    We look for a schedule that cancels the full integrand in~\eqref{eq:ve_cov_error_full}. Define
    \[
      B_t(\lambda) \eqdef \frac{\sigma_t \dot \sigma_t}{\lambda+\sigma_t^2},
      \qquad
      q_\alpha \eqdef 1-2\alpha-\alpha^2.
    \]
    Then~\eqref{eq:ve_cov_error_full} rewrites as
    \begin{align*}
      \Delta^{\Sigma,[1]}(\lambda)
      = -\lambda \int_0^{T} \left(\frac{\lambda}{\lambda+\sigma_s^2}\right)^\alpha \Big( \dot B_s(\lambda) + q_\alpha B_s(\lambda)^2 \Big)\, ds.
    \end{align*}
    Since $\left(\frac{\lambda}{\lambda+\sigma_t^2}\right)^\alpha>0$, cancelling the full integrand in~\eqref{eq:ve_cov_error_full} is equivalent to
    \begin{align} \label{eq:B_t_eq}
      \dot B_t(\lambda) + q_\alpha B_t(\lambda)^2 = 0.
    \end{align}
    If $q_\alpha=0$, then $B_t(\lambda)$ is constant. Since
    \[
      \frac{d}{dt}\log\!\left(1+\frac{\sigma_t^2}{\lambda}\right)
      = 2B_t(\lambda),
    \]
    the function $\log\!\left(1+\sigma_t^2/\lambda\right)$ is affine, and the boundary condition $\sigma_T = \sigma_{\max}$ gives
    \begin{align*}
      1+\frac{(\sigma_t^\star)^2}{\lambda}
      =
      \left(1+\frac{\sigma_{\max}^2}{\lambda}\right)^{t/T}.
    \end{align*}
    Hence
    \begin{align*}
      (\sigma_t^*)^2
      =
      \lambda\left[
      \left(1+\frac{\sigma_{\max}^2}{\lambda}\right)^{t/T}
      -1
      \right].
    \end{align*}
    Now, if $q_\alpha \neq 0$, define
    \[
      P_t(\lambda)
      \eqdef
      \left(1+\frac{\sigma_t^2}{\lambda}\right)^{q_\alpha/2}.
    \]
    Since $\dot P_t(\lambda)=q_\alpha B_t(\lambda)P_t(\lambda)$,  condition~\eqref{eq:B_t_eq} rewrites as
    \begin{align*}
      \ddot P_t(\lambda)=0,
    \end{align*}
    i.e. $P_t(\lambda)$ affine. Using the boundary conditions:
    \[
      P_0(\lambda)=1,
      \qquad
      P_T(\lambda)=\left(1+\frac{\sigma_{\max}^2}{\lambda}\right)^{q_\alpha/2},
    \]
    we get
    \begin{align*}
      P_t^\star(\lambda)
      &=
      1
      +
      \frac{t}{T}
      \left(
      \left(1+\frac{\sigma_{\max}^2}{\lambda}\right)^{q_\alpha/2}
      - 1
      \right).
    \end{align*}
    and in terms of $\sigma_t$:
    \begin{align} \label{eq:sigma_t_star_proof}
      (\sigma_t^*)^2
      &=
      \lambda\left[
      \left(
      1
      +
      \frac{t}{T}
      \left(
      \left(1+\frac{\sigma_{\max}^2}{\lambda}\right)^{q_\alpha/2}
      - 1
      \right)
      \right)^{2/q_\alpha}
      -1
      \right].
    \end{align}
  \end{proof}

  \begin{rem}
   Note that the above $\sigma_t^*$ is obtained by canceling the integrand from the first-order covariance error~\eqref{eq:cov_error_gauss} and not using the simplification obtained under $\sigma_{\max} \to \infty$ using Assumption~\ref{ass:regularity_sched_opt}. Indeed, as $t \to 0$, the above solution~\eqref{eq:sigma_t_star_proof} behaves as $\sigma_t = O(\sqrt{t})$ and thus does not satisfy the assumption $\xi_t = \dot \sigma_t \sigma_t \to 0$.
  \end{rem}

\subsubsection{Optimal noise schedule for VP}
\label{app:optimal_noise_schedule_VP}
\begin{prop}[Single-eigendirection optimal noise schedule for VP]
  \label{prop:opt_sigma_vp_ode}
  With the rescaling $\eta_t = \left(\frac{c}{c+\sigma_t^2}\right)^{1/2}$, $\mu_\data = 0$ and $\alpha = 0$, the following schedule cancels the first-order one-dimensional Fréchet error:
  \begin{align*}
    \sigma_t^*(\lambda)^2
    =
    \frac{c\lambda(1-z_t^2)}{\lambda z_t^2 - c},
    \qquad
    z_t
    \eqdef
    1
    +
    \frac{t}{T}
    \left(
    \sqrt{\frac{c(\lambda+\sigma_{\max}^2)}{\lambda(c+\sigma_{\max}^2)}}
    -1
    \right).
  \end{align*}
\end{prop}

  \begin{proof}
    Choose a constant $c>0$ and set
    \[
      \eta_t = \left(\frac{c}{c+\sigma_t^2}\right)^{1/2}.
    \] 
    For this rescaling family,
    \[
      A_t = -\frac{\sigma_t\dot\sigma_t}{c+\sigma_t^2},
      \qquad
      B_t(\lambda) = \frac{\sigma_t\dot\sigma_t}{\lambda+\sigma_t^2}.
    \]
    Proposition~\ref{prop:mean_cov_Gaussian} then gives
    \begin{align*}
      \Delta^{\Sigma,[1]}(\lambda)
      = -\lambda \int_0^T \Big( \dot A_s + \dot B_s(\lambda) + \big(A_s+B_s(\lambda)\big)^2 \Big)\, ds.
    \end{align*}
    Define
    \[
      z_t
      \eqdef
      \sqrt{\frac{c(\lambda+\sigma_t^2)}{\lambda(c+\sigma_t^2)}}.
    \]
    Then
    \[
      \frac{\dot z_t}{z_t} = A_t + B_t(\lambda).
    \]
    Therefore
    \begin{align*}
      \Delta^{\Sigma,[1]}(\lambda)
      = -\lambda \int_0^T \frac{\ddot z_s}{z_s}\, ds.
    \end{align*}
    Hence cancelling the full integrand is equivalent to
    \begin{align*}
      \ddot z_t = 0.
    \end{align*}
    Therefore $z_t$ is affine. Using $\sigma_0=0$ and $\sigma_T=\sigma_{\max}$ yields
    \begin{align*}
      z_t^\star
      =
      1
      +
      \frac{t}{T}
      \left(
      \sqrt{\frac{c(\lambda+\sigma_{\max}^2)}{\lambda(c+\sigma_{\max}^2)}}
      -1
      \right)
      = z_t.
    \end{align*}
    Since
    \[
      z_t^2 = \frac{c(\lambda+\sigma_t^2)}{\lambda(c+\sigma_t^2)},
    \]
    solving for $\sigma_t^2$ gives
    \begin{align*}
      (\sigma_t^*)^2
      =
      \frac{c\lambda(1-z_t^2)}{\lambda z_t^2 - c}.
    \end{align*}
    When $\sigma_{\max}=\infty$, this reduces to
    \begin{align*}
      z_t
      =
      1
      +
      \frac{t}{T}
      \left(
      \sqrt{\frac{c}{\lambda}}
      -1
      \right).
    \end{align*}
    Finally, when $\lambda=c$, one has $z_t\equiv 1$, hence $\Delta^{\Sigma,[1]}(c)=0$ for every schedule.
  \end{proof}

    \begin{rem}
   Note again that the above $\sigma_t^*$ is obtained by canceling the integrand from the first-order covariance error~\eqref{eq:cov_error_gauss} and not using the simplification obtained under $\sigma_{\max} \to \infty$ using Assumption~\ref{ass:regularity_sched_opt}. As for VE, as $t \to 0$, the above solution behaves as $\sigma_t = O(\sqrt{t})$ and thus does not satisfy the assumption $\xi_t = \dot \sigma_t \sigma_t \to 0$.
  \end{rem}

\subsubsection{Proof of Proposition~\ref{prop:sigma_improves}}
\label{app:sigma_improves}

\begin{proof}

\noindent\textbf{Part (i): improvement over polynomial schedules.}

For VE, the leading discretization covariance error is, using the notations of Proposition~\ref{prop:mean_cov_Gaussian}
\begin{equation}
\label{eq:error_VE}
\Delta^{\Sigma,[1]}(\lambda_i;c_\sigma,\beta)
=-(1-\alpha^2)\lambda_i\int_0^T
N_t(\lambda_i)^\alpha B_t(\lambda_i)^2\,dt
-\big[\lambda_iN_t(\lambda_i)^\alpha B_t(\lambda_i)\big]_0^T.
\end{equation}
with
\[
N_t(\lambda_i)=\frac{\lambda_i}{\lambda_i+\sigma_t^2},
\qquad
B_t(\lambda_i)=\frac{\sigma_t\dot\sigma_t}{\lambda_i+\sigma_t^2}.
\]
Fix $\beta>1$. The parameterization $\sigma_t(c_\sigma,\beta)$~\eqref{eq:sigma_opt_1D} gives, for
$c_\sigma>0$,
\[
\left(1+\frac{\sigma_t(c_\sigma,\beta)^2}{c_\sigma}\right)^{1/(2\beta)}
=1+\frac tT\left[
\left(1+\frac{\sigma_{\max}^2}{c_\sigma}\right)^{1/(2\beta)}-1
\right].
\]
Differentiating gives
\begin{align} \label{eq:sigdotsig}
\sigma_t\dot\sigma_t 
={}&\frac{\beta c_\sigma}{T}
\left[\left(1+\frac{\sigma_{\max}^2}{c_\sigma}\right)^{1/(2\beta)}-1\right]
\left(1+\frac{\sigma_t(c_\sigma,\beta)^2}{c_\sigma}\right)^{1-1/(2\beta)}.
\end{align}
Substituting and using the change of variable $\sigma = \sigma_t$ we get for the integral term from~\eqref{eq:error_VE}
\[
\begin{aligned}
\lambda_i\int_0^T N_t(\lambda_i)^\alpha B_t(\lambda_i)^2\,dt
={}& \lambda_i^{\alpha+1}
\int_0^T
\frac{(\sigma_t\dot\sigma_t)^2}
{(\lambda_i+\sigma_t^2)^{\alpha+2}}\,dt.\\
&= 
\frac{\beta c_\sigma^{1/(2\beta)}}{T}
\left[
\left(1+\frac{\sigma_{\max}^2}{c_\sigma}\right)^{1/(2\beta)}-1
\right]\lambda_i^{\alpha+1}
\int_0^{\sigma_{\max}}
\frac{\sigma(\sigma^2+c_\sigma)^{1-1/(2\beta)}}
{(\lambda_i+\sigma^2)^{\alpha+2}}\,d\sigma.
\end{aligned}
\]
For the second boundary term from~\eqref{eq:error_VE}, using
\[
\lambda_iN_t(\lambda_i)^\alpha B_t(\lambda_i)
=
\left(\frac{\lambda_i}{\lambda_i+\sigma_t^2}\right)^{\alpha+1}
\sigma_t\dot\sigma_t.
\]
we get
\[
\begin{aligned}
-\big[\lambda_iN_t(\lambda_i)^\alpha B_t(\lambda_i)\big]_0^T
={}&
\frac{\beta c_\sigma^{1/(2\beta)}}{T}
\left[
\left(1+\frac{\sigma_{\max}^2}{c_\sigma}\right)^{1/(2\beta)}-1
\right]\\
&\times\left[
c_\sigma^{1-1/(2\beta)}
-
\left(\frac{\lambda_i}{\lambda_i+\sigma_{\max}^2}\right)^{\alpha+1}
(c_\sigma+\sigma_{\max}^2)^{1-1/(2\beta)}
\right].
\end{aligned}
\]
Overall:
\begin{align}
\label{eq:Delta_gen}
&\Delta^{\Sigma,[1]}(\lambda_i;c_\sigma,\beta)
={}\underbrace{\frac{\beta c_\sigma^{1/(2\beta)}}{T}
\left[\left(1+\frac{\sigma_{\max}^2}{c_\sigma}\right)^{1/(2\beta)}-1\right]}_{(1)}\\
&\times \Bigg[
 \underbrace{c_\sigma^{1-1/(2\beta)}}_{(2)} \underbrace{-\left(\frac{\lambda_i}{\lambda_i+\sigma_{\max}^2}\right)^{\alpha+1}
\big(c_\sigma+\sigma_{\max}^2\big)^{1-1/(2\beta)}}_{(3)} \underbrace{-(1-\alpha^2)\lambda_i^{\alpha+1}
\int_0^{\sigma_{\max}}
\frac{\sigma(\sigma^2+c_\sigma)^{1-1/(2\beta)}}
{(\lambda_i+\sigma^2)^{\alpha+2}},d\sigma}_{(4)}
\Bigg] \nonumber .
\end{align}
As $c_\sigma\downarrow0$, it converges to its value for the polynomial schedule, giving overall $\Delta^{\Sigma,[1]}(\lambda_i;0,\beta)$.

We now expand as $c_\sigma\downarrow0$. 

The term (2) satisfies for $\beta>1$
\[
c_\sigma^{1-1/(2\beta)}=o(c_\sigma^{1/(2\beta)}).
\]
The term (3) minus its value at $c_\sigma=0$ is $O(c_\sigma)$. For the term (4), for $\sigma>0$,
\[
0\leq\partial_{c_\sigma}
(\sigma^2+c_\sigma)^{1-1/(2\beta)}
\leq\left(1-\frac1{2\beta}\right)\sigma^{-1/\beta}.
\]
Multiplying this bound by
$\sigma/(\lambda_i+\sigma^2)^{\alpha+2}$ gives an integrable
function on $[0,\sigma_{\max}]$. Thus, the term (4) minus its value at $c_\sigma=0$ is $O(c_\sigma)$.  

Hence the second line in~\eqref{eq:Delta_gen} differs from its value
at $0$ by $o(c_\sigma^{1/(2\beta)})$. 

The ratio of (1) to its polynomial limit is
\begin{align*}
\frac{c_\sigma^{1/(2\beta)}}{\sigma_{\max}^{1/\beta}}
\left[\left(1+\frac{\sigma_{\max}^2}{c_\sigma}\right)^{1/(2\beta)}-1\right]
&=\left(1+\frac{c_\sigma}{\sigma_{\max}^2}\right)^{1/(2\beta)}
-\left(\frac{c_\sigma}{\sigma_{\max}^2}\right)^{1/(2\beta)}\\
&=1-\left(\frac{c_\sigma}{\sigma_{\max}^2}\right)^{1/(2\beta)}
+o\!\left(\left(\frac{c_\sigma}{\sigma_{\max}^2}\right)^{1/(2\beta)}\right).
\end{align*}
Multiplying the two expansions yields
\[
\Delta^{\Sigma,[1]}(\lambda_i;c_\sigma,\beta)
=\Delta^{\Sigma,[1]}(\lambda_i;0,\beta)
\left[1-\left(\frac{c_\sigma}{\sigma_{\max}^2}\right)^{1/(2\beta)}\right]
+o\!\left(\left(\frac{c_\sigma}{\sigma_{\max}^2}\right)^{1/(2\beta)}\right).
\]
Squaring the expansion above and summing
over the fixed finite spectrum in~\eqref{eq:gaussian_expansion_Fréchet}  implies
$\mathcal E(c_\sigma,\beta)<\mathcal E(0,\beta)$ for every
sufficiently small $c_\sigma>0$.

\noindent\textbf{Part (ii): improvement over the optimal geometric schedule.}

We first study the Gaussian leading Fréchet error
$\mathcal E_{\mathrm{geo}}(c_\sigma)
\eqdef\mathcal E(\sigma_t^{\mathrm{geo}})$
as $d\to\infty$ and show that its minimizer
$c_{\sigma}^\star$ tends to zero.
We then expand the error for our schedule $\mathcal E( c = c_\sigma^*,\beta)$ as $\beta\to\infty$. We then show that the coefficient of $1/\beta$ is negative for sufficiently large $d$. Thus, a sufficiently large finite $\beta$ strictly
improves upon the optimal geometric schedule.

We then first consider the geometrical schedule. For $c_\sigma>0$, let $H=\log(1+\sigma_{\max}^2/c_\sigma)$ such that 
$$
(\sigma_t^{geo})^2=c_\sigma(e^{H t/T}-1).
$$
Let's first get the formula of $\mathcal E_{\mathrm{geo}}(c_\sigma)$.

With $u=(\sigma_t^{geo})^2$, we have
\[
\frac{du}{dt}=\frac HT(u+c_\sigma),
\qquad
\sigma_t\dot\sigma_t=\frac{H}{2T}(u+c_\sigma).
\]
Thus the integral contribution in~\eqref{eq:error_VE} becomes
\[
\lambda_i\int_0^T N_t(\lambda_i)^\alpha B_t(\lambda_i)^2\,dt
=\frac{H}{4T}\lambda_i^{1+\alpha}
\int_0^{\sigma_{\max}^2}
\frac{u+c_\sigma}{(\lambda_i+u)^{\alpha+2}}\,du.
\]
The boundary contribution is
\[
-\big[\lambda_iN_t(\lambda_i)^\alpha B_t(\lambda_i)\big]_0^T
=\frac{H}{2T}\left[
c_\sigma-(\sigma_{\max}^2+c_\sigma)
\left(\frac{\lambda_i}{\lambda_i+\sigma_{\max}^2}\right)^{1+\alpha}
\right].
\]
Combining these contributions gives
\begin{align*}
\Delta_i^{\mathrm{geo}}(c_\sigma)
=\frac{H}{2T}\Bigg[
&c_\sigma-(\sigma_{\max}^2+c_\sigma)
\left(\frac{\lambda_i}{\lambda_i+\sigma_{\max}^2}\right)^{1+\alpha}-\frac{1-\alpha^2}{2}\lambda_i^{1+\alpha}
\int_0^{\sigma_{\max}^2}
\frac{u+c_\sigma}{(\lambda_i+u)^{\alpha+2}}\,du
\Bigg].
\end{align*}
Using
\[
\lambda_i^{1+\alpha}\int_0^{\sigma_{\max}^2}
\frac{du}{(\lambda_i+u)^{\alpha+2}}
=\frac{1}{1+\alpha}\left[
1-\left(\frac{\lambda_i}{\lambda_i+\sigma_{\max}^2}\right)^{1+\alpha}
\right],
\]
we get
\[
\Delta_i^{\mathrm{geo}}(c_\sigma)
=\frac{H}{2T}(q_i c_\sigma-e_i)
\]
and thus 
\begin{align} \label{eq:Egeoc}
\mathcal E_{\mathrm{geo}}(c_\sigma)
=\frac{H^2}{16T^2}
\sum_{i=1}^d\frac{(q_i c_\sigma-e_i)^2}{\lambda_i} = \frac{H^2}{16T^2} \left[ c_\sigma^2 \sum_{i=1}^d\frac{q_i^2}{\lambda_i}  - 2c_\sigma \sum_{i=1}^d\frac{q_ie_i}{\lambda_i} + \sum_{i=1}^d\frac{e_i^2}{\lambda_i} \right]
\end{align}
where
\begin{align}
\label{eq:qi}
q_i&=\frac{1+\alpha}{2}\left[
1-\left(\frac{\lambda_i}{\lambda_i+\sigma_{\max}^2}\right)^{1+\alpha}
\right],\\
e_i&=\sigma_{\max}^2
\left(\frac{\lambda_i}{\lambda_i+\sigma_{\max}^2}\right)^{1+\alpha}
+\frac{1-\alpha^2}{2}\lambda_i^{1+\alpha}
\int_0^{\sigma_{\max}^2}
\frac{u}{(\lambda_i+u)^{\alpha+2}}\,du.
\label{eq:ei}
\end{align}
Both are positive. 

We now look at the behavior of $\mathcal E_{\mathrm{geo}}(c_\sigma)$ for large dimension $d$, with the power law spectrum $\lambda_i=\lambda_{\max}i^{-p}$ with $p>1$.  Using $q_i\to(1+\alpha)/2$ as $i \to \infty$, we have
\begin{align} \label{eq:sumqili}
\sum_{i=1}^d\frac{q_i^2}{\lambda_i} \sim\frac{(1+\alpha)^2}{4\lambda_{\max}(p+1)}d^{p+1} 
\end{align}
Moreover  $e_i\leq C\lambda_i[1+\log(1+\frac{\sigma_{\max}^2}{\lambda_i})]$, thus
$$
0<\sum_{i=1}^{\infty}\frac{e_i^2}{\lambda_i}<\infty.
$$
Finally, for \(\lambda_i=\lambda_{\max}i^{-p}\),
\[
\log\left(1+\frac{\sigma_{\max}^2}{\lambda_{i}}\right)
\leq \log\left(1+\frac{\sigma_{\max}^2}{\lambda_{\max}}\right)
+p\log i.
\]
Since \(0<q_i\leq1\), it follows that
\[
\sum_{i=1}^d\frac{q_i e_i}{\lambda_i}
\leq C\sum_{i=1}^d(1+\log i)
=O(d\log d).
\]
We locate the minimizer $c_\sigma^*$ of the error $\mathcal E_{\mathrm{geo}}$. From the above scalings, taking $c_\sigma=\lambda_{\max}d^{-(p+1)/2}$ in~\eqref{eq:Egeoc} gives
$\mathcal E_{\mathrm{geo}}(c)=O((\log d)^2)$,
whereas for $c_\sigma$ bounded away from zero, the error is $\sim d^{p+1}$. Also, the error diverges as $c_\sigma \downarrow0$ for each fixed $d$. Thus, since $(\log d)^2 = o(d^{p+1})$, for
sufficiently large $d$, a minimizer $c_{\sigma}^\star$
is attained in $(0,\infty)$. Moreover, $c_{\sigma}^\star\to0$ as $d\to\infty$.
Its stationarity equation is from~\eqref{eq:Egeoc} \[
0=2HH'Q+H^2Q'.
\]
Since
\[
Q'(c_\sigma)=2\sum_{i=1}^d
\frac{q_i(q_i c_\sigma-e_i)}{\lambda_i},
\qquad
H'(c_\sigma)
=-\frac{1}{c_\sigma(1+c_\sigma/\sigma_{\max}^2)},
\]
it writes:
\begin{equation}
\label{eq:geo_stationarity}
\left(1+\frac{c_\sigma^*}{\sigma_{\max}^2}\right)H(c_\sigma^*) c_\sigma^*
\sum_{i=1}^d\frac{q_i(q_i c^*_\sigma-e_i)}{\lambda_i}
=\sum_{i=1}^d\frac{(q_i c^*_\sigma-e_i)^2}{\lambda_i}.
\end{equation}
The right-hand side is positive, thus 
$$c_\sigma^* \geq \frac{\sum_i \frac{q_i e_i}{\lambda_i}}{\sum_i \frac{q_i^2}{\lambda_i}} \geq C d^{-(p+1)}, \quad C>0$$
which implies $H(c_\sigma^*)=O(\log d)$.
Rearranging~\eqref{eq:geo_stationarity}
gives
\begin{align*}
&\left(1+\frac{c_\sigma}{\sigma_{\max}^2}-\frac1H\right)
Hc_\sigma^2\sum_i\frac{q_i^2}{\lambda_i}=\sum_i\frac{e_i^2}{\lambda_i}
+\left(1+\frac{c_\sigma}{\sigma_{\max}^2}-\frac2H\right)
Hc_\sigma\sum_i\frac{q_i e_i}{\lambda_i}.
\end{align*}
Both prefactors tend to $1$. 
By Young's inequality,
\[
Hc_\sigma\sum_i\frac{q_i e_i}{\lambda_i}
\leq\frac14Hc_\sigma^2\sum_i\frac{q_i^2}{\lambda_i}
+H\frac{(\sum_i q_i e_i/\lambda_i)^2}
{\sum_i q_i^2/\lambda_i}.
\]
The preceding bounds give
\[
H\frac{(\sum_i q_i e_i/\lambda_i)^2}
{(\sum_i q_i^2/\lambda_i)(\sum_i e_i^2/\lambda_i)}
=O\bigl(d^{1-p}(\log d)^3\bigr)=o(1).
\]
thus the last term in Young's inequality is $o(\sum_i e_i^2/\lambda_i)$.
Substituting this bound into the rearranged stationarity
equation therefore gives
\[
Hc_\sigma^2\sum_i\frac{q_i^2}{\lambda_i}
=O\left(\sum_i\frac{e_i^2}{\lambda_i}\right).
\]
Since $H\to\infty$, both the $q_i^2$ and mixed terms in~\eqref{eq:Egeoc} are negligible compared with
$\sum_i e_i^2/\lambda_i$, giving
\[
\mathcal E_{\mathrm{geo}}(c_\sigma^*)
\sim\frac{H^2}{16T^2}\sum_i\frac{e_i^2}{\lambda_i} = O(H^2)
\]

We now expand around the geometric schedule, keeping $d$ and
$c_\sigma = c_{\sigma}^\star$ fixed as $\beta\to\infty$.
Expanding equation~\eqref{eq:sigdotsig} for large $\beta$ gives
\[
\sigma_t\dot\sigma_t
=\frac{H(c_\sigma+u)}{2T}
\left[1+\frac{H/2-\log(1+u/c_\sigma)}{2\beta}
+O(\beta^{-2})\right].
\]
Substitution into~\eqref{eq:error_VE}:
\[
\begin{aligned}
\Delta^{\Sigma,[1]}(\lambda_i;c_\sigma,\beta)
={}&\left(1-\frac{H}{4\beta}\right)
\Delta_i^{\mathrm{geo}}(c_\sigma)\\
&+\frac{H}{4T\beta}\Bigg\{c_\sigma H
-\frac{1-\alpha^2}{2}\lambda_i^{1+\alpha}
\int_0^{\sigma_{\max}^2}\log\left(\frac{c_\sigma+\sigma_{\max}^2}{c_\sigma+u}\right)
\frac{c_\sigma+u}{(\lambda_i+u)^{\alpha+2}}\,du\Bigg\}
+O(\beta^{-2}).
\end{aligned}
\]
For the Fréchet error, recall that
\[
\mathcal E(c_\sigma,\beta)
=\frac14\sum_{i=1}^d
\frac{\Delta^{\Sigma,[1]}(\lambda_i;c_\sigma,\beta)^2}{\lambda_i}.
\]
we get an expansion of the form
\begin{equation}
\label{eq:geo_beta_expansion}
\mathcal E(c_\sigma,\beta)
=\mathcal E_{\mathrm{geo}}(c_\sigma)+\frac{A_d(c_\sigma)}{\beta}
+O(\beta^{-2}).
\end{equation}
Collecting the coefficient of $1/\beta$ gives
\begin{align*}
A_d(c_\sigma)
={}&-\frac H2\mathcal E_{\mathrm{geo}}(c_\sigma)\\
&+\frac{H^2}{16T^2}\sum_{i=1}^d
\frac{q_i c_\sigma-e_i}{\lambda_i}
\Bigg[
c_\sigma H
-\frac{1-\alpha^2}{2}\lambda_i^{1+\alpha}
\int_0^{\sigma_{\max}^2}
\log\left(\frac{c_\sigma+\sigma_{\max}^2}{c_\sigma+u}\right)
\frac{c_\sigma+u}{(\lambda_i+u)^{\alpha+2}}\,du
\Bigg].
\end{align*}

Since $\alpha\geq0$ and
$\log(\frac{c_\sigma+U}{c_\sigma+u})\leq\log(\frac{U}{u})$,
the integral term above is nonnegative and bounded by
\[
\begin{aligned}
&\lambda_i^{1+\alpha}\int_0^{\sigma_{\max}^2}
\log\left(\frac{c_\sigma+\sigma_{\max}^2}{c_\sigma+u}\right)
\frac{c_\sigma+u}{(\lambda_i+u)^{\alpha+2}}\,du \leq c_\sigma H+C\lambda_i
\left[1+\log\left(1+\frac{\sigma_{\max}^2}{\lambda_i}\right)\right]^2.
\end{aligned}
\]
Indeed, split $(c_\sigma + u)$ in two terms. For the $c_\sigma$ term:
\[
\log\left(\frac{c_\sigma+\sigma_{\max}^2}{c_\sigma+u}\right)
\leq
\log\left(\frac{c_\sigma+\sigma_{\max}^2}{c_\sigma}\right)
=\log\left(1+\frac{\sigma_{\max}^2}{c_\sigma}\right)
=H.
\]
For the \(u\) term, use
\[
\log\frac{c_\sigma+\sigma_{\max}^2}{c_\sigma+u}
\leq\log\frac{\sigma_{\max}^2}{u},
\qquad
\frac{\lambda_i^{1+\alpha}}{(\lambda_i+u)^{\alpha+2}}
\leq\frac{\lambda_i}{(\lambda_i+u)^2},
\]
Thus this term is bounded by
\[
\lambda_i\int_0^{\sigma_{\max}^2}
\frac{u}{(\lambda_i+u)^2}\log\frac{\sigma_{\max}^2}{u}\,du.
\]
If \(\sigma_{\max}^2\geq\lambda_i\), splitting at \(u=\lambda_i\) gives
\[
\begin{aligned}
&\leq \frac1{\lambda_i}\int_0^{\lambda_i}
u\log\frac{\sigma_{\max}^2}{u}\,du
+\lambda_i\int_{\lambda_i}^{\sigma_{\max}^2}
\frac1u\log\frac{\sigma_{\max}^2}{u}\,du\\
&=\lambda_i\left[
\frac14+\frac12\log\frac{\sigma_{\max}^2}{\lambda_i}
+\frac12\log^2\frac{\sigma_{\max}^2}{\lambda_i}
\right].
\end{aligned}
\]
If \(\sigma_{\max}^2<\lambda_i\), it is at most \(\sigma_{\max}^4/(4\lambda_i)\leq\lambda_i/4\).
Both cases are bounded by
\[
C\lambda_i\left[1+\log\left(1+\frac{\sigma_{\max}^2}{\lambda_i}\right)\right]^2,
\]
which proves the stated inequality.

The preceding integral bound therefore implies
\begin{align*}
&\left|A_d(c_\sigma)
+\frac H2\mathcal E_{\mathrm{geo}}(c_\sigma)\right|\leq CH^2\sum_{i=1}^d
\frac{|q_i c_\sigma-e_i|}{\lambda_i}
\left[
c_\sigma H+\lambda_i
\left(1+\log\left(1+\frac{\sigma_{\max}^2}{\lambda_i}\right)\right)^2
\right].
\end{align*}
We now prove that the sum term is uniformly bounded when $d\to \infty$.
Since $q_i\leq1$,  $|q_i c_\sigma-e_i|\leq c_\sigma+e_i$
\begin{align*}
&\left|A_d(c_\sigma)
+\frac H2\mathcal E_{\mathrm{geo}}(c_\sigma)\right|\\
&\quad\leq CH^2\Bigg[
c_\sigma^2H\sum_{i=1}^d\frac1{\lambda_i}
+c_\sigma H\sum_{i=1}^d\frac{e_i}{\lambda_i}\\
&\qquad\qquad
+c_\sigma\sum_{i=1}^d
\left(1+\log\left(1+\frac{\sigma_{\max}^2}{\lambda_i}\right)\right)^2
+\sum_{i=1}^d e_i
\left(1+\log\left(1+\frac{\sigma_{\max}^2}{\lambda_i}\right)\right)^2
\Bigg].
\end{align*}
Using $\sum_i e_i/\lambda_i=O(d\log d)$ and $H=O(\log d)$
to absorb the mixed term
$c_\sigma H\sum_i e_i/\lambda_i$
into $O(c_\sigma d(\log d)^2)$.
\begin{align*}
&\left|A_d(c_\sigma)
+\frac H2\mathcal E_{\mathrm{geo}}(c_\sigma)\right|\leq CH^2\Bigg[
c_\sigma^2H\sum_{i=1}^d\frac1{\lambda_i}
+c_\sigma d(\log d)^2
+\sum_{i=1}^d e_i
\left(1+\log\left(1+\frac{\sigma_{\max}^2}{\lambda_i}\right)\right)^2
\Bigg].
\end{align*}

Finally, from its definition~\eqref{eq:ei}
\[
e_i\leq C\lambda_i
\left[1+\log\left(1+\frac{\sigma_{\max}^2}{\lambda_i}\right)\right],
\]
Indeed, since \(\alpha\in[0,1]\), the first term satisfies
\[
\sigma_{\max}^2
\left(\frac{\lambda_i}{\lambda_i+\sigma_{\max}^2}\right)^{1+\alpha}
\leq
\frac{\sigma_{\max}^2\lambda_i}
{\lambda_i+\sigma_{\max}^2}
\leq\lambda_i
\]
and the integral term
\[
\lambda_i^{1+\alpha}
\frac{u}{(\lambda_i+u)^{\alpha+2}}
=
\lambda_i\left(\frac{\lambda_i}{\lambda_i+u}\right)^\alpha
\frac{u}{(\lambda_i+u)^2}
\leq\frac{\lambda_i}{\lambda_i+u}.
\]
Therefore,
\[
e_i\leq\lambda_i+
\frac{\lambda_i}{2}\int_0^{\sigma_{\max}^2}\frac{du}{\lambda_i+u}
=\lambda_i\left[
1+\frac12\log\left(1+\frac{\sigma_{\max}^2}{\lambda_i}\right)
\right].
\]
We get
\begin{align*}
&\left|A_d(c_\sigma)
+\frac H2\mathcal E_{\mathrm{geo}}(c_\sigma)\right| \leq CH^2\Bigg[
c_\sigma^2H\sum_{i=1}^d\frac1{\lambda_i}
+c_\sigma d(\log d)^2
+\sum_{i=1}^d\lambda_i
\left(1+\log\left(1+\frac{\sigma_{\max}^2}{\lambda_i}\right)\right)^3
\Bigg].
\end{align*}
We now show that the entire bracket is bounded independently of $d$.
Since $q_i\geq q_1>0$
\[
c_\sigma^2H\sum_{i=1}^d\frac1{\lambda_i}
\leq\frac1{q_1^2}
c_\sigma^2H\sum_{i=1}^d\frac{q_i^2}{\lambda_i}
\]
We got before that
\[
Hc_\sigma^2\sum_i\frac{q_i^2}{\lambda_i}
=O\left(\sum_i\frac{e_i^2}{\lambda_i}\right).
\]
The above bound on $e_i$ and the power-law spectrum imply
\[
0<\frac{e_1^2}{\lambda_{\max}}
\leq\sum_{i=1}^d\frac{e_i^2}{\lambda_i}
\leq C\sum_{i=1}^\infty i^{-p}(1+\log i)^2<\infty,
\]
since $p>1$. Thus the first term is $O(1)$.

Moreover, from~\eqref{eq:sumqili}, $\sum_i q_i^2/\lambda_i =O(d^{p+1})$, so
\[
c_\sigma=O\left(d^{-(p+1)/2}H^{-1/2}\right).
\]
Consequently,
\[
c_\sigma d(\log d)^2
=O\left(d^{-(p-1)/2}H^{-1/2}(\log d)^2\right)=o(1),
\]
because $p>1$ and $H\to\infty$.
Finally,
\[
\sum_{i=1}^d\lambda_i
\left(1+\log\left(1+\frac{\sigma_{\max}^2}{\lambda_i}\right)\right)^3
\leq C\sum_{i=1}^\infty i^{-p}(1+\log i)^3<\infty.
\]
The entire bracket is therefore bounded independently of $d$,
and we obtain
\[
\left|A_d(c_{\sigma}^\star)
+\frac H2\mathcal E_{\mathrm{geo}}(c_{\sigma}^\star)\right|
=O(H^2).
\]

On the other hand, we saw previously that 
\[
\mathcal E_{\mathrm{geo}}(c_{\sigma}^\star)
\sim\frac{H^2}{16T^2}\sum_{i=1}^d\frac{e_i^2}{\lambda_i}
= O(H^2).
\]
Therefore
\[
\frac{A_d(c_{\sigma}^\star)}
{\mathcal E_{\mathrm{geo}}(c_{\sigma}^\star)}
=-\frac H2+O(1).
\]
Moreover $H\to\infty$ and thus $A_d(c_{\sigma}^\star)<0$ for all sufficiently large $d$.

Fix any such $d$. In~\eqref{eq:geo_beta_expansion},
the remainder is $O(\beta^{-2})$, whereas the correction
$A_d(c_{\sigma,d}^\star)/\beta$ is strictly negative
and of order $\beta^{-1}$. Thus, for all sufficiently
large finite $\beta$,
\[
\mathcal E(c_{\sigma}^\star,\beta)
<
\mathcal E_{\mathrm{geo}}(c_{\sigma}^\star)
=\min_{c_\sigma>0}\mathcal E_{\mathrm{geo}}(c_\sigma).
\]
\end{proof}

\section{Experiments}

\subsection{Experimental setup}
\label{app:details_expe}

For empirical image sampling, we use the DeepInverse library~\citep{tachella2025deepinverse} and
the diffusion models from~\citep{karras2022elucidating}, trained on CIFAR-10, FFHQ, and ImageNet. The CIFAR-10 model generates $32 \times 32$ images, whereas the FFHQ and ImageNet models generate $64 \times 64$ images. The CIFAR-10 and FFHQ models use the NCSN architecture from~\citep{song2019generative}, while the ImageNet model uses the ADM architecture~\citep{dhariwal2021diffusion}.

Each model was trained with a specific rescaling schedule $\eta_t$---VE for CIFAR-10 and FFHQ, and VP for ImageNet. However, at each model evaluation, we re-normalize the input and output to match our target rescaling $\eta_t$. We assume that the choice of $\sigma_t$ used during training does not affect the quality of the score estimate. We use $\sigma_{\max} = 80$ during sampling (which is smaller or equal to the maximum training value).

To estimate FID, we use between $10{,}000$ and $30{,}000$ generated images for unconditional experiments and the full training set as reference. For class-conditional ImageNet experiments, we use $5{,}000$ generated images per class and the ~$1{,}300$ training images of that class as reference. Uncertainty is reported as $95\%$ bootstrap confidence intervals over $256$ resamples of the generated and reference features.Figure~\ref{fig:empiricalspectrum} (left) shows the spectra of the empirical covariance estimated on each dataset.

\begin{figure}[h]
    \centering
    \begin{subfigure}[t]{0.35\linewidth}
    \centering
    \includegraphics[width=\linewidth]{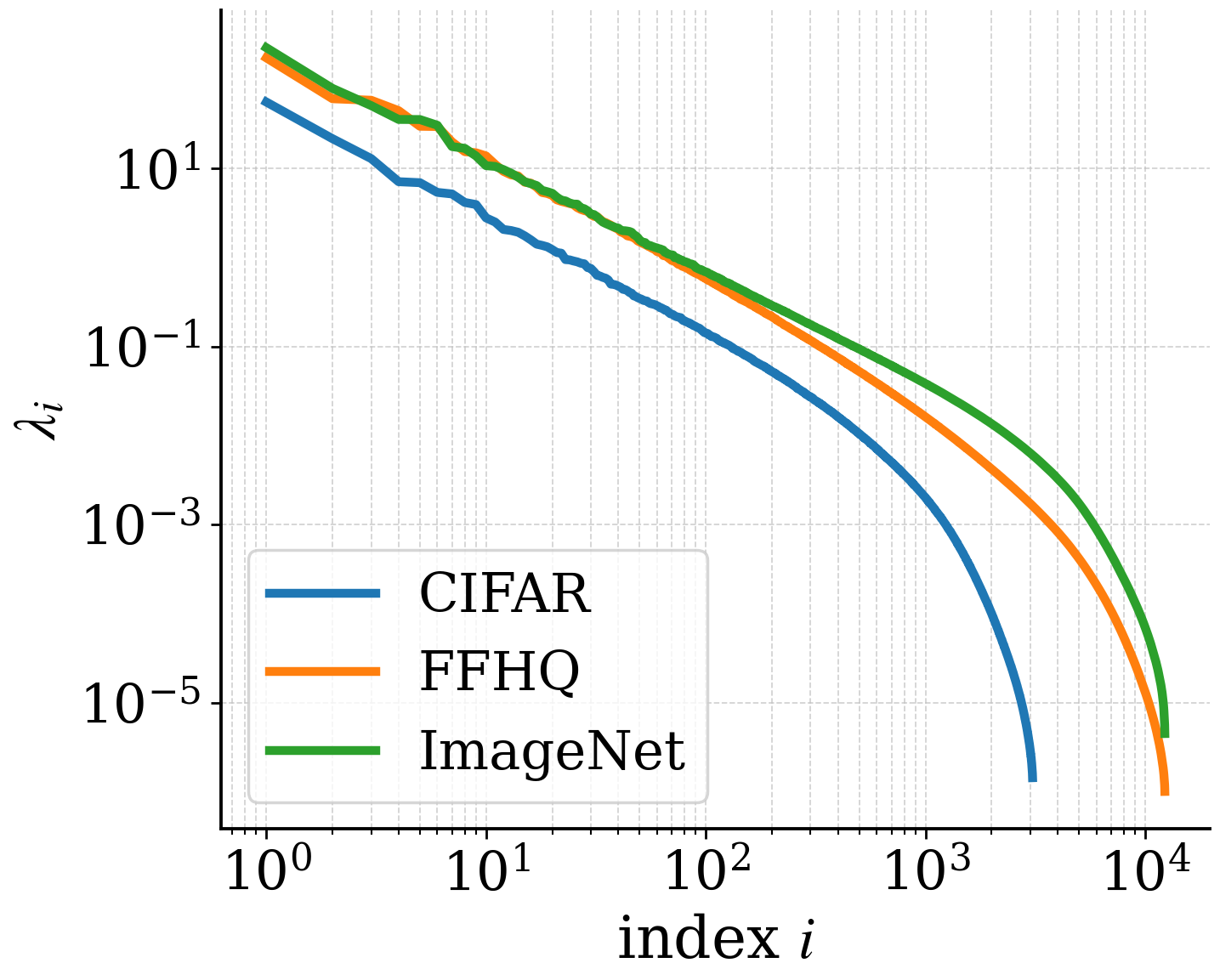}
     \caption{raw datasets}
     \end{subfigure}
    \begin{subfigure}[t]{0.35\linewidth}
    \centering
    \includegraphics[width=\linewidth]{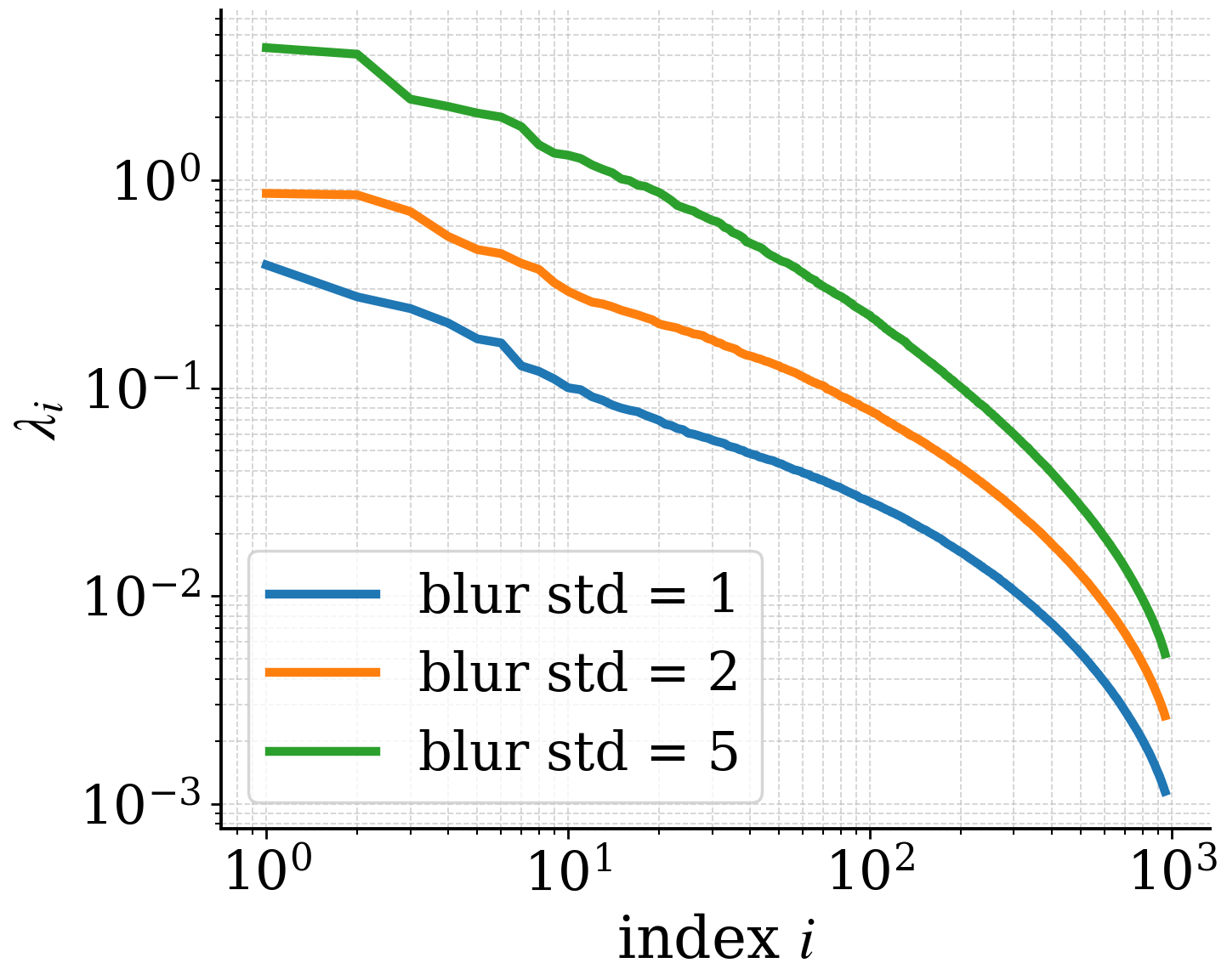}
     \caption{posterior distributions}
     \end{subfigure}
     \caption{Empirical covariance spectra used in the experiments: raw image datasets for unconditional sampling (a), and FFHQ posterior distributions obtained with Gaussian blur of varying standard deviation (b).}
    \label{fig:empiricalspectrum}
\end{figure}

All experiments are performed on single NVIDIA H100 or A100 GPUs.

\subsection{Posterior sampling experiment}
\label{app:posterior_sampling}

In Section~\ref{sec:opt_params}, we also consider posterior sampling problems of the form $p(x \mid y)$, where $y = k * x + w$ is a blurred and noisy observation of an FFHQ image $x$ (see Figure~\ref{fig:alpha_posterior_images}); here $k$ is a Gaussian blur kernel with varying standard deviation and $w$ is Gaussian noise with standard deviation $0.05$.

To sample from $p(x \mid y)$ without retraining a dedicated posterior model, we use the moment-matching method of~\citet{rozet2024learning}. At each step of the sampling algorithm~\eqref{eq:disc}, in order to compute the score of the posterior
$$ \nabla_{x_t} \log p(x_t |y) = \nabla_{x_t} \log p(x_t) + \nabla_{x_t} \log p(y | x_t) $$ using a single pretrained score model for $\nabla \log p(x_t)$, this method proposes an approximation of the term
\begin{align*}
\nabla_{x_t} \log p(y | x_t) &= \frac{1}{p(y | x_t) }\nabla_{x_t} \log \int p(y |x_0) p(x_0 | x_t) dx_0 \\
&= \frac{\int p(y |x_0) \nabla_{x_t} p(x_0 | x_t) dx_0}{p(y | x_t) }
\end{align*}
They propose to approximate $p(x_0 | x_t)$ in the above identity by a Gaussian $\mathcal{N}(\EE[x_0 | x_t], \cov[x_0| x_t])$ with mean and covariance given by the first and second-order Tweedie formulas. See~\citet{rozet2024learning} for more details. Note that this approximation is exact when the data distribution $p_\data$ is Gaussian. In practice, computing
$\nabla_{x_t}\log p(y\mid x_t)$ requires at each sampling step to compute the Jacobian of the score model and to compute the inverse of the
moment-matched measurement covariance. This covariance inversion is handled only approximately with a single conjugate-gradient step.

Because posterior sampling is substantially more expensive, we estimate the empirical Fr\'echet distance from $1{,}000$ generated samples. The empirical covariance thus estimates up to rank at most $10^3$, introducing finite-sample effects that may bias the estimated Fréchet distances.

For each compared parameter value ($c_\eta$ in Section~\ref{sec:opt_params}), the reference statistics are computed from posterior samples generated with the same parameter $c_\eta$, using $500$ discretization steps. For the theoretical curves, we use a single reference mean and covariance estimated by aggregating posterior reference samples across the entire $c_\eta$ range. The corresponding empirical spectra are plotted in Figure~\ref{fig:empiricalspectrum} (right).

\subsection{More details on the optimal diffusion-term parameter $\alpha$}
\label{app:alpha_per_eig}

We show in Figure~\ref{fig:FID_alpha_FFHQ_eig} the empirical (a,b) and the theoretical (c) Fréchet errors, decomposed along each eigendirection of the data covariance $\Sigma_\data$. By Proposition~\ref{prop:mean_cov_Gaussian}, this decomposition is exact in the Gaussian case.
For the empirical errors, we project the empirical covariance of the generated samples onto the eigenbasis of the dataset (empirical) covariance. We report this per-eigendirection error using spectra extracted from Inception space in (a) and image space in (b). The theoretical curve in (c) is computed using the empirical dataset spectrum extracted in image space.

The spectral trend predicted by Proposition~\ref{prop:opt_alpha} and illustrated in (c) remains clearly visible in both empirical curves: directions associated with smaller eigenvalues are minimized at smaller values of~$\alpha$, whereas high-variance directions favor larger values of $\alpha$. As expected, the image-space Fréchet error agrees more closely with the theory, since both use the same reference spectrum. Although the Inception-space transformation is not represented in the theory, the same qualitative trend is still observed.

\begin{figure}[h]
    \centering
    \begin{subfigure}{0.32\linewidth}
    \centering
    \includegraphics[width=\linewidth]{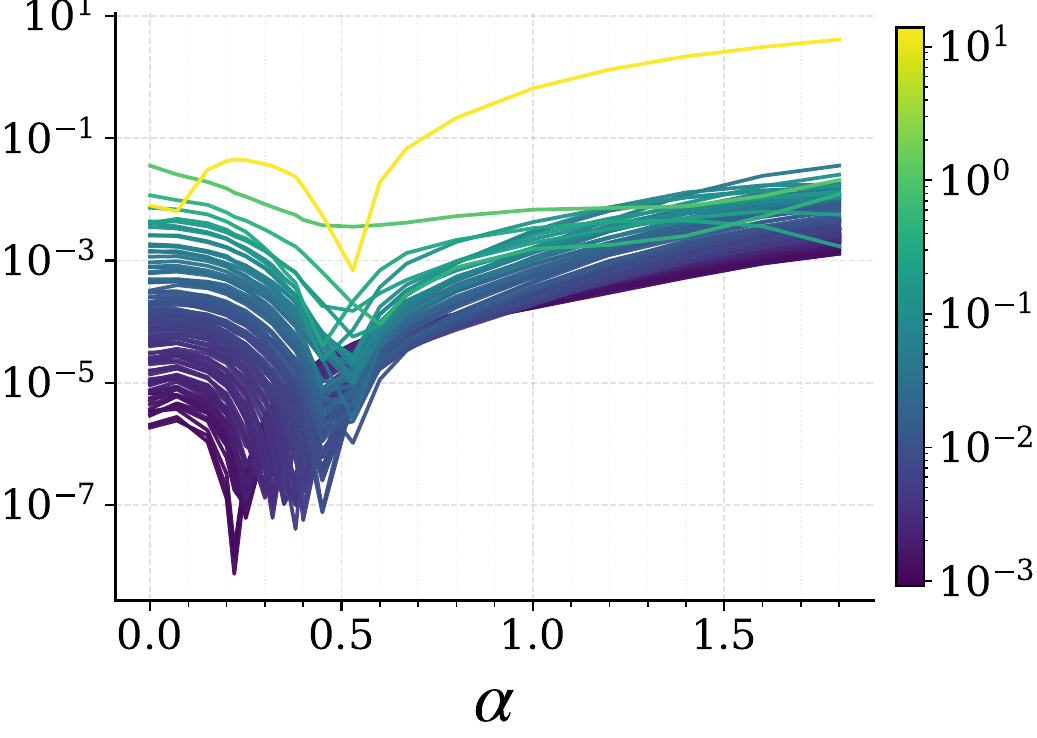}
    \caption{ \centering FFHQ sampling \\(inception space)}
     \end{subfigure}
     \begin{subfigure}{0.32\linewidth}
    \centering
    \includegraphics[width=\linewidth]{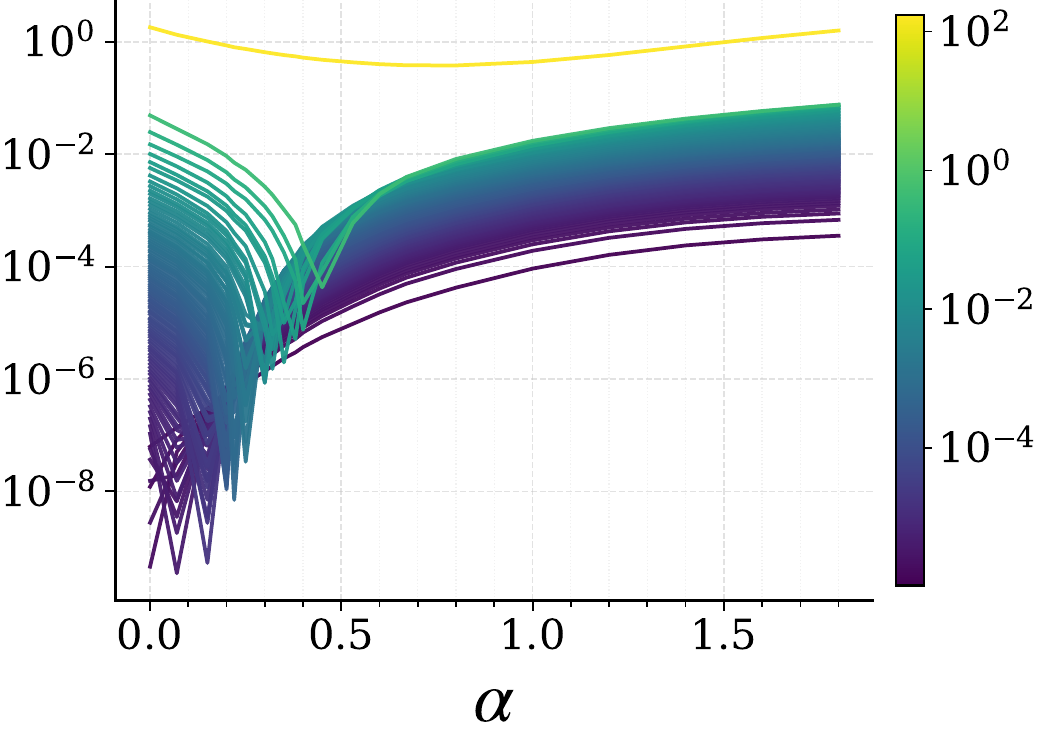}
    \caption{\centering FFHQ sampling \\ (image space)}
     \end{subfigure}
    \begin{subfigure}{0.32\linewidth}
    \centering
    \includegraphics[width=\linewidth]{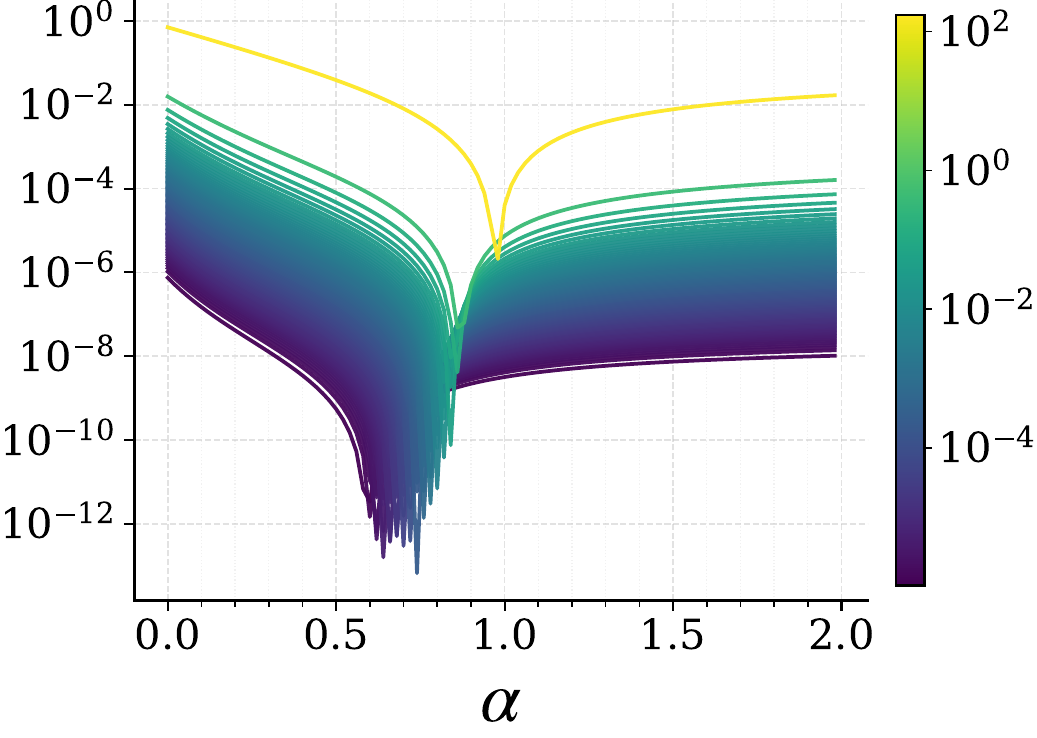}
    \caption{ \centering Gaussian theory \\ (image space)}
     \end{subfigure}
    \caption{Per-eigendirection Fréchet error w.r.t $\alpha$ for FFHQ sampling with VE and $K=100$. Empirical errors are projected onto the data covariance eigenbasis in Inception space (a) and image space (b); Gaussian theory is shown in image space (c).
    Colors indicate eigenvalues $\lambda_i$.}
    \label{fig:FID_alpha_FFHQ_eig}
    \vspace{-0.2cm}
\end{figure}

Figure~\ref{fig:alpha_images} shows examples of FFHQ samples generated with $K=100$ steps. Each column uses a different $\alpha$ and each row uses the same random seed. Note that $\alpha=0$ gives over-smoothed images that look close to the mean of the FFHQ dataset. Increasing $\alpha$ increases variability and eventually creates visual artifacts for large $\alpha$. As explained in Section~\ref{sec:opt_params}, this visual behavior can be explained by our first-order Gaussian theoretical error: using Proposition~\ref{prop:mean_cov_Gaussian} the generated variance along an eigendirection $\lambda$ is:
$$\operatorname{Var}_{\lambda}(Y_K)  = \lambda + \gamma \Delta^{\Sigma, [1]}(\lambda, \alpha) + o(\gamma)$$
For VE, using~\eqref{eq:ve_cov_error_beta}, $\Delta^{\Sigma, [1]}(\lambda, \alpha)$ is non-decreasing in $\alpha$, negative for $0 \leq \alpha < 1$ and positive for $\alpha > 1$. Thus, we get that $\alpha = 0$ minimizes the generated variance $\operatorname{Var}_{\lambda}(Y_K)$ in every direction $\lambda$, and thus produces samples that are closer to the mean.

\subsection{Variance Exploding with $\alpha = 0$ is adequate for (posterior) mean estimation}
\label{app:mean_estim}

The above observation can be particularly insightful for posterior mean estimation. In the context of image inverse problems, diffusion models can be used to sample from conditional distributions $p(x|y)$ where $y$ is the observed degraded image. In this situation, the goal is often to estimate $\EE[x|y]$ rather than sampling from the full posterior. This mean is better estimated by averaging when samples concentrate as much as possible around it. In this section, we show that one can exploit the discretization error to make the samples collapse toward $\mu_{\data}$, i.e., to minimize
\begin{align*}
\mathcal V_K(\alpha)
\eqdef
\EE \|\hat Y_K-\mu_{\data}\|^2 =
\tr{\hat \Sigma_K} + \norm{\hat\mu_K-\mu_{\data}}^2.
\end{align*}
Note that $\mathcal V_K$ combines the effect of the residual spread of the sampled law around its own mean
and the bias of the sampled mean relative to the target mean.
Under the Gaussian data assumption, using the notations of Proposition~\ref{prop:mean_cov_Gaussian}, we get the following expansion:
\begin{equation}
\label{eq:mmse_expansion}
\mathcal V_K
=
\tr{\Sigma_{\data}}
+
\gamma \sum_{i=1}^d \Delta^{\Sigma,[1]}(\lambda_i)
+
\gamma^2
\left[
\sum_{i=1}^d \Delta^{\mu,[1]}(\lambda_i)^2 (U^\top \mu_\data)_i^2
+
\sum_{i=1}^d \Delta^{\Sigma,[2]}(\lambda_i)
\right]
+ O(\gamma^3).
\end{equation}
Equation~\eqref{eq:mmse_expansion} shows that, when the above first and second order coefficients are negative, discretization error (i.e. using a positive stepsize $\gamma$) can reduce the distance to the mean.
At first order in $\gamma$, this distance is only controlled by the covariance error, and the mean error contributes only at order
$\gamma^2$. At first order, the parameters that minimize $\mathcal V_K$ are thus those that minimize the total first-order covariance correction $\sum_{i=1}^d \Delta^{\Sigma,[1]}(\lambda_i).$

In the Variance Exploding case, \eqref{eq:ve_cov_error_beta} shows that the first-order covariance correction is minimized at $\alpha=0$. Hence, within the VE family, the deterministic reverse ODE is the best first-order choice for concentrating samples around the data mean. Compared to Variance Preserving (VP), for a fixed noise schedule, if $\lambda_i \leq 1$ for all $i$, then direct comparison of \eqref{eq:ve_cov_error_sigma} and \eqref{eq:vp_error_sigma} shows that VE with $\alpha=0$ yields a more negative first-order covariance correction than any VP choice $\alpha \geq 0$, and therefore a smaller first-order value of $\mathcal{V}_K$. Moreover, for VE (and not for VP), the mean error $\Delta^{\mu,[1]} = 0$ which also cancels the second-order mean effect.

To illustrate this, we perform deblurring posterior sampling using the Moment Matching method described in Section~\ref{app:posterior_sampling}. Figure~\ref{fig:alpha_posterior_images} shows posterior samples obtained with VE, using $K=50$ discretization steps and different values of $\alpha$. For small values of $\alpha$, the variability across samples is low: the samples are concentrated around the estimated posterior mean. In contrast, larger values of $\alpha$ produce substantially more variability across samples. Thus, $\alpha=0$ makes it possible to estimate the posterior mean with very few samples. However, this posterior mean tends to be smoother than the original image, which is a classical limitation of posterior-mean estimation in image restoration~\citep{blau2018perception,kawar2021stochastic}.

\begin{figure}[h]
\centering
\setlength{\tabcolsep}{1pt}
\renewcommand{\arraystretch}{0.9}
\begin{tabular}{@{}ccccc@{}}
$\alpha=0$ & $\alpha=0.1$ & $\alpha=0.25$ & $\alpha=0.75$ &$\alpha=1.$ \\
\includegraphics[width=0.12\linewidth]{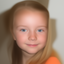} &
\includegraphics[width=0.12\linewidth]{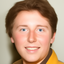} &
\includegraphics[width=0.12\linewidth]{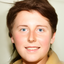} &
\includegraphics[width=0.12\linewidth]{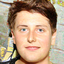} &
\includegraphics[width=0.12\linewidth]{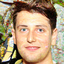}
\\
\includegraphics[width=0.12\linewidth]{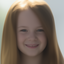} &
\includegraphics[width=0.12\linewidth]{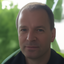} &
\includegraphics[width=0.12\linewidth]{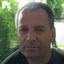} &
\includegraphics[width=0.12\linewidth]{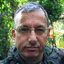} &
\includegraphics[width=0.12\linewidth]{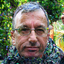} \\
\includegraphics[width=0.12\linewidth]{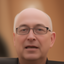} &
\includegraphics[width=0.12\linewidth]{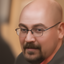} &
\includegraphics[width=0.12\linewidth]{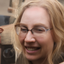} &
\includegraphics[width=0.12\linewidth]{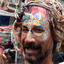} &
\includegraphics[width=0.12\linewidth]{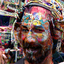}
\end{tabular}
\caption{FFHQ samples for increasing $\alpha$ with VE and $K=100$ steps. Small $\alpha$ yields smoother, mean-seeking samples, while larger $\alpha$ increases variability.}
\label{fig:alpha_images}
\end{figure}

\begin{figure}[h]
\centering
\setlength{\tabcolsep}{1pt}
\renewcommand{\arraystretch}{0.9}
\begin{tabular}{@{}ccccc@{}}
\multicolumn{1}{c}{Initial $x$} &
\multicolumn{1}{c}{Observation $y$} & && \\
\multicolumn{1}{c}{
\includegraphics[width=0.12\linewidth]{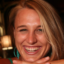}
} &
\multicolumn{1}{c}{
\includegraphics[width=0.12\linewidth]{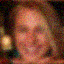}
} & & & \\[4pt]
$\alpha=0$ & $\alpha=0.1$ & $\alpha=0.2$ & $\alpha=0.3$ &$\alpha=0.4$ \\
\includegraphics[width=0.12\linewidth]{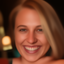} &
\includegraphics[width=0.12\linewidth]{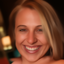}  &
\includegraphics[width=0.12\linewidth]{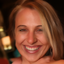} &
\includegraphics[width=0.12\linewidth]{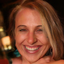}  &
\includegraphics[width=0.12\linewidth]{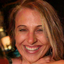}
\\
\includegraphics[width=0.12\linewidth]{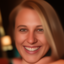} &
\includegraphics[width=0.12\linewidth]{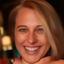}  &
\includegraphics[width=0.12\linewidth]{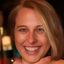} &
\includegraphics[width=0.12\linewidth]{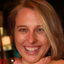}  &
\includegraphics[width=0.12\linewidth]{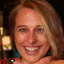} \\
\includegraphics[width=0.12\linewidth]{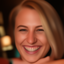} &
\includegraphics[width=0.12\linewidth]{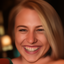}  &
\includegraphics[width=0.12\linewidth]{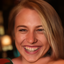} &
\includegraphics[width=0.12\linewidth]{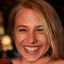}  &
\includegraphics[width=0.12\linewidth]{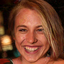} \\
\includegraphics[width=0.12\linewidth]{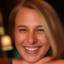} &
\includegraphics[width=0.12\linewidth]{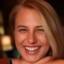}  &
\includegraphics[width=0.12\linewidth]{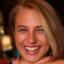} &
\includegraphics[width=0.12\linewidth]{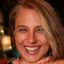}  &
\includegraphics[width=0.12\linewidth]{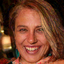}
\end{tabular}
\caption{Deblurring posterior samples for increasing $\alpha$ with VE and $K=50$ steps. Each row corresponds to a different seed. The observation has been obtained by $y = k*x + w$ with $k$ a Gaussian blur kernel with standard deviation $1.5$ and $w$ a Gaussian noise with standard deviation $0.05$. See Section~\ref{app:posterior_sampling} for more details on the posterior sampling algorithm.}
\label{fig:alpha_posterior_images}
\end{figure}

\subsection{More details on the optimal rescaling schedule $\eta_t$}
\label{app:c_per_eig}

Figure~\ref{fig:opt_c_per_eig} shows the empirical (a,b) and the theoretical (c) Fréchet errors from Figure~\ref{fig:opt_c_datasets} (top), but here decomposed along each eigendirection of the data covariance $\Sigma_\data$. See Section~\ref{app:alpha_per_eig} for more details on the per-$\lambda_i$ computations.

On the theoretical Figure (c), for each eigendirection, the curve is minimized near $c_\eta \approx \lambda_i$, consistently with the first-order cancellation condition~\eqref{eq:opt_rescaling} at $\alpha=0$. Thus, large-variance directions favor larger values of $c_\eta$, while low-variance directions favor smaller ones. This spectral ordering is also visible in the empirical decompositions (a,b), although the curves are noisier in Inception space.

This non-uniform behavior across eigendirections also explains the evolution of the optimal $c^*$ observed in Figure~\ref{fig:opt_c_datasets} for an anisotropic spectrum: datasets with larger effective scale favor larger optimal $c_\eta$, while a steeper spectrum shifts the normalized optimum $c_\eta/\lambda_{\max}$ toward smaller values.

\begin{figure}[h]
    \centering
    \begin{subfigure}{0.32\linewidth}
    \centering
    \includegraphics[width=\linewidth]{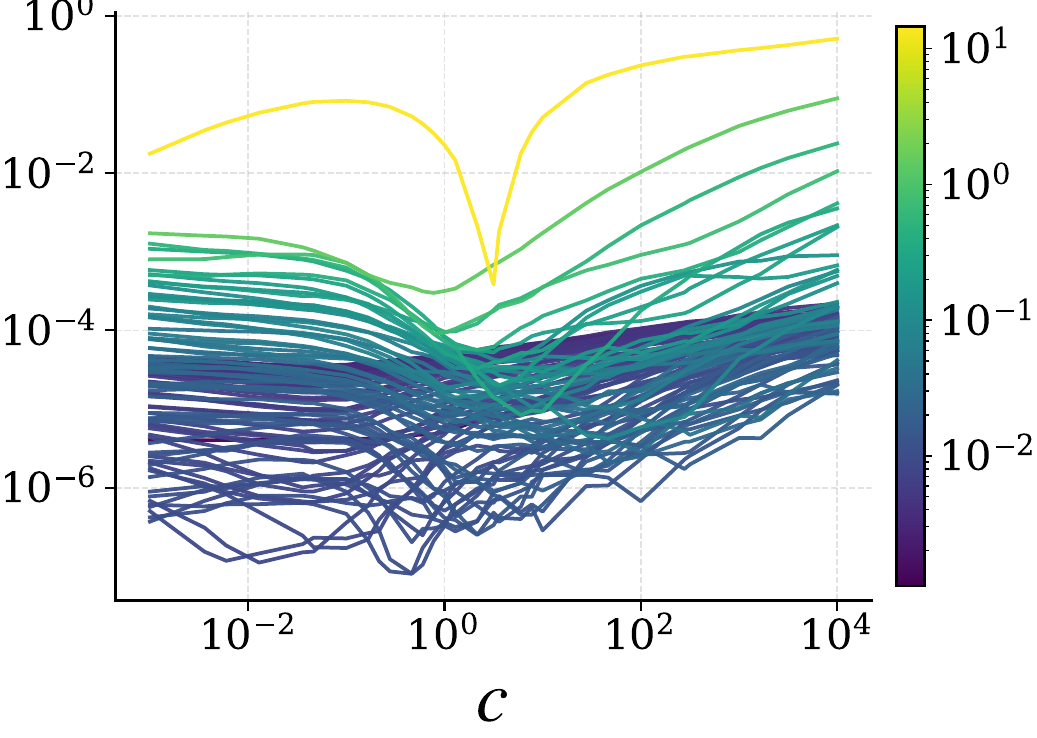}
    \caption{ \centering FFHQ sampling \\(inception space)}
     \end{subfigure}
     \begin{subfigure}{0.32\linewidth}
    \centering
    \includegraphics[width=\linewidth]{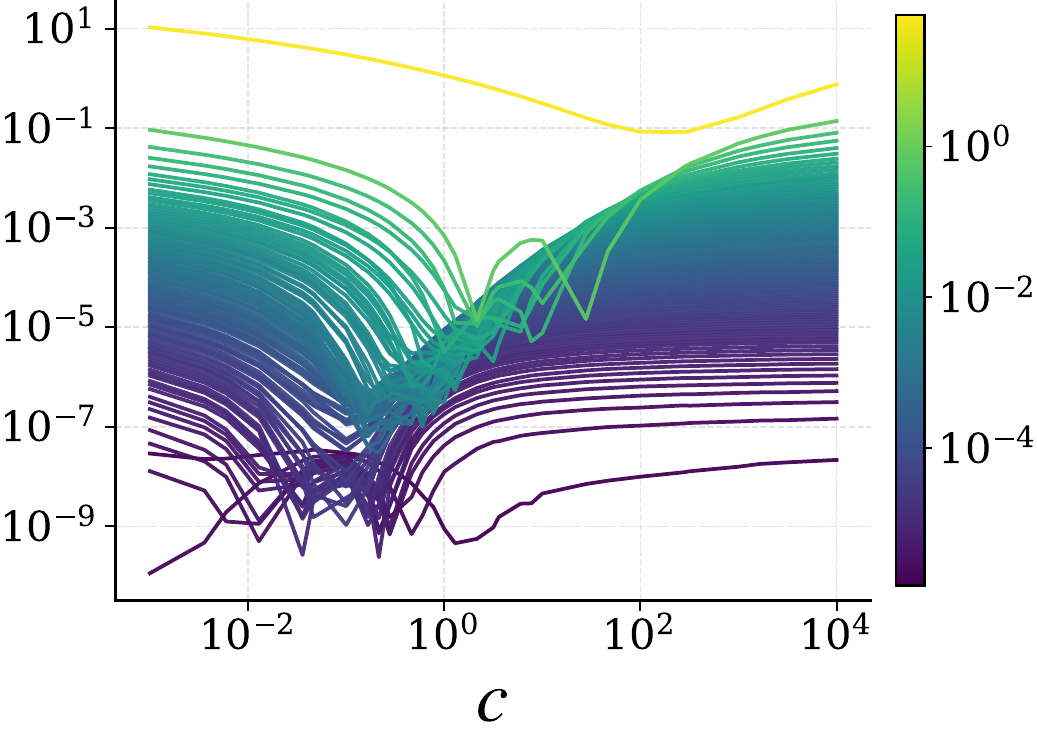}
    \caption{\centering FFHQ sampling \\ (image space)}
     \end{subfigure}
    \begin{subfigure}{0.32\linewidth}
    \centering
    \includegraphics[width=\linewidth]{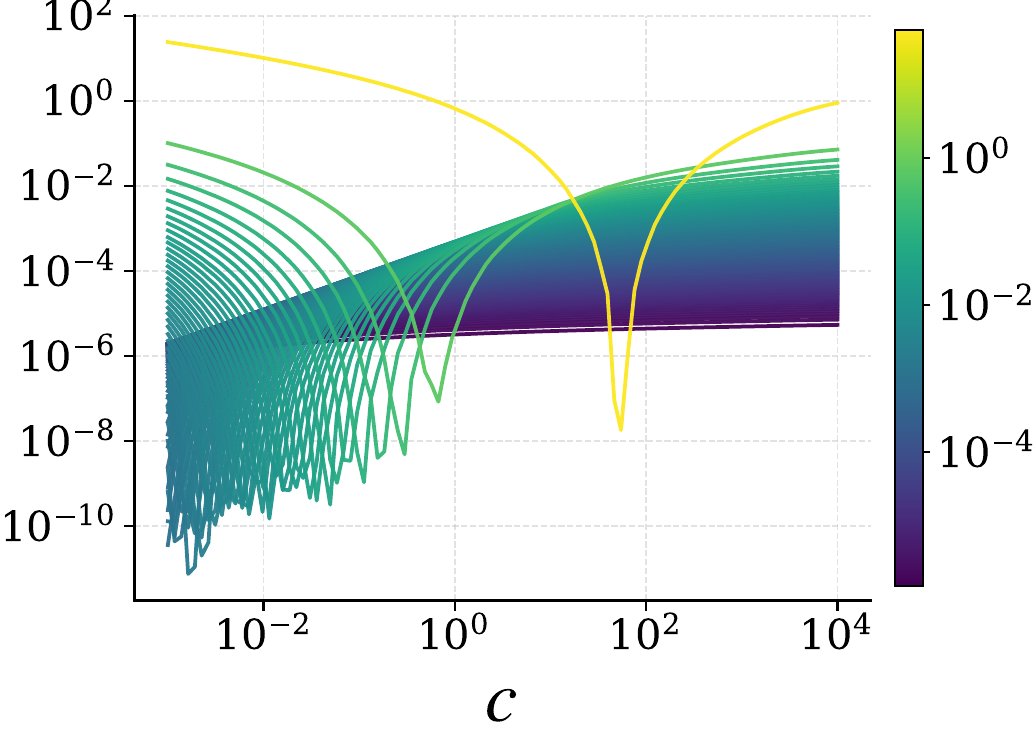}
    \caption{ \centering Gaussian theory \\ (image space)}
     \end{subfigure}
    \caption{Per-eigendirection Fréchet error w.r.t $c_\eta$ for FFHQ sampling with $K=100$. Empirical errors are projected onto the data covariance eigenbasis in Inception space (a) and image space (b); Gaussian theory is shown in image space (c).
Colors indicate eigenvalues $\lambda_i$.}
    \label{fig:opt_c_per_eig}
\end{figure} \vspace{-0.25cm}

\paragraph{Class-conditional sampling.}
\label{app:c_classes}
To vary the spectrum within a single model, we repeat the $c_\eta$ sweep with class-conditional ImageNet $64\times64$ sampling (ADM), fixing the class label of every sample. We show three classes in Figure~\ref{fig:intro_empirical_motivation}. Samples use $K=50$ Euler steps, $\alpha=0$, $\sigma_t=\sigma_{\max}(t/T)^5$ and $5000$ images per value of $c_\eta$, and the FID of each class is computed against all of its training images. 

As for the datasets of Figure~\ref{fig:opt_c_datasets}, both the empirical FID and the Gaussian FD computed from the class spectrum are minimized at a $c_\eta^\star$ that increases with $\lambda_{\max}$. Moreover, in all three classes, the FID optimum lies below the Gaussian one, the direction predicted by the scale-mixture correction of Proposition~\ref{prop:opt_eta_powerlaw}(ii) for $p\gtrsim0.65$.

\begin{figure}[h]
    \centering
    \begin{subfigure}[t]{0.37\linewidth}
    \captionsetup{skip=0pt}
    \centering
    \includegraphics[width=\linewidth]{figures/c/classes3_fid.pdf}
    \caption{Conditional ImageNet FID}
    \end{subfigure}
    \hfill
    \begin{subfigure}[t]{0.37\linewidth}
    \captionsetup{skip=0pt}
    \centering
    \includegraphics[width=\linewidth]{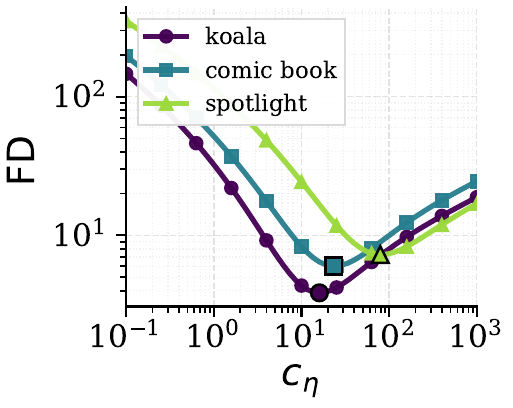}
    \caption{Gaussian FD}
    \end{subfigure}
    \hfill
    \begin{subfigure}[t]{0.22\linewidth}
    \captionsetup{skip=0pt}
    \vspace{-3cm}
    \centering
    \footnotesize
    \renewcommand{\arraystretch}{1.15}
    \setlength{\tabcolsep}{3pt}
    \begin{tabular}{@{}lcc@{}}
    \hline
    Class & $\lambda_{\max}$ & $p$ \\
    \hline
    {\color{vir1} koala} & 72 & 1.34 \\
    {\color{vir2} comic book} & 156 & 1.16 \\
    {\color{vir3!75!black} spotlight} & 714 & 1.76 \\
    \hline
    \end{tabular}
    \end{subfigure}
    \caption{Class-conditional ImageNet $64\times64$ sampling (ADM) for three classes with different covariance spectra, with $K=50$, $\alpha=0$ and $\sigma_t=\sigma_{\max}(t/T)^5$: empirical FID against each class's training images (a) and Gaussian FD from each class spectrum (b), as functions of $c_\eta$. Dots mark the minimizers. The table reports the leading eigenvalue $\lambda_{\max}$ and the fitted power-law exponent $p$ of each class spectrum.}
    \label{fig:opt_c_classes}
\end{figure}

\end{document}

%% file: macro.tex
\usepackage{hyperref}       %
\usepackage{url}            %
\usepackage{booktabs}       %
\usepackage{amsfonts}       %
\usepackage{amsmath}
\usepackage[utf8]{inputenc} %
\usepackage[T1]{fontenc}    %
\usepackage{amssymb}
\usepackage{nicefrac}       %
\usepackage{microtype}      %
\usepackage{xcolor}         %
\usepackage{caption}
\usepackage{subcaption}
\usepackage{amsfonts}
\usepackage{url}
\usepackage{amsthm}
\usepackage{mathtools}
\usepackage{xcolor}
\usepackage{stmaryrd}
\usepackage{hyperref}
\usepackage{empheq}
\usepackage{wrapfig}
\usepackage{graphicx}
\usepackage{enumitem}
\newtheorem{defn}{Definition}
\newtheorem{prop}{Proposition}
\newtheorem{lem}{Lemma}

\newtheorem{thm}{Theorem}
\newtheorem{ass}{Assumption}
\newtheorem{rem}{Remark}

\newcommand{\norm}[1]{\left\|{#1}\right\|}
\newcommand{\tr}[1]{\operatorname{Tr}\left({#1}\right)}
\newcommand{\eqdef}{\coloneqq}
\newcommand{\id}{\operatorname{Id}}

\newcommand{\RR}{\mathbb{R}}
\newcommand{\EE}{\mathbb{E}}
\newcommand{\data}{\mathrm{data}}
\newcommand{\FD}{\operatorname{FD}}
\newcommand{\cov}{\operatorname{Cov}}

\newcommand{\Law}{\operatorname{Law}}
\newcommand{\diag}[1]{\operatorname{Diag}\left({#1}\right)}

\definecolor{linkcolor}{RGB}{82, 82, 192}
\definecolor{vir1}{HTML}{440154}
\definecolor{vir2}{HTML}{277E8E}
\definecolor{vir3}{HTML}{9BD93C}
\definecolor{blur1}{HTML}{9E0142}
\definecolor{blur2}{HTML}{F46D43}
\definecolor{blur3}{HTML}{E6AB02}
\hypersetup{
    colorlinks=true,
    linkcolor=linkcolor,
    citecolor=linkcolor
}